\pdfoutput=1

\documentclass[11pt]{article}

\usepackage{amsmath}
\usepackage{amsthm}
\usepackage{amssymb}
\usepackage{amsfonts}
\usepackage[top = 1 in, bottom = 1.2 in, left = 1 in, right = 1 in]{geometry}
\usepackage{color}
\usepackage{mathrsfs}
\usepackage[normalem]{ulem}
\usepackage{longtable}
\usepackage{hyperref}
\usepackage{cleveref}[noabbrev]
\usepackage{dsfont}
\usepackage{graphicx}
\usepackage{caption}
\usepackage{subcaption}
\usepackage{float}
\usepackage{mathtools}
\usepackage{etoolbox}
\usepackage{thmtools}
\usepackage[textsize=tiny]{todonotes}
\usepackage{soul}

\usepackage{enumitem}

\newtheorem{theorem}{Theorem}
\newtheorem{lemma}{Lemma}

\declaretheoremstyle[headfont=\bf,bodyfont=\normalfont]{ex}

\declaretheoremstyle[bodyfont=\normalfont]{rm}

\DeclareMathOperator*{\argmin}{arg\,min}
\newcommand{\normmm}{{\vert\kern-0.25ex\vert\kern-0.25ex\vert}}
\newcommand{\bignormmm}{{\big\vert\kern-0.25ex\big\vert\kern-0.25ex\big\vert}}
\newcommand{\Bignormmm}{{\Big\vert\kern-0.25ex\Big\vert\kern-0.25ex\Big\vert}}

\newcommand{\defn}{\ensuremath{:\,=}}
\newcommand{\proj}{\ensuremath{\Pi}}

\makeatletter
\long\def\@makecaption#1#2{
 \vskip 0.8ex
 \setbox\@tempboxa\hbox{\small {\bf #1:} #2}
 \parindent 1.5em 
 \dimen0=\hsize
 \advance\dimen0 by -3em
 \ifdim \wd\@tempboxa >\dimen0
 \hbox to \hsize{
 \parindent 0em
 \hfil
 \parbox{\dimen0}{\def\baselinestretch{0.96}\small
 {\bf #1.} {#2}
 }
 \hfil}
 \else \hbox to \hsize{\hfil \box\@tempboxa \hfil}
 \fi
}
\makeatother

\usepackage[mathscr]{euscript}

\newcommand{\Prob}{\ensuremath{\mathbb{P}}}
\newcommand{\Exp}{\ensuremath{\mathbb{E}}}
\newcommand{\Var}{\ensuremath{{\rm Var}}}
\newcommand{\Real}{\ensuremath{\mathds{R}}}
\newcommand{\Natural}{\ensuremath{\mathds{N}}}
\newcommand{\Int}{\ensuremath{\mathds{Z}}}
\newcommand{\plaincon}{\ensuremath{c}}
\newcommand{\PlainCon}{\ensuremath{C}}

\newcommand{\const}[1]{\ensuremath{\plaincon_{#1}}}
\newcommand{\constnew}{\ensuremath{\plaincon'}}
\newcommand{\constnewnew}{\ensuremath{\plaincon''}}
\newcommand{\consttil}[1]{\ensuremath{\widetilde{\plaincon}_{#1}}}
\newcommand{\Const}[1]{\ensuremath{\PlainCon_{#1}}}

\newcommand{\projH}{\ensuremath{\proj_{\RKHS}}}
\newcommand{\projG}{\ensuremath{\proj_{\mathds{G}}}}
\newcommand{\projperp}{\ensuremath{\proj_{\perp}}}
\newcommand{\cpdP}{\ensuremath{\mathcal{L}}}
\newcommand{\cpdPHH}{\ensuremath{\cpdP_{\RKHS, \RKHS}}}
\newcommand{\cpdPpH}{\ensuremath{\cpdP_{\perp, \RKHS}}}
\newcommand{\cpdPHp}{\ensuremath{\cpdP_{\RKHS, \perp}}}
\newcommand{\cpdPpp}{\ensuremath{\cpdP_{\perp, \perp}}}
\newcommand{\cpdr}{\ensuremath{\mathsf{b}}}
\newcommand{\schur}{\ensuremath{\mathcal{S}}}
\newcommand{\one}{\ensuremath{\mathds{1}}}

\newcommand{\Hamming}{\ensuremath{\varrho_H}}

\newcommand{\Dataset}{\ensuremath{\mathcal{D}}}

\newcommand{\VstarMt}{\ensuremath{\boldsymbol{V}^*}}
\newcommand{\VstarMtbase}{\ensuremath{\VstarMt_0}}
\newcommand{\VstarMtper}{\ensuremath{\VstarMt}}
\newcommand{\VstarMtperp}{\ensuremath{\VstarMt_{\perp}}}
\newcommand{\VstarMtbaseperp}{\ensuremath{\VstarMt_{0, \, \perp}}}
\newcommand{\VstarMtperperp}{\ensuremath{\VstarMtper_{\perp}}}

\newcommand{\VstarMtbasepar}{\ensuremath{\VstarMt_{0, \, \RKHSMt}}}
\newcommand{\VstarMtperpar}{\ensuremath{\VstarMtper_{\RKHSMt}}}

\newcommand{\Vstar}{\ensuremath{V^*}}
\newcommand{\Vstartilde}{\ensuremath{\widetilde{V}^*}}
\newcommand{\Vstarbar}{\ensuremath{\overline{V}^*}}
\newcommand{\Vstarm}[1]{\ensuremath{V_{#1}^*}}
\newcommand{\Vstartildem}[1]{\ensuremath{\Vstartilde_{#1}}}
\newcommand{\Vstarbase}{\ensuremath{\Vstar}}

\newcommand{\Vstarperp}{\ensuremath{\Vstar_{\perp}}}
\newcommand{\Vstarperpm}[1]{\ensuremath{\Vstar_{#1, \, \perp}}}

\newcommand{\Vstarpar}{\ensuremath{\Vstar_{\RKHS}}}
\newcommand{\Vstarparm}[1]{\ensuremath{\Vstar_{#1, \, \RKHS}}}
\newcommand{\Vstarbasepar}{\ensuremath{\Vstarbase_{0, \, \RKHS}}}

\newcommand{\mytheta}{f}
\newcommand{\thetastar}{\ensuremath{\mytheta^*}}

\newcommand{\thetahat}{\ensuremath{\widehat{\mytheta}}}

\newcommand{\traj}{\ensuremath{\boldsymbol{\tau}}}

\newcommand{\lbnormperp}{\ensuremath{\bar{\varrho}_{\perp}}}

\newcommand{\return}{\ensuremath{G}}
\newcommand{\returnhat}{\ensuremath{\widehat{\return}}}
\newcommand{\ReturnHat}[2]{\ensuremath{\widehat{G}_{#1}^{#2}}}
\newcommand{\ssum}[2]{\ensuremath{{\sum_{#1}^{#2}}\,}}
\newcommand{\Step}{\ensuremath{K}}
\newcommand{\step}{\ensuremath{k}}

\newcommand{\stdfun}{\ensuremath{\sigma}}

\newcommand{\stdfunbar}{\ensuremath{\bar{\stdfun}}}
\newcommand{\approxerr}{\ensuremath{\sigma_a}}
\newcommand{\lbapproxerr}{\ensuremath{\overline{\sigma}_a}}

\newcommand{\varm}{\ensuremath{\mathsf{M}}}
\newcommand{\vara}{\ensuremath{\mathsf{A}}}

\newcommand{\Deltahat}{\ensuremath{\widehat{\Delta}}}

\newcommand{\Gammahat}{\ensuremath{\widehat{\Gamma}}}
\newcommand{\GammaHat}{\Gammahat}

\newcommand{\specfun}{\ensuremath{\varrho}}

\newcommand{\alphabold}{\ensuremath{\boldsymbol{\alpha}}}

\newcommand{\PackNum}{\ensuremath{M}}
\newcommand{\idxpack}{\ensuremath{m}}
\newcommand{\idxpacknew}{\ensuremath{m'}}
\newcommand{\idxpackdag}{\ensuremath{m^{\dagger}}}
\newcommand{\RKHS}{\ensuremath{\mathds{H}}}
\newcommand{\RKHSMt}{\ensuremath{\mathcal{H}}}

\newcommand{\Lmu}{\ensuremath{L^2(\distr)}}
\newcommand{\Ltwo}[1]{\ensuremath{L^2(#1)}}

\newcommand{\MRP}{\ensuremath{\mathscr{I}}}
\newcommand{\MRPMt}{\ensuremath{\boldsymbol{\mathscr{I}}}}
\newcommand{\MRPMtbase}{\ensuremath{\MRPMt_0}}
\newcommand{\MRPMtper}{\ensuremath{\MRPMt}}
\newcommand{\MRPbase}{\ensuremath{\MRP_0}}

\newcommand{\Ker}{\ensuremath{\mathcal{K}}}

\newcommand{\MRPclass}{\ensuremath{\mathfrak{C}}}

\newcommand{\reward}{\ensuremath{r}}

\newcommand{\discount}{\ensuremath{\gamma}}
\newcommand{\discounttil}{\ensuremath{\widetilde{\discount}}}
\newcommand{\discounteff}{\ensuremath{\overline{\discount}}}
\newcommand{\bou}{\ensuremath{b}}

\newcommand{\Ho}{\ensuremath{H}}
\newcommand{\Hoeff}{\ensuremath{\overline{\Ho}}}

\newcommand{\noise}{\ensuremath{\zeta}}
\newcommand{\noisebase}{\ensuremath{\noise_0}}
\newcommand{\noisebasenew}{\ensuremath{\widetilde{\noise}_0}}
\newcommand{\lbnoise}{\ensuremath{\overline{\noise}}}

\newcommand{\TransOp}{\ensuremath{\mathcal{P}}}

\newcommand{\TransOpm}[1]{\ensuremath{\TransOp_{#1}}}
\newcommand{\TransOpbase}{\ensuremath{\TransOp_0}}

\newcommand{\TransOptilde}{\ensuremath{\widetilde{\mathcal{P}}}}
\newcommand{\TransOptildem}[1]{\ensuremath{\TransOptilde_{#1}}}

\newcommand{\conj}{\ensuremath{*}}
\newcommand{\IdOp}{\ensuremath{\mathcal{I}}}

\newcommand{\IdOpH}{\ensuremath{\IdOp_{\RKHS}}}
\newcommand{\IdOpp}{\ensuremath{\IdOp_{\perp}}}
\newcommand{\StateSp}{\ensuremath{\mathcal{X}}}

\newcommand{\Bell}{\ensuremath{\mathcal{T}}}
\newcommand{\BellOp}[1]{\ensuremath{\Bell^{(#1)}}}
\newcommand{\BellOphat}[1]{\ensuremath{\widehat{\Bell}^{(#1)}}}
\newcommand{\weight}[1]{\ensuremath{w_{#1}}}
\newcommand{\bweight}{\ensuremath{\boldsymbol{w}}}
\newcommand{\bweightr}{\ensuremath{\boldsymbol{\omega}_r}}
\newcommand{\bweightrper}{\ensuremath{\widetilde{\boldsymbol{\omega}}_r}}
\newcommand{\bweightperp}{\ensuremath{\boldsymbol{\omega}_{\perp}}}

\newcommand{\featureMt}[1]{\ensuremath{\boldsymbol{\phi}_{#1}}}
\newcommand{\featureMtper}[1]{\ensuremath{\widetilde{\boldsymbol{\phi}}_{#1}}}
\newcommand{\FeatureMt}{\ensuremath{\boldsymbol{\Phi}}}

\newcommand{\featureMtperp}{\ensuremath{\featureMt{\perp}}}
\newcommand{\featureMtperperp}{\ensuremath{\featureMtper{\perp}}}

\newcommand{\TransMt}{\ensuremath{{\bf P}}}

\newcommand{\invmix}{\ensuremath{\varsigma}}
\newcommand{\Dp}{\ensuremath{\Delta p}}
\newcommand{\Dq}{\ensuremath{\Delta q}}
\newcommand{\rewardMt}{\ensuremath{\boldsymbol{r}}}
\newcommand{\TransMtbase}{\ensuremath{\TransMt_0}}
\newcommand{\TransMtper}{\ensuremath{{\TransMt}}}
\newcommand{\feature}[1]{\ensuremath{\phi_{#1}}}
\newcommand{\featureper}[1]{\ensuremath{\widetilde{\phi}_{#1}}}

\newcommand{\featureperperp}{\ensuremath{\featureper{\perp}}}

\newcommand{\CovOp}{\ensuremath{\Sigma_{\rm cov}}}
\newcommand{\CovOphat}{\ensuremath{\widehat{\Sigma}_{\rm cov}}}

\newcommand{\CrOp}{\ensuremath{\Sigma_{\rm cr}}}
\newcommand{\CrOpw}{\ensuremath{\CrOp^{(\bweight)}}}
\newcommand{\CrOphat}{\ensuremath{\widehat{\Sigma}_{\rm cr}}}
\newcommand{\CrOpwhat}{\ensuremath{\CrOphat^{(\bweight)}}}

\newcommand{\by}{\ensuremath{\boldsymbol{y}_0}}
\newcommand{\byhat}{\ensuremath{\widehat{\boldsymbol{y}}_0}}

\newcommand{\CovMt}{\ensuremath{{\bf K}_{\rm cov}}}

\newcommand{\CrMtw}{\ensuremath{{\bf K}_{\rm cr}^{(\bweight)}}}
\newcommand{\IdMt}{\ensuremath{{\bf I}}}
\newcommand{\IdMtper}{\ensuremath{\widetilde{\IdMt}}}

\newcommand{\yvec}{\ensuremath{\boldsymbol{y}}}

\newcommand{\widgraph}[2]{\includegraphics[keepaspectratio,width=#1]{#2}}

\newcommand{\numobs}{\ensuremath{n}}
\newcommand{\numobsnew}{\ensuremath{\widetilde{\numobs}}}

\newcommand{\delcrit}{\ensuremath{\delta_\numobs}}

\newcommand{\delcritnew}{\ensuremath{u_{\numobs}}}
\newcommand{\delcritnewtil}{\ensuremath{\widetilde{u}_{\numobs}}}
\newcommand{\delcritv}{\ensuremath{v_{\numobs}}}
\newcommand{\delcritvtil}{\ensuremath{\widetilde{v}_{\numobs}}}

\newcommand{\statdim}{\ensuremath{d_\numobs}}

\newcommand{\eig}[1]{\ensuremath{\mu_{#1}}}

\newcommand{\Borel}{\ensuremath{\mathscr{B}}}
\usepackage{upgreek} \newcommand{\distr}{\ensuremath{\upmu}}

\newcommand{\distrnew}{\ensuremath{\upnu}}

\newcommand{\distrm}[1]{\ensuremath{\distr_{#1}}}
\newcommand{\distrbar}{\ensuremath{\overline{\distr}}}

\newcommand{\ridge}{\ensuremath{\lambda_\numobs}}

\newcommand{\distrMt}{\ensuremath{\boldsymbol{\upmu}}}
\newcommand{\distrMtbase}{\ensuremath{\distrMt_0}}
\newcommand{\distrMtper}{\ensuremath{{\distrMt}}}

\newcommand{\chisqrad}{\ensuremath{\nu_{\numobs}}}

\newcommand{\Fclass}{\ensuremath{\mathscr{F}}}

\newcommand{\newrad}{\ensuremath{R}}
\newcommand{\unibou}{\ensuremath{\kappa}}
\newcommand{\CI}{\ensuremath{{\rm CI}}}

\newcommand{\mixtime}{\ensuremath{\tau_*}}
\newcommand{\hittime}{\ensuremath{\tau}}
\newcommand{\hitS}{\ensuremath{S}}
\newcommand{\hitSbar}{\ensuremath{\overline{S}}}
\newcommand{\hitStil}{\ensuremath{\widetilde{S}}}
\newcommand{\mixtimebar}{\ensuremath{\overline{\tau}_*}}

\newcommand{\funclass}{\ensuremath{\mathscr{F}}}
\newcommand{\funclasstil}{\ensuremath{\widetilde{\mathscr{F}}}}

\newcommand{\funnew}{\ensuremath{g}}

\newcommand{\Term}{\ensuremath{T}}

\newcommand{\state}{\ensuremath{x}}
\newcommand{\statetil}{\ensuremath{\widetilde{\state}}}
\newcommand{\statenew}{\ensuremath{\state'}}
\newcommand{\statetilnew}{\ensuremath{\statetil'}}
\newcommand{\statey}{\ensuremath{y}}
\newcommand{\State}{\ensuremath{X}}
\newcommand{\Statenew}{\ensuremath{\State'}}
\newcommand{\Statetil}{\ensuremath{\widetilde{\State}}}
\newcommand{\Statetilnew}{\ensuremath{\Statetil'}}
\newcommand{\diff}{\ensuremath{d}}
\newcommand{\dx}{\ensuremath{\diff \state}}

\DeclarePairedDelimiterX{\anglep}[1]{(}{)}{#1}
\makeatletter
\newcommand{\@spanstar}[1]{{\rm span}\anglep*{#1}}
\newcommand{\@spannostar}[2][]{{\rm span}\anglep[#1]{#2}}
\newcommand{\Span}{\@ifstar\@spanstar\@spannostar}
\makeatother

\DeclarePairedDelimiterX{\dfun}[2]{(}{)}{#1 ; #2}
\makeatletter
\newcommand{\@trunstar}[2]{\chi\dfun*{#1}{#2}}
\newcommand{\@trunnostar}[3][]{\chi\dfun[#1]{#2}{#3}}
\newcommand{\trun}{\@ifstar\@trunstar\@trunnostar}
\makeatother

\DeclarePairedDelimiterX{\inprod}[2]{\langle}{\rangle}{#1, \, #2}

\DeclarePairedDelimiterX{\kulldiv}[2]{(}{)}{#1\;\delimsize\|\;#2}
\makeatletter
\newcommand{\@kullstar}[2]{D_{\text{KL}}\kulldiv*{#1}{#2}}
\newcommand{\@kullnostar}[3][]{D_{\text{KL}}\kulldiv[#1]{#2}{#3}}
\newcommand{\kull}{\@ifstar\@kullstar\@kullnostar}
\makeatother

\makeatletter
\newcommand{\@chistar}[2]{\chi^2\kulldiv*{#1}{#2}}
\newcommand{\@chinostar}[3][]{\chi^2\kulldiv[#1]{#2}{#3}}
\newcommand{\chidiv}{\@ifstar\@chistar\@chinostar}
\makeatother

\makeatletter
\newcommand{\@hilinstar}[2]{\inprod*{#1}{#2}_{\RKHS}}
\newcommand{\@hilinnostar}[3][]{\inprod[#1]{#2}{#3}_{\RKHS}}
\newcommand{\hilin}{\@ifstar\@hilinstar\@hilinnostar}
\makeatother

\DeclarePairedDelimiterX{\norm}[1]{\|}{\|}{#1}
\makeatletter
\newcommand{\@normstar}[1]{\norm*{#1}_{\RKHS}}
\newcommand{\@normnostar}[2][]{\norm[#1]{#2}_{\RKHS}}
\newcommand{\hilnorm}{\@ifstar\@normstar\@normnostar}
\makeatother

\makeatletter
\newcommand{\@TVnormstar}[1]{\norm*{#1}_{\rm TV}}
\newcommand{\@TVnormnostar}[2][]{\norm[#1]{#2}_{\rm TV}}
\newcommand{\TVnorm}{\@ifstar\@TVnormstar\@TVnormnostar}
\makeatother

\DeclareFontFamily{U}{matha}{\hyphenchar\font45}
\DeclareFontShape{U}{matha}{m}{n}{
	<-6> matha5 <6-7> matha6 <7-8> matha7
	<8-9> matha8 <9-10> matha9
	<10-12> matha10 <12-> matha12
}{}
\DeclareSymbolFont{matha}{U}{matha}{m}{n}

\DeclareFontFamily{U}{mathx}{\hyphenchar\font45}
\DeclareFontShape{U}{mathx}{m}{n}{
	<-6> mathx5 <6-7> mathx6 <7-8> mathx7
	<8-9> mathx8 <9-10> mathx9
	<10-12> mathx10 <12-> mathx12
}{}
\DeclareSymbolFont{mathx}{U}{mathx}{m}{n}

\DeclareMathDelimiter{\vvvert} {0}{matha}{"7E}{mathx}{"17}%

\DeclarePairedDelimiterX{\opnorm}[1]{\vvvert}{\vvvert}{#1}

\makeatletter
\newcommand{\@hilopnormstar}[1]{\opnorm*{#1}_{\RKHS}}
\newcommand{\@hilopnormnostar}[2][]{\opnorm[#1]{#2}_{\RKHS}}
\newcommand{\hilopnorm}{\@ifstar\@hilopnormstar\@hilopnormnostar}
\makeatother

\makeatletter
\newcommand{\@muopnormstar}[1]{\opnorm*{#1}_{\distr}}
\newcommand{\@muopnormnostar}[2][]{\opnorm[#1]{#2}_{\distr}}
\newcommand{\muopnorm}{\@ifstar\@muopnormstar\@muopnormnostar}
\makeatother

\makeatletter
\newcommand{\@supnormstar}[1]{\norm*{#1}_{\infty}}
\newcommand{\@supnormnostar}[2][]{\norm[#1]{#2}_{\infty}}
\newcommand{\supnorm}{\@ifstar\@supnormstar\@supnormnostar}
\makeatother

\makeatletter
\newcommand{\@munormstar}[1]{\norm*{#1}_{\distr}}
\newcommand{\@munormnostar}[2][]{\norm[#1]{#2}_{\distr}}
\newcommand{\munorm}{\@ifstar\@munormstar\@munormnostar}
\makeatother

\makeatletter
\newcommand{\@distrnormstar}[2]{\norm*{#1}_{#2}}
\newcommand{\@distrnormnostar}[3][]{\norm[#1]{#2}_{#3}}
\newcommand{\distrnorm}{\@ifstar\@distrnormstar\@distrnormnostar}
\makeatother

\makeatletter
\newcommand{\@psinormstar}[2]{\norm*{#2}_{\psi_{#1}}}
\newcommand{\@psinormnostar}[3][]{\norm[#1]{#3}_{\psi_{#2}}}
\newcommand{\psinorm}{\@ifstar\@psinormstar\@psinormnostar}
\makeatother

\newcommand{\Rep}[1]{\ensuremath{\Phi_{#1}}}

\newcommand{\eventA}{\ensuremath{\mathcal{A}}}
\newcommand{\eventB}{\ensuremath{\mathcal{B}}}

\newcommand{\supZ}{\ensuremath{Z_{\numobs}}}
\newcommand{\supZbar}{\ensuremath{\overline{Z}_{\numobs}}}

\newcommand{\supH}{\ensuremath{H_{\numobs}}}
\newcommand{\supHbar}{\ensuremath{\overline{H}_{\numobs}}}

\newcommand{\termone}{\ensuremath{\nu}}
\newcommand{\termmtg}{\ensuremath{m}}
\newcommand{\termaprx}{\ensuremath{a}}
\newcommand{\stdmtg}{\ensuremath{\sigma_m}}
\newcommand{\lbstdmtg}{\ensuremath{\overline{\sigma}_m}}

\newcommand{\Dtermmtg}[1]{\ensuremath{\Delta\termmtg_{#1}}}
\newcommand{\eigfun}[1]{\ensuremath{\phi_{#1}}}

\newcommand{\bU}{\ensuremath{{\bf U}}}
\newcommand{\bu}{\ensuremath{\boldsymbol{u}}}
\newcommand{\bfun}{\ensuremath{\boldsymbol{f}}}
\newcommand{\bUbase}{\ensuremath{\bU_0}}
\newcommand{\bUper}{\ensuremath{{\bU}}}

\newcommand{\diag}{\ensuremath{{\rm diag}}}

\newcommand{\newradbar}{\ensuremath{\bar{\newrad}}}

\newcommand{\lbtheta}{\ensuremath{\vartheta}}

\newcommand{\scalar}{\ensuremath{\bar{c}}}
\newcommand{\scalarzero}{\ensuremath{\scalar_0}}
\newcommand{\scalarone}{\ensuremath{\scalar_1}}
\newcommand{\scalaroneone}{\ensuremath{\scalar_{1,1}}}
\newcommand{\scalaronetwo}{\ensuremath{\scalar_{1,2}}}
\newcommand{\scalartwo}{\ensuremath{\scalar_2}}

\newcommand{\Termtil}{\ensuremath{\widetilde{\Term}}}
\newcommand{\DTerm}{\ensuremath{\Delta\Term}}
\newcommand{\Terma}{\ensuremath{\Term_{a}}}
\newcommand{\Terms}{\ensuremath{\Term_{m}}}
\newcommand{\Termtila}{\ensuremath{\Termtil_{a}}}
\newcommand{\Termtils}{\ensuremath{\Termtil_{m}}}
\newcommand{\Termtilai}[1]{\ensuremath{\Termtil_{#1, \, a}}}
\newcommand{\Termtilsi}[1]{\ensuremath{\Termtil_{#1, \, m}}}
\newcommand{\DTerma}{\ensuremath{\DTerm_{a}}}
\newcommand{\DTerms}{\ensuremath{\DTerm_{m}}}
\newcommand{\DTermai}[1]{\ensuremath{\DTerm_{#1, \, a}}}
\newcommand{\DTermsi}[1]{\ensuremath{\DTerm_{#1, \, m}}}

\newcommand{\bZ}{\ensuremath{{\bf Z}}}
\newcommand{\bZbase}{\ensuremath{\widetilde{\bZ}}}
\newcommand{\DbZ}{\ensuremath{\Delta\bZ}}
\newcommand{\yvecbase}{\ensuremath{\widetilde{\yvec}}}
\newcommand{\Dyvec}{\ensuremath{\Delta\yvec}}

\newcommand{\tvec}{\ensuremath{\boldsymbol{t}}}

\newcommand{\tvectils}{\ensuremath{\widetilde{\tvec}_{m}}}

\newcommand{\tvectila}{\ensuremath{\widetilde{\tvec}_{a}}}

\newcommand{\walsh}[1]{\ensuremath{W_{#1}}}
\newcommand{\interval}[2]{\ensuremath{\Delta_{#2}^{(#1)}}}
\newcommand{\idxitv}{\ensuremath{u}}
\newcommand{\numitv}{\ensuremath{U}}

\newcommand{\idxstate}{\ensuremath{i}}
\newcommand{\idxstatenew}{\ensuremath{i'}}

\newcommand{\packvec}[1]{\ensuremath{\alphabold_{#1}}}
\newcommand{\pack}[2]{\ensuremath{\alphabold_{#1}(#2)}}

\newcommand{\ValueMt}[1]{\ensuremath{\boldsymbol{V}^{(#1)}}}
\newcommand{\ValueMtpar}[2]{\ensuremath{\boldsymbol{V}_{#2, \, \RKHSMt}^{#1}}}
\newcommand{\ValueMtperp}[2]{\ensuremath{\boldsymbol{V}_{#2, \, \perp}^{#1}}}

\newcommand{\Idxset}[1]{\ensuremath{\mathscr{T}(#1)}}
\newcommand{\idxset}[1]{\ensuremath{\mathscr{J}(#1)}}
\newcommand{\iidY}[1]{\ensuremath{Y_{#1}}}

\newcommand{\DiidY}[1]{\ensuremath{\Delta \iidY{#1}}}

\newcommand{\iidXi}[1]{\ensuremath{\Xi_{#1}}}

\newcommand{\hittimelb}{\ensuremath{\underline{\hittime}}}

\newcommand{\Tail}{\ensuremath{E}}

\newcommand{\rewardnorm}{\ensuremath{\varrho_r}}
\newcommand{\Vstarnorm}{\ensuremath{\varrho_V}}
\newcommand{\Dim}{\ensuremath{d}}

\newcommand{\Errone}{\ensuremath{\epsilon^2_m}}
\newcommand{\Errtwo}{\ensuremath{\epsilon^2_a}}

\newcommand{\epsilonbdnew}{\ensuremath{\epsilon}}

\newcommand{\variota}{\ensuremath{i}}

\newcommand{\Ellipse}{\ensuremath{\mathcal{E}}}
\newcommand{\Ellipsetil}{\ensuremath{\widetilde{\Ellipse}}}

\newcommand{\Constprob}{\ensuremath{c^\dagger}}

\newcommand{\Constdistrnew}{\ensuremath{\Const{\distrnew}}}

\usepackage{mdframed}
\usepackage{lipsum}
\usepackage{tcolorbox}
\tcbuselibrary{theorems}

\newtcbtheorem{emphbox}{}{colback=black!5,colframe=white!45!black,fonttitle=\bfseries}{th}
\newtcbtheorem{theorembox}{Theorem}{colback=black!5,colframe=white!45!black,fonttitle=\bfseries}{th}
\newtcbtheorem{lemmabox}{Lemma}{colback=black!5,colframe=white!45!black,fonttitle=\bfseries}{th}

\newenvironment{carlist}
{\begin{list}{$\bullet$}
		{\setlength{\topsep}{0in} \setlength{\partopsep}{0in}
			\setlength{\parsep}{0in} \setlength{\itemsep}{\parskip}
			\setlength{\leftmargin}{0.25in} \setlength{\rightmargin}{0.08in}
			\setlength{\listparindent}{0in} \setlength{\labelwidth}{0.08in}
			\setlength{\labelsep}{0.1in} \setlength{\itemindent}{0in}}}
	{\end{list}}

\newcommand{\Deltamtg}{\ensuremath{\Delta_{\mathsf{M}}}}
\newcommand{\Deltaaprx}{\ensuremath{\Delta_{\mathsf{A}}}}

\newcommand{\TD}{\ensuremath{{\rm TD}}}

\newcommand{\eligtr}{\ensuremath{\boldsymbol{z}}}
\newcommand{\Aop}{\ensuremath{{\mathcal{A}}}}
\newcommand{\bb}{\ensuremath{\boldsymbol{b}}}
\newcommand{\bbhat}{\ensuremath{\boldsymbol{\widehat{b}}}}
\newcommand{\Aophat}{\ensuremath{{\widehat{\mathcal{A}}}}}
\newcommand{\thetahatback}{\ensuremath{\thetahat^{\rm back}}}

\newcommand{\alphahat}{\ensuremath{\widehat{\alpha}}}
\newcommand{\alphavechat}{\ensuremath{\widehat{\boldsymbol{\alpha}}}}

\newcommand{\Length}{\ensuremath{L}}

\newcommand{\real}{\ensuremath{\Real}}

\newcommand{\myassumption}[3]{
	\begin{enumerate}[label={\bf{{{(#1)}}}}]
		\item \label{#2} {#3}
	\end{enumerate}
}

\newcommand{\mystatement}[3]{
	\begin{enumerate}[label={{{{(#1)}}}}]
		\item \label{#2} {#3}
	\end{enumerate}
}

\newcommand{\myunder}[2]{\underbrace{#1}_{\mbox{\small{#2}}}}

\newcommand{\myprecsim}{\ensuremath{\lesssim}}

\begin{document}

\begin{center}
  {\bf \LARGE Policy evaluation from a single path: \\
  Multi-step methods, mixing and mis-specification} \\ \vspace{1em}

 {\large{
 \begin{tabular}{ccc}
 Yaqi Duan$^\star$ && Martin J. Wainwright$^{\diamond,
 \dagger,\star}$
 \end{tabular}

 \medskip

 \begin{tabular}{c}
 Department of Electrical Engineering and Computer
 Sciences$^\star$\\
 Department of Mathematics$^\dagger$ \\
 Massachusetts Institute of Technology
 \end{tabular}

 \medskip 
 \begin{tabular}{c}
 Department of Electrical Engineering and Computer
 Sciences$^\diamond$ \\
 Department of Statistics$^\diamond$ \\
 UC Berkeley \\
 \end{tabular}

 }}

 \medskip
 
 \today
\end{center}

\medskip

\begin{abstract}
We study non-parametric estimation of the value function of an
infinite-horizon $\gamma$-discounted Markov reward process (MRP) using
observations from a single trajectory.  We provide non-asymptotic
guarantees for a general family of kernel-based multi-step temporal
difference (TD) estimates, including canonical $K$-step look-ahead TD
for $K = 1, 2, \ldots$ and the TD$(\lambda)$ family for $\lambda \in
[0,1)$ as special cases.  Our bounds capture its dependence on Bellman
  fluctuations, mixing time of the Markov chain, any mis-specification
  in the model, as well as the choice of weight function defining the
  estimator itself, and reveal some delicate interactions between
  mixing time and model mis-specification.  For a given TD method
  applied to a well-specified model, its statistical error under
  trajectory data is similar to that of i.i.d. sample transition
  pairs, whereas under mis-specification, temporal dependence in data
  inflates the statistical error.  However, any such deterioration can
  be mitigated by increased look-ahead.  We complement our upper
  bounds by proving minimax lower bounds that establish optimality of
  TD-based methods with appropriately chosen look-ahead and weighting,
  and reveal some fundamental differences between value function
  estimation and ordinary non-parametric regression.
\end{abstract}


\section{Introduction}

Reinforcement learning (RL)---a group of data-driven methods for
sequential decision-making---has been the focus of substantial
research in recent years.  The underlying impetus is its great
potential in a range of applications, including clinical
treatment~\cite{komorowski2018artificial,zhao2011reinforcement},
inventory management~\cite{giannoccaro2002inventory}, and industrial
process control~\cite{spielberg2017deep,nian2020review}, among many
others.  A property common to many such applications is that there
exist rich data sets based on past trials, but that new experiments
are costly and/or dangerous. For such problems---which correspond to
what is known as the offline setting in reinforcement learning---it is
not possible to collect data adaptively in an online manner; rather,
it is necessary to develop procedures that apply to previously
collected datasets in batch.

A key challenge in the offline setting is to evaluate the performance
of a given policy using existing data.  The quality of a given policy
can be measured by its value function $\Vstar$---that is, the expected
sum of (discounted) rewards under a trajectory generated by running
the given policy.  This value function is central to many
applications.  For example, in the setting of clinical treatments, it
might correspond to the expected long-term survival rate of septic
patients~(e.g.,~\cite{komorowski2018artificial}, whereas in inventory
management, it measures the profits/losses of a company over time
(e.g.,~\cite{giannoccaro2002inventory}).  The problem of \emph{policy
evaluation} refers to the problem of estimating either the full value
function, or a linear functional thereof, such as its value
$\Vstar(\state_0)$ at a particular starting state $\state_0$.  In this
paper, we focus on estimating the full value function.

In practice, policy evaluation is rendered challenging by the
complexity of the underlying state space, which can be of finite size
but prohibitively large, or continuous in nature.  In most cases of
interest, it is essential to use some type of function approximation
to compute what is known as a projected fixed point associated with
the Bellman operator.  The simplest choice is to search over the
linear span of a fixed set of features, as in the classical method of
least-squares temporal differences
(e.g.,~\cite{bradtke1996linear,tsitsiklis1997analysis,boyan1999least,bertsekas2011dynamic,sutton2018reinforcement,bertsekas2022abstract,bhandari2018finite}).
More generally, one can make use of techniques for non-parametric
regression in order to approximate the value function by computing a
projected fixed point. In particular, in this paper, we study
projected fixed point approximations that are based on reproducing
kernel Hilbert spaces (RKHSs).

In many applications of policy evaluation, the dataset takes the form
of one or more trajectories collected by applying the policy of
interest.  Such trajectory-based sampling models present both
challenges and opportunities.  On one hand, the statistical dependence
induced by trajectory sampling requires technical innovation: the
resulting estimators can no longer be analyzed using the standard
results in empirical process theory and concentration of measure that
are tailored to the i.i.d. setting.  Instead, it is essential to
understand the mixing time of the Markov process, and its effect on
the accuracy of value function estimates.  At the same time,
trajectory-based data opens the possibility of using more
sophisticated multi-step methods.  Recall that from classical dynamic
programming, the value function $\Vstar$ can be characterized as the
unique fixed point of the Bellman operator $\Bell$, and the standard
approach is to empirically approximate the projected fixed point
associated with this operator.  Given trajectory data, it becomes
possible to form empirical approximations to multi-step versions of
the Bellman operator---of the form $\BellOp{\bweight} \defn
\ssum{\step=1}{\Step} \weight{\step} \BellOp{\step}$ where the integer
$\Step \geq 1$ is the \emph{look-ahead parameter}, and $\bweight \in
\real_+^\Step$ is a vector of non-negative weights summing to one.
The TD$(\lambda)$-family is a well-known instance of this type of
approach.  Given the wide range of possible choices of look-ahead
$\Step$ and weight vector $\bweight{}$, one naturally wonders how to
make principled choices, and in particular ones that lead to better
estimators.  These types of questions, while long recognized as being
important in reinforcement learning
(e.g.,~\cite{jaakkola1993convergence,baird1995residual,bertsekas1996neuro,singh1998analytical,boyan1999least,yu2009convergence,mann2016adaptive,bhandari2018finite}),
are far from completely resolved.  In particular, what would be
desirable---and the goal of this paper---is theory that gives a very
precise understanding of the trade-offs involved, along with some
actionable guidelines for the practitioner.

In this paper, we explore these fundamental issues in the context of
$\discount$-discounted Markov reward processes, and using methods for
approximate policy evaluation based on reproducing kernel Hilbert
spaces
(e.g.,~\cite{shawe2004kernel,berlinet2011reproducing,Gu02,wainwright2019high}).
Our main contributions are to provide a non-asymptotic
characterization of the statistical properties of a broad class of
kernel-based policy evaluation procedures, with a particular emphasis
on how temporal dependencies and model mis-specification affect the
estimation error.  Our theory reveals some surprising phenomena, and
also provides guidance on the choice of look-ahead in multi-step
methods.

\subsection{A preview}

It is helpful to examine some simulations, so as to reveal the
phenomena of interest here, and provide a preview of the theory to
come.  We begin by providing some context for the results to be shown.
Given some function class $\Fclass$, suppose that we use a projected
fixed point procedure to compute an estimate $\thetahat$ of the value
function $\Vstar$.  As we describe in more detail
in~\Cref{SecEmpirical}, the estimate $\thetahat$ is a data-based
approximation of the idealized (population-level) fixed point
$\thetastar$ associated with our procedure.  Thus, using the triangle
inequality, the overall error $\|\thetahat - \Vstar\|$ in our estimate
$\thetahat$ can be upper bounded as
\begin{align*}
\|\thetahat - \Vstar\| & \leq \myunder{\|\thetahat -
  \thetastar\|}{Estimation error} + \myunder{\|\thetastar -
  \Vstar\|}{Approximation error}.
\end{align*}
The approximation error is deterministic in nature, and thus not
affected by changes in how the data is collected.\footnote{To be
clear, it is affected by the choice of function class $\Fclass$, along
with the operator used to define the projected fixed point.}
Accordingly, we focus our attention on the estimation error term, also
known as the statistical error.

Here we show some plots of the mean-squared estimation error $\Exp
\|\thetahat - \thetastar\|^2$ for estimates $\thetahat$ obtained using
various types of kernel-based least-squares temporal difference (LSTD)
estimates of the value function; see~\Cref{sec:set-up_def} for
details.


\paragraph{Comparison between i.i.d. and trajectory data:}

The first natural question is when there are differences between LSTD
estimates based on i.i.d. samples versus a single trajectory, and how
changes in the sampling model interact with the degree of
mis-specification.\footnote{Following standard statistical
terminology, we say that the model is mis-specified if the true value
function lies outside the RKHS used to compute the kernel-based LSTD
estimate.}  Here we explore this question via simulations using a
kernel-LSTD estimate with the standard look-ahead ($\Step = 1$), and
with a kernel function whose eigenvalues decay at a polynomial
rate.\footnote{Examples of kernels with this eigendecay include the
Laplacian kernel, and various types of spline kernels.}
See~\Cref{sec:experiment} (and in particular
equation~\eqref{eq:def_exp_eigvalue}) for more details.
\Cref{FigNsampMis} provides plots of the mean-squared error (MSE) of
the kernel-LSTD estimate versus the sample size $\numobs$, as applied
to two different sampling models.  Dashed lines correspond to
i.i.d. samples of successive pairs $\{(\state_i, \state_i')
\}_{i=1}^\numobs$ with $\state_i$ drawn from the stationary
distribution, and $\state_i'$ from the transition distribution.  Solid
lines correspond to a dataset consisting of pairs extracted from a
single trajectory $\traj = (\state_1, \state_2, \ldots,
\state_{\numobs})$, initialized with $\state_1$ drawn from the
stationary distribution.  In~\Cref{FigNsampMis}, panels (a) and (b)
show comparisons using Markov reward processes with or without model
mis-specification, respectively.  These plots reveal a number of
interesting phenomena to be explained:
\begin{figure}[t!]
  \begin{center}
    \begin{tabular}{cc}
      \widgraph{0.45\textwidth}{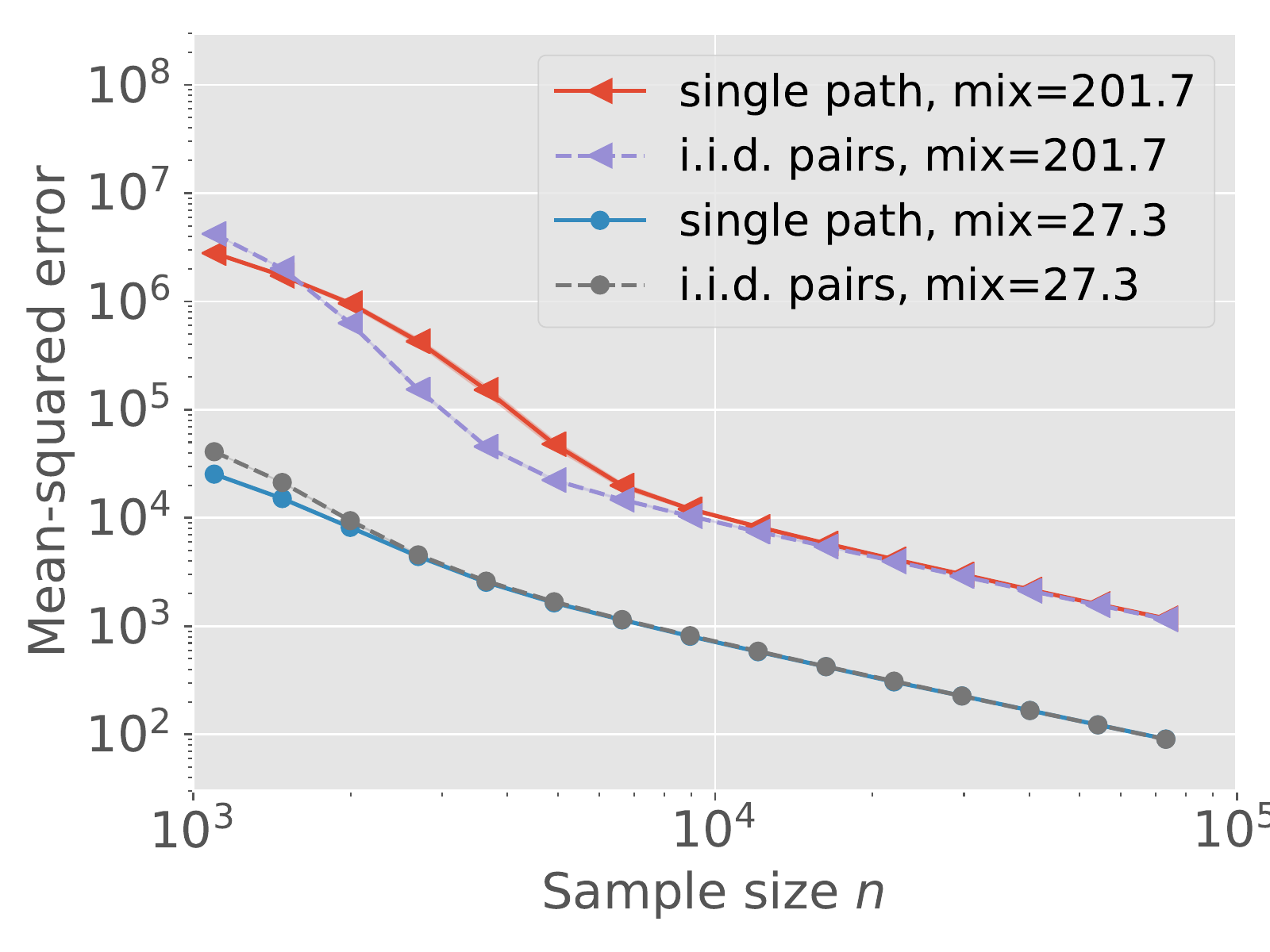}
      &
      \widgraph{0.45\textwidth}{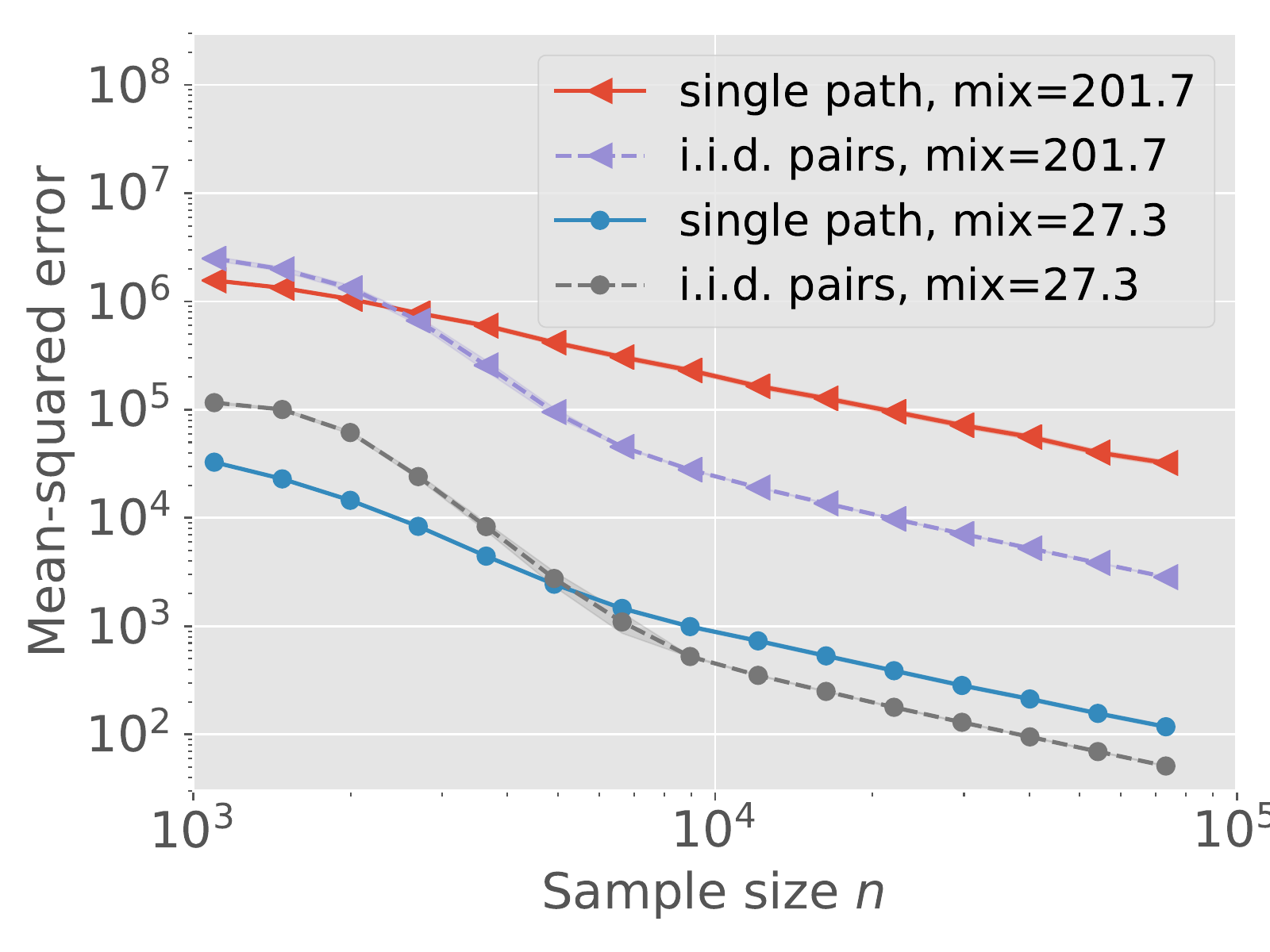}
      \\
      (a) Without model mis-specification. & (b) With model
      mis-specification.
    \end{tabular}
    \caption{Log-log plots of the mean-square error (MSE) versus the
      sample size $\numobs$ for standard (one-step) LSTD using data
      from a single path versus dataset formed of i.i.d. transition
      pairs. For each point on each curve on each plot, the MSE was
      approximated by taking a Monte Carlo average over $5000$ trials,
      with shaded areas delineating $3$ times standard errors.  (a) No
      model mis-specification: little difference between
      i.i.d. dataset and trajectory dataset.  This finding holds for
      both the fast and slow-mixing chains.  (b) With model
      mis-specification: MSE from trajectory data is much larger than
      the i.i.d.  case, with the degradation being more significant
      in the slow-mixing case.}
    \label{FigNsampMis}
  \end{center}
\end{figure}
\begin{itemize} \itemsep = -.1em
  \item If the MRP is well-specified (\Cref{FigNsampMis}(a)), then
    after the sample size exceeds some certain threshold, the MSEs of
    using trajectory data and i.i.d. pairs have extremely similar
    scalings. Thus, for a well-specified model, estimates from a
    single path are as good as those from i.i.d. pairs despite the
    (potentially complicating) temporal dependence in data.
  \item In contrast, when there is a material amount of model
    mis-specification (\Cref{FigNsampMis}(b)), the temporal dependence
    starts to affect the estimation error. For a sufficiently large
    sample size, using a single path leads to a worse estimate.  This
    reduction in quality is monotonic in the amount of dependence: for
    a more slowly mixing chain (i.e., with stronger dependence), the
    difference between i.i.d. and trajectory sampling is even larger.
\end{itemize}


\paragraph{Different step-lengths in TD methods:}

A second natural question is the effect of different look-ahead
$\Step$ on a multi-step projected fixed point estimate.
\Cref{FigNsampTD} provides comparisons of TD estimates with look-ahead
lengths $\Step \in \{1, 5, 10 \}$, as applied to a discounted MRP with
$\discount = 0.9$.  We conducted three groups of experiments in total,
corresponding to the following types of MRP instances: (i) slowly
mixing but well-specified (no mis-specification); (ii) large
mis-specification but rapidly mixing; or (iii) large mis-specification
and slowly mixing.  As indicated in the figure, panel (a) involves the
first two cases~(i) and (ii), whereas panel (b) provides results for
case (iii).

\begin{figure}[h!]
  \begin{center}
    \begin{tabular}{cc}
      \widgraph{0.45\textwidth}{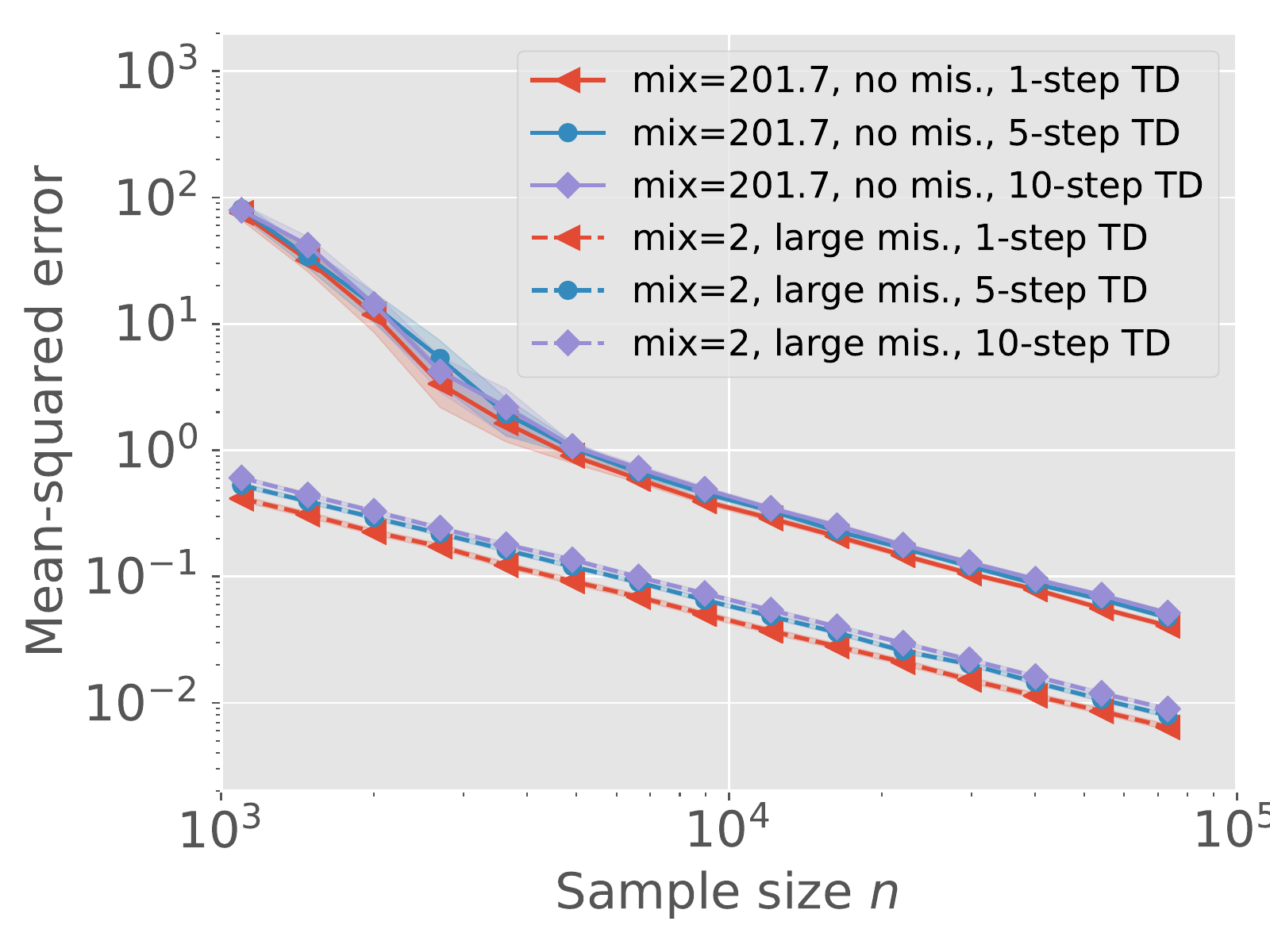}
      &
      \widgraph{0.45\textwidth}{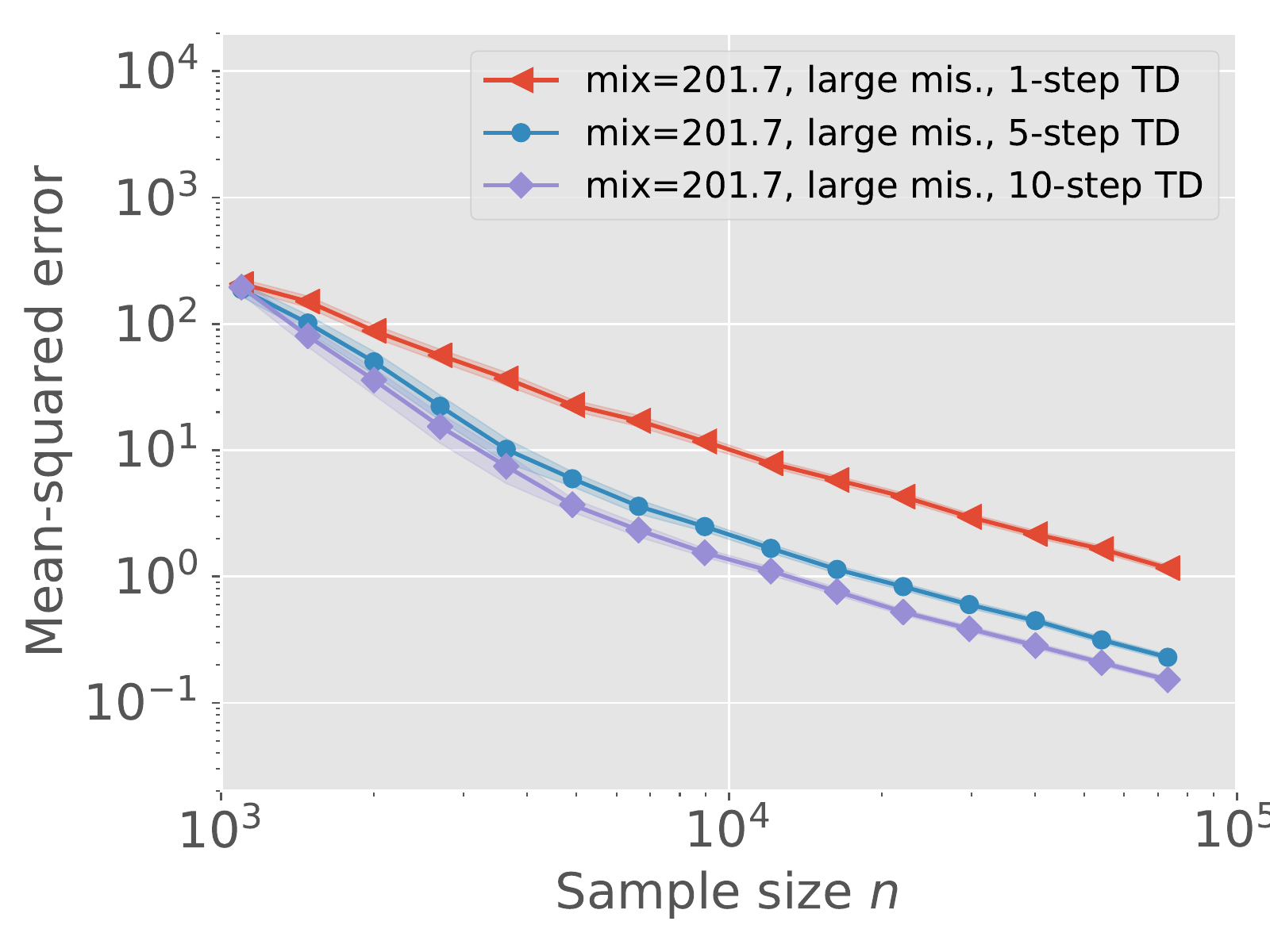} \\
        {\small{ (a) No mis-specification \emph{or} fast mixing}}
    &  {(b) \small{Large mis-specification \emph{and}  slow
                mixing.}}
    \end{tabular}
    \caption{Log-log plots of the mean-square error (MSE) versus the
      sample size $\numobs$ for different multi-step temporal
      difference (TD) estimates when using data from a single
      path. For each point on each curve on each plot, the MSE was
      approximated by taking a Monte Carlo average over $5000$ trials;
      $3$ times sample errors are shown by the shaded area. (a) No
      mis-specification \emph{or} fast mixing: When there are enough
      samples, the MSEs of TD estimates with different look-ahead have
      similar scales. (b) Large mis-specification \emph{and} slowly
      mixing: The MSE is smaller for larger step $\Step$ in the TD
      estimates.}
    \label{FigNsampTD}
  \end{center}
\end{figure}

From panel (a), we see that, for both cases (i) and (ii), the choice
of look-ahead $\Step$ has little effect; all three methods ($\Step \in
\{1, 5, 10 \}$) behave very similarly.  This behavior should be
contrasted with case (iii): as shown in panel (b), in this setting,
increasing the look-ahead $\Step$ leads to substantial reductions in
the MSE.  Thus, while some settings are unaffected by look-ahead
choice, changing $\Step$ does have a very significant effect for a
model that is both mis-specified and slowly mixing.  The theory to be
developed in this paper will explain this and other related phenomena.


\subsection{Related work}

This paper builds upon our earlier work~\cite{duan2021optimal}, in
which we studied the properties of the standard one-step ($\Step = 1$)
least-squares temporal difference (LSTD) estimate in its kernelized
form.  This past work was restricted to the so-called generative
setting, in which observations consist of i.i.d. state-next-state
pairs.  In contrast, the major challenge addressed here is to move
beyond this highly idealized setting by accommodating the
trajectory-based models that arise in practice, and to provide a
precise characterization of a much broader class of multi-step
estimates.  As noted above, this generality allows us to characterize
the delicate interaction between mixing time, mis-specification, and
choice of look-ahead.

There is large body of past work on analyzing LSTD procedures as
applied with i.i.d. data
(e.g.,\cite{munos2008finite,farahmand2016regularized,liu2015finite,fan2020theoretical,long20212}).
Of most direct relevance here is a line of past work on policy
evaluation and optimization for trajectory-based models.  In early
work, Antos al.~\cite{antos2008learning} studied policy iteration
using single trajectory generated under a fixed policy. Under a
$\beta$-mixing condition, they proved various non-asymptotic bounds on
both the estimation of the value function, as well as the
sub-optimality of the associated policy.  Their analysis, involving
VC-crossing dimension to measure the function complexity, guarantees
consistency as the trajectory length increases, but the underlying
rates are slow (and hence sub-optimal).  Focusing on the special case
of linear function approximation, Lazaric et
al.~\cite{lazaric2012finite} proved non-asymptotic bounds for both
standard LSTD and least-squares policy iteration; their bounds both
the feature dimension, and the smallest eigenvalue of the Gram matrix.
Bhandari et al.~\cite{bhandari2018finite} provided non-asymptotic
bounds for temporal difference learning.  When applied to data from a
single Markov trajectory, their bounds involve a multiplicative factor
of the mixing time relative to the i.i.d. case.  In application to
$\TD(\lambda)$ algorithms, their analysis does not capture the
possible benefits of increased $\lambda$ in reducing statistical
estimation error that we document in this work.  It should be noted
that bounds in the aforementioned
papers~\cite{antos2008learning,lazaric2012finite,bhandari2018finite}
do not isolate the variance structure of the policy evaluation
problem, which is essential to establishing the statistical optimality
of the estimates.

Some recent work, involving a subset of the current authors, does
isolate this variance structure in the linear case.  Mou et
al.~\cite{mou2021optimal} studied stochastic approximation procedures
for solving linear fixed point equations over $\real^d$, given
observations from a single trajectory of an underlying Markov
chain. Among the consequences of their general theory are
instance-dependent guarantees for the MSE of $\TD(\lambda)$
methods. In the context of this paper, this analysis was limited to
finite-rank kernels, whereas our primary interest is the more general
non-parametric instantiations of kernels.  We provide a more detailed
comparison between our results and this work
in~\Cref{sec:ub_exp_linear}.

\subsection{Our contributions and paper organization}

Our main contribution is a precise analysis, including both achievable
upper bounds and fundamental lower bounds, on policy evaluation based
on one or more observed trajectories.  We state and prove two main
results.  Our first result (\Cref{thm:ub}) applies to broad class of
kernel-based projected fixed point estimators, and gives
high-probability upper bounds on the associated estimation error.
These upper bounds are specified in terms of a \emph{signal-to-noise
ratio}, or SNR for short, one which captures the essential difficulty
of value function estimation.  We identify two different types of
fluctuations, denoted by $\stdmtg$ and $\approxerr$ respectively, that
correspond to martingale noise, and error due to the Bellman residual,
respectively. The martingale noise exhibits behavior similar to that
of independent random variables, whereas the temporal dependence in
the underlying Markov chain interacts with the Bellman residual to
form $\approxerr$.  Our characterization of these interactions allows
us to predict the phenomena illustrated in the preceding simulations,
and has a number of interesting implications.  As one example,
consider the natural intuition about multi-step TD methods---as
written about in past work on the
topic~\cite{bertsekas1996neuro,boyan1999least,yu2009convergence}---that
increasing look-ahead, which is known to reduce approximation error,
will increase the estimation error.  The results in this paper reveal
many scenarios in which estimation error is not increased by larger
choices of look-ahead parameter; other factors dictate the limits of
choosing look-ahead.

We complement our upper bounds with a minimax lower bound
(\Cref{thm:lb}) that prescribes fundamental limits to any policy
evaluation procedure over certain kernel classes.  Both of the
preceding two terms also appear in the lower bound: the martingale
term $\stdmtg$ arises from perturbations to value functions via the
dynamics of the Markov chain, whereas the term $\approxerr$ is induced
by shifts of the stationary distribution.

An important take-away from this paper is the interaction between
model mis-specification---i.e, the gap between the true value function
$\Vstar$ and its best approximation within a given class
$\Fclass$---and the statistical difficulty of estimating the best
approximation.  This interaction should be contrasted with static
non-parametric regression problems, for which it is straightforward to
disentangle the approximation and estimation errors.  In this paper,
we measure model mis-specification in an instance-dependent (and hence
not worst-case way), as either the $L^2$-distance between $\Vstar$ and
its projection onto $\Fclass$, or the Bellman residual associated with
the projected fixed point.  This instance-dependence provides a more
refined view than worst-case notions, such ``realizability'' or
``completeness''
(e.g.,~\cite{munos2008finite,farahmand2016regularized,chen2019information,duan2020minimax,uehara2021finite,duan2021risk,zanette2021exponential}),
along with approximate versions
thereof~\cite{munos2008finite,chen2019information,uehara2021finite,duan2021risk},
that have been used to specify approximation error in past work on
reinforcement learning.  However, it should be noted that the global
nature of our measure of approximation error makes it more restrictive
than pointwise notions that have been used for estimating functionals
of value functions (e.g.,~\cite{ZanWai22_Galerkin_Conf}).


\paragraph{Paper organization:}
The remainder of this paper is organized as
follows. In~\Cref{sec:set-up}, we begin by introducing the background
of Markov reward process, value function estimation, multi-step
Bellman equations and estimates as well as reproducing kernel Hilbert
spaces (RKHSs) used for function
approximation. In~\Cref{sec:main,sec:lb}, we present the statements of
non-asymptotic upper bounds (\Cref{sec:ub}), some discussion of their
consequences (\Cref{sec:ub_exp,sec:ub_simple}), and minimax lower
bounds (\Cref{sec:lb}).  \Cref{sec:proof} is devoted to the proofs of
the upper and lower bounds, accompanied by interpretations of the
terms that set the noise levels.


\paragraph{Notation:}
Throughout the paper, we use $\PlainCon, \plaincon, \const{0}$ etc. to
denote universal constants whose numerical values may very from line
to line. For any positive integer $D$, let $[D]$ be the collection of
numbers $\{ 1,2,\ldots,D \}$.  Given a distribution $\distr$, we
define the $\Lmu$-norm $\munorm{f} \defn \sqrt{\int f^2 \distr(\dx)}$.
We also make use of the supremum norm $\supnorm{f} \defn \sup_{\state
  \in \StateSp} |f(\state)|$.  For two measures $p$ and $q$ with $p$
absolutely continuous with respect to $q$, we define the
Kullback–Leibler (KL) divergence $\kull{p}{q} \defn \Exp_p \big[ \log
  \big( \frac{\diff p}{\diff q} \big) \big]$, along with the
$\chi^2$-divergence $\chi^2\kulldiv{p}{q} \defn \Exp_q \big[ \big(
  \frac{\diff p}{\diff q} - 1 \big)^2 \big]$.


\section{Background and problem set-up}
\label{sec:set-up}

In this section, we provide background and then set up the problem to
be studied in this paper.  We begin in~\Cref{sec:set-up_MRP} with
background on Markov reward processes, value functions, and multi-step
Bellman equations.  In~\Cref{sec:set-up_fixpoint}, we discuss how
projected fixed points can be used to approximate the solution to the
Bellman equation.  \Cref{sec:set-up_def} introduces the reproducing
kernel Hilbert spaces (RKHS) that we use as approximating function
classes in this paper, along with the kernel-based multi-step temporal
difference (TD) estimates that they define.


\subsection{Markov reward processes and Bellman operators}
\label{sec:set-up_MRP}

For a given discount factor $\discount \in [0,1)$, a
  $\discount$-discounted Markov reward process consists of a
  time-homogeneous Markov chain on a state space $\StateSp$, combined
  with a reward function $\reward$ that maps each state $\state$ to a
  scalar reward $\reward(\state)$. The Markov chain is defined by a
  transition function $\TransOp$, so that when the chain is in state
  $\state$ at the current time, it transitions to a random state
  $\Statenew$ drawn according to a probability distribution
  $\TransOp(\cdot \mid \state)$.

The value function measures the expected value of a geometrically
discounted sum of the rewards over a trajectory of the Markov chain.
In particular, for each possible starting state $\state \in \StateSp$,
we define
\begin{align}
\label{eq:def_Vstar}
\textstyle \Vstar(\state) \defn \Exp\big[ \sum_{t=0}^\infty
  \discount^t \, \reward(\State_{t}) \bigm| \State_0 = \state \big],
\end{align}
where the expectation is taken over a trajectory $(\state, \State_1,
\State_2, \ldots)$ that is governed by the probability transition
operator~$\TransOp$.  The existence and well-definedness of the value
function $\Vstar$ is guaranteed under mild conditions.  It is
convenient to rewrite definition~\eqref{eq:def_Vstar} in
operator-theoretic notation as \mbox{$\Vstar = (\IdOp - \discount \,
  \TransOp)^{-1} \reward$.}

In this paper, we study the problem of estimating the value function
$\Vstar$ based on a set of observations from a single trajectory
$\traj = (\state_1, \state_2, \ldots, \state_{\numobs}) \in
\StateSp^{\numobs}$ from the Markov chain, where $\state_1$ is drawn
from the stationary distribution.  We assume that the reward function
$\reward$ is known, so that the rewards $\reward(\state_i)$ are also
given.  Our results extend to the case of unknown reward function, but
we study the known reward case for the bulk of our analysis so as to
draw attention to the differences between multi-step Bellman operators
(all of which share the same reward structure).

Letting $\distr$ correspond to the stationary distribution of the
Markov chain, we measure the error associated with an estimate
$\thetahat$ of $\Vstar$ in terms of the squared-$\Lmu$-norm
\begin{align*}
  \munorm{\thetahat - \Vstar}^2 \defn \Exp \big[ (\thetahat(\State) -
    \Vstar(\State))^2 \big].
\end{align*}

The estimates studied in this paper are established based on the
observation that, for any positive integer $\step = 1, 2, \ldots$, the
value function $\Vstar$ is the solution to the $\step^{th}$-order
Bellman fixed point equation
\begin{align*}
\Vstar(\state) = \reward(\state) + \discount \, \Exp_{\State_1 \mid
  \state} \big[ \reward(\State_1) \big] + \ldots + \discount^{\step-1}
\, \Exp_{\State_{\step-1} \mid \state} \big[ \reward(\State_{\step-1})
  \big] + \discount^{\step} \, \Exp_{\State_{\step} \mid \state} \,
\big[ \Vstar(\State_{\step}) \big].
\end{align*}
For natural reasons, we refer to the integer $\step$ as the number of
\emph{look-ahead steps}.

For future reference, we introduce a more concise formulation of this
fixed point relation as \mbox{$\Vstar = \BellOp{\step}(\Vstar)$,}
where the \emph{$\step$-step Bellman operator} $\BellOp{\step}$ is
given by
\begin{subequations}
\begin{align}
\label{EqnBellmanStep}
\big(\BellOp{\step}(f)\big)(\state) \defn \Exp \bigg[ \, \sum_{\ell =
    0}^{\step-1} \discount^\ell \reward(\State_\ell) +
  \discount^{\step} f(\State_{\step}) \mid \State_0 = \state \, \bigg]
\qquad & \text{for any $f \in \Lmu$ and $\state \in \StateSp$}.
\end{align}
More generally, we can form convex combinations of operators of this
type.  As one possible formalization, fix a positive integer $\Step
\geq 1$, and consider the class of all \emph{weighted $\Step$-step
Bellman operators}
\begin{align}
\label{EqnDefnWeightedBell}
\BellOp{\bweight} \defn \ssum{\step=1}{\Step} \weight{\step}
\BellOp{\step},
\end{align}
\end{subequations}
where the non-negative weight vector \mbox{$\bweight = \begin{pmatrix}
    \weight{1} & \ldots & \weight{\Step}
	\end{pmatrix}$} ranges over the probability simplex in $\Real^{\Step}$.
Given these constraints, it can be verified that any such weighted
operator $\BellOp{\bweight}$ also has the original value function
$\Vstar$ as its unique fixed point.  Notice that if we observe a
single trajectory of length $\numobs$, we can (in principle) try to
approximate a $\Step$-step weighted Bellman operator for any
\mbox{$\Step \in \{1, 2, \ldots, \numobs-1 \}$.}


\subsection{Approximate policy evaluation by projected fixed points}
\label{sec:set-up_fixpoint}

For problems with sufficiently complex state spaces, it often becomes
necessary to seek only approximate solutions to the Bellman equations.
A broad class of such approximation procedures arise via the framework
of projected fixed points, which we introduce here.

Let $\mathds{G}$ be a convex and closed set of functions contained
within $\Lmu$ that is used to approximate the value function.  Given
any such class, we can define the projection operator $\projG: \Lmu
\rightarrow \mathds{G}$ via $\projG(f) \defn \argmin_{g \in
  \mathds{G}} \munorm{g - f}$.  Using this projection operator
$\projG$ and the weighted Bellman operator $\BellOp{\bweight}$ from
equation~\eqref{EqnDefnWeightedBell}, we define the \emph{projected
fixed point}
\begin{align}
  \label{EqnDefnProjFix}
  \thetastar = \projG \big( \BellOp{\bweight}(\thetastar) \big).
\end{align}
To be clear, different choices of the weight vector $\bweight \in
\Step$ and function space $\mathds{G}$ lead to different projected
fixed points $\thetastar$, but so as to avoid clutter, we suppress
this dependence in our notation.  Given any weight vector $\bweight
\in \Real^{\Step}$, we define the \emph{effective discount factor} of
the $\bweight$-weighted TD estimate as
\begin{align}
\label{EqnDefnEffDiscount}
\discounteff & \defn \ssum{\step=1}{\Step} \weight{\step}
\discount^{\step}.
\end{align}
Note that $\discounteff \leq \discount$ for any choice of the weight
vector in the probability simplex.

Let us consider a few examples to illustrate. As a first example, in
the standard $\Step$-step temporal difference method, the weight
vector is given by $\weight{\Step} = 1$, and $\weight{\ell} = 0$ for
$\ell \neq \Step$. This choice leads to the effective discount factor
$\discounteff = \discount^{\Step}$. Given an integer $\Step \geq 1$, a
second example is the $\Step$-truncated $\TD(\lambda)$ method, in
which the weight vector takes the form
\begin{align*}
\bweight = \frac{1-\lambda}{1-\lambda^{\Step}} \begin{bmatrix} 1 &
\lambda & \lambda^2 & \ldots & \lambda^{\Step-1}
\end{bmatrix} \qquad \mbox{for some $\lambda \in [0,1)$.}
\end{align*}
This choice leads to an effective discount factor $\discounteff =
\frac{\discount \, (1-\lambda)}{1-\lambda\discount}
\frac{1-\lambda^{\Step}\discount^{\Step}}{1-\lambda^{\Step}}$. If we
take the limit as $\Step \rightarrow \infty$, then we see that
\begin{align*}
\discounteff \rightarrow \frac{\discount \,
	(1-\lambda)}{1-\lambda\discount}, \quad \mbox{and} \quad
(1-\discounteff)^{-1} \rightarrow
\frac{1-\lambda\discount}{1-\discount}.
\end{align*}


\subsection{Kernel-based multi-step temporal difference estimates}
\label{sec:set-up_def}

In addition to the weight vector, the second choice in the
specification of a projected fixed point~\eqref{EqnDefnProjFix} is the
approximating function class.  In this paper, we study multi-step
temporal difference (TD) estimates that are based on reproducing
kernel Hilbert spaces, or RKHSs for short.  Any such function class
space (RKHS) is defined by a bivariate function $\Ker: \StateSp \times
\StateSp \rightarrow \Real$ that is symmetric in its arguments, and
positive semidefinite.  Within the RKHS framework, such a kernel is
used to define an inner product $\hilin{\cdot}{\cdot}$ along with the
associated Hilbert norm $\hilnorm{\cdot}$.  Throughout this paper, we
assume that the function space $\RKHS$ contains all constant
functions.

Given such a Hilbert space, the population-level objects of interest
in this paper are the projected fixed points~\eqref{EqnDefnProjFix}
defined by the operators $(\projH, \BellOp{\bweight})$ for some choice
of weight vector $\bweight$.  Using the machinery of reproducing
kernels, it turns out this such projected fixed points can be defined
in an equivalent, and arguably more explicit, manner in terms of the
covariance and cross-covariance operators associated with the Hilbert
space.  For each $\state \in \StateSp$, we define the function
$\Rep{\state}(\cdot) \defn \Ker(\cdot, \state)$.  By classical RKHS
theory, this function is the representer of evaluation, meaning that
we have $\hilin{f}{\Rep{\state}} = f(\state)$ for all $f \in \RKHS$.


\subsubsection{Population-level kernel-LSTD estimates}

We now describe how the (population-level) projected fixed
points~\eqref{EqnDefnProjFix} can be written as the solution of a
linear operator equation defined in terms of covariance and
cross-covariance operators associated with the RKHS.  The
\emph{covariance operator} is a mapping on $\RKHS$ given by
\begin{subequations}
 \begin{align}
 \label{eq:def_CovOp}
 \CovOp & \defn \Exp_{\State \sim \distr} \big[
   \Rep{\State} \otimes \Rep{\State} \big].
 \end{align}
In more explicit terms, using the representer property of
$\Rep{\State}$, the covariance operator performs the mapping $f
\mapsto \CovOp(f) \defn \Exp_{\State \sim \distr} \big[ \Rep{\State}
  f(\State) \big]$, and so is a linear operator.  Similarly, for a
weight vector $\bweight$ in the $\Step$-dimensional probability
simplex, we define the \emph{$\bweight$-weighted cross-covariance
operator}
\begin{align}
\label{eq:def_CrOpw}
 \CrOpw & \defn \Exp_{\{ \State_{\step} \}_{\step=0}^{\Step} \sim
   (\distr, \TransOp)} \bigg[ \Rep{\State_0} \otimes \Big\{
   \sum_{\step=1}^{\Step} \weight{\step} \discount^{\step}
   \Rep{\State_{\step}} \Big\} \bigg],
\end{align}
along with the \emph{$\bweight$-weighted reward function}
\begin{align}
\label{eq:def_by}
 \by & \defn \Exp_{\{ \State_{\step} \}_{\step=0}^{\Step} \sim
 (\distr, \TransOp)} \bigg[ \Rep{\State_0} \Big\{
 \sum_{\step=1}^{\Step} \weight{\step} \, \ssum{\ell=1}{\step-1}
 \discount^\ell \reward(\State_\ell) \Big\} \bigg].
 \end{align}
\end{subequations}

As shown in~\Cref{AppEquiv}, the projected fixed point $\thetastar$
defined as in equation~\eqref{EqnDefnProjFix} with the pair $(\projH,
\BellOp{\bweight})$ is the unique solution to the linear operator
equation
\begin{align}
  \label{eq:proj_Bell}
  \CovOp \, \thetastar = \CovOp \, \reward + \by + \CrOpw \, \thetastar \,.
\end{align}
In~\Cref{append:backward}, we also provide an equivalent form of the
fixed point equation~\eqref{eq:proj_Bell}, which is known as a
``backward'' formula.

\subsubsection{Empirical kernel-LSTD estimate}
\label{SecEmpirical}
The operators $(\CovOp, \CrOpw)$ and function $\by$ defining the
population-level projected fixed point~\eqref{eq:proj_Bell} are
unknown to us, but can be approximated using data.  Given the observed
trajectory $(\state_1, \ldots, \state_{\numobs})$, we can form the
empirical estimates
\begin{subequations}
 \begin{align}
 \label{eq:def_ophat}
 \CovOphat & \defn \frac{1}{\numobs-\Step} \sum_{t=1}^{\numobs -
 \Step} \Rep{\state_t} \otimes \Rep{\state_t} \, , \\ \CrOpwhat &
 \defn \frac{1}{\numobs-\Step} \sum_{t=1}^{\numobs - \Step}
 \Rep{\state_t} \otimes \Big\{ \sum_{\step=1}^{\Step}
 \weight{\step} \discount^{\step} \Rep{\state_{t+\step}} \Big\} \,
 , \quad \mbox{and} \label{eq:def_CrOpwhat} \\ \byhat & \defn \frac{1}{\numobs-\Step}
 \sum_{t=1}^{\numobs-\Step} \Rep{\state_t} \Big\{
 \sum_{\step=1}^{\Step} \weight{\step} \, \ssum{\ell=1}{\step-1}
 \discount^\ell \reward(\state_{t+\ell}) \Big\} \, . \label{eq:def_byhat}
 \end{align}
\end{subequations}

For a user-defined regularization parameter $\ridge > 0$, the
empirical LSTD estimate $\thetahat$ corresponds to the solution of the
equation
\begin{align}
 \label{eq:def_thetahat}
(\CovOphat + \ridge \IdOp) \, \thetahat \; = \; (\CovOphat + \ridge
 \IdOp) \, \reward + \byhat + \CrOpwhat \, \thetahat.
\end{align}
Note that computing this estimate, in the form given here, involves
solving an operator equation in the Hilbert space.  However, as
described in~\Cref{AppComputation}, this computation can be reduced to
the solution of a linear system over $\real^\numobs$.

Similar to the empirical LSTD estimate $\thetahat$, it is also
possible to compute a closely related (but slightly different)
estimate $\thetahatback$ using the backward formulation of the
projected Bellman
equation~\eqref{eq:proj_Bell}. See~\Cref{sec:backward_estimate} for
details of this formulation.  The difference between $\thetahat$ and
$\thetahatback$ is negligible when the sample size is relatively
large.  The backward variant $\thetahatback$ is closely related to the
standard description of temporal difference learning as a form of
stochastic approximation; this relation is also described
in~\Cref{sec:backward_estimate}.

\subsubsection{Connection to empirical Bellman operator}

From a conceptual point of view, it is also useful to view the
estimator $\thetahat$ as the fixed point of an empirical Bellman
operator $\BellOphat{\bweight}$. For each $\step \in \{1, \ldots,
\Step \}$, and time step $t \in \{ 1, 2, \ldots, \numobs - \step \}$,
define the $\step$-step future return as
\begin{align}
\label{eq:def_returnhat}
\ReturnHat{t+1}{t+\step}(f) \defn \ssum{\ell=1}{\step-1}
\discount^\ell \reward(\state_{t+\ell}) + \discount^\step
f(\state_{t+\step}).
\end{align}
In terms of these future returns, the \emph{empirical Bellman
operator} is given by
\begin{align*}
 f \mapsto \BellOphat{\bweight}(f) & \defn \reward + \argmin_{h \in
 \RKHS} \left \{ \frac{1}{\numobs - \Step} \sum_{t=1}^{\numobs-\Step} \Big(
 h(\state_t) - \ssum{\step=1}{\Step} \weight{\step} \,
 \ReturnHat{t+1}{t+\step} (f) \Big)^2 + \ridge \hilnorm{h}^2 \right \},
\end{align*}
defined for any $f \in \Lmu$.  With this notation, it can be shown
that the estimate $\thetahat$ can equivalently be defined by the fixed
point equation $\thetahat = \BellOphat{\bweight}(\thetahat)$.


\section{Non-asymptotic upper bounds on multi-step kernel LSTD}
\label{sec:main}

In this section, we develop some non-asymptotic theory for the
estimation error associated with the function $\thetahat$ computed
using multi-step kernel LSTD method.  From the introduction, its
overall error as an estimate of the true value function $\Vstar$ is
upper bounded as
\begin{align}
  \label{eq:ub_err_decomp}
  \munorm{\thetahat - \Vstar} \leq \myunder{\munorm{\thetahat -
      \thetastar}}{Estimation error} + \myunder{\munorm{\thetastar -
      \Vstar}}{Approximation error}.
\end{align}
The approximation error $\munorm{\thetastar - \Vstar}$ is
deterministic in nature, and controlled by the richness of the
underlying RKHS, as well as the choice of weight vector $\bweight{}$
in a multi-step TD method.  The goal of this section is to
characterize the statistical estimation error $\munorm{\thetahat -
  \thetastar}$ associated with estimating the projected fixed point
$\thetastar$.  In our discussion, we return to comment about the
overall error including the approximation error.


\subsection{Non-asymptotic upper bounds}
\label{sec:ub}

In this part, we provide non-asymptotic upper bounds on the
statistical error $\munorm{\thetahat - \thetastar}$.
\Cref{sec:stat_ub_thm} provides the statement of the upper bound
involving solutions to a critical inequality.  \Cref{sec:SNR_lb}
presents bounds on the noise level that appears in the critical
inequality.  \Cref{sec:multi_traj} discusses how to generalize the
theorem to data collected from multiple trajectories.

\subsubsection{Statement of upper bound}
\label{sec:stat_ub_thm}

Our analysis relies on the following mixing condition, which involves
a scalar $\mixtime \geq 1$, known as the \emph{mixing time}, and a
nonnegative constant $\Constdistrnew < \infty$.
\myassumption{MIX$(\mixtime)$}{assump:mixing}{There exists a
  probability measure $\distrnew$ on $\StateSp$ with
  \mbox{$\supnorm[\big]{\tfrac{\diff \, \distrnew}{\diff \, \distr}}
    \leq 1 + \Constdistrnew < \infty$} such that
\begin{align}
\label{cond:minor}
\inf_{\state \in \StateSp} \TransOp(B \mid \state) \; \geq
\frac{1}{\mixtime} \distrnew(B) \qquad \mbox{for all sets $B$ in the
  Borel $\sigma$-field $\Borel(\State)$.}
\end{align}
  } 
We note that the mixing condition~\eqref{cond:minor} is slightly
stronger\footnote{To clarify, the mixing condition~\eqref{cond:minor}
is equivalent to assuming that the state space $\StateSp$ is
$\distrnew_1$-small with \mbox{$\distrnew_1 \defn \tfrac{1}{\mixtime}
  \, \distrnew$}.  From known results (e.g., Theorem 16.0.2 in the
book~\cite{meyn2012markov}), uniform ergodicity ensures that the state
space $\StateSp$ is $\distrnew_m$-small for some integer $m \geq 1$,
and some non-trivial measure $\distrnew_m$.} than uniform ergodicity.
It has been used in various papers
(e.g.,~\cite{adamczak2008tail,adamczak2015exponential,dedecker2015subgaussian,lemanczyk2021general})
that establish concentration inequalities for Markov chains. \\

In addition to this mixing condition, our analysis imposes some
boundedness conditions on the kernel function, as well as the
covariance operator $\CovOp$ (cf. definition~\eqref{eq:def_CovOp})
that it induces.  This operator acts on the space $\Lmu$ as
\begin{align*}
(\CovOp \, f)(\cdot) = \int_{\StateSp} \Ker(\state, \cdot) \,
  f(\state) \, \distr(\dx) \qquad \text{for any function $f \in
    \Lmu$}.
\end{align*}
Under mild regularity conditions, this covariance operator has a discrete
collection of eigenvalues $\{\mu_j\}_{j=1}^\infty$ along with associated
eigenfunctions $\{\feature{j} \}_{j=1}^\infty$, orthonormal in $\Lmu$.
We impose the following regularity condition:
\myassumption{KER$(\bou, \unibou)$}{ass:kernel}{The kernel function
  $\Ker$ and eigenfunctions $\{ \feature{j} \}_{j = 1}^{\infty}$ are
  uniformly bounded---viz.
\begin{align}
 \label{eq:def_bou}
 \sup_{\state \in \StateSp} \sqrt{\Ker(\state, \state)} \leq \bou
 \qquad \text{and} \qquad \sup_{j \in \Int_+} \supnorm{\feature{j}}
 \leq \unibou \, .
\end{align}
}

We now turn to the other ingredients that underlie our main result.
The \emph{kernel complexity} at scale $\delta$ is given by
\begin{align}
\mathcal{C}(\delta) \defn \sqrt{ \sum_{j=1}^{\infty} \min\big\{
  \frac{\eig{j}}{\delta^2}, 1 \big\}}.
\end{align}
Our main result specifies a critical $\delcrit$ in terms of an
inequality of the form $\mathcal{C}(\delta) \leq (\mbox{SNR}) \:
\delta$.  Here $\mbox{SNR}$ is a \emph{signal-to-noise ratio},
and it is controlled by the following properties of the
underlying problem:
\begin{description}
\item[Effective timescale:] Recalling the
  definition~\eqref{EqnDefnEffDiscount} of the effective discount
  factor $\discounteff \equiv \discounteff(\bweight)$, we use $\Hoeff
  \defn (1-\discounteff)^{-1}$ to denote the effective timescale
  associated with a $\bweight$-weighted TD method.
\item[Bellman fluctuations:] We measure the variability of the
  $\Step$-step Bellman operator via
 \begin{subequations}
 \begin{align}
 \label{eq:def_std}
 \stdmtg(\thetastar) & \defn \sum_{\ell=1}^{\Step} \discount^\ell
 \sqrt{\Exp\Big[ \Var \big[ \big( \sum_{\step=\ell}^{\Step}
       \weight{\step} \, \BellOp{\step-\ell}(\thetastar)\big)
       (\Statenew) \; \bigm| \; \State \big] \Big] } \, ,
 \end{align}
 where $\State$ is drawn from the stationary distribution $\distr$ and
 $(\State, \Statenew)$ are successive samples from the Markov chain
 $\TransOp$.
 \item[Bellman residual and mixing:] When the value function $\Vstar$
   does not belong to the space $\RKHS$, the projected fixed point
   $\thetastar$ differs from $\Vstar$, and hence the Bellman residual
   $\BellOp{\bweight}(\thetastar) - \thetastar$ is non-zero.  In this
   case, our bounds involve an additional noise term, given by
\begin{align}
\label{eq:def_approx}     
 \approxerr(\thetastar) & \defn 2 \sqrt{ \mixtime}
 \munorm[\big]{\BellOp{\bweight}(\thetastar) - \thetastar} \; \Big\{ 1
 + \tfrac{1}{4} \log \tfrac{\supnorm{\BellOp{\bweight}(\thetastar) -
     \thetastar}}{\munorm{\BellOp{\bweight}(\thetastar) - \thetastar}}
 \Big\}
 \end{align}
   where $\mixtime$ is the mixing time.
 \end{subequations}
\end{description}

We are now ready to describe the inequality that determines the
estimation error in our main result.  Consider a user-defined radius
$\newrad$ such that
\begin{subequations}
\begin{align}
 \label{cond:newrad}
  \newrad \geq \max\big\{ \hilnorm{\thetastar - \reward}, \;
  \tfrac{\supnorm{\reward}}{\bou} \big\},
\end{align}
along with the \emph{effective noise level}
\begin{align}
 \label{eq:def_noisebase}
 \noisebase \defn \Hoeff \, \big\{ \stdmtg(\thetastar) +
 \approxerr(\thetastar) \big\}.
\end{align}
Using these quantities, the estimation error in our first main result
is determined by the \emph{critical inequality}
\begin{align}
  \label{eq:critineq}
\tag{\CI(\noisebase)} \myunder{\mathcal{C}(\delta)}{Kernel complexity}
& \leq \frac{\sqrt{\numobs} \newrad}{ \unibou \; \noisebase } \;
\delta.
\end{align}
\end{subequations}
Concretely, we let $\delcrit(\noisebase)$ be the smallest positive
solution to the critical inequality~\ref{eq:critineq}.
\begin{theorem}[Non-asymptotic upper bound]
\label{thm:ub}
Under the mixing condition~\ref{assump:mixing} and the kernel
boundedness condition~\ref{ass:kernel}, consider the kernel-LSTD
method based on a look-ahead $\Step \leq \mixtime/\Constdistrnew$.
Suppose that the sample size $\numobs$ is large enough to ensure that
\begin{align}
 \label{EqnSampleLowerBound}
 \newrad^2 \, \delcrit^2(\noisebase) \; \leq \; \plaincon \,
 \frac{(1-\discounteff) \, \noisebase^2}{\sqrt{(\mixtime + \Step) \,
     \numobs}}.
\end{align}
Then for any regularization parameter $\ridge \geq c_0 \,
\delcrit^2(\noisebase) (1-\discounteff) \, \log \numobs$, the
projected fixed point \mbox{$\thetahat \equiv \thetahat(\ridge)$}
satisfies the bound
\begin{align}
  \label{eq:thm_ub}
  \munorm{\thetahat - \thetastar}^2 \leq c_1 \, \newrad^2 \Big\{
  \delcrit^2(\noisebase) \, \log^2 \numobs +
  \frac{\ridge}{1-\discounteff} \Big\}
  \end{align}
with probability at least $1 - \const{2} \, \exp\big( - \Constprob \,
\tfrac{\numobs \, \delta^2(\noisebase)}{\bou^2} \big)$, \mbox{where
  $\Constprob \; \defn \; \const{3} \; \frac{(1-\discounteff)^2 (1 -
    \discount)^2}{\mixtime + \Step}$.}
\end{theorem}
\noindent 
See~\Cref{sec:proof_ub} for the proof.  \\

\medskip

\noindent In order to interpret the result, a few comments are in
order.  Beginning with the sample size
condition~\eqref{EqnSampleLowerBound}, as we show in the sequel, the
estimation error $\delcrit^2(\noisebase)$ typically drops at a rate
faster than $1/\sqrt{\numobs}$, so that
condition~\eqref{EqnSampleLowerBound} can always be satisfied for a
sufficiently large $\numobs$.  With the minimal choice of
regularization parameter $\ridge = c_0 \delcrit^2(\noisebase)
(1-\discounteff) \, \log \numobs$, \Cref{thm:ub} guarantees that
$\munorm{\thetahat - \thetastar}^2 \myprecsim \newrad^2
\delcrit^2(\noisebase) \, \log^2(\numobs)$ with high probability.
Thus, disregarding logarithmic factors, the mean-squared error is
determined by the critical radius $\delcrit^2$.  In~\Cref{sec:ub_exp},
we compute this critical radius for a number of typical
eigensequences.

\paragraph{Proof overview:}
The first step, of a relatively straightforward nature, is to derive
an optimization-theoretic inequality satisfied by the estimator. The
second step---and the bulk of our technical effort---is devoted to
bounding various terms that appear in this basic inequality.  Due to
the Markovian nature of the data, we need to derive high probability
bounds on the suprema of empirical processes associated with Markov
chains.  This dependence structure precludes the use of various
techniques, including symmetrization, that are standard in the
analysis of empirical processes defined with independent sampling.
Our strategy is to ``reduce'' the Markovian case to a block-based
version of independent sampling.  In particular, we invoke the mixing
condition~\ref{assump:mixing} so as to argue that a ``splitting''
procedure can reduce a single trajectory into the union of independent
blocks, each of a random length (depending on an underlying hitting
time).  This reduction allows us to leverage known concentration
inequalities in the independent case so as to control the Markovian
setting of interest here.


\subsubsection{Bounding the noise level $\noisebase$}
\label{sec:SNR_lb}

The bound~\eqref{eq:thm_ub} from~\Cref{thm:ub} holds, in weakened
form, for any upper bound on the noise level $\noisebase$.
Accordingly, in order to develop intuition for the behavior of our
bounds, it is useful to derive such an upper bound that decouples into
a variance term along with a form of approximation error.  In
particular, let us define the \emph{expected Bellman variance}
\begin{align}
\label{eq:def_stdfun}
\stdfun^2(\Vstar) \defn \Exp_{\State \sim \distr} \big[ \Var [
    \Vstar(\Statenew) \bigm| \State ] \big],
\end{align}
associated with the true value function.  Recall that $\mixtime \geq
1$ is the mixing time, $\Ho = (1-\discount)^{-1}$ stands for the
effective horizon, and define the error $\Vstarperp \defn \Vstar -
\projH(\Vstar)$ in the projection\footnote{To be clear, the projection
$\projH(\Vstar)$ is, in general, \emph{not} the same as the projected
fixed point $\thetastar$.} of $\Vstar$ onto the function class.  With
this notation, it can be shown that the effective noise $\noisebase$
defined in equation~\eqref{eq:def_noisebase} is upper bounded as
\begin{align}
\label{eq:def_noisebasenew}
\noisebase & \leq \noisebasenew \defn \constnew \, \Big\{
\underbrace{\Ho \; \stdfun(\Vstar)}_{\text{uncertainty}} \; + \;
\underbrace{\Hoeff \sqrt{\max\{ \Ho, \, \mixtime\}} \;
  \munorm{\Vstarperp}}_{\text{model error}} \Big\}
\end{align}
where the pre-factor $\constnew \equiv \constnew(\thetastar)$ depends
only on the logarithmic quantity $\log
\tfrac{\supnorm{\BellOp{\bweight}(\thetastar) -
    \thetastar}}{\munorm{\BellOp{\bweight}(\thetastar) -
    \thetastar}}$.  See~\Cref{app:cor:ub_new} for the proof of this
claim, where we also derive some other bounds that are of independent
interest.

To understand the behavior of this upper
bound~\eqref{eq:def_noisebasenew}, suppose that the model error
portion of $\noisebasenew$ is negligible relative to the variance term
$\Ho \, \stdfun(\Vstar)$; for example, this will be the case when the
model mis-specification is zero or small (in the sense that $\Vstar$
is close to $\RKHS$, so that $\munorm{\Vstarperp} \approx 0$).  In
this regime of negligible mis-specification, our theory makes a number
of interesting predictions of a qualitative nature:
\begin{itemize}
\item Observe that the leading term $\Ho \, \stdfun(\Vstar)$ is
  independent of the mixing time $\mixtime$.  Consequently, once the
  sample size $\numobs$ exceeds a finite threshold threshold
  (cf. condition~\eqref{EqnSampleLowerBound}), the estimation
  error~$\munorm{\thetahat - \thetastar}^2$ should \emph{not} be
  affected by the dependence in the trajectory sampling model.  Apart
  from constant differences, the error from trajectory-based data
  should scale as if we were given $\numobs$ i.i.d. sample
  transitions---that is, a collection of i.i.d. pairs $\{
  (\statetil_i, \statetilnew_i) \}_{i=1}^{\numobs}$ with $\statetil_i
  \sim \distr$ and $\statetilnew_i \sim \TransOp(\cdot \mid
  \statetil_i)$.  This prediction is confirmed by the simulation
  results reported in~\Cref{FigNsampMis}(a).  Note that this behavior
  is rather interesting, since it is often the case in non-parametric
  problems that dependent sampling models cause degradation in the
  behavior of estimators.
\item A second useful property is that the term $\Ho \,
  \stdfun(\Vstar)$ remains invariant to the choice of the weight
  vector $\bweight$.  Consequently, in the regime of negligible
  mis-specification, no matter what type of TD method is chosen---with
  possibilities including $\Step$-step TD method for $\Step \in
  \Int_+$, or $\TD(\lambda)$ for any $\lambda \in [0,1)$---the
    estimation error $\munorm{\thetahat - \thetastar}^2$ should scale
    in a similar manner.  Thus, the flexibility in the choice of TD
    method does not have any benefits for reducing estimation error.
    To be clear, it can still reduce the approximation error in the
    decomposition~\eqref{eq:ub_err_decomp}, since the effective
    discount factor can be reduced.
\end{itemize}

It should be emphasized that in other regimes, careful choices of the
weight vector $\bweight$ can reduce the estimation error.  More
precisely, this choice can reduce the effective horizon $\Hoeff$,
which in turn can reduce the model error portion of the effective
noise bound $\noisebasenew$, as well as the approximation
error~$\munorm{\thetastar - \Vstar}^2$.  Reductions in $\Hoeff$ can be
achieved by choosing a larger look-ahead parameter $\Step$ in a
multi-step TD method, or a larger value of $\lambda \in [0,1)$ in the
${\rm TD}(\lambda)$ family of methods.

In some special cases, the model error in
equation~\eqref{eq:def_noisebasenew} may dominate the uncertainty term
$\Ho \, \stdfun(\Vstar)$. For instance, suppose that the Markov chain
is nearly deterministic in nature, so that the conditional variances
$\stdfun^2(\Vstar)$ and $\stdmtg^2(\thetastar)$ are very close to
zero.  In this case, the estimation error is mainly determined by the
Bellman residual and mixing term $\approxerr(\thetastar)$.
Alternatively, suppose that the reward function is ``sparse'', meaning
that it is close to zero for most states, and the value function
inherits this structure.  In this case, the variances
$\stdfun^2(\Vstar)$ and $\stdmtg^2(\thetastar)$ will again be small,
and model mis-specification may be the dominant factor.


\subsubsection{Generalization to multiple episodes}
\label{sec:multi_traj}

Up to now, we have assumed access to a single trajectory of length
$\numobs$.  Suppose instead that, for some length parameter $\Length
\geq 2$, we have access to $\lceil \numobs/\Length \rceil$
trajectories, of each length $\Length$.  For a dataset of this type,
we can prove a non-asymptotic upper bound similar to that
in~\Cref{thm:ub}. The modified result differs in the definition of
$\approxerr(\thetastar)$, and in the sample size
condition~\eqref{EqnSampleLowerBound}.  In particular, we make the
alternative definition
\begin{align}
\label{eq:def_approx_multi}     
\approxerr(\thetastar) & \defn 2 \,
\munorm{\BellOp{\bweight}(\thetastar) - \thetastar} \; \min \bigg\{
\sqrt{\Length}, \, \sqrt{ \mixtime} \Big\{ 1 + \tfrac{1}{4} \log
\tfrac{\supnorm{\BellOp{\bweight}(\thetastar) -
    \thetastar}}{\munorm{\BellOp{\bweight}(\thetastar) - \thetastar}}
\Big\} \bigg\},
\end{align}
and we propose the modified sample size condition
\begin{align*}
\newrad^2 \, \delcrit^2(\noisebase) \; \leq \; \plaincon \,
\frac{(1-\discounteff) \, \noisebase^2}{\sqrt{\min\{\Length, \, \mixtime +
    \Step\} \, \numobs}} \, .
\end{align*}
See~\Cref{sec:proof_ub_overview} for the extension of our proof to
multiple episodes. \\

Suppose the trajectory length parameter $\Length$ is small relative to
the mixing time---that is, $\Length \ll \mixtime$.  In this case, the
approximation error term~$\approxerr(\thetastar)$ associated with the
multi-step estimator~\eqref{eq:def_approx_multi} is much smaller than
that in definition~\eqref{eq:def_approx}. Thus, we see that the
independence between episodes, when compared to the fully dependent
case of a single trajectory, reduces the effective noise level
\mbox{$\noisebase = \Hoeff \big\{ \stdmtg(\thetastar) +
  \approxerr(\thetastar) \big\}$.}

There are trade-offs, however, in that when the length parameter $L$
is larger, it becomes possible to use multi-step TD methods with
greater look-ahead $\Step$.  This flexibility allows us to choose
vectors~$\bweight$ in the larger space $\Real^{\Step}$, and thereby
reduce the effective horizon $\Hoeff$ that also enters the
effective noise level.

Our theory shows that the independence among episodes in data
collection can reduce estimation error only when either (a) the
Bellman residual is significant; or (b) the sample size is relatively
small (so that the sample size condition~\eqref{EqnSampleLowerBound}
fails to hold).  Otherwise, when there is no model mis-specification
and the sample size exceeds a threshold, then TD-based procedures for
policy evaluation show little difference between the single trajectory
and multiple trajectory cases. \\

As a final sanity check, consider the special case of
i.i.d. transition pairs \mbox{$\Dataset = \{ (\state_i, \statenew_i)
  \}_{i=1}^{\numobs}$} and the standard one-step LSTD estimate.  This
particular setting was the focus of our earlier
work~\cite{duan2021optimal}. In the language of the current paper, it
corresponds to the special case $L = 2$ and $\Step = 1$, which leads
to the effective horizon \mbox{$\Hoeff = (1 - \discount)^{-1}$}.  In
this special case, the Bellman fluctuation~\eqref{eq:def_std} takes
the simpler form
\begin{align*}
\stdmtg(\thetastar) = \discount
\sqrt{\Exp\big[\Var[\thetastar(\Statenew) \mid \State]\big]} =
\sqrt{\Exp\Big[ \big( \BellOp{1}(\thetastar)(\State) - \reward(\State)
    - \discount \, \thetastar(\Statenew) \big)^2 \Big]} \, .
\end{align*}
The effective noise level \mbox{$\noisebase = \Hoeff \big\{
  \stdmtg(\thetastar) + \approxerr(\thetastar) \big\}$} is then given
by
\begin{align}
\label{eq:noibase_iid}  
\noisebase & \asymp \Hoeff \Big\{ \discount
\sqrt{\Exp\big[\Var[\thetastar(\Statenew) \mid \State]\big]} +
\munorm[\big]{\BellOp{1}(\thetastar) - \thetastar} \Big\} \; \asymp \;
\Hoeff \sqrt{\Exp\Big[ \big( \thetastar(\State) - \reward(\State) -
    \discount \, \thetastar(\Statenew) \big)^2 \Big]}.
\end{align}
For i.i.d. observations and the standard one-step LSTD estimate, the
critical inequality~\ref{eq:critineq} with $\noise = \noisebase$ given
in equation~\eqref{eq:noibase_iid} is consistent with the ``fast
rate'' established in Theorem 1(b) of our earlier
work~\cite{duan2021optimal}.


\subsection{Consequences for specific kernels}
\label{sec:ub_exp}

All of our rates are stated in terms of the solution
$\delcrit(\noisebase)$ to the critical inequality~\ref{eq:critineq},
so that it is helpful to consider the form of $\delcrit$ for specific
classes of kernels.  Suppose that we use the minimal regularization
parameter $\ridge = \const{0} \, \delcrit^2(\noisebase) \,
(1-\discounteff) \, \log \numobs$.  With this choice, \Cref{thm:ub}
implies the upper bound
\begin{align}
\label{eq:thm_ub_new}
\munorm{\thetahat - \thetastar}^2 \leq \underbrace{\const{1} (1 +
  \const{0})}_{\consttil{}} \; \newrad^2 \,
\delcrit^2(\noisebase) \, \log^2 \numobs \, .
\end{align}
Let us consider explicit forms of this simpler bound for concrete
classes of kernels.


\subsubsection{Finite-rank or linear kernels}
\label{sec:ub_exp_linear}

We begin with the special case of a finite-rank kernel.  In this case,
there exists a $\Dim$-dimensional feature mapping $\varphi: \StateSp
\rightarrow \Real^{\Dim}$ such that $\Ker(\state, \statey) =
\inprod{\varphi(\state)}{\varphi(\statey)}$, and the RKHS corresponds
to functions that are linear in this feature vector.  Note that the
vector $\varphi(\state) \in \Real^{\Dim}$ plays the role of the
representer of evaluation. Let $\CovOp = \Exp_{\distr}\big[
  \varphi(\State) \, \varphi(\State)^{\top} \big] = \sum_{j=1}^{\Dim}
\eig{j} v_j v_j^{\top}$ be an eigen decomposition of covariance matrix
$\CovOp \in \Real^{\Dim \times \Dim}$. We assume that
\begin{align*}
\bou \defn \sup_{\state \in \StateSp} \sqrt{\Ker(\state, \state)} =
\sup_{\state \in \StateSp} \, \norm{\varphi(\state)}_2 < +\infty \quad
\text{and} \quad \unibou \defn \sup_{\state \in \StateSp} \sup_{j \geq
  1} | \inprod{v_j}{\varphi(\state)} | < + \infty.
\end{align*}

As a corollary of \Cref{thm:ub}, we can show that as long as the
sample size $\numobs$ is lower bounded as $\sqrt{\numobs/(\mixtime +
  \Step)} \, \geq \, \plaincon \; \unibou^2 \, \Dim \, \Hoeff$, then
the estimate $\thetahat$ satisfies the bound
\begin{align}  
\label{eq:ub_linear_b}    
\munorm{\thetahat - \thetastar}^2 & \leq c' \; \bigg\{
\underbrace{\frac{\unibou^2 \, \stdmtg^2(\thetastar)}{(1 -
    \discounteff)^2} \; \frac{\Dim}{\numobs}}_{\Errone} \; + \;
\underbrace{\frac{\unibou^2 \, \munorm{\BellOp{\bweight}(\thetastar) -
      \thetastar}^2}{(1-\discounteff)^2} \,
  \frac{\Dim}{(\numobs/\mixtime)}}_{\Errtwo} \bigg\} \,
\log^2(\numobs)
\end{align}
with high probability.  See \Cref{sec:proof_ub_linear_b} for the
proof.

As emphasized by our equation lay-out, the first term $\Errone$ decays
as $\numobs^{-1}$, whereas the second term $\Errtwo$ decays as
$(\numobs/\mixtime)^{-1}$, so that the effective sample size is
smaller by a factor of the mixing time.  Alternatively, we can
refactor the right-hand side of the bound~\eqref{eq:ub_linear_b} as
\begin{align*}
\frac{\unibou^2}{(1 - \discounteff)^2} \frac{\Dim}{\numobs} \Big \{
\stdmtg^2(\thetastar) + \mixtime \munorm{\BellOp{\bweight}(\thetastar)
  - \thetastar}^2 \Big \},
\end{align*}
so we see explicitly the transition between the dominance between the
martingale noise terms $\stdmtg^2(\thetastar)$ and the term involving
the interaction between mixing time and model error.


It is also worthwhile comparing to past work (involving one of the
current authors) on stochastic approximation for projected linear
equations.  More precisely, in the special case of $\TD(\lambda)$, let
us compare our inequality~\eqref{eq:ub_linear_b} with Corollary~4 in
the paper~\cite{mou2021optimal}.  The right-hand side of
equation~\eqref{eq:ub_linear_b} is an upper bound (up to logarithmic
factors) on the first equation~(39b) in the
paper~\cite{mou2021optimal} which involves a trace term characterizing
the limiting variance.  The result given here, by detaching the
Bellman residual from the variance term, is possibly easier to
interpret. The bound~\eqref{eq:ub_linear_b} is tighter than certain
results in the paper~\cite{mou2021optimal}; in particular, compare to
the inequality below equation~(40) in their paper.


\subsubsection{Kernels with $\alpha$-polynomial decay}
\label{sec:ub_exp_alpha}

Another important class of kernels are those with eigenvalues that
satisfy the \emph{$\alpha$-polynomial decay condition}
\begin{align}
 \eig{j} \leq \plaincon \, j^{-2\alpha} \qquad \text{for some exponent
   $\alpha > \tfrac{1}{2}$}\,.
\end{align}
Kernels that exhibit eigendecay of this type include various types of
Sobolev spaces, spline kernels, and the Laplacian kernel; see Chapter
13 in the book~\cite{wainwright2019high} and references therein for
further details.

Based on~\Cref{thm:ub}, it follows that as long as the sample size
$\numobs$ is large enough to ensure that \mbox{$\newrad^2 \,
  \delcrit^2(\noisebase) \leq \plaincon \, \frac{(1-\discounteff) \,
    \noisebase^2}{\sqrt{(\mixtime + \Step) \, \numobs}}$}, then
\begin{align}
\label{eq:ub_alpha_b}
\munorm{\thetahat - \thetastar}^2 \leq \constnew \;
\newrad^{\frac{2}{2\alpha+1}} \, \bigg\{ \underbrace{\frac{\unibou^2
    \, \stdmtg^2(\thetastar)}{(1 - \discounteff)^2} \;
  \frac{1}{\numobs}}_{\Errone} \; + \; \underbrace{\frac{\unibou^2 \,
    \munorm{\BellOp{\bweight}(\thetastar) -
      \thetastar}^2}{(1-\discounteff)^2} \,
  \frac{1}{(\numobs/\mixtime)}}_{\Errtwo}
\bigg\}^{\frac{2\alpha}{2\alpha+1}} \log^2 (\numobs) \, ,
\end{align}
with high probability. See \Cref{sec:proof_ub_alpha_b} for the proof
of this claim.

Similar comments can be made about the two terms in the
bound~\eqref{eq:ub_alpha_b}.  In terms of sample size dependence,
both terms decay with the familiar non-parametric exponent
$\frac{2 \alpha}{2 \alpha + 1}$, but the second term has effective
sample size $\numobs/\mixtime$, so reduced by the mixing time.


\subsection{Simpler bounds in some special cases}
\label{sec:ub_simple}

Our upper bounds take simpler and more interpretable forms in
particular regimes of parameters, as we discuss here.

\subsubsection{Cases where the model error is negligible}
\label{sec:ub_noerr}

Recalling the definition~\eqref{eq:def_stdfun} of the variance
term~$\stdfun^2(\thetastar)$, we first discuss settings in which
\mbox{$\noisebasenew \lesssim \Ho \, \stdfun(\thetastar)$} in
equation~\eqref{eq:def_noisebasenew}.  In particular, we claim that if
either the temporal dependence is mild (i.e., the mixing time
$\mixtime$ is relatively small), and/or the model mis-specification is
small (i.e. $\munorm{\Vstarperp}$ is small), then the model error term
$\Hoeff \sqrt{\max\{ \Ho, \, \mixtime\}} \; \munorm{\Vstarperp}$ is
dominated by the uncertainty term $\Ho \, \stdfun(\thetastar)$.

\paragraph{Mild temporal dependence:}
Now suppose that the mixing time $\mixtime$ is relatively
small---concretely, say $\mixtime \leq \Ho = (1-\discount)^{-1}$---
and moreover, that the chain is geometrically ergodic, meaning that
  \begin{align}
    \label{assump:Lmu_ergo}
    \munorm{\TransOp f} \leq (1 - \mixtime^{-1}) \, \munorm{f} \qquad
    \text{for any function $f \in \RKHS$ such that $\distr(f) = 0$.}
  \end{align}
Under these conditions, we claim that the approximation error
$\Vstarperp \defn \Vstar - \projH(\Vstar)$ can be upper bounded as
\begin{align}
  \label{eq:Vstarperp<stdfun}
  \munorm{\Vstarperp} \, \leq \, \sqrt{\mixtime} \, \stdfun(\Vstar).
\end{align}
See~\Cref{append:proof:eq:Vstarperp<stdfun} for the proof of this
auxiliary claim. \\

Given the bound~\eqref{eq:Vstarperp<stdfun}, we have $\noisebasenew
\leq \constnew \, \big\{ 1 \, + \, \Hoeff \sqrt{\mixtime/\Ho} \big\}
\, \Ho \, \stdfun(\Vstar)$, and so we may choose a weight vector
$\bweight \in \Real^{\Step}$ to ensure that $\Hoeff \myprecsim \sqrt{\Ho
  / \mixtime}$.  This condition translates to
\begin{align*}
\Step \gtrsim \sqrt{\Ho \, \mixtime}, \quad \mbox{or} \quad
(1-\lambda)^{-1} \gtrsim \sqrt{\Ho \, \mixtime}
\end{align*}
for a $\Step$-step TD method, or $\TD(\lambda)$ method, respectively.
Putting together the pieces, we conclude that \mbox{$\noisebasenew
  \myprecsim \Ho \, \stdfun(\Vstar)$.}


\paragraph{Model mis-specification is small:}
We turn to the cases where $\munorm{\Vstarperp}$ is upper bounded by
$\munorm{\Vstarperp}^2 \leq \min\big\{ \Ho^{-1}, \, \mixtime^{-1}
\big\} \; \stdfun^2(\Vstar)$. Since we always have $\Hoeff \leq \Ho$
for any weight vector $\bweight$, the
bound~\eqref{eq:def_noisebasenew} on noise level $\noisebase$ then
reduces to $\noisebase \leq \noisebasenew \leq \constnewnew \, \Ho \,
\stdfun(\Vstar)$ for some $\constnewnew \geq 2 \, \constnew$. \\

In the two cases above, the estimation error $\munorm{\thetahat -
  \thetastar}^2$ is determined by solutions to critical
inequality~\ref{eq:critineq} with $\noise = \constnewnew \, \Ho \,
\stdfun(\Vstar)$.  In the concrete examples with finite-rank kernels
or kernels with $\alpha$-polynomial decay, the nonasymptotic upper
bounds in equations~\eqref{eq:ub_linear_b}~and~\eqref{eq:ub_alpha_b}
have simple forms as shown below:
\begin{subequations}
  \begin{align}
    \label{eq:ub_linear_simple}
    \munorm{\thetahat - \thetastar}^2 & \lesssim
    \; \frac{\unibou^2 \, \stdfun^2(\Vstar)}{(1 - \discount)^2} \;
    \frac{\Dim \, \log^2 \numobs}{\numobs} \, , \\
    \label{eq:ub_alpha_simple}
    \munorm{\thetahat - \thetastar}^2 & \lesssim \, \newrad^2 \,
    \Big( \frac{\unibou^2 \, \stdfun^2(\Vstar)}{\newrad^2 \, (1 -
      \discount)^2} \; \frac{1}{\numobs}
    \Big)^{\frac{2\alpha}{2\alpha+1}} \log^2 \numobs \, .
  \end{align}
\end{subequations}

Disregarding logarithmic factors, the upper
bound~\eqref{eq:ub_linear_simple} on $\munorm{\thetahat -
  \thetastar}^2$ is the same as the ones in Corollary~1(b) of
paper~\cite{duan2021optimal} and the bound~\eqref{eq:ub_alpha_simple}
coincides with Corollary~2(b) of paper~\cite{duan2021optimal}.  It
shows that the estimation error of using trajectory data $(\state_1,
\state_2, \ldots, \state_{\numobs})$ scales as if we are conducting
standard LSTD estimate using i.i.d. transition pairs $\{(\statetil_i,
\statetilnew_i)\}_{i=1}^{\numobs}$ with $\statetil_i \sim \distr$ and
$\statetilnew_i \sim \TransOp(\cdot \mid \state_i)$. It is also worth
noting that the bound~\eqref{eq:ub_linear_simple} holds for any choice
of weights $\bweight$. Therefore, the estimation errors of different
TD methods have similar scales.

Although the non-asymptotic bounds~\eqref{eq:ub_linear_simple} and
\eqref{eq:ub_alpha_simple} are invariant to mixing time $\mixtime$ and
weight vector $\bweight$, the parameters determine the burn-in time,
i.e. the lower bound condition on sample size~$\numobs$.  For example,
as the mixing time $\mixtime$ or the effective timescale $\Hoeff$
grows, it gets harder to satisfy the sample size lower bound
$\sqrt{\numobs/(\mixtime + \Step)} \geq \plaincon \, \unibou^2 \Dim \,
\Hoeff$ in the finite-rank case. We remark that these conditions are
not tight and could be improved, which we leave open as an interesting
direction for future work.


\subsubsection{Bounds under uniform structural conditions}
\label{sec:ub_normbd}

In the analysis of RL algorithms, it is standard to impose various
types of uniform structural conditions on the MRP.  In this section,
we explore the consequences of our instance-dependent results for two
such structural constraints: (i) a uniform bound on the reward
function $\reward$; and (ii)~a $\Lmu$-norm upper bound on the value
function $\Vstar$.  Our theory shows that different choices of TD
parameters should be made in these two settings.

\paragraph{Uniformly bounded reward:}
Suppose that the reward function is uniformly bounded---viz.
$\supnorm{\reward} \leq \rewardnorm$ for some finite constant
$\rewardnorm$---and that the weight vector $\bweight$ is chosen to
ensure that
\begin{subequations}
\begin{align}
 \label{eq:Hoeff_r}
 \Hoeff \equiv \Hoeff(\bweight) \myprecsim \; \Big \{ 1 +
 \frac{\Ho}{\mixtime} \Big \}
\end{align}
The bound~\eqref{eq:Hoeff_r} can be ensured by setting
\begin{align}
\label{eq:Choices_r}
\Step \gtrsim \min\{ \Ho, \, \mixtime \} \qquad \mbox{for $\Step$-step
  TD, or} \quad \qquad (1-\lambda)^{-1} \gtrsim \min\{ \Ho, \,
\mixtime \} \quad \mbox{for ${\rm TD}(\lambda)$.}
\end{align}
Moreover, we also prove in~\Cref{sec:proof_noisebasenew_r} that, with these choices, the noise level
$\noisebase$ is bounded as
\begin{align}
\label{eq:noisebasenew_r}
\noisebase \myprecsim \; \Ho \sqrt{\max\{ \Ho, \, \mixtime \}} \;
\rewardnorm \, .
\end{align}
\end{subequations}

This bound on the effective noise level, in turn, has consequences for
specific kernel classes.  For example, when using finite-rank kernels,
we have
\begin{subequations}
\label{eqn:thm_ub_new_r}  
\begin{align}
  \label{eq:thm_ub_new_r_linear}
  \munorm{\thetahat - \thetastar}^2 \; \myprecsim \;
  \underbrace{\rewardnorm^2 \; \Ho^2 \, \max\{ \Ho, \, \mixtime \} \;
    \frac{\Dim}{\numobs}}_{\epsilonbdnew^2} \, \log^2 \numobs \, .
\end{align}
On the other hand, for kernels with $\alpha$-polynomial decay, we have
\begin{align}
  \label{eq:thm_ub_new_r_alpha}
  \munorm{\thetahat - \thetastar}^2 \; \myprecsim \; \underbrace{
    \rewardnorm^2 \; \Ho^2 \, \Big( \frac{\max\{ \Ho, \, \mixtime
      \}}{\numobs} \Big)^{\frac{2\alpha}{2\alpha+1}} }
  _{\epsilonbdnew^2} \, \log^2 \numobs \, .
\end{align}
\end{subequations}
Both the bounds hold with probability at least $1 - \const{2} \,
\exp\big( - \Constprob \, \tfrac{\numobs \, \epsilonbdnew^2}{\bou^2 \,
	\Ho^2 \rewardnorm^2} \big)$.  


\paragraph{Value function with bounded $\Lmu$-norm:}

Now suppose that $\munorm{\Vstar} \leq \Vstarnorm$ for some finite
$\Vstarnorm$, and the weight vector $\bweight$ is chosen to ensure
that
\begin{subequations}
\begin{align}
\label{eq:Hoeff_V}
\Hoeff \equiv \Hoeff(\bweight) \myprecsim \min \Big \{ \sqrt{\Ho}, \, 1
+ \frac{\Ho}{\sqrt{\mixtime}} \Big \}.
\end{align}
The
bound~\eqref{eq:Hoeff_V} can be satisfied by choosing
\begin{align}
  \Step \gtrsim \min \big \{ \Ho, \sqrt{\Ho + \mixtime} \big \} \quad
  \mbox{in $\Step$-step TD, or} \quad (1 - \lambda)^{-1} \gtrsim \min
  \big \{ \Ho, \sqrt{\Ho + \mixtime} \big\} \quad \mbox{in
    $\TD(\lambda)$ .}
\end{align}
In~\Cref{sec:proof_noisebasenew_V}, we prove that the effective noise level is then bounded as
\begin{align}
\label{eq:noisebasenew_V}
\noisebase \myprecsim \max \{ \Ho, \sqrt{\mixtime} \} \, \Vstarnorm.
\end{align}
\end{subequations}

As before, there are concrete consequences for specific kernel
classes.  For a linear kernel, we have
\begin{subequations}
\label{eqn:thm_ub_new_V}
\begin{align}
  \label{eq:thm_ub_new_V_linear}
  \munorm{\thetahat - \thetastar}^2 \; \myprecsim \;
  \underbrace{\Vstarnorm^2 \; \max\{ \Ho^2, \, \mixtime \} \;
    \frac{\Dim}{\numobs}}_{\epsilonbdnew^2} \, \log^2 \numobs \, ;
\end{align}
whereas for kernels with $\alpha$-polynomial decay,
\begin{align}
  \label{eq:thm_ub_new_V_alpha}
  \munorm{\thetahat - \thetastar}^2 \; \myprecsim \;
  \underbrace{\Vstarnorm^2 \; \Big( \frac{\max\{ \Ho^2, \, \mixtime
      \}}{\numobs}
    \Big)^{\frac{2\alpha}{2\alpha+1}}}_{\epsilonbdnew^2} \, \log^2
  \numobs.
\end{align}
\end{subequations}
Both of these bounds hold with probability at least $1 - \const{2} \,
\exp\big( - \Constprob \, \tfrac{\numobs \, \epsilonbdnew^2}{\bou^2 \,
  \Vstarnorm^2} \big)$. By comparison with the
bounds~\eqref{eqn:thm_ub_new_r} for the bounded reward case, we see
that estimation error is increased; this change is to be expected,
since we have imposed only the milder condition of a $L^2$-bounded
value function.


\section{General lower bounds on policy evaluation}
\label{sec:lb}

Thus far, we have stated and discussed a number of achievable results
on the problem of policy evaluation.  In this section, we turn to the
complementary question of fundamental lower bounds.   


\subsection{Set-up for lower bounds}

We prove minimax lower bounds over a class $\MRPclass$ of Markov
reward processes (MRPs) defined on the state space \mbox{$\StateSp =
  [0,1]$}.  Each MRP in the family is constructed to have a unique
stationary distribution $\distr$.  For any given MRP instance
$\MRP(\reward, \TransOp, \discount)$, suppose that we observe
trajectory $\traj = (\state_1, \state_2, \ldots, \state_{\numobs}) \in
\StateSp^{\numobs}$ of length $\numobs$ generated by the Markov chain
defined by the transition distribution $\TransOp$, initialized from
the stationary distribution~$\distr$. An estimator $\thetahat$ is a
mapping from any given trajectory $\traj$ to an element in a function
space $\RKHS$ used to approximates the value function $\Vstar$ of
$\MRP$. We measure the estimation error $\thetahat - \Vstar$ in
$\Lmu$-norm.  Letting $\bar{\RKHS}$ denote the closure of the RKHS, we
define the projection
\begin{align*}
\Vstarpar \defn \proj_{\distr}(\Vstar) \equiv \arg \min_{f \in
  \bar{\RKHS}} \munorm{f - \Vstar}
  \end{align*}
of the value function onto the RKHS, along with the associated error
function $\Vstarperp \defn \Vstar - \Vstarpar$.  The Pythagorean
theorem holds for this type of projection, so that we can write
\begin{align}
\label{eq:lb_err_decomp}
\munorm{\thetahat - \Vstar}^2 = \munorm{\thetahat - \Vstarpar}^2 +
\munorm{\Vstarperp}^2.
\end{align}
The model mis-specification error $\munorm{\Vstarperp}^2$ in
equation~\eqref{eq:lb_err_decomp} is non-vanishing as the sample size
grows, therefore, we are more interested in the analyzing the other
term $\munorm{\thetahat - \Vstarpar}^2$.  Our lower bound says that
for suitable classes $\MRPclass$ indexed by parameters $(\newradbar,
\stdfunbar, \lbnormperp, \mixtimebar)$, if we measure the projected
error $\munorm{\thetahat - \Vstarpar}^2$ of any estimator $\thetahat$
uniformly over the family $\MRPclass$, then it is always lower bounded
by $\const{1} \, \newradbar^2 \, \delcrit^2$ for a universal constant
$\const{1} > 0$. The parameters $\stdfunbar, \lbnormperp$ and
$\mixtimebar$ reflect the constraints on the conditional variance
$\stdfun^2(\Vstar)$, model mis-specification $\munorm{\Vstarperp}$ and
mixing time $\mixtime$ respectively. The radius $\delcrit > 0$ is
determined by a critical inequality similar as the upper bound.  Let
us now elaborate on these issues.


\subsubsection{Families of MRPs and regular kernels}

 Our lower bound involves a fixed RKHS $\RKHS \subset
 \Real^{\StateSp}$ that contains the constant function
 $\one(\cdot)$. The RKHS $\RKHS$ is induced by a kernel function
 \mbox{$\Ker: \StateSp \times \StateSp \rightarrow \Real$}.  We use
 $\eig{1} \geq \eig{2} \geq \eig{3} \geq \ldots \geq 0$ to denote the
 ordered eigenvalues of the kernel integral operator defined by $\Ker$
 (with underlying Lebesgue measure $\distrbar$), and we use
 $\{\feature{j} \}_{j=1}^\infty$ to denote the associated sequence of
 eigenfunctions, orthonormal in $\Ltwo{\distrbar}$.  We assume the
 regularity conditions
 \begin{align}
 \sum_{j=1}^{\infty} \eig{j} \leq \frac{\bou^2}{4} \qquad \text{and}
 \qquad \sup_{j = 1, 2, \ldots} \supnorm{\feature{j}} \leq \unibou = 2
 \, ,
 \end{align}
 for some finite $\bou > 0$.  These conditions ensure that the kernel
 $\Ker$ satisfies Assumption~\ref{ass:kernel}.

We consider MRPs with stationary distribution $\distr$ close to the
Lebesgue measure $\distrbar$. In particular, we assume the Pearson
$\chi^2$-divergence satisfies
\begin{align}
\label{eq:lb_constraint_chisqrad}
\chidiv{\distr}{\distrbar} \; \leq \; \chisqrad \, ,
\end{align}
where $\chisqrad > 0$ is a radius to be determined later.  Our family
of MRPs $\MRPclass$ are defined by parameters $(\newradbar,
\stdfunbar, \lbnormperp, \mixtimebar)$.  In particular, we have the
constraints
\begin{subequations}
\label{eq:lb_constraints}
\begin{align}
 \label{eq:lb_constraint_newradbar}  
 & \max \big \{ \hilnorm{\Vstarpar}, \,
 \tfrac{\supnorm{\reward}}{\bou} \big\} \; \leq \; \newradbar, \\
  & \stdfun(\Vstar) = \Exp_{\State \sim \distr} \big[ \Var[ \,
     \Vstar(\Statenew) \mid \State \, ] \big] \; \leq \; \stdfunbar^2
 \; ; \label{eq:lb_constraint_stdfunbar} \\
& \munorm{\Vstarperp} \; \leq \; \lbnormperp \;
 ; \label{eq:lb_constraint_lbnormperp} \\
 & \text{The mixing condition~\ref{assump:mixing} holds with $\mixtime
   = \mixtimebar$.}
\label{eq:lb_constraint_mixtimebar}
\end{align}
\end{subequations}
We say that an MRP $\MRP$ is $(\newradbar, \stdfunbar, \lbnormperp,
\mixtimebar)$-valid if it satisfies
conditions~\eqref{eq:lb_constraint_newradbar}-\eqref{eq:lb_constraint_mixtimebar}
above as well as the $\chi^2$-divergence
bound~\eqref{eq:lb_constraint_chisqrad} with respect to the stationary
distribution~$\distr$.  More formally, we define
\begin{align}
\label{eq:def_MRPclass}
\MRPclass(\newradbar, \stdfunbar, \lbnormperp, \mixtimebar) \equiv
\MRPclass(\newradbar, \stdfunbar, \lbnormperp, \mixtimebar; \;
\reward, \discount, \RKHS) \defn \big\{ ~ \text{MRP $\MRP(\reward,
  \TransOp, \discount)$ is $(\newradbar, \stdfunbar,
  \lbnormperp,\mixtimebar)$-valid.} ~ \big\} \, .
\end{align}
The reward function $\reward \in \RKHS$ and discount factor $\discount
\in (0,1)$ are shared by all elements in $\MRPclass$ and are given. As
such, the ``hardness'' in our minimax lower bound is induced purely by
uncertainty about the transition kernel $\TransOp$, which translates
to uncertainty about the value function.

As counterparts of the Bellman fluctuation $\stdmtg(\thetastar)$ and
the Bellman residual and mixing term~$\approxerr(\thetastar)$ given in
equations~\eqref{eq:def_std}~and~\eqref{eq:def_approx}, we define
\begin{align*}
\lbstdmtg \defn (1 - \discount)^{-1} \stdfunbar \qquad \text{and}
\quad \lbapproxerr \defn \sqrt{\mixtimebar} \, \lbnormperp, \quad
\mbox{along with the noise level $\lbnoise \defn \lbstdmtg +
  \lbapproxerr$.}
\end{align*}

As in past work~\cite{yang2017randomized,duan2021optimal}, our lower
bounds apply to the class of regular kernels, for which the
eigenvalues decay in a non-pathological way.  For a given radius
$\delcrit > 0$, define \mbox{$\statdim(\delcrit) \defn \max\{ j \mid
  \eig{j} \geq \delcrit^2 \}$} to be the statistical dimension.  We
say a kernel is regular if
\begin{align}
\label{cond:regular}
\frac{2 \lbnoise^2}{\newradbar^2} \statdim \, \; \geq \; \const{} \,
\numobs \, \delcrit^2 \, .
\end{align}
The \emph{regularity condition}~\eqref{cond:regular} precludes certain
types of ill-behaved kernels for which (even in the setting of
ordinary non-parametric regression) kernel ridge estimators are no
longer optimal; see Yang et al.~\cite{yang2017randomized} for further
discussion of this issue.


\subsection{Statement of lower bound}
\label{sec:lb_statement}

With this set-up, we are now equipped to state our lower bound, which
involves the smallest positive solution $\delcrit$ to the critical
inequality
\begin{align}
\label{eq:CI_lb}
\sqrt{\sum_{j=1}^{\infty} \min\big\{ \frac{\eig{j}}{\delta^2}, \, 1
  \big\}} \; \leq \; \frac{\sqrt{\numobs} \, \newradbar}{2 \;
  \lbnoise} \, \delta.
\end{align}
In our result, we require that the sample size is sufficiently large
to ensure that
\begin{align}
\label{cond:n>}
\newradbar^2 \, \delcrit^2 \leq \plaincon \;
\frac{\lbnoise^2}{\mixtimebar} \qquad \text{and} \qquad \delcrit \leq
\plaincon \, \bou.
\end{align}
Moreover, in our construction of the MRP family $\MRPclass$, we set
the radius $\chisqrad \defn \frac{\mixtimebar \, \newradbar^2 \,
  \delcrit^2}{400 \, \lbnoise^2}$ in the
\mbox{$\chi^2$-divergence} condition~\eqref{eq:lb_constraint_chisqrad}.
  We consider groups of parameters $(\newradbar, \stdfunbar,
  \lbnormperp, \mixtimebar)$ that satisfy the following constraints:
 \begin{subequations}
 \label{cond:lb}
 \begin{align}
 \label{eq:def_newradbar}
 \newradbar & \; \geq \; \max\big\{ \tfrac{2}{\sqrt{\eig{2}}} \, (
 \stdfunbar + \lbnormperp ), \; \tfrac{1}{4\bou} \, \stdfunbar \big\}
 \, , \\
 \label{cond:norm_ratio}
 \frac{1}{50} \; \stdfunbar \, (1-\discount)^{-1} \;
 \sqrt{\frac{\statdim}{\numobs}} & \; \stackrel{(i)}{\leq} \;
 \lbnormperp \stackrel{(ii)}{\leq} \; \frac{1}{108} \; \stdfunbar \;
 \min \big\{ (1-\discount)^{-1}, \, \sqrt{\mixtimebar} \big\} \, , \\
\label{cond:mixtimebar}
 \mixtimebar & \; \geq \; \Ho = (1-\discount)^{-1} \, .
 \end{align}
\end{subequations}
 The following result applies to any regular
 sequence~\eqref{cond:regular} of eigenvalues $\{ \eig{j}
 \}_{j=1}^\infty$, and quadruple $(\newradbar, \stdfunbar,
 \lbnormperp, \mixtimebar)$ satisfying
 inequalities~\eqref{eq:def_newradbar}--~\eqref{cond:mixtimebar}.  In
 the following statement, we use $(\const{1}, \const{2})$ to denote
 positive constants, and we use $\Vstarpar(\MRP)$ to denote the
 projection (onto the closure of $\RKHS$) of the value function
 induced by a particular $\MRP$ within the family
 $\MRPclass(\newradbar, \stdfunbar, \lbnormperp, \mixtimebar)$.
\begin{theorem}[Minimax lower bound]
  \label{thm:lb}
Given a sample size satisfying the lower bound~\eqref{cond:n>}, there
is a reproducing kernel Hilbert space $\RKHS$ with eigenvalues
$\{\eig{j} \}_{j=1}^\infty$, a reward function $\reward \in \RKHS$ and
a family of MRPs $\MRPclass(\newradbar, \stdfunbar, \lbnormperp,
\mixtimebar)$ such that
\begin{align}
\label{eq:def_lb}
\inf_{\thetahat} \sup_{\MRP \in \MRPclass(\newradbar, \stdfunbar,
  \lbnormperp, \mixtimebar)} \Prob \Big( \; \munorm{\thetahat \, - \,
  \Vstarpar(\MRP)}^2 \; \geq \; \const{1} \; \newrad^2 \, \delcrit^2 \;
\Big) \; \geq \; \const{2},
\end{align}
where $\delcrit$ is the smallest positive solution to the
inequality~\eqref{eq:CI_lb}. 
\end{theorem}
\noindent 
See~\Cref{sec:proof_lb} for the proof of this claim. \\

\medskip

The quantity $\delcrit$ in the lower bound~\eqref{eq:def_lb} depends
on the effective noise level \mbox{$\lbnoise = \lbstdmtg +
  \lbapproxerr$} via the critical inequality~\eqref{eq:CI_lb}.  We
note that the Bellman fluctuation $\lbstdmtg = (1-\discount)^{-1}
\stdfunbar$ has appeared in past work on tabular and linear problems \cite{gheshlaghi2013minimax,pananjady2020instance,khamaru2020temporal,duan2021optimal}.
However, a lower bound that also involves the model
mis-specification and mixing term \mbox{$\lbapproxerr =
  \sqrt{\mixtime} \, \lbnormperp$}---as we have given here---is novel.

In particular, the minimax lower bound~\eqref{eq:def_lb} strengthens
Theorem~2 in our previous paper~\cite{duan2021optimal}, which applied
only to standard $1$-step procedures with i.i.d. data, and without
model mis-specification.  The result given here shows that the
difficulty of estimating the projected value $\Vstarpar(\MRP)$
function depends on the degree of model mis-specification via the term
$\lbapproxerr$.  Thus, we see that there is actually a
fundamental---and rather interesting---difference between policy
evaluation and standard non-parametric regression.  In the latter
setting, the usual bounds on the excess risk \emph{do not} depend on
the degree of model mis-specification: rather, they depend only on the
complexity of function class used by the procedure itself. In policy
evaluation, the counterpart of excess risk is given by
\begin{align*}
\munorm{\thetahat - \Vstarpar}^2 = \munorm{\thetahat - \Vstar}^2 -
\inf_{f \in \RKHS} \munorm{\Vstar - f}^2,
\end{align*}
and~\Cref{thm:lb} gives a lower bound on this quantity. The appearance
of $\lbapproxerr = \sqrt{\mixtime} \, \lbnormperp$ in this lower bound
shows that, in contrast to non-parametric regression, the ``excess
risk'' does depend on the degree of mis-specification.


\subsection{Comparison with upper bound}

It is worthwhile comparing the minimax lower bound from~\Cref{thm:lb}
with our upper bounds.  One important take-away for the design of
multi-step TD estimates is the following: by taking a weight vector
$\bweight \in \Real^{\Step}$ that ensures the effective timescale
$\Hoeff$ is of constant order, the resulting multi-step TD estimate is
nearly minimax-optimal---that is, with $\Lmu$-error matching the lower
bound in~\Cref{thm:lb} up to logarithmic factors.
  
Let us substantiate this prediction informally.  Since $\Hoeff$ is of
order one, and our lower bound involves the constraint $\mixtime \geq
\Ho$ (cf. equation~\eqref{cond:mixtimebar}), the noise
level~\eqref{eq:def_noisebasenew} scales as
\begin{align*}
  \noisebasenew \; \asymp \; \Ho \, \stdfun(\Vstar) +
  \sqrt{\mixtime} \, \munorm{\Vstarperp} .
\end{align*}
Thus, the noise parameters $\lbnoise$ and $\noisebasenew$ that appear
in the lower and upper bounds, respectively, are equal up to
constants, whence the lower bound~\eqref{eq:def_lb} and our upper
bound match up to logarithmic factors.
  
Recall that $\Vstarpar$ denotes the projection of $\Vstar$ onto
$\RKHS$, whereas $\thetastar$ denotes the projected fixed point
defined by $\RKHS$.  In general, these two objects can be very
different.  We always have the bound $\munorm{\Vstarpar - \Vstar}^2 \;
\leq \; \munorm{\thetastar - \Vstar}^2$, so that the direct projection
never has larger approximation error than the projected fixed point.
However, when the effective horizon $\Hoeff$ is of order one, these
two estimates are not too different, in that the associated
approximation errors \mbox{$\munorm{\Vstar - \Vstarpar}^2$} and
\mbox{$\munorm{\Vstar - \thetastar}^2$} exhibit the same scaling.
Indeed, since the weighted Bellman operator is
$\discounteff$-contractive, an elementary argument shows that
\begin{align*}
  \munorm{\thetastar - \Vstar}^2 & \leq \frac{1}{1 - \discounteff}
  \munorm{\Vstarpar - \Vstar}^2 \equiv \Hoeff \, \munorm{\Vstarpar -
    \Vstar}^2 \; \lesssim \; \munorm{\Vstarpar - \Vstar}^2,
\end{align*}
where the $\lesssim$-step follows from our assumption that $\Hoeff$
is of order one.
  
Thus, we have shown that a multi-step TD method, with the weight
vector chosen so as to ensure that $\Hoeff$ is order one, has
estimation and approximation error that match the minimax lower
bounds.  Thus, a multi-step TD estimate of this type is
minimax-optimal over the MRP family $\MRPclass$ in~\Cref{thm:lb},
as claimed.
  
We make a few more comments on \Cref{thm:lb}: \vspace{.5em}
  
\begin{carlist} \itemsep = .9em
\item
There are some differences in the minimal sample sizes under which the
bounds in~\Cref{thm:ub,thm:lb} hold. Given that $\numobs \geq
\mixtime$, the sample size condition~\eqref{cond:n>} in the lower
bound (~\Cref{thm:lb}) is less stringent than the
condition~\eqref{EqnSampleLowerBound} in the upper bound
(\Cref{thm:ub}).
\item 
In most cases,
conditions~\eqref{cond:norm_ratio}~and~\eqref{cond:mixtimebar} only
preclude instances with $\lbstdmtg \gtrsim \lbapproxerr$;
see~\Cref{sec:proof:lb_statement} for further discussion.  This should
not be of concern, since as noted above, main message from our minimax
lower bound is the importance of term $\lbapproxerr$ in determining the
hardness of policy evaluation problems (and for this purpose, the
regime $\lbstdmtg \lesssim \lbapproxerr$ is the relevant one).
\end{carlist}


\section{Proofs}
\label{sec:proof}

We now turn to the proof of the main theorems, with
\Cref{sec:proof_ub} devoted to the proof of our upper bounds
from~\Cref{thm:ub}, and \Cref{sec:proof_lb} providing the proof of the
minimax lower bounds stated in~\Cref{thm:lb}.


\subsection{Proof of~Theorem~\ref{thm:ub}}
\label{sec:proof_ub}

We prove the result itself in~\Cref{sec:proof_ub_overview}, with many
of the more technical aspects of the argument deferred to the
appendix.  In~\Cref{sec:proof_ub_intuition}, we provide intuition for
the terms $\stdmtg(\thetastar)$ and $\approxerr(\thetastar)$, and
establish some bounds that afford insight.

\subsubsection{Main argument}
 \label{sec:proof_ub_overview}

The proof of~\Cref{thm:ub} exploits some basic machinery introduced in
our past work~\cite{duan2021optimal}, but a number of new results are
required in order to handle multi-step TD methods along with dependent
sampling models.  For a given step length $\Step$ and weight vector
$\bweight \in \real^\Step$, we begin by introducing the functional
\begin{align}
  \label{eq:def_specfun}
\specfun(f) & \defn \bigg( \Exp \Big[ f^2(\State_0) -
  \ssum{\step=1}{\Step} \weight{\step} \discount^{\step} f(\State_0)
  f(\State_{\step}) \Big] \bigg)^{1/2} \, ,
\end{align}
where the sequence $(\State_0, \State_1, \ldots, \State_{\Step})$ is
generated from the Markov chain defined by the transition operator
$\TransOp$, with $\State_0$ drawn from the stationary
distribution~$\distr$.  From the Cauchy--Schwarz inequality, we always
have the lower bound
\begin{align}
 \Exp \bigg[ f^2(\State_0) - \ssum{\step=1}{\Step} \weight{\step}
   \discount^{\step} f(\State_0) f(\State_{\step}) \bigg] \geq ( 1 -
 \discounteff ) \, \munorm{f}^2 \geq 0,
 \end{align}
so that our definition of $\specfun$ is meaningful.  Our proof is
based on the following basic inequality satisfied by the
$\TD(\bweight)$-estimate $\thetahat$:
\begin{lemma}
\label{lemma:decomp}
The error $\Deltahat = \thetahat - \thetastar$ satisfies the
 inequality \vspace{-.6em}
\begin{align}
  \label{eq:decomp}
  (1- \discounteff) \, \munorm{\Deltahat}^2 \leq \specfun^2(\Deltahat)
  = \sum_{j=1}^3 \Term_j - \ridge \hilnorm{\Deltahat}^2, \vspace{-1em}
\end{align}
 where
\begin{subequations}
  \begin{align}
  \Term_1 & \defn \hilin[\big]{\Deltahat}{\CovOphat(\reward -
    \thetastar) + \byhat + \CrOpwhat \thetastar} \,
  , \label{eq:def_Term1} \\ \Term_2 & \defn \ridge
  \hilin[\big]{\Deltahat}{\reward - \thetastar} \, , \\ \Term_3 &
  \defn \hilin[\big]{\Deltahat}{(\Gamma - \Gammahat) \Deltahat} \quad
  \text{with $\Gammahat = \CovOphat - \CrOpwhat$ and $\Gamma = \CovOp
    - \CrOpw$} \, .
  \end{align}
\end{subequations}
\end{lemma}
\noindent This result can be proved based on a relatively
straightforward extension of the arguments used to prove a related
basic inequality from our past work~\cite{duan2021optimal}.


At the core of our proof---and requiring the bulk of our technical
effort---are upper bounds on the three terms on the right-hand side of
the basic inequality~\eqref{eq:decomp}.  Only the term $\Term_2$ is
easy to handle.  Recall that we are guaranteed to have
$\hilnorm{\thetastar - \reward} \leq \newrad$ by our
choice~\eqref{cond:newrad} of radius~$\newrad$. It follows that
\begin{align}
\label{eq:Term2}
  |\Term_2| \leq \ridge \hilnorm{\reward - \thetastar}
  \hilnorm{\Deltahat} \leq \frac{\ridge}{2} \big\{
  \hilnorm{\Deltahat}^2 + \newrad^2 \big\}.
\end{align}

We derive high probability upper bounds on the other two terms
$\Term_1$ and $\Term_3$ in in~\Cref{lemma:Term1,lemma:Term3},
respectively.  The bounds~\eqref{eq:Term1}~and~\eqref{eq:Term3} hold
for the critical radius $\delcrit \equiv \delcrit(\noise)$ for any $\noise
\geq \noisebase$.
\begin{lemma}
\label{lemma:Term1}
If $\Constdistrnew \, \Step \leq \mixtime$, then
\begin{align}
\label{eq:Term1}
|\Term_1| \; \leq \; \plaincon \, (1-\discounteff) \, \delcrit^2
\big\{ \hilnorm{\Deltahat}^2 + \newrad^2 \big\} \, \log \numobs +
\plaincon \, \newrad \, (1-\discounteff) \, \delcrit
\munorm{\Deltahat} \, \log \numobs \,
 \end{align}
with probability $1 - \constnew \, \exp\big( - \Constprob \,
\tfrac{\numobs \, \delcrit^2}{\bou^2} \big)$, where $\Constprob =
\constnewnew \, \frac{(1 - \discounteff)^2 (1 - \discount)^2}{\mixtime
  + \Step}$.
\end{lemma}
\noindent See~\Cref{sec:proof_Term1} for the proof of this claim.

\begin{lemma}
\label{lemma:Term3}
Given a sample size $\numobs$ satisfying
condition~\eqref{EqnSampleLowerBound}, we have
\begin{align}
  \label{eq:Term3}
  |\Term_3| & \leq c \, (1-\discounteff) \, \delcrit^2 \big\{
  \hilnorm{\Deltahat}^2 + \newrad^2 \big\} \, \log \numobs +
  \tfrac{1}{2} \, \rho^2(\Deltahat)
\end{align}
with probability $1 - \constnew \, \exp\big( -
\frac{\constnewnew}{\mixtime + \Step} \, \tfrac{\numobs \,
  \delcrit^2}{\bou^2} \big)$.
\end{lemma}
\noindent See~\Cref{sec:proof_Term3} for the proof of this claim. \\

The proofs of~\Cref{lemma:Term1,lemma:Term3} exploit regeneration
techniques for Markov chains~\cite{meyn2012markov}.  At a high level,
the minorization condition~\ref{assump:mixing} allows us to define a
``splitting'' of the Markov chain, such that a single trajectory is
split into sub-trajectories of random lengths, but such that the
sub-trajectories are independent of one another.  Based on this
splitting, we can then apply Talagrand-type inequalities to
i.i.d. sub-exponential random variables to derive non-asymptotic upper
bounds on the random processes; in particular, for this step, we make
use of some bounds due to Adamczak~\cite{adamczak2008tail}. \\


Given~\Cref{lemma:decomp,lemma:Term1,lemma:Term3} and
inequality~\eqref{eq:Term2}, we are now ready to finish the proof
of~\Cref{thm:ub}.  In particular, we need to show that the
bound~\eqref{eq:thm_ub} holds when the regularization parameter
$\ridge$ is suitably lower bounded, as stated in~\Cref{thm:ub}.

Replacing each of the terms $\{\Term_i\}_{i=1}^3$ with their
corresponding upper bounds in inequalities~\eqref{eq:Term1},
\eqref{eq:Term2} and \eqref{eq:Term3}, we find that
\begin{multline*}
  \frac{1 - \discounteff}{2} \, \munorm{\Deltahat}^2 \leq \tfrac{1}{2}
  \, \specfun^2(\Deltahat) \\ \leq \plaincon \, \newrad \,
  (1-\discounteff) \, \delcrit \munorm{\Deltahat} \, \log \numobs +
  \hilnorm{\Deltahat}^2 \Big\{ 2 \plaincon(1 - \discounteff)
  \delcrit^2 \, \log \numobs - \frac{1}{2} \ridge \Big\} + \newrad^2
  \Big\{ 2 \plaincon (1 - \discounteff) \delcrit^2 \, \log \numobs +
  \frac{1}{2} \ridge \Big\} \, .
  \end{multline*}
Setting $\ridge \geq 4 \plaincon(1 - \discounteff) \delcrit^2 \, \log
\numobs$ ensures that the second term is ensured to be negative,
whence
\begin{align*}
\frac{1 - \discounteff}{2} \, \munorm{\Deltahat}^2 \; \leq \;
\plaincon \, \newrad \, (1-\discounteff) \, \delcrit
\munorm{\Deltahat} \, \log \numobs + \ridge \, \newrad^2 \, .
\end{align*}
Solving this quadratic inequality for $\munorm{\Deltahat}$
yields
\begin{align*}
  \munorm{\Deltahat}^2 \; \leq \; \constnew \, \newrad^2 \Big\{
  \delcrit^2 \, \log^2 \numobs + \frac{\ridge}{1 - \discounteff}
  \Big\}
  \end{align*}
as claimed for $\delta = \delcrit$, so the bound given
in~\Cref{thm:ub} is valid.


\paragraph{Extension to multiple episodes:}

We consider policy evaluation with observations from multiple episodes
as discussed in \Cref{sec:multi_traj}. The proof of non-asymptotic
upper bounds in this case involves slightly different analyses of the
empirical processes, compared with that of a single path. When
deriving bounds on terms $|\Term_1|$ and $|\Term_3|$, we can leverage
the independence among different episodes and therefore simplify the
proof. Distinct from the current proofs of
\Cref{lemma:Term1,lemma:Term3}, there is no need to apply the
minorization condition to partition the chain into blocks when using
data from multiple episodes. Since the episodes are i.i.d.,
we can apply the Talagrand's inequality for i.i.d.
random variables to analyze $|\Term_1|$ and $|\Term_3|$. The modified
analysis leads to the results stated in \Cref{sec:multi_traj}.


\subsubsection{Intuition for $\stdmtg(\thetastar)$ and $\approxerr(\thetastar)$}
 \label{sec:proof_ub_intuition}

Let us provide some intuition for how the standard deviation
$\stdmtg(\thetastar)$ and model mis-specification error
$\approxerr(\thetastar)$, from equations~\eqref{eq:def_std} and
equation~\eqref{eq:def_approx} respectively, enter the upper
bound.  These two terms arise from the analysis of $\Term_1$ and
plays an important role in the proof of~\Cref{lemma:Term1}.

Introducing the shorthand $\numobsnew \defn \numobs - \Step$, the term
$\Term_1$ in the basic inequality~\eqref{eq:decomp} can be written as
\mbox{$\Term_1 = \frac{1}{\numobsnew} \sum_{t=1}^{\numobsnew} \,
  \Deltahat(\state_t) \, \termone_t$,} where the random variable
$\termone_t$ is given by
\begin{align}
\label{eq:def_termone}
 \termone_t \equiv \termone\big(\state_t^{t+\Step}\big) = \bigg\{
 \underbrace{\ssum{\step=1}{\Step} \weight{\step} \,
   \ReturnHat{t+1}{t+\step}(\thetastar) - \Exp\Big[
     \sum_{\step=1}^{\Step} \weight{\step} \,
     \ReturnHat{t+1}{t+\step}(\thetastar) \Bigm| \state_t
     \Big]}_{\termmtg_t \equiv \termmtg(\state_{t+1}^{t+\Step})}
 \bigg\} + \underbrace{\big( \BellOp{\bweight}(\thetastar) -
   \thetastar \big)(\state_t)}_{\termaprx_t \equiv
   \termaprx(\state_t)} \, .
\end{align}
In equation~\eqref{eq:def_termone}, the random variable $\termone_t$
is decomposed into the sum of two parts, $\termmtg_t$ and
$\termaprx_t$. In our analysis, we show that fluctuations associated
with first term $\termmtg_t$ lead to standard deviation
$\stdmtg(\thetastar)$ in definition~\eqref{eq:def_noisebase} of noise
level~$\noisebase$, while the term~$\termaprx_t$ gives rise to the
model mis-specification error $\approxerr(\thetastar)$.

It is worth noting that $\Exp[\termmtg_t \mid \state_t] = 0$, so that
the term $\termmtg_t$ can viewed as a ``martingale'' component.
Consequently, it is possible to prove concentration bounds for the
rescaled sum $\frac{1}{\numobsnew} \sum_{t=1}^{\numobsnew}
\Deltahat(\state_t) \, \termmtg_t$ that are similar to those
obtainable for i.i.d.  triples $\{ (\state_t, \termmtg_t)
\}_{t=1}^{\numobsnew}$.  Our analysis also shows that
$\stdmtg^2(\thetastar)$ can be viewed as an approximation of a
variance-like quantity associated with $\termmtg_t$. In this sense,
the effective sample size related to $\stdmtg(\thetastar)$ is
approximately $\numobs$, which is consistent with the form of
$\Errone$ in inequalities~\eqref{eq:ub_linear_b}
and~\eqref{eq:ub_alpha_b}.

The second term $\termaprx_t$ reflects the Bellman residual and is
nonzero when $\Vstar \notin \RKHS$. When analyzing the concentration
of random process $\numobsnew^{-1} \sum_{t=1}^{\numobsnew}
\Deltahat(\state_t) \, \termaprx_t$, the effective sample size is
reduced due to the temporal dependence in Markov chain, with the worst
case reduction being down to $(\numobs/\mixtime)$. This phenomenon is
also observed in inequalities~\eqref{eq:ub_linear_b} and
\eqref{eq:ub_alpha_b}, where $\Errtwo$ has a factor of
$(\numobs/\mixtime)^{-1}$. The reduced sample size explains the
scaling of the model mis-specification error $\approxerr^2(\thetastar)
\asymp \mixtime \, \munorm{\BellOp{\bweight}(\thetastar) -
  \thetastar}^2$. \\

Central to our analysis of term $\Term_1$ is control of the
``asymptotic variance'' terms
\begin{subequations}
  \begin{align}
\label{eq:def_varm}     
\varm(\feature{j}) & \defn \Exp\bigg[ \feature{j}(\State_0) \,
  \termmtg(\State_{1}^{\Step}) \; \ssum{t=-\infty}{\infty}
  \feature{j}(\State_t) \, \termmtg\big(\State_{t+1}^{t+\Step}\big)
  \bigg], \mbox{ and} \\
\label{eq:def_vara}
\vara(\feature{j}) & \defn \Exp\bigg[ \feature{j}(\State_0) \,
  \termaprx(\State_0) \; \ssum{t=-\infty}{\infty}
  \feature{j}(\State_t) \, \termaprx(\State_t) \bigg],
\end{align}
\end{subequations}
where $(\ldots, \State_{-t}, \ldots, \State_{-1}, \State_0, \State_1,
\ldots, \State_t, \ldots)$ is a stationary Markov chain governed by
transition kernel $\TransOp$, and the function $\feature{j}$ is the
$j^{th}$ eigenfunction of the covariance operator $\CovOp$. The
following result provides bounds on $\varm(\feature{j})$ and
$\vara(\feature{j})$:
\begin{lemma}
\label{lemma:var}
For any $j = 1, 2, \ldots$, we have the bounds
  \begin{align}
\label{eq:varclaims}    
  \varm(\feature{j}) \stackrel{(a)}{\leq} \unibou^2 \,
  \stdmtg^2(\thetastar), \quad \mbox{and} \quad \vara(\feature{j})
  \stackrel{(b)}{\leq} \unibou^2 \, \approxerr^2(\thetastar),
 \end{align}
where $\stdmtg(\thetastar)$ and $\approxerr(\thetastar)$ were defined
in equations~\eqref{eq:def_std} and~\eqref{eq:def_approx},
respectively.
\end{lemma}

\begin{proof}
Let us prove these two claims here, so as to provide intuition for the
different variance terms.
\paragraph{Proof of inequality~\eqref{eq:varclaims}(a):}
We begin by observing that the sum
$\termmtg\big(\State_{t+1}^{t+\Step}\big)$ admits the martingale
decomposition $\termmtg\big(\State_{t+1}^{t+\Step}\big) =
\ssum{\ell=1}{\Step} \Dtermmtg{\ell}(\State_{t+\ell-1},
\State_{t+\ell})$, where
\begin{align}
\label{eq:Dtermmtg}  
\Dtermmtg{\ell}(\State_{t+\ell-1}, \State_{t+\ell}) & \defn
\discount^{\ell} \, \ssum{\step=\ell}{\Step} \weight{\step} \, \big\{
\big( \BellOp{\step - \ell}(\thetastar) \big)(\State_{t+\ell}) - \big(
\TransOp \, \BellOp{\step - \ell}(\thetastar) \big)(\State_{t+\ell-1})
\big\}.
\end{align}
Each term $\Dtermmtg{\ell}(\State_{t+\ell-1}, \State_{t+\ell})$ is
measurable in the $\sigma$-algebra generated by $\State_{t+\ell}$ and
satisfies $\Exp\big[ \Dtermmtg{\ell}(\State_{t+\ell-1},
  \State_{t+\ell}) \bigm| \State_{t+\ell-1} \big] = 0$.  Combining
this martingale decomposition with the definition~\eqref{eq:def_varm}
of $\varm(\feature{j})$, we find that
\begin{align*}
\varm(\feature{j}) & = \Exp\bigg[ \feature{j}(\State_0) \, \Big\{
  \ssum{\ell=1}{\Step} \Dtermmtg{\ell}(\State_{\ell-1}, \State_\ell)
  \Big\} \; \sum_{t=-\infty}^{\infty} \feature{j}(\State_t) \, \Big\{
  \ssum{\ell=1}{\Step} \Dtermmtg{\ell}(\State_{t+\ell-1},
  \State_{t+\ell}) \Big\} \bigg] \\
& = \ssum{\ell=1}{\Step} \ssum{t=\ell-\Step}{\ell-1} \Exp\Big[
  \feature{j}(\State_0) \, \Dtermmtg{\ell}(\State_{\ell-1},
  \State_\ell) \; \feature{j}(\State_t) \,
  \Dtermmtg{\ell-t}(\State_{\ell-1}, \State_\ell) \Big] \, .
 \end{align*}
 We apply the Cauchy--Schwarz inequality to each term in the summation and obtain
\begin{align}
  \label{eq:varm3}
 \varm(\feature{j}) & \leq \ssum{\ell=1}{\Step}
 \ssum{t=\ell-\Step}{\ell-1} \Exp\Big[ \big\{ \feature{j}(\State_0) \,
   \Dtermmtg{\ell}(\State_{\ell-1}, \State_\ell) \big\}^2
   \Big]^{\frac{1}{2}} \; \Exp\Big[ \big\{ \feature{j}(\State_t) \,
   \Dtermmtg{\ell-t}(\State_{\ell-1}, \State_\ell) \big\}^2
   \Big]^{\frac{1}{2}} \, .
\end{align}
Denote $\Exp\big[ ( \Dtermmtg{\ell} )^2 \big] = \Exp_{\State \sim
  \distr, \Statenew \sim \TransOp(\cdot \mid \State)}\big[
  \Dtermmtg{\ell}^2(\State, \Statenew) \big]$ for short.  Using the
uniform bound $\supnorm{\feature{j}} \leq \unibou$, we find that
\begin{align*}
\Exp \Big[ \big\{ \feature{j}(\State_0) \,
  \Dtermmtg{\ell}(\State_{\ell-1}, \State_\ell) \big\}^2 \Big] \leq
\unibou^2 \; \Exp\Big[ \big\{ \Dtermmtg{\ell}(\State_{\ell-1},
  \State_\ell) \big\}^2 \Big] = \unibou^2 \; \Exp\big[ (
  \Dtermmtg{\ell} )^2 \big] \, ,
\end{align*}
where the final equality is guaranteed by the stationarity of the
process $(\ldots, \State_{-1}, \State_0, \State_1, \ldots)$.
Inequality~\eqref{eq:varm3} then implies that
\begin{align*}
\varm(\feature{j}) & \leq \unibou^2 \, \sum_{\ell=1}^{\Step}
 \ssum{t=\ell-\Step}{\ell-1} \Exp\big[ ( \Dtermmtg{\ell} )^2
   \big]^{\frac{1}{2}} \; \Exp\big[ (\Dtermmtg{\ell-t})^2
   \big]^{\frac{1}{2}} = \unibou^2 \, \bigg\{ \sum_{\ell=1}^{\Step}
 \sqrt{\Exp\big[ ( \Dtermmtg{\ell} )^2 \big]} \bigg\}^2 \, .
\end{align*}
Finally, from the definition~\eqref{eq:Dtermmtg} of $\Dtermmtg{\ell}$,
we have the equivalence \mbox{$\ssum{\ell=1}{\Step} \Exp\big[ (
    \Dtermmtg{\ell} )^2 \big]^{\frac{1}{2}} = \stdmtg(\thetastar)$},
which establishes the claim~\eqref{eq:varclaims}(a).

\paragraph{Proof of inequality~\eqref{eq:varclaims}(b):}

We now turn to analyze $\vara(\feature{j})$. Applying the
Cauchy--Schwarz inequality yields
\begin{align}
 \vara(\feature{j}) & = \Exp\big[ \feature{j}^2(\State_0) \,
   \termaprx^2(\State_0) \big] + 2 \, \ssum{t=1}{\infty} \Exp\big[
   \feature{j}(\State_0) \, \termaprx(\State_0) \;
   \Exp[\feature{j}(\State_t) \, \termaprx(\State_t) \mid \State_0 ]
   \big] \notag \\
& \leq \Exp\big[ \feature{j}^2(\State_0) \, \termaprx^2(\State_0)
   \big] + 2 \, \ssum{t=1}{\infty} \Exp\big[ \feature{j}^2(\State_0)
   \, \termaprx^2(\State_0) \big]^{\frac{1}{2}} \; \Exp\big[
   \Exp[\feature{j}(\State_t) \, \termaprx(\State_t) \mid \State_0 ]^2
   \big]^{\frac{1}{2}} \notag \\
& = \munorm{\feature{j} \, \termaprx}^2 + 2 \, \ssum{t=1}{\infty}
 \munorm{ \feature{j} \, \termaprx} \;
 \munorm[\big]{\TransOp^t(\feature{j} \, \termaprx)} \leq 2 \,  \munorm{ \feature{j} \, \termaprx} \, \sum_{t=0}^{\infty} \;
 \munorm[\big]{\TransOp^t(\feature{j} \, \termaprx)} \, . \label{eq:vara<}
\end{align}
In the sequel, we use the minorization condition~\ref{assump:mixing}
to bound the series $\ssum{t=0}{\infty} \munorm{\TransOp^t(\feature{j}
  \, \termaprx)}$.

According to Theorem~16.2.4 in the book~\cite{meyn2012markov} (see
also Theorem~3.1.4 in the paper~\cite{chilina2006f}), the minorization
condition~\eqref{cond:minor} ensures that the Markov chain is
uniformly ergodic and satisfies
\begin{align}
\label{eq:geoerg}
\norm[\big]{\TransOp^t(\cdot \mid \state) - \distr(\cdot)}_{\rm TV}
\leq 2 \, ( 1 - \mixtime^{-1} )^t \qquad \text{for any state $\state
  \in \StateSp$ and step $t \in \Natural$} \, .
\end{align}
Recall that $\thetastar$ is the solution to the projected Bellman
equation $\thetastar = \projH \, \BellOp{\bweight}(\thetastar)$,
therefore, for any eigenfunction $\feature{j} \in \RKHS$,
\begin{align*}
  \distr(\feature{j} \, \termaprx) = \inprod{\feature{j}}{\termaprx}_{\distr} = \inprod[\big]{\feature{j}}{ \BellOp{\bweight}(\thetastar) - \thetastar}_{\distr} = \inprod[\big]{\feature{j}}{ \projH \, \BellOp{\bweight}(\thetastar) - \thetastar}_{\distr} = 0 \, .
\end{align*}
For any function $f \in \Lmu$ satisfying $\distr(f) = 0$, it follows
from the definition of the total variation divergence that
$\supnorm[\big]{\TransOp^t \, f} \leq \norm[\big]{\TransOp^t(\cdot
  \mid \state) - \distr(\cdot)}_{\rm TV} \, \supnorm{f} \leq 2 \,
( 1 - \mixtime^{-1} )^t \, \supnorm{f}$, which implies that
\begin{subequations}
\begin{align}
  \label{eq:termaprx1}
  \munorm[\big]{\TransOp^t(\feature{j} \, \termaprx)} \leq
  \supnorm[\big]{\TransOp^t(\feature{j} \, \termaprx)} \leq 2 \,
  ( 1 - \mixtime^{-1} )^t \, \supnorm{\feature{j} \, \termaprx} \, .
\end{align}
In the series $\ssum{t=0}{\infty} \munorm[\big]{\TransOp^t(\feature{j}
  \, \termaprx)}$, we apply the bound~\eqref{eq:termaprx1} to terms
with $t \geq \hittime$, where the truncation level \mbox{$\hittime
  \defn \mixtime \, \log
  \tfrac{\supnorm{\termaprx}}{\munorm{\termaprx}}$.}  For terms with
$t = 1,2, \ldots, \hittime-1$, we use the bound
\begin{align}
\label{eq:termaprx2}
  \munorm[\big]{\TransOp^t(\feature{j} \, \termaprx)} \leq
  \munorm{\feature{j} \, \termaprx} \, .
\end{align}
\end{subequations}
Here the inequality is ensured by the stationarity of distribution
$\distr$ under the transition kernel $\TransOp$.

By combining the bounds~\eqref{eq:termaprx1}~and~\eqref{eq:termaprx2},
we find that
\begin{align}
  \sum_{t=0}^{\infty} \, \munorm[\big]{\TransOp^t(\feature{j} \,
    \termaprx)} & \leq \hittime \; \munorm{\feature{j} \, \termaprx} +
  \sum_{t=\hittime+1}^{\infty} 2 \, ( 1 - \mixtime^{-1} )^t \,
  \supnorm{\feature{j} \, \termaprx} \notag \\ & \leq \supnorm{\feature{j}} \,
  \big\{ \hittime \, \munorm{\termaprx} + 2 \, ( 1 - \mixtime^{-1} )^{\hittime}
  \, \mixtime \, \supnorm{\termaprx} \big\} \, . \label{eq:termaprx0}
\end{align}
Since $( 1 - \mixtime^{-1} )^{\mixtime} \leq 1/e$, we have $( 1 - \mixtime^{-1} )^{\hittime} \leq \tfrac{\munorm{\termaprx}}{\supnorm{\termaprx}} $.
It follows that
\begin{align}
  \label{eq:series<}
  \sum_{t=0}^{\infty} \, \munorm[\big]{\TransOp^t(\feature{j} \, \termaprx)} \leq \supnorm{\feature{j}} \, \big\{ \hittime \, \munorm{\termaprx} + 2 \, \mixtime \, \munorm{\termaprx} \big\} = \supnorm{\feature{j}} \cdot \mixtime \, \munorm{\termaprx} \, \big\{ 2 + \log \tfrac{\supnorm{\termaprx}}{\munorm{\termaprx}} \big\} \, .
\end{align}
Using inequality~\eqref{eq:vara<} to upper bound the quantity
$\sum_{t=0}^{\infty} \, \munorm[\big]{\TransOp^t(\feature{j} \,
  \termaprx)}$ appearing in equation~\eqref{eq:series<}, we find that
\begin{align*}
  \vara(\feature{j}) \leq 2 \, \supnorm{ \feature{j}} \,
  \munorm{\termaprx} \; \ssum{t=0}{\infty}
  \munorm[\big]{\TransOp^t(\feature{j} \, \termaprx)} \leq
  \supnorm{\feature{j}}^2 \cdot 2 \, \mixtime \, \munorm{\termaprx}^2
  \, \big\{ 2 + \log \tfrac{\supnorm{\termaprx}}{\munorm{\termaprx}}
  \big\} \, .
\end{align*}
In order to conclude the proof of the bound~\eqref{eq:varclaims}(b),
it suffices to note that \mbox{$\supnorm{\feature{j}} \leq \unibou$}
along with the bound \mbox{$2 \, \mixtime \, \munorm{\termaprx}^2 \,
  \big\{ 2 + \log \tfrac{\supnorm{\termaprx}}{\munorm{\termaprx}}
  \big\} \leq \approxerr^2(\thetastar)$.}
\end{proof}

\subsection{Proof of minimax lower bounds}
 \label{sec:proof_lb}

In this part, we prove the minimax lower bound stated
in~\Cref{thm:lb}.

\subsubsection{High-level overview}
\label{sec:proof_lb_overview}

We use a version of Fano's method for proving the lower bound; see
Chapter 15 in the book~\cite{wainwright2019high} for more details on
this and related methods for lower bounds.  We give a roadmap to the
argument here, with more technical results and their proofs deferred
to the appendices.  At the core of our proof is the construction of a
family of Markov reward processes (MRPs) whose value functions are
sufficiently ``well-separated'', but such that the models themselves
are as ``statistically close'' as possible.  Herein lies the novel and
technically challenging aspect of our proof, since capturing
dependence on the relevant problem parameters appearing in our upper
bound requires a rather delicate construction.

In order to do so, we begin with a simple discrete-state MRP, and then
``tensorize'' it to build continuous state MRP models.  We build a
large collection of MRP models $\{ \MRP_{\idxpack} \}_{\idxpack \in
  [\PackNum]} \subset \MRPclass$ as well as a linear space $\RKHS
\subset \Real^{\StateSp}$. The MRP model $\MRP_{\idxpack} =
\MRP(\reward, \TransOpm{\idxpack}, \discount)$ has a transition kernel
$\TransOpm{\idxpack}$, and a common reward function $\reward \in
\RKHS$.  Let $\distrm{\idxpack}$ be the unique stationary distribution
associated with the transition kernel $\TransOpm{\idxpack}$,
i.e. $\distrm{\idxpack}$ satisfies \mbox{$\distrm{\idxpack}(\state) =
  \int_{\StateSp} \TransOpm{\idxpack}(\statenew \mid \state) \,
  \distrm{\idxpack}(\diff\statenew)$}. We denote the value function of
MRP~$\MRP_{\idxpack}$ by \mbox{$\Vstarm{\idxpack} \defn (\IdOp -
  \discount \, \TransOpm{\idxpack})^{-1} \, \reward \in
  \Real^{\StateSp}$}. The projection of $\Vstarm{\idxpack}$ onto (the
closure of) linear space $\RKHS$ under measure $\distrm{\idxpack}$ is
set as
\begin{align}
\label{eq:def_Vstarpar}
\Vstarparm{\idxpack} = \proj_{\distrm{\idxpack}}(\Vstarm{\idxpack})
\defn \argmin_{f \in \text{closure of } \RKHS} \; \norm{ \, f - \Vstarm{\idxpack} \,
}_{\distrm{\idxpack}}.
\end{align}
We also take $\Vstarperpm{\idxpack} \defn \Vstarm{\idxpack} -
\Vstarparm{\idxpack}$.

In our value estimation problem, the dataset \mbox{$\Dataset = \{
  (\state_{t-1}, \state_t) \}_{t\in[\numobs]}$} is collected from a
trajectory \mbox{$\traj = (\state_1, \state_2, \ldots,
  \state_{\numobs})$}, where $\state_1$ is drawn from the stationary
distribution $\distrm{\idxpack}$ and the trajectory is governed by
transition kernel $\TransOpm{\idxpack}$. We denote the law of
trajectory $\traj$ by $\TransOpm{\idxpack}^{1:\numobs}$.

Fano's method yields minimax lower bounds on an estimation problem by
arguing that there is an $\PackNum$-ary hypothesis testing problem
that is as least as hard. We consider choosing an index $J$ uniformly
over $[\PackNum]$ and observations $\Dataset = \{ (\state_{t-1},
\state_t) \}_{t\in[\numobs]}$ are generated from MRP $\MRP_J$. By a
standard form of Fano's method (cf. \S 15.3.2 in the
book~\cite{wainwright2019high}), we have
\begin{multline*}
\inf_{\thetahat} \sup_{\idxpackdag \in [\PackNum]}
\Prob_{\idxpackdag}\Big( \; \distrnorm[\big]{\thetahat -
  \Vstarparm{\idxpackdag}}{\distrbar} \; \geq \; \tfrac{1}{2} \,
\min_{\idxpack \neq \idxpacknew} \distrnorm[\big]{\Vstarparm{\idxpack}
  - \Vstarparm{\idxpacknew}}{\distrbar} \; \Big) \\
\geq 1 - \frac{\log 2 + \max_{\idxpack, \idxpacknew \in [\PackNum]}
  \kull[\big]{\TransOpm{\idxpack}^{1:\numobs}}{\TransOpm{\idxpacknew}^{1:\numobs}}}{\log
  \PackNum} \, .
\end{multline*}
By our construction, we have $\{ \MRP_{\idxpack} \}_{\idxpack \in
  [\PackNum]} \subset \MRPclass(\newradbar, \stdfunbar, \lbnormperp,
\mixtimebar)$ for a group of pre-specified parameters $(\newradbar,
\stdfunbar, \lbnormperp, \mixtimebar)$. Furthermore, we can prove that
\begin{align}
\label{cond:densityratio>}
\frac{\diff \, \distrm{\idxpack}}{\diff \, \distrbar} (\state) \geq
\frac{1}{2} \qquad \text{for any $\state \in \StateSp$}.
\end{align}
It follows that
\begin{multline}
\label{eq:lb0}
\inf_{\thetahat} \sup_{\idxpackdag \in [\PackNum]}
\Prob_{\idxpackdag}\Big( \; \distrnorm[\big]{\thetahat -
  \Vstarparm{\idxpackdag}}{\distrm{\idxpackdag}}^2 \; \geq \;
\tfrac{1}{8} \, \min_{\idxpack \neq \idxpacknew}
\distrnorm[\big]{\Vstarparm{\idxpack} -
  \Vstarparm{\idxpacknew}}{\distrbar}^2 \; \Big) \\
\geq 1 - \frac{\log 2 + \max_{\idxpack, \idxpacknew \in [\PackNum]}
  \kull[\big]{\TransOpm{\idxpack}^{1:\numobs}}{\TransOpm{\idxpacknew}^{1:\numobs}}}{\log
  \PackNum} \, .
\end{multline}
We further show that our constructed MRP instances $\{
\MRP_{\idxpack}\}_{\idxpack \in [\PackNum]}$ satisfy inequalities
\begin{subequations}
\label{eq:cond}
\begin{align}
\label{eq:KL<0}  
\max_{\idxpack, \idxpacknew \in [\PackNum]}
\kull[\big]{\TransOpm{\idxpack}^{1:\numobs}}{\TransOpm{\idxpacknew}^{1:\numobs}}
& \leq \frac{\statdim}{45} \, \qquad \text{and} \\
\label{eq:valuegap>0}
\min_{\idxpack \neq \idxpacknew} \; \distrnorm[\big]{
  \Vstarparm{\idxpack} \, - \, \Vstarparm{\idxpacknew} }{\distrbar}^2
& \geq 8 \, \const{1} \; \newradbar^2 \, \delcrit^2.
\end{align}
\end{subequations}
Substituting the bound~\eqref{eq:KL<0} into
inequality~\eqref{eq:lb0}, we find the right-hand side is greater than
a positive scalar $0.1$. By replacing $\min_{\idxpack \neq
  \idxpacknew} \distrnorm[\big]{\Vstarparm{\idxpack} -
  \Vstarparm{\idxpacknew}}{\distrbar}^2$ in inequality~\eqref{eq:lb0}
with its lower bound~\eqref{eq:valuegap>0}, we derive \ref{eq:def_lb}
and finish the proof of \Cref{thm:lb}. \\

Thus, the remaining portion of the proof (and the novel and innovative
aspects of our argument) is devoted to the following two steps:
\begin{itemize}
 \item Constructing an RKHS $\RKHS$ and a group of MRP instances $\{
   \MRP_{\idxpack} \}_{\idxpack \in [\PackNum]}$ that belong to
   familly~$\MRPclass$ given in definition~\eqref{eq:def_MRPclass}. We give the details of the construction in \Cref{sec:lb_construction}.
 \item Verifying the claims \eqref{cond:densityratio>},
   \eqref{eq:KL<0} and \eqref{eq:valuegap>0}. The proofs are provided
   in \Cref{append:proof_lb}.
 \end{itemize}


\subsubsection{Intuition for $\lbstdmtg$ and $\lbapproxerr$}
\label{sec:lb_proof_intuition}

It is useful to examine the roles of the terms $\lbstdmtg$ and
$\lbapproxerr$ in the minimax lower bound.  In this section, we
develop some auxiliary results of independent interest that provide an
alternative interpretation for these terms, one that is complementary
to the asymptotic variance perspective given
in~\Cref{sec:proof_ub_intuition}.
 
For any pair of distinct indices $\idxpack, \idxpacknew \in
[\PackNum]$, the difference between the projected value functions
satisfies
 \begin{align}
 \Vstarparm{\idxpack} \, - \, \Vstarparm{\idxpacknew} = \proj_{\distrm{\idxpack}} \Vstarm{\idxpack} \, - \, \proj_{\distrm{\idxpacknew}} \Vstarm{\idxpacknew} = \underbrace{\proj_{\distrm{\idxpack}} (\Vstarm{\idxpack} \, - \, \Vstarm{\idxpacknew})}_{\Deltamtg} + \underbrace{(\proj_{\distrm{\idxpack}} - \, \proj_{\distrm{\idxpacknew}}) \, \Vstarm{\idxpacknew}}_{\Deltaaprx} \, .
 \end{align}
 The term $\Deltamtg$ reflects the difference in value functions $\Vstarm{\idxpack}$ and $\Vstarm{\idxpacknew}$, whereas the term $\Deltaaprx$ is due to the shift in stationary distributions $\distrm{\idxpack}$ and $\distrm{\idxpacknew}$ used in projections. Our construction ensures that the $\Ltwo{\distrbar}$-norm of these two terms are given by
 \begin{align}
 \label{eq:Deltamtgandaprx}
 \distrnorm{ \Deltamtg }{\distrbar} \; \asymp \; \lbstdmtg \sqrt{\frac{\statdim}{\numobs}} \qquad \text{and} \qquad \distrnorm{ \Deltaaprx }{\distrbar} \; \asymp \; \lbapproxerr \sqrt{\frac{\statdim}{\numobs}} \, .
 \end{align}
 Moreover, it approximately holds that
 \mbox{$\distrnorm[\big]{\Vstarparm{\idxpack} \, - \,
     \Vstarparm{\idxpacknew}}{\distrbar} \asymp \distrnorm{ \Deltamtg
   }{\distrbar} + \distrnorm{ \Deltaaprx }{\distrbar} \asymp \lbnoise
   \, \sqrt{\frac{\statdim}{\numobs}}$,} which implies that
 $\distrnorm[\big]{\Vstarparm{\idxpack} \, - \,
   \Vstarparm{\idxpacknew}}{\distrbar}^2 \asymp \newradbar^2 \,
 \delcrit^2$ under the kernel regularity
 condition~\eqref{cond:regular}. \\
 
The term $\Deltaaprx$ reveals a form of distribution shift that arises
in policy evaluation (but not in a classical non-parametric
regression).  Moreover, the term $\Deltaaprx$ leads to the quantity
$\lbapproxerr = \sqrt{\mixtime} \, \lbnormperp$ in our lower bound, as
shown by the following auxiliary result:
 \begin{lemma}
   \label{lemma:lb_approxerr}
The function $\Deltaaprx \; = \; (\proj_{\distrm{\idxpack}} -
\proj_{\distrm{\idxpacknew}}) \, \Vstarm{\idxpacknew}$ satisfies the equivalence
\begin{align*}
  \inprod{f}{\Deltaaprx}_{\distrm{\idxpack}} = \inprod[\big]{f}{\Vstarperpm{\idxpacknew}}_{\distrm{\idxpack} - \distrm{\idxpacknew}} \qquad \text{for any function $f \in \RKHS$.}
\end{align*}
Here the inner product of any two integrable functions $f_1$, $f_2$ with respect to the signed measure $\distrm{\idxpack} - \distrm{\idxpacknew}$ is defined as \mbox{$\inprod{f_1}{f_2}_{\distrm{\idxpack} - \distrm{\idxpacknew}} \defn \int f_1 f_2 \, \diff \distrm{\idxpack} - \int f_1 f_2 \, \diff \distrm{\idxpacknew}$}.
 \end{lemma}

In our construction, the ``divergence'' between $\distrm{\idxpack}$
and $\distrm{\idxpacknew}$ has the order of $\sqrt{\mixtime \statdim /
  \numobs}$ and the ``norm'' of $\Vstarperpm{\idxpacknew}$ is
approximately $\lbnormperp$. Therefore, the overall magnitude of
$\Deltaaprx$ is given by
\begin{align*}
   \distrnorm{\Deltaaprx}{\distrbar} \asymp \sqrt{\mixtime \statdim / \numobs} \cdot \lbnormperp = \lbapproxerr \sqrt{\statdim/\numobs} \, ,
\end{align*}
as shown in equation~\eqref{eq:Deltamtgandaprx}.  Let us now prove
\Cref{lemma:lb_approxerr}.
\begin{proof}
For notational convenience, we take $\idxpack = 1$ and $\idxpacknew =
2$. For any function $f \in \RKHS$, by definitions of the projections $\proj_{\distrm{1}}$
and $\proj_{\distrm{2}}$, it holds that
\begin{subequations}
\begin{align}
\label{eq:lb_projV}
\inprod[\big]{f}{\Vstarm{2} - \proj_{\distrm{1}} \Vstarm{2}}_{\distrm{1}} & = 0 \qquad \text{and} \\
\label{eq:lb_projV2}
\inprod[\big]{f}{\Vstarm{2} - \proj_{\distrm{2}} \Vstarm{2}}_{\distrm{2}} & = \inprod[\big]{f}{\Vstarperpm{2}}_{\distrm{2}} = 0 \, .
 \end{align}
\end{subequations}
We substitute the term $\proj_{\distrm{1}} \Vstarm{2}$ in equation~\eqref{eq:lb_projV} with $\proj_{\distrm{2}} \Vstarm{2} + \Deltaaprx$ and find that
\begin{align*}
  0 & = \inprod[\big]{f}{\Vstarm{2} - \proj_{\distrm{2}} \Vstarm{2} - \Deltaaprx}_{\distrm{1}} = \inprod[\big]{f}{\Vstarperpm{2} - \Deltaaprx}_{\distrm{1}} \, ,
\end{align*}
which further implies
\begin{align}
  \label{eq:lb_projV0}
  \inprod[\big]{f}{\Deltaaprx}_{\distrm{1}} & = \inprod[\big]{f}{\Vstarperpm{2}}_{\distrm{1}} \, .
\end{align}
Subtracting equations~\eqref{eq:lb_projV0} and \eqref{eq:lb_projV2} then yields $\inprod[\big]{f}{\Deltaaprx}_{\distrm{1}} = \inprod[\big]{f}{\Vstarperpm{2}}_{\distrm{1} - \distrm{2}}$, 
as stated in \Cref{lemma:lb_approxerr}.
\end{proof}


\section{Discussion}
\label{SecDiscussion}

In this paper, we analyzed non-asymptotic statistical properties of
kernel-based multi-step temporal difference (TD) methods. In
particular, we investigated how facts such as temporal dependence in
the samples and/or mis-specification in the model influence the
statistical estimation error. In the presence of trajectory data, our
theory also shows when and to what extent multi-step TD methods
improve the quality of estimates.  The main contribution of our work
was a non-asymptotic upper bound on the estimation error in any
$\bweight$-weighted TD estimate, and a minimax lower bound over
subclasses of MRPs that sets a fundamental limit for any
estimator. Our theory shows that the upper and lower bound match each
other for some properly chosen TD methods, and therefore exhibits the
minimax optimality of the TD estimates.
  
Our work leaves open a number of intriguing questions; let us mention
a few of them here to conclude.  First, it would be interesting to
develop a principled method for parameter selection in
$\bweight$-weighted TD that can be implemented without
population-level knowledge. Currently, our theory involves some
quantities that are non-trivial to estimate using data, for instance,
the norm of Bellman residual and the mixing time. Second, the scope of
the paper is restricted to the on-policy setting in reinforcement
learning. The generalization of the theory to off-policy evaluation
remains challenging. It is interesting to determine whether, and if so
under what conditions, off-policy procedures can be devised to benefit
from multi-step predictive models.  Third, our theory shows that
various properties of the model and data---including temporal
dependence, mixing, model mis-specification and so---all affect the
quality of estimates in coupled ways.  Another interesting direction
is how to use possible freedom in data collection so as to develop
adaptive procedures that minimize the estimation error.


\subsection*{Acknowledgements}

This work was partially supported by Office of Naval Research Grant
ONR-N00014-21-1-2842, NSF-CCF grant 1955450, and NSF-DMS grant 2015454
to MJW.


\appendix


\section{Properties of kernel LSTD estimates}

In this appendix, we derive some important properties of kernel-based
projected fixed points.

\subsection{Covariance-based representation of projected fixed point}
\label{AppEquiv}

In this appendix, we show that the covariance-based representation of
the kernel fixed point in equation~\eqref{eq:proj_Bell} is equivalent
to the projected Bellman equation~\eqref{EqnDefnProjFix} with the set
$\mathds{G}$ replaced by the closure of the RKHS $\RKHS$.

First, we claim that the projected Bellman equation 
\begin{subequations}
  \begin{align}
    \label{EqnDefnProjFix_new}
    \thetastar = \projH \BellOp{\bweight} (\thetastar)
  \end{align}
is equivalent to the linear operator equation
\begin{align}
  \label{eq:proj_Bell_new}
\CovOp \, \thetastar = \Exp_{\State \sim \distr} \big[ \Rep{\State} \,
  \big(\BellOp{\bweight} (\thetastar)\big)(\State) \big] \, .
  \end{align}
\end{subequations}
In order to establish this equivalence, we notice that the right-hand
side of equation~\eqref{eq:proj_Bell_new} satisfies
\begin{align*}
 \Exp_{\State \sim \distr} \big[ \Rep{\State} \,
   \big(\BellOp{\bweight} (\thetastar)\big)(\State) \big] =
 \Exp_{\State \sim \distr} \big[ \Rep{\State} \, \big(\projH
   \BellOp{\bweight} (\thetastar)\big)(\State) \big] = \CovOp \,
 \projH \BellOp{\bweight} (\thetastar),
  \end{align*}
from which it can be seen that equation~\eqref{eq:proj_Bell_new}
implies equation~\eqref{EqnDefnProjFix_new}. We next prove the
converse.  For any function $f \in \RKHS$ satisfying the equality
$\CovOp f = 0$, it holds that \mbox{$\munorm{f}^2 = \hilin{f}{\CovOp
    \, f} = 0$}, which further implies \mbox{$f(\state) = 0$} for
$\distr$-a.e. state $\state$, so that the only solution to the
equation $\CovOp \, f = 0$ is the zero function. Since the kernel
representation~\eqref{eq:proj_Bell_new} ensures $\CovOp \, \thetastar
= \CovOp \, \projH \BellOp{\bweight} (\thetastar)$, it follows that
solution to equation~\eqref{eq:proj_Bell_new} satisfies the projected
Bellman equation~\eqref{EqnDefnProjFix_new}, where we have used the
uniqueness of the solution to the equation~\mbox{$\CovOp \, f = 0$}. We
have thus
established the claimed equivalence.

We now show that the right-hand side of
equation~\eqref{eq:proj_Bell_new} can be re-written as $\CovOp \,
\reward + \by + \CrOpw \thetastar$ with $\CrOpw$ and $\by$ given in
definitions~\eqref{eq:def_CrOpw}~and~\eqref{eq:def_by}.  By combining
equations~\eqref{EqnBellmanStep}~and~\eqref{EqnDefnWeightedBell}, we
find that
\begin{align}
  \big(\BellOp{\bweight}(\thetastar)\big)(\state) & \defn \Exp \bigg[
    \, \sum_{\step=1}^{\Step} \weight{\step} \sum_{\ell = 0}^{\step-1}
    \discount^\ell \reward(\State_\ell) + \sum_{\step=1}^{\Step}
    \weight{\step} \discount^{\step} \thetastar(\State_{\step}) \mid
    \State_0 = \state \, \bigg] \notag \\ & = \reward(\state) +
  \underbrace{\Exp \bigg[ \, \sum_{\step=1}^{\Step} \weight{\step}
      \sum_{\ell = 1}^{\step-1} \discount^\ell \reward(\State_\ell)
      \mid \State_0 = \state \, \bigg]}_{g_1(\state)} +
  \underbrace{\Exp \bigg[ \, \sum_{\step=1}^{\Step} \weight{\step}
      \discount^{\step} \thetastar(\State_{\step}) \mid \State_0 =
      \state \, \bigg]}_{g_2(\state)}
  \end{align}
for any state $\state \in \StateSp$.  We multiply the three terms
$\reward(\state)$, $g_1(\state)$ and $g_2(\state)$ with the
representer of evaluation $\Rep{\state}$ and take the expectations
over distribution $\distr$. It follows that
\begin{align*}
  \Exp_{\State \sim \distr} \big[ \Rep{\State} \, \reward(\State)
    \big] & = \CovOp \, \reward \, , \\ \Exp_{\State \sim \distr}
  \big[ \Rep{\State} \, g_1(\State) \big] & = \Exp \bigg[
    \Rep{\State_0} \sum_{\step=1}^{\Step} \weight{\step} \sum_{\ell =
      1}^{\step-1} \discount^\ell \reward(\State_\ell) \bigg] = \by \,
  , \\ \Exp_{\State \sim \distr} \big[ \Rep{\State} \, g_2(\State)
    \big] & = \Exp \bigg[ \Rep{\State_0} \sum_{\step=1}^{\Step}
    \weight{\step} \discount^{\step} \thetastar(\State_{\step}) \bigg]
  = \CrOpw \thetastar \, .
  \end{align*}
The right-hand side of equation~\eqref{eq:proj_Bell_new} satisfies
\begin{align*}
  \Exp_{\State \sim \distr} \big[ \Rep{\State} \, \big(\BellOp{\bweight} (\thetastar)\big)(\State) \big] = \CovOp \, \reward + \by + \CrOpw \thetastar \, ,
\end{align*}
therefore, equation~\eqref{eq:proj_Bell_new} is equivalent to the
kernel fixed point equation \eqref{eq:proj_Bell}.  Finally, combining
the pieces shows that equations~\eqref{eq:proj_Bell}
and~\eqref{EqnDefnProjFix_new} are equivalent.


\subsection{Efficient computation of kernel LSTD estimates}
\label{AppComputation}

In this section, we provide an explicit matrix-form expression for the
kernel LSTD estimates defined in equation~\eqref{eq:def_thetahat}.  We
introduce a shorthand $\numobsnew \defn \numobs - \Step$.  Define the
kernel covariance matrices $\CovMt \in \Real^{\numobsnew \times
  \numobsnew}$ and $\CrMtw \in \Real^{\numobsnew \times \numobsnew}$
with entries
\begin{align}
  \label{eq:def_KerMt}
  \CovMt(i,j) \defn \Ker(\state_i, \state_j) \, / \, \numobsnew,
  \qquad \text{and} \qquad \CrMtw(i,j) \defn \sum_{\step=1}^{\Step}
  \weight{\step} \discount^{\step} \, \Ker(\state_{i+\step}, \state_j)
  \, / \, \numobsnew
  \end{align}
for $i,j = 1,2,\ldots,\numobsnew$.  Let $\yvec \in \Real^{\numobsnew}$
be a vector representation of the compound rewards given by
\begin{align}
  \label{eq:def_yvec}
  \yvec(i) \defn \frac{1}{\sqrt{\numobsnew}} \, \sum_{\step=1}^{\Step}
  \weight{\step} \sum_{\ell=1}^{\step} \discount^{\ell} \,
  \reward(\state_{i+\ell}) \, .
\end{align}
The following lemma provides an explicit linear-algebraic expression
for the solution $\thetahat$ of equation~\eqref{eq:def_thetahat}:
\begin{lemma}[Kernel-based computation]
\label{lemma:thetahat_kernel}
The kernel-LSTD estimate $\thetahat$ takes the form $\thetahat =
\reward + \frac{1}{\sqrt{\numobs - \Step}} \sum_{t=1}^{\numobs-\Step}
\alphahat_t \, \Ker(\cdot, \state_t)$, where the coefficient vector
$\alphavechat \in \Real^{\numobsnew}$ solves the linear system
\begin{align}
  \label{eq:def_thetahat_vec}
  \big( \CovMt \; + \; \ridge \, \IdMt_{\numobsnew} \; - \; \CrMtw \big) \, \alphavechat \; = \; \yvec \, ,
\end{align}
where the vector $\yvec \in \Real^{\numobsnew}$ and matrices
$\CovMt, \CrMtw \in \Real^{\numobsnew \times \numobsnew}$ are
defined in
equations~\eqref{eq:def_KerMt}~and~\eqref{eq:def_yvec}.
\end{lemma}

\begin{proof}
  By reformulating the definition of $\thetahat$, we find that it
  satisfies the relation
  \begin{align}
    \label{eq:def_thetahat_new}
    (\CovOphat - \CrOpwhat) \, (\thetahat - \reward) + \ridge \, (\thetahat - \reward) \; = \; \byhat + \CrOpwhat \, \reward.
  \end{align}
Recall the definitions of $\byhat$ and $\CrOpwhat$ in
equations~\eqref{eq:def_byhat}~and~\eqref{eq:def_CrOpwhat}.  The
right-hand side of equation~\eqref{eq:def_thetahat_new} can be
rewritten as
\begin{align}
  \label{eq:def_thetahat_new_RHS}
  \byhat + \CrOpwhat \, \reward & = \frac{1}{\numobsnew}
  \sum_{t=1}^{\numobs-\Step} \Rep{\state_t} \Big\{
  \sum_{\step=1}^{\Step} \weight{\step} \,
  \ssum{\ell=1}{\step} \discount^\ell \,
  \reward(\state_{t+\ell}) \Big\} =
  \frac{1}{\sqrt{\numobsnew}} \, \big[ \Rep{\state_1},
    \Rep{\state_2}, \ldots, \Rep{\state_{\numobsnew}} \big] \,
  \yvec \, ,
\end{align}
where $\yvec \in \Real^{\numobsnew}$ is the vector defined in
equation~\eqref{eq:def_yvec}.  We further consider the explicit
expression for the term $(\CovOphat - \CrOpwhat) \, (\thetahat -
\reward)$ on the left-hand side of
equation~\eqref{eq:def_thetahat_new}. Using the vector-form
representation of estimate $\thetahat =
\reward + \frac{1}{\sqrt{\numobs - \Step}} \sum_{t=1}^{\numobs-\Step}
\alphahat_t \, \Ker(\cdot, \state_t)$, we find that
\begin{align*}
(\CovOphat - \CrOpwhat) \, (\thetahat - \reward) & =
  \frac{1}{\numobsnew} \sum_{t=1}^{\numobsnew} \Rep{\state_t} \otimes
  \Big\{ \Rep{\state_t} - \sum_{\step=1}^{\Step} \weight{\step}
  \discount^{\step} \Rep{\state_{t+\step}} \Big\} \, (\thetahat -
  \reward) \\
& = \frac{1}{\numobsnew \sqrt{\numobsnew}} \sum_{t=1}^{\numobsnew}
  \sum_{s=1}^{\numobsnew} \Rep{\state_t} \, \Big\{ \Ker(\state_t,
  \state_s) - \sum_{\step=1}^{\Step} \weight{\step} \discount^{\step}
  \Ker(\state_{t+\step}, \state_s) \Big\} \, \alphahat_s \\ & =
  \frac{1}{\sqrt{\numobsnew}} \, \big[ \Rep{\state_1}, \Rep{\state_2},
    \ldots, \Rep{\state_{\numobsnew}} \big] \, \big( \CovMt - \CrMtw
  \big) \, \alphavechat \, .
\end{align*}
Moreover, since $\thetahat - \reward = \frac{1}{\sqrt{\numobsnew}} \,
\big[ \Rep{\state_1}, \Rep{\state_2}, \ldots,
  \Rep{\state_{\numobsnew}} \big] \, \alphavechat$, the left-hand side
of equation~\eqref{eq:def_thetahat_new} takes an equivalent form
\begin{align}
  \label{eq:def_thetahat_new_LHS}
  (\CovOphat - \CrOpwhat) \, (\thetahat - \reward) + \ridge \,
  (\thetahat - \reward) = \frac{1}{\sqrt{\numobsnew}} \, \big[
    \Rep{\state_1}, \Rep{\state_2}, \ldots, \Rep{\state_{\numobsnew}}
    \big] \, \big( \CovMt + \ridge \, \IdMt_{\numobsnew} - \CrMtw
  \big) \, \alphavechat \, .
\end{align}
By combining
expressions~\eqref{eq:def_thetahat_new_RHS}~and~\eqref{eq:def_thetahat_new_LHS}
with equation~\eqref{eq:def_thetahat_new}, we conclude that the
$\alphavechat$ vector in definition~\eqref{eq:def_thetahat} of
estimate $\thetahat$ satisfies the linear
system~\eqref{eq:def_thetahat_vec}, as claimed in
\Cref{lemma:thetahat_kernel}.
\end{proof}


\subsection{A backward form of the covariance-based representation}
\label{append:backward}

In this part, we establish an equivalent form of the projected fixed
point equation~\eqref{eq:proj_Bell}, which also provides a backward
view of the multi-step TD methods.
  
Let $(\State_{-\Step}, \State_{-\Step+1}, \ldots, \State_0) \in
\State^{\Step+1}$ be a stationary process governed by the Markov
transition kernel~$\TransOp$, i.e. $\State_{-\step} \sim \distr$ and
$\State_{-\step+1} \sim \TransOp(\cdot \mid \State_{-\step})$ for any
$\step = 0,1,\ldots,\Step$. We take an eligibility trace function
\begin{align}
  \label{eq:def_eligtrace}
  \eligtr \defn \sum_{\step=0}^{\Step-1} \sum_{\ell = \step+1}^{\Step} \weight{\ell} \,
  \discount^{\step} \, \Rep{\State_{-\step}}  \; \in \;  \RKHS \, .
  \end{align}
Given the function $\eligtr$, we define an operator $\Aop: \RKHS
\rightarrow \RKHS$ and a function $\bb$ as \mbox{$\Aop \defn \Exp\big[
    \eligtr \otimes ( \Rep{\State_0} - \discount \, \Rep{\State_1} )
    \big]$} and \mbox{$\bb \defn \Exp \big[ \reward(\State_0) \;
    \eligtr \big]$.}  With this notation, we claim that the fixed
point $\thetastar$ to the projected Bellman operator $\projH
\BellOp{\bweight}$ satisfies the relation \mbox{$\Aop \, \thetastar =
  \bb$.}

Let us prove this claim.  We use the backward sequence
$(\State_{-\Step}, \State_{-\Step+1}, \ldots, \State_0)$ to rewrite
the operator $\CrOpw$ and function $\CovOp \, \reward + \by$ in
equation~\eqref{eq:proj_Bell}. It follows from their
definitions~\eqref{eq:def_CrOpw}~and~\eqref{eq:def_by} that
\begin{subequations}
  \begin{align}
    \CrOpw & = \Exp \bigg[ \Big\{ \sum_{\step=1}^{\Step}
      \weight{\step} \discount^{\step} \Rep{\State_{-\step}}
      \Big\} \otimes \Rep{\State_0} \bigg] \qquad
    \text{and} \label{eq:CrOpw_back} \\
    \CovOp \, \reward + \by
    & = \Exp \bigg[ \Big\{
      \sum_{\step=0}^{\Step-1} \sum_{\ell = \step+1}^{\Step} \weight{\ell}
      \discount^{\step} \Rep{\State_{-\step}} \Big\} \, \reward(\State_0) \bigg] = \Exp \big[ \reward(\State_0) \, \eligtr \big] = \bb \, . \label{eq:bb}
  \end{align}
\end{subequations}
  	
We next establish the relation between operator $\Aop$ and the
difference $\CovOp - \CrOp$.  Applying the
definition~\eqref{eq:def_eligtrace} of the eligibility trace $\eligtr$
yields
\begin{align}
  \Aop & = \Exp \big[ \eligtr \otimes \{ \Rep{\State_0} - \discount
    \Rep{\State_1} \} \big] = \Exp \big[ \eligtr \otimes
    \Rep{\State_0} \big] - \discount \, \Exp \big[ \eligtr \otimes
    \Rep{\State_1} \big] \notag \\ & = \Exp \bigg[ \Big\{
    \sum_{\step=0}^{\Step-1} \sum_{\ell = \step+1}^{\Step}
    \weight{\ell} \, \discount^{\step} \Rep{\State_{-\step}} \Big\}
    \otimes \Rep{\State_0} \bigg] - \Exp \bigg[ \Big\{
    \sum_{\step=1}^{\Step} \sum_{\ell = \step}^{\Step} \weight{\ell}
    \, \discount^{\step} \Rep{\State_{-\step}} \Big\} \otimes
    \Rep{\State_0} \bigg] \notag \\ & = \Exp \bigg[ \Big\{ \sum_{\ell
      = 1}^{\Step} \weight{\ell} \, \Rep{\State_0} \Big\} \otimes
    \Rep{\State_0} \bigg] - \Exp \bigg[ \Big\{ \sum_{\step=1}^{\Step}
    \weight{\step} \, \discount^{\step} \Rep{\State_{-\step}} \Big\}
    \otimes \Rep{\State_0} \bigg] = \CovOp - \CrOpw,
\label{eq:Aop}
\end{align}
where the last line follows from the equality~$\sum_{\ell=1}^{\Step}
\weight{\ell} = 1$ and the expression for $\CrOpw$ in
equation~\eqref{eq:CrOpw_back}.  We combine
equalities~\eqref{eq:bb}~and~\eqref{eq:Aop} with the projected fixed
point equation~\eqref{eq:proj_Bell} and find that the population-level
solution $\thetastar$ satisfies the claimed relation \mbox{$\Aop \,
  \thetastar = \bb$.}


\subsection{The backward estimate and stochastic approximation}
\label{sec:backward_estimate}

We propose an estimate $\thetahatback$ based on the backward equation
\mbox{$\Aop \, \thetastar = \bb$.}  Let $\thetahatback$ be the
solution to equation
\begin{align} \label{eq:def_thetahatback}
  \big(\Aophat + \ridge \, \IdOp\big) \; \thetahatback \; = \; \bbhat
  + \ridge \, \reward \, ,
\end{align}
where $\ridge > 0$ is the a ridge parameter. Given the sample
trajectory $(\state_1, \ldots, \state_{\numobs})$, the operator
$\Aophat: \RKHS \rightarrow \RKHS$ and function $\bbhat \in \RKHS$ are
defined as
\begin{align*}
  \Aophat \defn \frac{1}{\numobs - \Step} \sum_{t=\Step}^{\numobs-1}
  \eligtr_t \otimes (\Rep{\state_t} - \discount \, \Rep{\state_{t+1}})
  \, \qquad \text{and} \qquad \bbhat \defn \frac{1}{\numobs - \Step}
  \sum_{t=\Step}^{\numobs-1} \reward(\state_t) \; \eligtr_t
\end{align*}
with the eligibility trace \mbox{$\eligtr_t \defn
  \sum_{\step=0}^{\Step-1} \sum_{\ell = \step+1}^{\Step} \weight{\ell}
  \, \discount^{\step} \, \Rep{\state_{t-\step}}$.}  Note that the
definition of $\thetahatback$ in equation~\eqref{eq:def_thetahatback}
can be viewed as a variant of definition~\eqref{eq:def_thetahat} of
estimate $\thetahat$ after shifting some of the indices for at most
$\Step$ steps. Therefore, as data accumulate, the difference between
estimates $\thetahat$ and $\thetahatback$ is negligible. Our analyses
for the estimate $\thetahat$ given in equation~\eqref{eq:def_thetahat}
also apply to $\thetahatback$ in equation~\eqref{eq:def_thetahatback}.

We can equivalently interpret the estimate $\thetahatback$ as the
output of a specific stochastic approximation procedure. To see this,
we initialize with $\thetahatback_0 \defn \reward$ and $\Aophat_0
\defn (\numobs - \Step) \, \ridge \; \IdOp$, and then, for time steps
$t = 0,1, \ldots, \numobs - \Step - 1$, we iteratively compute
\begin{multline}
  \label{eq:SA}
  \thetahatback_{t+1} \defn \thetahatback_t + \const{t} \;
  \Aophat_t^{-1} \eligtr_{t+\Step} \, \big\{ \reward(\state_{t+\Step})
  + \discount \, \thetahat_t (\state_{t+\Step+1}) -
  \thetahat_t(\state_{t+\Step}) \big\} \\ \text{where} \quad \const{t}
  \defn \big\{ 1 + \inprod[\big]{\Rep{\state_{t+\Step}} - \discount \,
    \Rep{\state_{t+\Step+1}}}{\Aophat_t^{-1}
    \eligtr_{t+\Step}}_{\RKHS} \big\}^{-1}
  \end{multline}
and \mbox{$\Aophat_{t+1} \defn \Aophat_t + \eligtr_{t+\Step} \otimes
  \big( \Rep{\state_{t+\Step}} - \discount \, \Rep{\state_{t+\Step+1}}
  \big)$.}  Then the estimate $\thetahatback$ from
equation~\eqref{eq:def_thetahatback} is the output at $t = \numobs -
\Step$, i.e. $\thetahatback = \thetahatback_{\numobs - \Step}$. The
key update step~\eqref{eq:SA} in the procedure is a rescaled
semi-gradient TD update.  The theory developed in this paper also
provides guarantees for the stochastic approximation procedure
described above.

\section{Details of simulations}
\label{sec:experiment}

In this part, we provide details of the simulation results reported in
\Cref{FigNsampMis,FigNsampTD}.
  
\paragraph{Families of MRPs:}
We begin by describing the families of Markov reward processes used in
our simulations. For any mixing time $\mixtime \geq 1$, we define a transition
kernel $\TransOp$ over state space $\StateSp = [0,1]$ as
\begin{align*}
  \TransOp(\statenew \mid \state) & \defn \begin{cases} 1 - \mixtime^{-1} / 2 \quad & \text{if $\state, \statenew \in [0,
        \tfrac{1}{2})$ or $\state, \statenew \in [\tfrac{1}{2}, 1]$},
        \\ \mixtime^{-1} /2 & \text{otherwise}.
  \end{cases}
  \end{align*}
Note that the transition function $\TransOp$ above satisfies the
mixing condition~\ref{assump:mixing}, and has the Lebesgue measure on
$\StateSp$ as its unique stationary distribution $\distr$.

The reward function takes the form
\begin{align*}
\reward(\state) & = \begin{cases} \reward_0 \, \big( \cos\lbtheta + \sqrt{2} \,
  \sin\lbtheta \big) & \text{if $\state \in [0,\tfrac{1}{4})$},
    \\ \reward_0 \, \big( \cos\lbtheta - \sqrt{2} \, \sin\lbtheta \big) & \text{if $\state \in
      [\tfrac{1}{4}, \tfrac{1}{2})$}, \\ - \reward_0 \, \cos\lbtheta & \text{if
        $\state \in [\tfrac{1}{2}, 1]$},
\end{cases}
\end{align*}
where the parameter $\lbtheta \in [0,\pi/2]$ is determined by our
choice of the function space $\RKHS$ (as shown later).  We rescale the
reward function by a scalar $\reward_0 > 0$ so that the curves of
different groups are separated on the figure, but $\reward_0$ remains
the same for each single configuration.  In panels (a) and (b)
of~\Cref{FigNsampMis}, we pick a same $\reward_0$ for experiments with
a same mixing time $\mixtime$, whereas in panels (a) and (b)
of~\Cref{FigNsampTD}, scalar $\reward_0$ is the same for each pair of
parameters $(\mixtime, \lbtheta)$.  We use discount factor $\discount
\defn 0.9$ for all experiments.


\paragraph{Reproducing kernel Hilbert space:}  We compute the
  kernel-based TD estimates using a RKHS $\RKHS$ defined in terms of
  the Walsh basis functions.  (We note that these same basis functions
  and RKHS were used to prove our minimax lower bound.) Let $\{
  \feature{j} \}_{j=1}^{\infty}$ be a group of orthonormal functions
  in $\Lmu$ and $\{ \eig{j} \}_{j=1}^{\infty}$ be a series of
  pre-specified eigenvalues. We take the kernel $\Ker$ as
  \begin{align*}
  	\Ker(\state, \statey) = \sum_{j=1}^{\infty} \; \eig{j} \;
        \feature{j}(\state) \, \feature{j}(\statey) \qquad \text{for
          any $\state, \statey \in \StateSp$},
  \end{align*}
  and the space $\RKHS$ as the associated RKHS. In particular, the
  features $\{ \feature{j} \}_{j=1}^{\infty}$ take the form
  \begin{subequations}
  	\begin{align*}
  	\feature{2j+1} & \defn \tfrac{1}{2} \big\{ \walsh{2j} -
        \walsh{2j+1} + \walsh{4j} + \walsh{4j+1} \big\} \, ,
        \\ \feature{2j+2} & \defn \tfrac{1}{2} \big\{ \walsh{2j} -
        \walsh{2j+1} + \walsh{4j} + \walsh{4j+1} \big\} \cdot \big\{
        \walsh{1} \, \cos\lbtheta \; + \; \tfrac{1}{\sqrt{2}} \, (
        \walsh{2} + \walsh{3} ) \, \sin\lbtheta \big\} \, ,
  	\end{align*}
  \end{subequations}
  which are identical to the bases given in
  equations~\eqref{eq:def_feature}.  Here $\walsh{j}$ denotes the
  $j^{th}$ Walsh function and the angle $\lbtheta \in [0, \pi/2]$
  controls the model mis-specification error.  We consider two series
  of eigenvalues:
  \begin{subequations}
  \begin{align}
  \label{eq:def_exp_eigvalue}
  \text{($0.6$-)polynomial decay:} & \quad \eig{j} \defn j^{-6/5} & \text{in \Cref{FigNsampMis}},
  \\
  \text{exponential decay:} & \quad \eig{j} \defn \exp
  \big\{-(j-1)^2\big\} & \text{in \Cref{FigNsampTD}},
  \end{align}
  \end{subequations}
  for any $j = 1,2,\ldots$.

By comparing to our definition of the reward function, we have
$\reward = \feature{2} \in \RKHS$ by construction.

\paragraph{Simulation set-up:}
For all of the simulation results reported
in~\Cref{FigNsampMis,FigNsampTD}, we pick the sample size~$\numobs$ as
\begin{align*}
\numobs & \in \big\{ \lfloor \exp(7 + 0.3 i) \rfloor \mid i = 0, 1,
\ldots, 14 \big\} = \{ 1096, 1480, 1998, \ldots, 54176, 73130\} \, .
\end{align*}
In each i.i.d. trial, we set the ridge parameter $\ridge \defn 0.01 \,
\times \, \delcrit^2 \, (1 - \discounteff)$, with $\delcrit$ denoting
the smallest positive solution to inequality~$\CI(\noisebase)$.

\paragraph{Results in~\Cref{FigNsampMis}:} We compare the
performance of the standard $1$-step TD estimate when using data from
either a single path or i.i.d. transition pairs. In particular, we
consider the following two regimes:
\begin{itemize} \itemsep = -.2em
\item a single path $(\state_1, \state_2, \ldots,
  \state_{\numobs})$ with $\state_1 \sim \distr$ and
  $\state_{t+1} \sim \TransOp(\cdot \mid \state_t)$ for $t =
  1,2,\ldots,\numobs-1$;
\item i.i.d. transition pairs $\{ (\statetil_i, \statetilnew_i)
  \}_{i=1}^{\numobs}$ with $\statetil_i \stackrel{i.i.d.}{\sim}
  \distr$ and $\statetilnew_i \sim \TransOp(\cdot \mid \statetil_i)$.
\end{itemize}
We use the polynomial kernel in all experiments
in~\Cref{FigNsampMis}. Moreover, we set the mixing time $\mixtime \in
\big\{ e^4/2, \, e^6/2 \big\} = \{ 27.3, \, 201.7 \}$ for both panels
(a) and (b) in Figures~\ref{FigNsampMis}. In panel (a), we choose
$\lbtheta = 0$ so as to ensure that there is no model
mis-specification error; whereas in panel (b), we set $\lbtheta
=\pi/4$ so that the mis-specification error is large.
  
In \Cref{FigNsampMis}(b), we report the average truncated mean squared
error $\min \big\{\munorm{\thetahat - \thetastar}^2, \; 100 \,
\munorm{\thetastar}^2 \big\}$ for each point, so as to avoid
distortions caused by very large errors (which happens occasionally
for the smaller sample sizes in our simulations).

\paragraph{Results in~\Cref{FigNsampTD}:} We compare the
performances of $1$, $5$ and $10$-step TD methods under different
regimes.  In this group of simulations, we use data collected from a
single path $(\state_1, \state_2, \ldots, \state_{\numobs})$ and conduct RKHS approximation using a kernel with exponentially decaying eigenvalues.  \\ In
\Cref{FigNsampTD}(a), we choose the pair of mixing time and angle
parameters $(\mixtime, \lbtheta)$ as
  \begin{align*}
    (\mixtime, \lbtheta) & \in \big\{ \big(2, \pi/16\big), \big(e^6/2,
    0\big) \big\} \, ,
  \end{align*}
  which correspond to MRPs that are fast mixing or have zero model
  mis-specification. \\
In~\Cref{FigNsampTD}(b), we choose the pair $(\mixtime, \lbtheta) =
\big(e^6/2, \pi/16\big)$ so that the MRP is slowly mixing and the
model mis-specification error is significant.


\section{Proof of noise upper bound~\eqref{eq:def_noisebasenew}}
\label{app:cor:ub_new}

The upper bound~\eqref{eq:def_noisebasenew} on noise level $\noisebase$ is a consequence of some
bounds of independent interest, which relate the Bellman fluctuation
$\stdmtg(\thetastar)$ to the conditional variance of value function
$\Vstar$, and bound the Bellman residual
$\munorm{\BellOp{\bweight}(\thetastar) - \thetastar}$ in terms of the
model mis-specification error.  We summarize more formally as:
\begin{lemma}
\label{lemma:std}
\begin{enumerate} 
\item[(a)]
\label{lemma:stdmtg}
We have the upper bound
\begin{subequations}
  \begin{align}
    \label{eq:ub_std}
    \stdmtg(\thetastar) \leq \frac{\discount \,
      (1-\discounteff)}{1-\discount} \, \stdfun(\Vstar) +
    \sqrt{\frac{\discount \, (1-\discounteff)}{1-\discount}}
    \munorm{\thetastar - \Vstar} \, ,
  \end{align}
where 
$\stdfun^2(\Vstar) = \Exp_{\State \sim
    \distr} \big[ \Var [ \Vstar(\Statenew) \bigm| \State ]
          \big]$ is defined in equation~\eqref{eq:def_stdfun}.
\item[(b)]
\label{lemma:approxerr}  
In terms of the operator $\cpdP \defn \ssum{\step=1}{\Step}
\weight{\step} \discount^{\step} \TransOp^{\step}$ and the projection
error $\Vstarperp \defn \Vstar - \projH(\Vstar)$, we have the
relations
\begin{align}
\label{eq:ub_approx}
\munorm[\big]{\BellOp{\bweight}(\thetastar) - \thetastar} & =
\munorm[\big]{\big( \projperp (\IdOp - \cpdP)^{-1} \projperp
  \big)^{-1} \Vstarperp} \; \leq \; 2 \, \munorm{\Vstarperp}.
\end{align}
\end{subequations}
\end{enumerate}
\end{lemma}
\noindent
See~\Cref{sec:proof:lemma:stdfun,sec:proof:lemma:approxerror} for the
proofs of these claims. \\

We now apply inequalities~\eqref{eq:ub_std}~and~\eqref{eq:ub_approx}
so as to prove the bound~\eqref{eq:def_noisebasenew}.  By known
results~\cite{tsitsiklis1997analysis,bertsekas2011dynamic,yu2010error,MouPanWai22}, we always have the upper bound
$\munorm{\thetastar - \Vstar} \lesssim (1-\discounteff)^{-1/2}
\munorm{\Vstarperp}$, so that inequality~\eqref{eq:ub_std} implies
\begin{align}
\label{eq:ub_stdnew}
\stdmtg(\thetastar) \lesssim \frac{\discount \,
  (1-\discounteff)}{1-\discount} \, \stdfun(\Vstar) +
\sqrt{\frac{\discount}{1-\discount}} \ \munorm{\Vstarperp} \, .
 \end{align}
Recall the definitions of the effective timescale \mbox{$\Hoeff =
  (1-\discounteff)^{-1}$} and the effective horizon \mbox{$\Ho
  = (1 - \discount)^{-1}$}. Substituting
inequalities~\eqref{eq:ub_approx} and \eqref{eq:ub_stdnew} into
equation~\eqref{eq:def_noisebase} then yields the claimed
bound~\eqref{eq:def_noisebasenew}.


\subsection{Proof of Lemma~\ref{lemma:std}(a)}
\label{sec:proof:lemma:stdfun}

Using the triangle inequality, we have the upper bound
\mbox{$\stdmtg(f) \leq \stdmtg(\Vstar) + \Delta \stdmtg(f)$,} where
\begin{subequations}
 \label{eq:stdfun_decomp_first}  
\begin{align}
  \stdmtg(\Vstar)
  & \defn \sum_{\ell=1}^{\Step} \sum_{\step=l}^{\Step}
\weight{\step} \discount^{\ell} \, \sqrt{\Exp_{\State \sim \distr}
  \big[ \Var [ \Vstar (\Statenew) \mid \State \big] } \qquad
  \text{and} \\
\Delta \stdmtg(f) & \defn \sum_{\ell=1}^{\Step}
\sum_{\step=\ell}^{\Step} \weight{\step} \discount^{\ell} \,
\sqrt{\Exp_{\State \sim \distr} \Big[ \Var \big[ \big( \Vstar -
      \BellOp{\step-\ell}(f)\big) (\Statenew) \bigm| \State \big]
    \Big] } \, .
\end{align}
\end{subequations}
We control each of these two terms in turn.

\paragraph{Handling the term $\stdmtg(\Vstar)$:}
Recalling the definition~\eqref{EqnDefnEffDiscount} of the effective
discount factor $\discounteff = \ssum{\step=1}{\Step} \weight{\step}
\discount^{\step}$, we have
\begin{align}
  \label{eq:horizoneff}
  \sum_{t= 1}^{\Step} \sum_{\step=t}^{\Step} \weight{\step}
  \discount^t = \sum_{\step = 1}^{\Step} \sum_{t=1}^{\step}
  \weight{\step} \discount^t = \frac{\discount}{1-\discount}
  \sum_{\step = 1}^{\Step} \weight{\step} (1-\discount^{\step}) =
  \frac{\discount}{1-\discount} \, \bigg\{ 1 - \ssum{\step=1}{\Step}
  \weight{\step} \discount^{\step} \bigg\} = \frac{\discount \,
    (1-\discounteff)}{1-\discount} \, .
\end{align}
Therefore, we have
\begin{align}
 \label{eq:stdfunV0}
 \stdmtg(\Vstar) = \frac{\discount \, (1-\discounteff)}{1-\discount}
 \sqrt{\Exp_{\State \sim \distr} \big[ \Var [ \Vstar (\Statenew) \mid
       \State \big] } = \frac{\discount \,
     (1-\discounteff)}{1-\discount} \, \stdfun(\Vstar) \, .
\end{align}

\paragraph{Handling the term
  $\Delta\stdmtg(f)$:}
We use the property $\Vstar -
\BellOp{\step}(f) = \discount^{\step} \TransOp^{\step} (\Vstar - f)$
for \mbox{$\step = 0, 1,\ldots, \Step$} and find that
\begin{align*}
 \Delta \stdmtg(f) & = \sum_{\ell=0}^{\Step-1}
 \sum_{\step=\ell+1}^{\Step} \weight{\step} \discount^{\step} \,
 \sqrt{\Exp_{\State \sim \distr} \Big[ \Var \big[ \big(
 \TransOp^{\ell} (\Vstar - f)\big) (\Statenew) \bigm| \State
 \big] \Big] } \, .
\end{align*}
By the Cauchy--Schwarz inequality,
\begin{align}
 \label{eq:Dstdfun}
 \Delta \stdmtg(f) & \leq \sqrt{\ssum{\ell=0}{\Step-1} \bigg\{
 \sum_{\step=\ell+1}^{\Step}\weight{\step} \discount^{\step}
 \bigg\}^2} \ \bigg\{ \sum_{\ell=0}^{\Step-1} \Exp_{\State \sim
 \distr} \Big[ \Var \big[ \big( \TransOp^{\ell} (\Vstar - f)\big)
 (\Statenew) \bigm| \State \big] \Big] \bigg\}^{1/2} \, .
\end{align}
Using the law of total variance, we have
\begin{multline}
 \sum_{\ell=0}^{\Step-1} \Exp_{\State \sim \distr} \Big[ \Var \big[
 \big( \TransOp^{\ell} (\Vstar - f)\big) (\Statenew) \bigm|
 \State \big] \Big] \\ = \Var_{\State \sim \distr} \big[ (\Vstar
 - f)(\State) \big] - \Var_{\State \sim \distr} \big[ (
 \TransOp^{\Step}(\Vstar - f) )(\State) \big] \leq \munorm{\Vstar -
 f}^2 \, . \label{eq:tv}
\end{multline}
Additionally, we have
\begin{align}
 \label{eq:coef_tv}
 \sum_{\ell=0}^{\Step-1} \bigg\{
 \sum_{\step=\ell+1}^{\Step}\weight{\step} \discount^{\step}
 \bigg\}^2 \leq \sum_{\ell=0}^{\Step-1}
 \sum_{\step=\ell+1}^{\Step}\weight{\step} \discount^{\step} =
 \sum_{\step=1}^{\Step} \sum_{\ell=0}^{\step-1} \weight{\step}
 \discount^{\step} \leq \sum_{\step=1}^{\Step} \sum_{\ell=1}^{\step}
 \weight{\step} \discount^{\ell} \stackrel{(i)}{=} \frac{\discount \,
 (1-\discounteff)}{1-\discount},
\end{align}
where equality~(i) is due to
equation~\eqref{eq:horizoneff}. Substituting
inequalities~\eqref{eq:tv}~and~\eqref{eq:coef_tv} into
inequality~\eqref{eq:Dstdfun} yields the upper bound
\begin{align}
 \label{eq:Dstdfun0}
 \Delta \stdmtg(f) \leq \sqrt{\frac{\discount \,
 (1-\discounteff)}{1-\discount}} \ \munorm{\Vstar - f} \, .
\end{align}

\vspace*{0.05in}

\noindent Combining inequalities~\eqref{eq:stdfunV0}
and~\eqref{eq:Dstdfun0} with the
decomposition~\eqref{eq:stdfun_decomp_first} completes the proof
of~\Cref{lemma:std}(a).


\subsection{Proof of Lemma~\ref{lemma:std}(b)}
\label{sec:proof:lemma:approxerror}

The Bellman operator $\BellOp{\bweight}$ can be written as
$\BellOp{\bweight}(f) = \cpdP(f) + \cpdr$, where
\begin{align*}
\cpdP \defn \ssum{\step=1}{\Step} \weight{\step} \discount^{\step}
\TransOp^{\step}, \quad \mbox{and} \quad \cpdr \defn
\ssum{\step=0}{\Step-1} \big( \ssum{\ell=\step+1}{\Step} \weight{\ell}
\big) \discount^{\step} \TransOp^{\step} \reward.
\end{align*}
As in Lemma 3 in the paper~\cite{MouPanWai22}, we use the
projections $\projH$ and $\projperp$ to define the operators
\begin{align*}
 \cpdPHH \defn \projH \cpdP \projH, \quad \cpdPpH \defn \projperp
 \cpdP \projH, \quad \cpdPHp \defn \projH \cpdP \projperp, \quad
 \cpdPpp \defn \projperp \cpdP \projperp \, .
\end{align*}
Using these notations, we recast the Bellman equation $\Vstar =
\BellOp{\bweight}(\Vstar) = \cpdP(\Vstar) + \cpdr$ as
\begin{subequations}
\begin{align}
\label{eq:Vstarpar}   
 \Vstarpar & = \cpdPHH \, \Vstarpar + \cpdPHp \, \Vstarperp + \projH
 \, \cpdr, \quad \mbox{and} \\
 \label{eq:Vstarperp}
 \Vstarperp & = \cpdPpH \,
 \Vstarpar + \cpdPpp \, \Vstarperp + \projperp \, \cpdr.
\end{align}
Consequently, the projected equation $\thetastar = \projH \big(
\BellOp{\bweight}(\thetastar) \big)$ has the equivalent representation
\begin{align}
 \label{eq:thetastarpar}  
 \thetastar & = \cpdPHH \, \thetastar + \projH \, \cpdr.
\end{align}
\end{subequations}

We now consider the Bellman residual $\thetastar -
\BellOp{\bweight}(\thetastar)$. Since $\projH + \projperp = \IdOp$, we
have
\begin{align}
 \label{eq:approxerr-2}
 \thetastar - \BellOp{\bweight}(\thetastar) & = \projperp \big( \cpdP
 (\thetastar) + \cpdr \big) = \cpdPpH \, \thetastar + \projperp \,
 \cpdr.
\end{align}
Equation~\eqref{eq:Vstarperp} implies that $\projperp \, \cpdr =
(\IdOpp - \cpdPpp) \, \Vstarperp - \cpdPpH \, \Vstarpar$, and
substituting this representation into equation~\eqref{eq:approxerr-2}
yields
\begin{align}
 \label{eq:approxerr0}
 \thetastar - \BellOp{\bweight}(\thetastar) & = (\IdOpp - \cpdPpp) \,
 \Vstarperp - \cpdPpH \, (\Vstarpar - \thetastar) \, .
\end{align}
By subtracting
equations~\eqref{eq:Vstarpar}~and~\eqref{eq:thetastarpar}, we obtain
$\Vstarpar - \thetastar = (\IdOp_{\RKHS} - \cpdPHH)^{-1} \, \cpdPHp \,
\Vstarperp$, and substituting this equality into
equation~\eqref{eq:approxerr0} yields
\begin{align*}
 \thetastar - \BellOp{\bweight}(\thetastar) & = \schur(\Vstarperp)
 \qquad \text{with } \schur \defn (\IdOpp - \cpdPpp) - \cpdPpH \,
 (\IdOpH - \cpdPHH)^{-1} \cpdPHp \, .
\end{align*}

Note that the operator $\schur$ is the Schur complement of the block
$\projH(\IdOp - \cpdP) \, \projH$ of operator $(\IdOp - \cpdP)$, and
hence \mbox{$\schur^{-1} = \projperp (\IdOp - \cpdP)^{-1} \projperp$.}
In other words, we have
\begin{align*}
\munorm[\big]{ \BellOp{\bweight}(\thetastar) - \thetastar} =
\munorm[\big]{\big( \projperp (\IdOp - \cpdP)^{-1} \projperp
  \big)^{-1} \, \Vstarperp} \, ,
\end{align*}
as claimed in the lemma. On the other hand, we have the upper bound
\mbox{$\muopnorm{\schur} \leq \muopnorm{\IdOp - \cpdP} \leq 2$}. It follows that \mbox{$\munorm[\big]{\BellOp{\bweight}(\thetastar) -
    \thetastar} \leq \muopnorm{\schur} \munorm{\Vstarperp} \leq 2 \,
  \munorm{\Vstarperp}$,} which establishes \Cref{lemma:std}(b).


\section{Proofs of the upper bound corollaries in Sections~\ref{sec:ub_exp}~and~\ref{sec:ub_simple}}

In this part, we collect the proofs of the consequences discussed
in~\Cref{sec:ub_exp,sec:ub_simple}.  In particular,
\Cref{sec:proof_ub_linear_b} contains the analysis of finite-rank
kernels shown in \Cref{sec:ub_exp_linear}; \Cref{sec:proof_ub_alpha_b}
is devoted to kernels with $\alpha$-polynomial decay given in
\Cref{sec:ub_exp_alpha}.  In
\Cref{sec:proof_noisebasenew_r,sec:proof_noisebasenew_V}, we consider
MRP instances with uniformly bounded rewards and bounded $\Lmu$-norm
on value functions respectively, as discussed in \Cref{sec:ub_simple}.
In~\Cref{append:proof:eq:Vstarperp<stdfun,append:proof:eq:Vstarperp<reward},
we prove the two inequalities from \Cref{sec:ub_noerr} and
\Cref{sec:proof_noisebasenew_r}, controlling the projection error
$\Vstarperp$.


\subsection{Proof of inequality~\eqref{eq:ub_linear_b} for finite-rank kernels}
\label{sec:proof_ub_linear_b}
For any parameter $\delta > 0$, the left-hand side of
inequality~\ref{eq:critineq} can be upper bounded as
\[ \mathcal{C}(\delta) = \! \sqrt{\sum_{j=1}^{\Dim} \min\big\{
  \frac{\eig{j}}{\delta^2}, 1 \big\}} \leq \!
\sqrt{\Dim} \, . \] Therefore, the choice $\delta \equiv \delta(\noisebase)
\defn \frac{\unibou \, \noisebase}{\newrad} \sqrt{\frac{\Dim}{\numobs}}$
satisfies the critical inequality~\ref{eq:critineq}. It follows that the bound~\eqref{eq:thm_ub} takes the form
\begin{align*}
  \munorm{\thetahat - \thetastar}^2 \leq \constnew \, \log^2 \numobs
  \, \newrad^2 \delta^2(\noisebase) & \leq \constnew \; \unibou^2 \;
  \noisebase^2 \; \frac{\Dim \, \log^2 \numobs}{\numobs} \\
  & \leq 2 \, \constnew \; \unibou^2 \; \Hoeff^2 \, \big\{
  \stdmtg^2(\thetastar) + \approxerr^2(\thetastar) \big\} \,
  \frac{\Dim\, \log^2 \numobs}{\numobs} \, ,
\end{align*}
which is equivalent to the claimed bound.  By solving the inequality
$\newrad^2 \, \delta^2(\noisebase) \lesssim \frac{(1-\discounteff) \,
  \noisebase^2}{\sqrt{(\mixtime + \Step) \, \numobs}}$, we find that
as long as $\sqrt{\numobs/(\mixtime + \Step)} \gtrsim \unibou^2 \,
\Dim \, \Hoeff$, the condition~\eqref{EqnSampleLowerBound} is
satisfied.

\subsection{Proof of inequality~\eqref{eq:ub_alpha_b} for kernels with $\alpha$-polynomial decay}
\label{sec:proof_ub_alpha_b}

We follow arguments similar to those in Corollary~2 from the
paper~\cite{duan2021optimal}; in particular, it can be shown that
$\delcrit^2(\noisebase) \asymp \Big( \frac{\unibou^2 \,
  \noisebase^2}{\newrad^2 \, \numobs}
\Big)^{\frac{2\alpha}{2\alpha+1}}$ satisfies the critical
inequality~\ref{eq:critineq}. Moreover, inequality $\newrad^2 \,
\delcrit^2(\noisebase) \leq \plaincon \, \frac{(1-\discounteff) \,
  \noisebase^2}{\sqrt{(\mixtime + \Step) \, \numobs}}$ holds when
sample size $\numobs$ exceeds a finite threshold.


  \subsection{Proof of inequality~\eqref{eq:noisebasenew_r} for instances with uniform bound on $\reward$}
  \label{sec:proof_noisebasenew_r}
  We claim that the regularity condition $\supnorm{\reward} \leq \rewardnorm$ implies the model mis-specification $\munorm{\Vstarperp}$ is upper bounded by
  \begin{align}
  \label{eq:Vstarperp<reward}
  \munorm{\Vstarperp} \, \leq \, \min\Big\{(1 -
  \discount)^{-1}, \, \mixtime \, \big\{ 2 + \tfrac{1}{2} \log \tfrac{\supnorm{\reward}^2}{\Var_{\distr}[\reward]} \big\} \Big\} \, \sqrt{\Var_{\distr}[\reward]} \; \lesssim \; \min\big\{ (1 - \discount)^{-1}, \, \mixtime \big\} \, \rewardnorm \;
  .
  \end{align}
  The proof of the claim is deferred to
  \Cref{append:proof:eq:Vstarperp<reward}. We comment that the upper
  bound \eqref{eq:Vstarperp<reward} is achievable, for example, when
  RKHS $\RKHS$ only contains constant functions. According to the law
  of total variance, we have $\stdfun^2(\Vstar) \leq \Ho \,
  \rewardnorm^2$ for $\Ho = (1 - \discount)^{-1}$. In
  definition~\eqref{eq:def_noisebasenew} of $\noisebasenew$, we
  replace $\stdfun(\Vstar)$ with its upper bound $\sqrt{\Ho} \,
  \rewardnorm$ and $\munorm{\Vstarperp}$ with its bound in
  inequality~\eqref{eq:Vstarperp<reward}. It follows that
  \begin{align} \label{eq:noisebasenew_reward} \noisebasenew \leq \constnew \, \big\{ 1 + \Hoeff \sqrt{\mixtime/\Ho} \sqrt{\min\{ 1, \, \mixtime/\Ho \}} \big\} \, \Ho^{3/2} \, \rewardnorm \, . \end{align}
  Under condition~\eqref{eq:Hoeff_r} on effective timescale $\Hoeff$,
  the inequality~\eqref{eq:noisebasenew_reward} reduces to the
  bound~\eqref{eq:noisebasenew_r}.


\subsection{Proof of inequality~\eqref{eq:noisebasenew_V} for instances with $\Lmu$-norm bound on $\Vstar$}
  \label{sec:proof_noisebasenew_V}
  The regularity condition $\munorm{\Vstar} \leq \Vstarnorm$ implies
  that $\munorm{\Vstarperp} \leq \munorm{\Vstar} \leq \Vstarnorm$ and
  \mbox{$\stdfun(\Vstar) \leq \munorm{\Vstar} \leq \Vstarnorm$} since
  the variance is dominated by the second moment.  It then follows
  from definition~\eqref{eq:def_noisebasenew} of $\noisebasenew$ that
  \mbox{$\noisebasenew \lesssim \big\{ \Ho + \Hoeff \sqrt{\max\{ \Ho ,
      \, \mixtime \}} \big\} \, \Vstarnorm$.}  Under
  condition~\eqref{eq:Hoeff_V}, this bound reduces to
  inequality~\eqref{eq:noisebasenew_V}.


\subsection{Proof of inequality~\eqref{eq:Vstarperp<stdfun}, bounding $\norm{\Vstarperp}_{\distr}$ using $\stdfun(\Vstar)$}
\label{append:proof:eq:Vstarperp<stdfun}

Recall that the value function is defined by $\Vstar = (\IdOp -
\discount \, \TransOp)^{-1} \, \reward$. Since $\distr(\Vstar) = (1 -
\discount)^{-1} \distr(\reward)$, the projection of $\Vstar$ onto the
one-dimensional subspace $\Span{\one}$ is given by $\distr(\Vstar)
\cdot \one = (1 - \discount)^{-1} \distr(\reward) \cdot \one$. Define
\begin{align}
  \label{eq:decomp_Vstar}
  \Vstarbar \defn \Vstar - (1 - \discount)^{-1} \distr(\reward) \cdot \one \, ,
\end{align}
which is the component of value function $\Vstar$ orthogonal to the
subspace $\Span{\one}$. The space $\RKHS$ has the constant function
$\one(\cdot)$ as a member, so that $\RKHS^{\perp} \subset
\Span{\one}^{\perp}$. We can view the function $\Vstarperp = \Vstar -
\Vstarpar$ as the projection of $\Vstarbar$ onto the space
$\RKHS^{\perp}$, which further implies
\begin{align}
\label{eq:Vstarperp<Vstarbar}
\munorm{\Vstarperp} \; \leq \; \munorm{\Vstarbar} \, .
\end{align}

Our next step is to calculate the variance $\stdfun^2(\Vstar)$ of
value function $\Vstar$. By definition, we have
\begin{subequations}
\begin{align}
\stdfun^2(\Vstar) = \Exp\big[ \Var[\Vstar(\Statenew) \mid \State]
  \big] & = \Exp\big[ \Exp[(\Vstar(\Statenew))^2 \mid \State] -
  \Exp[\Vstar(\Statenew) \mid \State]^2 \big] \notag \\
\label{eq:stdfun_decomp}
& = \Exp\big[(\Vstar(\Statenew))^2\big] - \Exp\big[(\TransOp
  \Vstar(\State))^2\big] = \munorm{\Vstar}^2 -
\munorm{\TransOp\Vstar}^2 \, .
\end{align}
Using the expression for $\Vstarbar$ in
equation~\eqref{eq:decomp_Vstar} and the
property~$\inprod[\big]{\one}{\Vstarbar}_{\distr} = 0$, we find that
\begin{align}
\label{eq:munormVstar<}
\munorm{\Vstar}^2 = (1-\discount)^{-2} \, (\distr(\reward))^2 +
\munorm{\Vstarbar}^2 \, .
\end{align}
As for the term $\munorm{\TransOp\Vstar}^2$, we note that
\mbox{$\TransOp \, \Vstar = \TransOp \big\{ \Vstarbar +
  (1-\discount)^{-1} \, \distr(\reward) \cdot \one \big\}  =
  (1-\discount)^{-1} \, \distr(\reward) \cdot \one + \TransOp
  \Vstarbar$.}  Since $\distr \, \TransOp \, \Vstarbar =
\distr(\Vstarbar) = 0$, we have $\TransOp \Vstarbar \in
\Span{\one}^{\perp}$, which implies
\begin{align}
\label{eq:munormPVstar<}
\munorm{\TransOp\Vstar}^2 = (1-\discount)^{-2} \, (\distr(\reward))^2
+ \munorm{\TransOp \Vstarbar}^2 \, .
\end{align}
\end{subequations}
Combining
equations~\eqref{eq:stdfun_decomp}~to~\eqref{eq:munormPVstar<} then yields \mbox{$\stdfun^2(\Vstar) =
\munorm{\Vstarbar}^2 - \munorm{\TransOp \Vstarbar}^2$}.

When the geometric ergodicity condition~\eqref{assump:Lmu_ergo} holds,
we have $\munorm{\TransOp \Vstarbar} \leq ( 1 - \mixtime^{-1} ) \,
\munorm{\Vstarbar}$, from which it follows that
\begin{align*}
\stdfun^2(\Vstar) \geq \big\{ 1-( 1 - \mixtime^{-1} )^2 \big\} \,
\munorm{\Vstarbar}^2 \geq \mixtime^{-1} \,
\munorm{\Vstarbar}^2 \, .
\end{align*}
Combined with
inequality~\eqref{eq:Vstarperp<Vstarbar}, this lower bound implies
that $\munorm{\Vstarperp} \leq \munorm{\Vstarbar} \leq \sqrt{\mixtime}
\, \stdfun(\Vstar)$, which completes the proof of
inequality~\eqref{eq:Vstarperp<stdfun}.


\subsection{Proof of inequality~\eqref{eq:Vstarperp<reward}, bounding $\norm{\Vstarperp}_{\distr}$ using $\Var_{\distr}[\reward]$}
\label{append:proof:eq:Vstarperp<reward}

Since \mbox{$\Vstar = (\IdOp - \discount \, \TransOp)^{-1} \,
  \reward$} and $(1-\discount)^{-1} \distr(\reward) \cdot \one =
(\IdOp - \discount \, \TransOp)^{-1} \, (\distr(\reward) \cdot \one)$,
the function~$\Vstarbar$ from equation~\eqref{eq:decomp_Vstar} can be
written as \mbox{$\Vstarbar = (\IdOp - \discount \, \TransOp)^{-1}
  \big(\reward - \distr(\reward) \cdot \one \big)$.}  Since the
distribution $\distr$ is stationary under the transition kernel
$\TransOp$, we have
\begin{align}
\label{eq:Vstarbar<1}
\munorm{\Vstarbar} \leq (1 - \discount)^{-1} \munorm[\big]{\reward -
  \distr(\reward) \cdot \one} = (1 - \discount)^{-1}
\sqrt{\Var_{\distr}[\reward]} \, .
\end{align}

On the other hand, we leverage the uniform ergodicity of Markov chain
$\TransOp$ to derive a bound involving mixing time $\mixtime$. 
We first write the function $\Vstarbar$ as the series $\Vstarbar = \sum_{t=0}^{\infty} \discount^t \TransOp^t \funnew$ with \mbox{$\funnew \defn \reward - \distr(\reward) \cdot \one$}. Similar to the bounds \eqref{eq:termaprx1}~and~\eqref{eq:termaprx2}, the mixing condition~\ref{assump:mixing} ensures \mbox{$\munorm{\TransOp^t \funnew} \leq \min \big\{ \munorm{\funnew}, 2 \, (1 - \mixtime^{-1})^t \supnorm{\funnew} \big\}$} for $t = 0,1,\ldots$.
We apply the bounds to the series
$\sum_{t=0}^{\infty} \discount^t \, \TransOp^t \, \funnew$ with a truncation level \mbox{$\hittime \defn (1 -
	\discount + \discount\mixtime^{-1} )^{-1} \log \tfrac{\supnorm{\funnew}}{\munorm{\funnew}}$}. It follows that
\begin{multline*}
  \munorm{\Vstarbar} \leq \sum_{t
  = 0}^{\infty} \discount^t \, \munorm{\TransOp^t \, \funnew}
  \leq \sum_{t = 0}^{\hittime-1} \discount^t \, \munorm{\funnew} + \sum_{t = \hittime}^{\infty} 2 \, \discount^t ( 1 - \mixtime^{-1} )^t \, \supnorm{\funnew} 
  \leq \hittime \, \munorm{\funnew} + \frac{2 \, \discount^{\hittime} ( 1 - \mixtime^{-1} )^{\hittime}}{1 - \discount + \discount \mixtime^{-1}} \, \supnorm{\funnew} \, .
\end{multline*}
Since $\discount^{\hittime} ( 1 - \mixtime^{-1} )^{\hittime} \leq \frac{\munorm{\funnew}}{\supnorm{\funnew}}$, we have
\begin{align*}
  \munorm{\Vstarbar} & \leq \hittime \, \munorm{\funnew} + \frac{2 \, \munorm{\funnew}}{1 - \discount + \discount \mixtime^{-1}} = \frac{\munorm{\funnew}}{1 - \discount + \discount \mixtime^{-1}} \big\{ 2 + \log \tfrac{\supnorm{\funnew}}{\munorm{\funnew}} \big\} \leq \mixtime \, \munorm{\funnew} \, \big\{ 2 + \log \tfrac{\supnorm{\funnew}}{\munorm{\funnew}} \big\} \, ,
\end{align*}
which implies that $\munorm{\Vstarbar} \leq \mixtime \, \big\{ 2 +
\tfrac{1}{2} \log \tfrac{\supnorm{\reward}^2}{\Var_{\distr}[\reward]}
\big\} \, \sqrt{\Var_{\distr}[\reward]}$.  Combining this inequality
and the bound~\eqref{eq:Vstarbar<1} with the
relation~$\munorm{\Vstarperp} \; \leq \; \munorm{\Vstarbar}$ in
inequality~\eqref{eq:Vstarperp<Vstarbar} completes the proof of the
claim~\eqref{eq:Vstarperp<reward}.


\section{Proof of auxiliary results in Section~\ref{sec:proof_ub}}
\label{append:proof_ub}

This section is devoted to the proofs of the non-asymptotic upper
bounds on the terms $\Term_1$ and $\Term_3$, as stated
in~\Cref{lemma:Term1,lemma:Term3}. Both of these proofs make use of a
modified Athreya-Ney random renewal time approach (cf. \S 5.1.3 in the
book~\cite{meyn2012markov} for details) to regenerating Markov
chains. We describe this regeneration method
in~\Cref{sec:split_chain}, and then turn bounding term~$\Term_1$ and
proving~\Cref{lemma:Term1} in~\Cref{sec:proof_Term1}.
\Cref{sec:proof_Term3} is devoted to the analysis of $\Term_3$ as
formalized in \Cref{lemma:Term3}.


\subsection{Regenerating a Markov chain}
\label{sec:split_chain}

We consider regenerating a Markov chain with a transition kernel
$\TransOp$ in the spirit of Athreya-Ney random renewal time approach
(cf. \S~5.1.3 in the book~\cite{meyn2012markov}).  The renewal times
divide the original chain into several independent blocks.  This
splitting procedure allows us to reduce the analysis of a trajectory
from the original Markov chain to the analysis of independent blocks
of random lengths. In particular, our construction ensures that each
block has a random length with mean $\mixtime + \Step$. The
independence of the blocks allows us to apply Talagrand's inequality
for sub-exponential random variables (e.g., Theorem~4 in the
paper~\cite{adamczak2008tail}) so as to control $\Term_1$ and
$\Term_3$. We establish bounds of the Bernstein type that involve
limiting variances consistent with the central limit theorem for
Markovian processes.

A technical comment: our construction of the random renewal times is
related to but slightly different from the one given in \S 5.1.3 of
Meyn and Tweedie~\cite{meyn2012markov}.  Given the minorization
condition~\eqref{cond:minor}, we construct the renewal times by
tossing a coin with probability of a head given by $\mixtime^{-1}$. In the standard approach, the coin is flipped at every
time step and the observation of states between two successive heads
form a block.  In our construction, we only toss the coin after the
first $\Step$ steps in each cycle. It ensures the lengths of the
blocks are at least~$\Step$, which is important to our analyses of
$\Term_1$ and $\Term_3$.  Let $\Idxset{s}$ be the collection of time
indices in the $s^{th}$ block.  Our construction ensures that for any
two different indices $s$ and $s'$, the two groups of random variables
$\{ \state_t \mid t \in \Idxset{s} \}$ and $\{ \state_t \mid t \in
\Idxset{s'} \}$ are independent.


\paragraph{Regeneration scheme:}  With this intuition in place, let us
give a more formal description of our scheme.  Take a series of
stopping times $\{ \hittime_s \}_{s \in \Int_+}$. Let $\{\hittime_1 -
\Step\} \cup \{\hittime_{s+1} - \hittime_s - \Step\}_{s \in \Int_+}$
be i.i.d. random variables such that
\begin{align*}
\Prob(\hittime_1 = \Step + t) = \Prob(\hittime_{s+1} - \hittime_s =
\Step + t) = ( 1 - \mixtime^{-1} )^{t-1} \mixtime^{-1} \qquad \text{for any $t
  \in \Int_+$}.
\end{align*}
Generate an initial state $\state_1$ from stationary distribution
$\distr$. At time step $t = 1,2,\ldots,\numobs-1$, we take a next
state $\state_{t+1}$ according to the following scheme:
\begin{itemize}
\item If $t = \hittime_s$ for some $s \in \Int_+$, then generate
  $\state_{t+1}$ independently from distribution $\distrnew$;
\item If $t = \hittime_s + \step$ for some $s \in \Int_+$ and $\step
  \in [\Step]$, then distribute $\state_{t+1}$ according to
  $\TransOp(\cdot \mid \State_t)$;
\item Otherwise, let $\state_{t+1} \sim ( 1 - \mixtime^{-1} )^{-1} \, \big\{
  \TransOp(\cdot \mid \State_t) - \mixtime^{-1} \, \distrnew(\cdot)
  \big\}$. The minorization condition~\eqref{cond:minor} ensures that
  this measure is well-defined.
\end{itemize}
\noindent It can be seen that after marginalizing out the stopping
times $\{\hittime_s\}_{s \in \Int_+}$, the process $(\state_1,
\state_2, \ldots, \state_{\numobs})$ constructed above follows the
transition kernel $\TransOp$.

We let $\hitS$ be the largest index of hitting time $\hittime_s$
before $\numobsnew = \numobs - \Step$, i.e. \mbox{$\hitS \defn \max\{
  s \mid \hittime_s < \numobsnew \}$}. For notational convenience, we
denote $\hittime_0 \defn 0$. Using the stopping times $\{ \hittime_s
\}_{s \in \Natural}$, we partition indices $[\numobsnew]$ into
$(\hitS+1)$ blocks $\Idxset{s}$ for $s=0,1,\ldots,\hitS-1$ and $\Tail$
and write
\begin{align}
 \label{eq:def_block}
 \Idxset{s} \defn \{ \hittime_s+1, \hittime_s+2, \ldots,
 \hittime_{s+1} \} \text{~~for $s = 0,1,\ldots,\hitS-1$} \quad
 \text{and} \quad \Idxset{\Tail} \defn \{\hittime_{\hitS}+1, \ldots,
 \numobsnew\} \, .
\end{align}
In expectation, the length of each block equals
$\Exp[\hittime_2 - \hittime_1] = \mixtime + \Step$. By our construction, the lengths of the blocks are lower bounded by $\Step$.

Suppose we are analyzing $\bweight$-weighted TD estimate with \mbox{$\dim \bweight = \Step$}. Then we need to consider the sums of some functionals of $\Step+1$ successive states \mbox{$\state_{t}^{t+\Step} \defn (\state_t, \state_{t+1}, \ldots, \state_{t+\Step})$} for \mbox{$t = 1,2,\ldots,\numobs-\Step$}.
The minimum length condition on the blocks and the independence among different blocks imply that the sums $\big\{ \sum_{t \in \Idxset{s}} f\big(\state_{t}^{t+\Step}\big) \big\}_{s \in \Int_+}$ form a one-dependent process for any fixed functional $f: \StateSp^{\Step+1} \rightarrow \Real$.
In the following analyses, we classify the random variables $\big\{ \sum_{t \in \Idxset{s}} f\big(\state_{t}^{t+\Step}\big) \big\}_{s \in \Int_+}$ by odd and even indices $s$ so that the elements in each class are i.i.d..


\subsection{Proof of Lemma~\ref{lemma:Term1}}
\label{sec:proof_Term1}

Using the blocks $\{ \Idxset{s} \}_{s=0}^{\hitS - 1} \cup \{
\Idxset{\Tail} \}$ given in equation~\eqref{eq:def_block}, we define
random variables
\begin{align*}
\iidY{s}(f) \defn \!\! \sum_{t \in \Idxset{s}} \!\! f(\state_t) \,
\termone_t \qquad \text{for $s \in \{ 0,1,\ldots,\hitS-1 \} \cup
  \{\Tail \}$},
\end{align*}
where $\termone_t$ was previously defined~\eqref{eq:def_termone}. We
introduce the shorthands \mbox{$\hitS_1 \defn
  \lfloor(\hitS-2)/2\rfloor$} and {$\hitS_2 \defn
  \lfloor(\hitS-3)/2\rfloor$,} and recast term $\Term_1$ as
\begin{align*}
 \Term_1 = \frac{1}{\numobsnew} \sum_{t = 1}^{\numobsnew}
 \Deltahat(\state_t) \, \termone_t = \frac{1}{\numobsnew} \,
 \iidY{0}(\Deltahat) + \frac{1}{\numobsnew} \sum_{s = 0}^{\hitS_1}
 \iidY{2s+1}(\Deltahat) + \frac{1}{\numobsnew} \sum_{s = 0}^{\hitS_2}
 \iidY{2s+2}(\Deltahat) + \frac{1}{\numobsnew} \,
 \iidY{\Tail}(\Deltahat) \, .
\end{align*}
For any fixed function $f \in \RKHS$, $\{ \iidY{2s+1}(f)
\}_{s=0}^{\hitS_1}$ and $\{ \iidY{2s+2}(f) \}_{s=0}^{\hitS_2}$ are two
groups of i.i.d. random variables and $\iidY{1}(f)$ and $\iidY{2}(f)$
are identically distributed. The i.i.d. property allows us to apply the Talagrand's inequality
(Theorem~4 in the paper~\cite{adamczak2008tail}) to analyze the concentration
property of the random processes. As preparation, we develop bounds on the expectation
and variance of $\iidY{1}(f)$. See \Cref{lemma:Term1_var} below.
\begin{lemma}
\label{lemma:Term1_var}
For any function $f \in \RKHS$, we have $\Exp[ \, \iidY{1}(f) \, ] =
0$. If we further suppose $\Constdistrnew \, \Step \leq \mixtime$, then
\begin{multline}
\label{eq:Term1_var}
\Exp \big[ \iidY{1}^2(f) \big] \, \big/ \, \Exp[\hittime_2 -
  \hittime_1] \\ \leq \min\Big\{ 18 \, \supnorm{f}^2 \, \big\{
\stdmtg^2(\thetastar) + \approxerr^2(\thetastar) \big\}, \; 32 \,
\munorm{f}^2 \cdot \tfrac{\bou^2 \newrad^2}{(1-\discount)^2} \, \Big( 9 \Step + \mixtime \big\{ 2 + \log \tfrac{\supnorm{f}}{\munorm{f}} \big\} \Big) \Big\} \, .
  \end{multline}
\end{lemma}
\noindent The proof is given in~\Cref{sec:proof:lemma:Term1_var}. \\

Recall that each block has expected length $\Exp[\hittime_{s+1} -
  \hittime_s] = \mixtime + \Step$. Therefore, with high probability,
we have $\max\big\{ \hitS_1, \hitS_2 \big\} \leq \hitSbar \defn
\lfloor \numobsnew / (\mixtime + \Step) \rfloor$. For scalars $u > 0$,
we define function classes $\funclass(u) \defn \{ f \in \RKHS \mid
\munorm{f} \leq u, \hilnorm{f} \leq \newrad \}$ and the family of
random variables
\begin{align}
 \supZbar^{(\variota)}(u) \defn \sup_{f \in
 \funclass(u)} \Big| \frac{1}{\numobsnew}
 \sum_{s=0}^{\hitSbar} \iidY{2s+\variota}(f) \, \Big|
 \quad \text{for $\variota \in \{1,2\}$} \, .
\end{align}
Note that for any fixed function $f \in \RKHS$, $\{ \iidY{2s+1}(f)
\}_{s=0}^{\hitSbar}$ and $\{ \iidY{2s+2}(f) \}_{s=0}^{\hitSbar}$ are
two groups of i.i.d. random variables following the same distribution,
therefore, $\supZbar^{(1)}(u)$ and $\supZbar^{(2)}(u)$ are identically
distributed, and have the same expectation. Let $\delcritnew > 0$ be
the smallest positive solution to the inequality
\begin{align}
\label{eq:def_delcritnew}
\Exp \big[ \supZbar^{(1)}(u) \big] = \Exp\big[ \supZbar^{(2)}(u) \big]
\leq \tfrac{1}{80} \, (1-\discounteff) \, u^2 \, .
\end{align}
We now derive a result that relates $\delcritnew$ to the critical
radius $\delcrit$ involved in \Cref{thm:ub}.
\begin{lemma}
\label{lemma:Term1_exp}
There is a universal constant $\const{0} \geq 1$ such that
$\delcritnew \leq \delcritnewtil \defn \const{0} \, \newrad \,
\delcrit$, where $\delcrit \equiv \delcrit(\noise)$ is the smallest
positive solution to the critical inequality~\eqref{eq:critineq} for
some $\noise \geq \noisebase$.
\end{lemma}
\noindent See \Cref{sec:proof:lemma:Term1_exp} for the proof.\\

We now consider the random variable \mbox{$\supZ(u) \defn \sup_{f \in
    \funclass(u)} \Big| \frac{1}{\numobsnew} \sum_{t=1}^{\numobsnew}
  \, f(\state_t) \, \termone_t \Big|$.}  We use Talagrand's inequality
to establish a high probability bound:
\begin{lemma}
\label{lemma:Term1_tail}
There are universal constants $\const{1}, \const{2} > 0$ such that
\begin{align}
 \Prob\big[ \supZ(\delcritnewtil) \geq  (1-\discounteff)
 \, \delcritnewtil^2 \, \log \numobs \big] \leq \const{1} \, \exp\big( -
 \Constprob \, \tfrac{\numobs \,
 \delcritnewtil^2}{\bou^2 \newrad^2} \big) = \const{1} \, \exp\big(
 - \const{0}^2 \, \Constprob \, \tfrac{\numobs \,
 \delcrit^2}{\bou^2} \big) \, ,
\end{align}
where $\Constprob = \const{2} \, \frac{(1-\discounteff)^2 (1-\discount)^2}{\mixtime + \Step}$.
\end{lemma}
\noindent See~\Cref{append:proof_Term1_tail} for the proof. \\

Similar to Lemma~8 in the paper~\cite{duan2021optimal}, we can show that
\mbox{$\supZ(\delcritnewtil) \leq  (1-\discounteff) \,
 \delcritnewtil^2 \, \log \numobs$} implies
\begin{align*}
|\Term_1| \leq (1-\discounteff) \, \Big\{ \delcritnewtil^2
\, \max\big\{ 1, \, \tfrac{\hilnorm{\Deltahat}}{\newrad} \big\} +
\delcritnewtil \, \munorm{\Deltahat} \Big\} \log \numobs \, ,
\end{align*}
which completes the proof of bound~\eqref{eq:Term1} in
\Cref{lemma:Term1}. \\

It remains to prove the three auxiliary \Cref{lemma:Term1_var,lemma:Term1_exp,lemma:Term1_tail}, and we
prove them
in~\Cref{sec:proof:lemma:Term1_var,sec:proof:lemma:Term1_exp,append:proof_Term1_tail},
respectively.


\subsubsection{Proof of Lemma~\ref{lemma:Term1_var}}
\label{sec:proof:lemma:Term1_var}

Consider the expectation $\Exp[ \, \iidY{1}(f) \,]$. By the law of
large numbers, we have
\begin{align}
\Exp[\, \iidY{1}(f) \,] & = \lim_{\hitS \rightarrow \infty}
\frac{1}{\hitS-1} \sum_{s=1}^{\hitS-1} \iidY{s}(f) = \Exp[\hittime_2 -
 \hittime_1] \cdot \lim_{\numobsnew \rightarrow \infty}
\frac{1}{\numobsnew} \, \bigg\{ \sum_{s = 0}^{\hitS-1}
\iidY{s}(f) + \iidY{\Tail}(f) \bigg\} \notag \\
& = \Exp[\hittime_2 - \hittime_1] \cdot \lim_{\numobsnew \rightarrow
 \infty} \frac{1}{\numobsnew} \sum_{t=1}^{\numobsnew} f(\state_t) \,
\termone_t = \Exp[\hittime_2 - \hittime_1] \cdot
\Exp_{\State_0 \sim \distr} \big[ f(\State_0) \,
 \termone(\State_{0}^{\Step}) \big] = 0   \label{eq:EiidY}
\end{align}
with probability one.  We can conclude that $\Exp[ \, \iidY{1}(f) \,]
= 0$. \\

Next we bound the second moment of $\iidY{1}(f)$. We claim that it can
be upper bounded as
\begin{subequations}
\begin{align}
 \label{eq:iidY<}
 \sqrt{\Exp\big[ (\iidY{1}(f))^2 \big] \, \big/ \, \Exp[\hittime_2 -
 \hittime_1]} \leq \sqrt{2 \, \big( \varm(f) + \vara(f) \big) } +
 2 \sqrt{\Exp\big[ (\DiidY{2}(f))^2 \big] \, \big/ \, \Exp[\hittime_2
 - \hittime_1]} \, ,
\end{align}
where $\varm(f)$ and $\vara(f)$ were defined in
equations~\eqref{eq:def_varm} and~\eqref{eq:def_vara}, respectively,
whereas the random variable $\DiidY{2}(f)$ is given by
\begin{align}
  \DiidY{2}(f) & \defn \iidY{1}(f) - \Exp[ \, \iidY{1}(f) \mid
    \state_{\hittime_2} ] = \sum_{t = \hittime_2 - \Step}^{\hittime_2}
  f(\state_t) \, \big\{ \termone_t - \Exp[\termone_t \mid
    \state_{\hittime_2}] \big\} \, . \label{eq:def_DiidY}
\end{align}
\end{subequations}

Taking the claim~\eqref{eq:iidY<} as given for the moment, let us
bound each of the terms $\varm(f)$, $\vara(f)$ and $\Exp\big[
  \DiidY{2}^2(f) \big]$ in turn.  Recall from~\Cref{lemma:var} that
\begin{align}
\label{eq:var1}
 \varm(f) \leq \supnorm{f}^2 \, \stdmtg^2(\thetastar) \qquad
 \text{and} \qquad \vara(f) \leq \supnorm{f}^2 \,
 \approxerr^2(\thetastar) \, .
\end{align}
Moreover, we can bound the term $\varm(f)$ using the $\Lmu$-norm
$\munorm{f}$ by
\begin{align}
\varm(f) & = \ssum{t=-\Step}{\Step-1} \Exp\big[ f(\State_0) \,
  \termmtg(\State_{1}^{\Step}) \; f(\State_t) \,
  \termmtg\big(\State_{t+1}^{t+\Step}\big) \big] \notag \\
\label{eq:varm2}  
& \stackrel{(i)}{\leq} \ssum{t=-\Step}{\Step-1} \Exp\big[
  f^2(\State_0) \, \termmtg^2(\State_{1}^{\Step}) \big]^{\frac{1}{2}}
\; \Exp\big[ f^2(\State_t) \,
  \termmtg^2\big(\State_{t+1}^{t+\Step}\big) \big]^{\frac{1}{2}} \;
\leq \munorm{f}^2 \cdot 2 \, \Step \, \supnorm{\termmtg}^2 \, ,
\end{align}
where step $(i)$ follows from the Cauchy--Schwarz inequality.  As for
term $\vara(f)$, similar to inequality~\eqref{eq:termaprx0}, we can
show that
\begin{align*}
\vara(f) & \leq 2 \, \hittime \; \munorm{f \, \termaprx}^2 + 4 \,
( 1 - \mixtime^{-1} )^{\hittime} \mixtime \; \supnorm{f \,
  \termaprx} \, \munorm{f \, \termaprx} \leq 2 \, \munorm{f}^2 \,
\supnorm{\termaprx}^2 \, \big\{ \hittime \; + \; 2 \, \mixtime \,
( 1 - \mixtime^{-1} )^{\hittime} \tfrac{\supnorm{f}}{\munorm{f}} \big\}
\end{align*}
for any $\hittime \in \Int_+$.  By letting $\hittime \defn \mixtime \,
\log \tfrac{\supnorm{f}}{\munorm{f}}$, we have
\begin{align}
\label{eq:vara2}
\vara(f) & \leq \munorm{f}^2 \cdot 2 \, \mixtime \,
\supnorm{\termaprx}^2 \, \big\{ 2 + \log
\tfrac{\supnorm{f}}{\munorm{f}} \big\} \, .
\end{align}

We now bound the $L^{\infty}$-norms $\supnorm{\termmtg}$ and
$\supnorm{\termaprx}$ using the radius $\newrad$ given in
condition~\eqref{cond:newrad}. Recall the definitions of functions
$\termmtg$ and $\termaprx$ in equation~\eqref{eq:def_termone}. Under
condition~\eqref{cond:newrad}, we have
\begin{align*}
\big| \returnhat_{t+1}^{t+\step}(\thetastar) \big| & \leq
\ssum{\ell=1}{\step} \discount^\ell \supnorm{\reward} +
\discount^\step \, \supnorm{\thetastar - \reward} \leq
\tfrac{1}{1-\discount} \, \bou \newrad \, .
\end{align*}
It follows that
\begin{align}
\label{eq:supnorm_termone}
\max \big\{ \supnorm{\termmtg}, \, \supnorm{\termaprx} \big\} \leq
\supnorm{\termone} \leq \supnorm{\thetastar - \reward} + \big|
\returnhat_{t+1}^{t+\step}(\thetastar) \big| \leq
\tfrac{2}{1-\discount} \, \bou\newrad \, .
\end{align}
Substituting the bound~\eqref{eq:supnorm_termone} into
inequalities~\eqref{eq:varm2} and \eqref{eq:vara2}, we find that
\begin{align}
\label{eq:var2}
\varm(f) \leq 8 \, \munorm{f}^2 \cdot \Step \, \tfrac{\bou^2
  \newrad^2}{(1-\discount)^2} \qquad \text{and} \qquad \vara(f) \leq 8
\, \munorm{f}^2 \cdot \mixtime \, \tfrac{\bou^2
  \newrad^2}{(1-\discount)^2} \, \big\{ 2 + \log
\tfrac{\supnorm{f}}{\munorm{f}} \big\} \, .
\end{align}

As for term $\DiidY{2}(f)$, we show later that
\begin{align}
\label{eq:DiidY<}  
 \Exp\big[ (\DiidY{2}(f))^2 \big] \, \big/ \, \Exp[\hittime_2 -
   \hittime_1] \leq \min\big\{ 2 \, \supnorm{f}^2 \,
 \stdmtg^2(\thetastar), \; 16 \, \munorm{f}^2 \cdot \Step \,
 \tfrac{\bou^2 \newrad^2}{(1-\discount)^2} \big\}.
\end{align}
Substituting the bounds~\eqref{eq:var1}, \eqref{eq:var2} and
\eqref{eq:DiidY<} into inequality~\eqref{eq:iidY<}, we obtain
inequality~\eqref{eq:Term1_var}, as stated in
\Cref{lemma:Term1_var}. \\

\noindent It remains to prove inequalities~\eqref{eq:iidY<}
and~\eqref{eq:DiidY<}.


\paragraph{Proof of claim~\eqref{eq:iidY<}:}

By the law of large numbers, in the limit as the sample size
$\numobsnew = \numobs - \Step \rightarrow \infty$, we have
\begin{subequations}
\begin{align}
 \label{eq:Esq1}     
&\frac{1}{\numobsnew} \, \bigg\{ \sum_{t=1}^{\numobsnew} f(\state_t) \,
\termone_t \bigg\}^2 ~~ \stackrel{a.s.}{\longrightarrow} ~~ \Exp\bigg[
  f(\State_0) \, \termone(\State_{0}^{\Step}) \;
  \ssum{t=-\infty}{\infty} f(\State_t) \, \termone(\State_{t: \,
    t+\Step}) \bigg], \\
\label{eq:Esq2}
&\frac{1}{\hitS} \, \bigg\{ \sum_{s=0}^{\hitS} \iidY{s}(f) \bigg\}^2 ~~
\stackrel{a.s.}{\longrightarrow} ~~ \Exp\big[ (\iidY{1}(f))^2 \big] +
2 \, \Exp\big[ \iidY{1}(f) \, \iidY{2}(f) \big], \quad \mbox{and} \\
\label{eq:Esq=}
&\frac{1}{\numobsnew} \, \bigg\{ \sum_{t=1}^{\numobsnew} f(\state_t) \,
\termone_t \bigg\}^2 \, - \; \frac{1}{\hitS} \, \bigg\{
\sum_{s=0}^{\hitS} \iidY{s}(f) \bigg\}^2 \Big/ ~ \Exp[\hittime_2 -
  \hittime_1] ~~ \stackrel{a.s.}{\longrightarrow} ~~ 0.
\end{align}
\end{subequations}

Substituting~equations \eqref{eq:Esq1} and \eqref{eq:Esq2} into equation~\eqref{eq:Esq=} yields
\begin{align}
\label{eq:iidYsq1}
  \frac{\Exp \big[ (\iidY{1}(f))^2 \big]}{\Exp[\hittime_2 -
      \hittime_1]} = \Exp\bigg[ f(\State_0) \,
    \termone(\State_{0}^{\Step}) \; \ssum{t=-\infty}{\infty}
    f(\State_t) \, \termone(\State_{t: \, t+\Step}) \bigg] - 2 \,
  \frac{\Exp\big[ \iidY{1}(f) \, \iidY{2}(f) \big]}{ \Exp[\hittime_2 -
      \hittime_1]}.
\end{align}

We first look at the second term in the right-hand side of
inequality~\eqref{eq:iidYsq1}. Since the random variable $\iidY{1}(f)
- \DiidY{2}(f) = \Exp[\,\iidY{1}(f) \mid x_{\hittime_2}\,]$ belongs to
the $\sigma$-field generated by $(\state_1, \state_2, \ldots,
\state_{\hittime_2})$, independent of $\iidY{2}(f)$, we have
\begin{align}
  \Exp\big[ \iidY{1}(f) \, \iidY{2}(f) \big] & = \Exp\big[ \iidY{1}(f)
    - \DiidY{2}(f)\big] \, \Exp[ \iidY{2}(f) ] + \Exp\big[
    \DiidY{2}(f) \, \iidY{2}(f) \big] = \Exp\big[ \DiidY{2}(f) \,
    \iidY{2}(f) \big] \notag \\ & \stackrel{(i)}{\geq} - \Exp\big[
    (\iidY{2}(f))^2 \big]^{\frac{1}{2}} \, \Exp\big[ (\DiidY{2}(f))^2
    \big]^{\frac{1}{2}} \stackrel{(ii)}{=} - \Exp\big[ (\iidY{1}(f))^2
    \big]^{\frac{1}{2}} \, \Exp\big[ (\DiidY{2}(f))^2
    \big]^{\frac{1}{2}} \, . \label{eq:iidYsq2}
\end{align}
Here the inequality~$(i)$ follows from the Cauchy--Schwarz
inequality. The equality~$(ii)$ holds since $\iidY{1}(f)$ and
$\iidY{2}(f)$ are identically distributed.

We next consider the first term in the right-hand side of
inequality~\eqref{eq:iidYsq1}. By the Fenchel--Young inequality,
we have
\begin{align*}
  \frac{1}{\numobsnew} \, \bigg\{ \sum_{t=1}^{\numobsnew} f(\state_t)
  \, \termone_t \bigg\}^2 \; \leq \; \frac{2}{\numobsnew} \, \bigg\{
  \sum_{t=1}^{\numobsnew} f(\state_t) \, \termmtg_t \bigg\}^2 \, + \,
  \frac{2}{\numobsnew} \, \bigg\{ \sum_{t=1}^{\numobsnew} f(\state_t)
  \, \termaprx_t \bigg\}^2.
\end{align*}
Moreover, note that \mbox{$\frac{1}{\numobsnew} \, \bigg\{
  \sum_{t=1}^{\numobsnew} f(\state_t) \, \termmtg_t \bigg\}^2 ~~
  \stackrel{a.s.}{\longrightarrow} ~~ \varm(f)$} and
\mbox{$\frac{1}{\numobsnew} \, \bigg\{ \sum_{t=1}^{\numobsnew}
  f(\state_t) \, \termaprx_t \bigg\}^2 ~~
  \stackrel{a.s.}{\longrightarrow} ~~ \vara(f)$.$\quad$}  Combining the
pieces yields
\begin{align}
\label{eq:iidYsq3}
\Exp\bigg[ f(\State_0) \, \termone(\State_{0}^{\Step}) \;
  \ssum{t=-\infty}{\infty} f(\State_t) \, \termone(\State_{t: \,
    t+\Step}) \bigg] \; \leq \; 2 \, \big\{ \varm(f) + \vara(f) \big\}
\, .
\end{align}       
Substituting the bounds~\eqref{eq:iidYsq2}~and~\eqref{eq:iidYsq3} into
inequality~\eqref{eq:iidYsq1} yields
\begin{multline*}
\frac{\Exp\big[ (\iidY{1}(f))^2 \big]}{\Exp[\hittime_2 - \hittime_1]}
\leq 2 \, \big\{ \varm(f) + \vara(f) \big\} + 2 \, \big\{ \Exp\big[
  (\iidY{1}(f))^2 \big] \, \big/ \, \Exp[\hittime_2 - \hittime_1]
\big\}^{\frac{1}{2}} \, \big\{ \Exp\big[ (\DiidY{2}(f))^2 \big] \,
\big/ \, \Exp[\hittime_2 - \hittime_1] \big\}^{\frac{1}{2}} \, .
 \end{multline*}
Solving this inequality yields the claimed upper
bound~\eqref{eq:iidY<}.


\paragraph{Proof of inequality~\eqref{eq:DiidY<}:}

Similar to equation~\eqref{eq:EiidY}, we can show that
\begin{align}
\label{eq:Efsq}
\Exp\bigg[ \sum_{t = \hittime_1 + 1}^{\hittime_2} f^2(\state_t) \bigg]
= \Exp[\hittime_2 - \hittime_1] \cdot \munorm{f}^2 \, .
\end{align}
Recall the decomposition $\DiidY{2}(f) = \iidY{1}(f) - \Exp[ \,
  \iidY{1}(f) \mid \state_{\hittime_2} ] = \sum_{t = \hittime_2 -
  \Step + 1}^{\hittime_2} f(\state_t) \, \big\{ \termone_t -
\Exp[\termone_t \mid \state_{\hittime_2}] \big\}$.  By the
Cauchy--Schwarz inequality, we have
        \begin{align*}
        \Exp\big[ ( \DiidY{2}(f) )^2 \big] & \leq \Exp\bigg[ \Big\{
          \ssum{t = \hittime_2 - \Step + 1}{\hittime_2} f^2(\state_t)
          \Big\} \, \Big\{ \ssum{t = \hittime_2 - \Step +
            1}{\hittime_2} \big( \termone_t - \Exp[\termone_t \mid
            \state_{\hittime_s}] \big)^2 \Big\} \bigg] \\ & \leq 4 \,
        \Step \, \supnorm{\termone}^2 \; \Exp\bigg[ \sum_{t =
            \hittime_1 + 1}^{\hittime_2} f^2(\state_t) \bigg]
        \stackrel{(i)}{=} \Exp[ \hittime_2 - \hittime_1 ] \,
        \munorm{f}^2 \cdot 4 \, \Step \, \supnorm{\termone}^2 \, ,
        \end{align*}
        where we have used equation~\eqref{eq:Efsq} in step $(i)$. By
        further applying the bound~\eqref{eq:supnorm_termone} on
        $\supnorm{\termone}$, we obtain
        \begin{align} \label{eq:DiidYsq1}
                \Exp\big[ ( \DiidY{2}(f) )^2 \big] & \leq \Exp[
                  \hittime_2 - \hittime_1 ] \; \munorm{f}^2 \cdot 16
                \, \Step \, \tfrac{\bou^2 \newrad^2}{(1-\discount)^2}
                \, .
        \end{align} 

In the sequel, we focus on proving
\begin{align}
\label{eq:DiidYsq2}
\Exp\big[ ( \DiidY{2}(f) )^2 \big] & \leq \Exp[ \hittime_2 -
  \hittime_1 ] \; \supnorm{f}^2 \cdot 2 \, \stdmtg^2(\thetastar) \, .
\end{align}
For any index $t \leq \Step$, we have
\begin{align}
  \label{eq:decomp_termone}
  \termone_{\hittime_2 - t} - \Exp\big[ \termone_{\hittime_2 - t}
    \bigm| \state_{\hittime_2} \big] = \Dtermmtg{t+1}(\distrnew,
  \state_{\hittime_2+1}) + \ssum{\ell=2}{\Step-t}
  \Dtermmtg{t+\ell}(\state_{\hittime_2+\ell-1},
  \state_{\hittime_2+\ell}) \, ,
\end{align}
where function $\Dtermmtg{t+\ell}: \StateSp^2 \rightarrow \Real$ is
given in definition~\eqref{eq:Dtermmtg} and
\begin{align}
  \label{eq:def_Dtermmtg_distrnew}
  \Dtermmtg{t+1}(\distrnew, \state_{\hittime_2+1}) \defn
  \discount^{t+1} \, \ssum{\step=t+1}{\Step} \weight{\step} \, \big\{
  \big( \BellOp{\step - t-1}(\thetastar) \big)(\state_{\hittime_2+1})
  - \distrnew \, \BellOp{\step - t-1}(\thetastar) \big\} \, .
\end{align}
Substituting equation~\eqref{eq:decomp_termone} into
equation~\eqref{eq:DiidY}, we find that
\begin{multline}
  \label{eq:DiidYsq}
  \Exp\big[ (\DiidY{2}(f))^2 \big] = \Exp \bigg[ \Big\{
    \sum_{t=1}^{\Step} f(\state_{\hittime_2-t+1}) \,
    \Dtermmtg{t}(\distrnew, \state_{\hittime_2+1}) +
    \sum_{\ell=2}^{\Step} \sum_{t=\ell}^{\Step}
    f(\state_{\hittime_2-t+\ell}) \,
    \Dtermmtg{t}(\state_{\hittime_2+\ell-1}, \state_{\hittime_2+\ell})
    \Big\}^2 \bigg] \\ = \Exp\bigg[ \Big\{ \ssum{t=1}{\Step}
    f(\state_{\hittime_2-t+1}) \, \Dtermmtg{t}(\distrnew,
    \state_{\hittime_2+1}) \Big\}^2 \bigg] + \sum_{\ell=2}^{\Step}
  \Exp\bigg[ \Big\{ \ssum{t=\ell}{\Step} f(\state_{\hittime_2-t+\ell})
    \, \Dtermmtg{t}(\state_{\hittime_2+\ell-1},
    \state_{\hittime_2+\ell}) \Big\}^2 \bigg] \, ,
\end{multline}
where the second equality is due to the property that \mbox{$\Exp\big[
    \Dtermmtg{t}(\state_{\hittime_2+\ell-1}, \state_{\hittime_2+\ell})
    \bigm| \state_{\hittime_2+\ell-1} \big] = 0$} for \mbox{$\ell =
  2,3,\ldots,\Step$}.  By the Cauchy--Schwarz inequality, the two
terms on the right-hand side of equation~\eqref{eq:DiidYsq} satisfy
\begin{subequations}
  \begin{align}
    \Exp \bigg[ \Big\{ \sum_{t=1}^{\Step} f(\state_{\hittime_2-t+1})
      \, \Dtermmtg{t}(\distrnew, \state_{\hittime_2+1}) \Big\}^2
      \bigg] & \leq \supnorm{f}^2 \bigg\{ \sum_{t=1}^{\Step}
    \sqrt{\Exp\big[ (\Dtermmtg{t}(\distrnew, \state_{\hittime_2+1}))^2
        \big]} \bigg\}^2 \, , \label{eq:Dtermmtg1} \\ \Exp\bigg[
      \Big\{ \sum_{t=\ell}^{\Step} f(\state_{\hittime_2-t+\ell}) \,
      \Dtermmtg{t}(\state_{\hittime_2+\ell-1},
      \state_{\hittime_2+\ell}) \Big\}^2 \bigg] & \leq \supnorm{f}^2
    \bigg\{ \sum_{t=\ell}^{\Step} \sqrt{\Exp\big[
        (\Dtermmtg{t}(\state_{\hittime_2+\ell-1},
        \state_{\hittime_2+\ell}))^2 \big]} \bigg\}^2 \,
    . \label{eq:Dtermmtg2}
  \end{align}
\end{subequations}
In the sequel, we estimate the terms $\Exp\big[
  (\Dtermmtg{t}(\distrnew, \state_{\hittime_2+1}))^2 \big]$ and $\Exp
\big[ (\Dtermmtg{t}(\state_{\hittime_2+\ell-1},
  \state_{\hittime_2+\ell}))^2 \big]$ in turn.

Recall the state $\state_{\hittime_2+1}$ is independently generated
from distribution $\distrnew$ conditioned on $(\state_1, \ldots,
\state_{\hittime_2})$. We introduce two independent random variables
$\State \sim \distr$ and $\Statetilnew \sim \distrnew$. It then
follows from definition~\eqref{eq:def_Dtermmtg_distrnew} of
$\Dtermmtg{t}(\distrnew, \state_{\hittime_2+1})$ that
\begin{align*}
  \Exp\big[ (\Dtermmtg{t}(\distrnew, \state_{\hittime_2+1}))^2 \big] =
  \Exp_{\State \sim \distr} \, \Exp_{\Statetilnew\sim\distrnew}\bigg[
    \Big\{ \discount^{t+1} \! \ssum{\step=t+1}{\Step} \!\!
    \weight{\step} \big\{ \big( \BellOp{\step - t-1}(\thetastar)
    \big)(\Statetilnew) - \distrnew \, \BellOp{\step -
      t-1}(\thetastar) \big\} \Big\}^2 \bigg] \, .
\end{align*}
Due to the bias-variance decomposition, we have
\begin{align}
  \Exp\big[ (\Dtermmtg{t}(\distrnew, \state_{\hittime_2+1}))^2 \big] &
  \leq \Exp_{\State \sim \distr} \, \Exp_{\Statetilnew\sim\distrnew}
  \bigg[ \Big\{ \discount^{t+1} \!  \ssum{\step=t+1}{\Step} \!
    \weight{\step} \big\{ \big( \BellOp{\step - t-1}(\thetastar)
    \big)(\Statetilnew) - \big(\TransOp \BellOp{\step -
      t-1}(\thetastar)\big)(\State) \big\} \Big\}^2 \bigg] \notag \\ &
  = \Exp_{\State \sim \distr} \, \Exp_{\Statetilnew\sim\distrnew}
  \big[ (\Dtermmtg{t}(\State, \Statetilnew))^2 \big] \,
  . \label{eq:Dtermmtg_distrnew_1}
\end{align}
We further conduct a change of measure by replacing $\Statetilnew \sim
\distrnew$ with $\Statenew \sim \TransOp(\cdot \mid \State)$. It
follows from the minorization condition~\eqref{cond:minor} that
\begin{align}
\Exp_{\State \sim \distr} \, \Exp_{\Statenew\sim\distrnew} \big[
  (\Dtermmtg{t}(\State, \Statenew))^2 \big]
\label{eq:Dtermmtg_distrnew_2}  
& = \mixtime \; \Exp_{\State \sim \distr, \,
  \Statenew\sim\TransOp(\cdot \mid \State)} \big[
  (\Dtermmtg{t}(\State, \Statenew))^2 \big].
\end{align}
Substituting
inequality~\eqref{eq:Dtermmtg_distrnew_2}~into~inequality~\eqref{eq:Dtermmtg_distrnew_1},
we conclude that
\begin{align}
\label{eq:Dtermmtg_distrnew_3}
\Exp \big[ (\Dtermmtg{t}(\distrnew, \state_{\hittime_2+1}))^2 \big]
\leq \mixtime \; \Exp_{\State \sim \distr, \, \Statenew \sim
  \TransOp(\cdot \mid \State)} \big[ (\Dtermmtg{t}(\State,
  \Statenew))^2 \big] \, .
\end{align}
Applying the bound~\eqref{eq:Dtermmtg_distrnew_3} to
inequality~\eqref{eq:Dtermmtg1}, we derive
\begin{align}
\label{eq:Dtermmtg_distrnew}
\Exp \bigg[ \Big\{ \sum_{t=1}^{\Step} f(\state_{\hittime_2-t+1}) \,
  \Dtermmtg{t}(\distrnew, \state_{\hittime_2+1}) \Big\}^2 \bigg] \leq
\supnorm{f}^2 \cdot \mixtime \, \bigg\{ \sum_{t=1}^{\Step} \sqrt{\Exp
  \big[ (\Dtermmtg{t})^2 \big]} \bigg\}^2 = \supnorm{f}^2 \cdot
\mixtime \; \stdmtg^2(\thetastar) \, .
\end{align}

As for the term $\Exp \big[ (\Dtermmtg{t}(\state_{\hittime_2+\ell-1},
  \state_{\hittime_2+\ell}))^2 \big]$ in
inequality~\eqref{eq:Dtermmtg2}, we first note that by our
construction of the regenerated chain, $\state_{\hittime_2 + 1} \sim
\distrnew$ and $\state_{\hittime_2 + \ell} \sim \TransOp(\cdot \mid
\state_{\hittime_2 + \ell - 1})$ for $\ell = 2, 3, \ldots,
\Step$. Therefore,
\begin{align*}
\Exp \big[ (\Dtermmtg{t}(\state_{\hittime_2+\ell-1},
  \state_{\hittime_2+\ell}))^2 \big] = \Exp_{\State_1 \sim \distrnew}
\big[ (\Dtermmtg{t}(\State_{\ell-1}, \State_{\ell}))^2 \big] \, ,
\end{align*}
where $(\State_1, \State_2, \ldots, \State_{\Step})$ is a Markov chain
generated by the transition kernel $\TransOp$ and initial
distribution~$\distrnew$. Recall that by the mixing
condition~\ref{assump:mixing}, we have $\tfrac{\diff \distrnew}{\diff
  \distr} \leq 1 + \Constdistrnew$. It follows that
\begin{multline*}
  \Exp \big[ (\Dtermmtg{t}(\state_{\hittime_2+\ell-1},
    \state_{\hittime_2+\ell}))^2 \big] = \Exp_{\State_1 \sim
    \distrnew} \big[ (\Dtermmtg{t}(\State_{\ell-1}, \State_{\ell}))^2
    \big] \\ \leq (1 + \Constdistrnew) \, \Exp_{\State_1 \sim
    \distr} \big[ (\Dtermmtg{t}(\State_{\ell-1}, \State_{\ell}))^2
    \big] = (1 + \Constdistrnew) \, \Exp_{\State \sim \distr, \,
    \Statenew \sim \TransOp(\cdot \mid \State)} \big[
    (\Dtermmtg{t}(\State, \Statenew))^2 \big] \, .
\end{multline*}
The last equality is due to the stationarity of distribution $\distr$
under the transition kernel $\TransOp$. Under the condition
$\Constdistrnew \, \Step \leq \mixtime$, we have
\begin{align}
  \label{eq:Dtermmtg_1}
  \Exp \big[ (\Dtermmtg{t}(\state_{\hittime_2+\ell-1},
    \state_{\hittime_2+\ell}))^2 \big] & \leq (1 + \mixtime / \Step)
  \; \Exp_{\State \sim \distr, \, \Statenew \sim \TransOp(\cdot \mid
    \State)} \big[ (\Dtermmtg{t}(\State, \Statenew))^2 \big] \, .
\end{align}
Combining inequalities~\eqref{eq:Dtermmtg_1}~and~\eqref{eq:Dtermmtg2}
yields the bound
\begin{align}
\label{eq:Dtermmtg_State}
\Exp\bigg[ \Big\{ \sum_{t=\ell}^{\Step} f(\state_{\hittime_2-t+\ell})
  \, \Dtermmtg{t}(\state_{\hittime_2+\ell-1},
  \state_{\hittime_2+\ell}) \Big\}^2 \bigg] & \leq \supnorm{f}^2 \cdot
(1 + \mixtime / \Step) \, \bigg\{ \sum_{t=\ell}^{\Step} \sqrt{\Exp
  \big[ (\Dtermmtg{t})^2 \big]} \bigg\}^2 \notag \\ & = \supnorm{f}^2
\cdot (1 + \mixtime / \Step) \; \stdmtg^2(\thetastar) \, .
\end{align}

Substituting the
bounds~\eqref{eq:Dtermmtg_distrnew}~and~\eqref{eq:Dtermmtg_State} into
equation~\eqref{eq:DiidYsq} yields
\begin{align}
  \Exp\big[ (\DiidY{2}(f))^2 \big] & \leq \supnorm{f}^2 \cdot (2 \,
  \mixtime + \Step) \; \stdmtg^2(\thetastar) \, .
\end{align}
Inequality~\eqref{eq:DiidYsq2} then follows from the property
$\Exp[\hittime_2 - \hittime_1] = \mixtime + \Step$. \\

\noindent Combining the
bounds~\eqref{eq:DiidYsq1}~and~\eqref{eq:DiidYsq2} yields the claimed
inequality~\eqref{eq:DiidY<}.



\subsubsection{Proof of Lemma~\ref{lemma:Term1_exp}}
\label{sec:proof:lemma:Term1_exp}

Recall that any function $f \in \RKHS$ has a decomposition $f =
\ssum{j=1}{\infty} f_j \, \eigfun{j}$, where $f_j \in \Real$ and
$\eigfun{j}$ is the $j^{th}$ eigenfunction. The constraints $\munorm{f}
\leq u$, $\hilnorm{f} \leq \newrad$ are then equivalent to
$\ssum{j=1}{\infty} f_j^2 \leq u^2$ and $\ssum{j=1}{\infty}
\frac{f_j^2}{\eig{j}} \leq \newrad^2$. Consequently, we have
\begin{align}
 \Exp \big[ \supZbar^{(1)}(u) \big] = \Exp \big[ \supZbar^{(2)}(u)
   \big] & = \frac{\newrad}{\numobs} \cdot \Exp \Bigg[
   \sup_{\begin{subarray}{c} \munorm{f} \leq u/\newrad \notag
       \\ \hilnorm{f} \leq 1 \end{subarray}} \bigg|
   \sum_{j=1}^{\infty} f_j \sum_{s=0}^{\hitSbar}
   \iidY{2s+1}(\feature{j}) \bigg| \Bigg] \! \\ & \leq
 \frac{\newrad}{\numobs} \cdot \Exp \Bigg[ \bigg\{ 2
   \sum_{j=1}^{\infty} \min\big\{ \tfrac{u^2}{\newrad^2}, \eig{j}
   \big\} \, \Big\{ \sum_{s=0}^{\hitSbar} \iidY{2s+1}(\feature{j})
   \Big\}^2 \bigg\}^{1/2} \Bigg] \notag \\ & \leq
 \frac{\newrad}{\numobs} \sqrt{ 2 \sum_{j=1}^{\infty} \min\big\{
   \tfrac{u^2}{\newrad^2}, \eig{j} \big\} \, \Exp \bigg[ \Big\{
     \sum_{s=0}^{\hitSbar} \iidY{2s+1}(\feature{j}) \Big\}^2 \bigg] }
 \, . \label{eq:delcritnew}
\end{align}
From~\Cref{lemma:Term1_var}, the random variables $\big\{
\iidY{2s+1}(\feature{j}) \big\}_{s=0}^{\hitSbar}$ are i.i.d. with
$\Exp\big[ \iidY{2s+1}(\feature{j}) \big] = 0$, whence
\begin{align*}
\Exp \bigg[ \Big\{ \sum_{s=0}^{\hitSbar} \iidY{2s+1}(\feature{j})
  \Big\}^2 \bigg] = (\hitSbar + 1) \; \Exp \big[
  \iidY{1}^2(\feature{j}) \big] \leq \frac{2 \,
  \numobs}{\Exp[\hittime_2 - \hittime_1]} \, \Exp \big[
  \iidY{1}^2(\feature{j}) \big] \, .
\end{align*}
It follows that
\begin{align*}
\Exp \big[ \supZbar^{(1)}(u) \big] = \Exp \big[ \supZbar^{(2)}(u)
  \big] \leq 2 \newrad \sqrt{ \frac{1}{\numobs} \sum_{j=1}^{\infty}
  \min\big\{ \tfrac{u^2}{\newrad^2}, \eig{j} \big\}} ~ \sup_{j \in
  \Int_+} \sqrt{\Exp\big[ (\iidY{1}(\feature{j}))^2 \big] \, \big/ \,
  \Exp[\hittime_2 - \hittime_1]} \; .
\end{align*}
\Cref{lemma:Term1_var} also guarantees that $\Exp\big[
  (\iidY{1}(\feature{j}))^2 \big] \, \big/ \, \Exp[\hittime_2 -
  \hittime_1] \leq 18 \, \supnorm{\feature{j}}^2 \, \big\{
\stdmtg^2(\thetastar) + \approxerr^2(\thetastar) \big\}$, whence
\begin{align}
 \Exp \big[ \supZbar^{(1)}(u) \big] = \Exp \big[ \supZbar^{(2)}(u)
   \big] & \leq 6 \, \unibou \,\newrad \, \big\{ \stdmtg(\thetastar) +
 \approxerr(\thetastar) \big\} \sqrt{ \frac{2}{\numobs}
   \sum_{j=1}^{\infty} \min\big\{ \tfrac{u^2}{\newrad^2}, \eig{j}
   \big\}} \notag \\ & \leq 6\sqrt{2} \, (1-\discounteff) \, \newrad^2
 \, \Big\{ \delcrit^2 + \frac{\delcrit \, u}{\newrad} \Big\} \,
 , \label{eq:EsupZbar<}
\end{align}
where the second step is ensured by the critical inequality
\ref{eq:critineq} with $\noise \geq \noisebase$. The smallest solution
$\delcritnew > 0$ to inequality~\eqref{eq:def_delcritnew} then
satisfies
\begin{align*}
 \frac{1}{40} \, (1 - \discounteff) \, \delcritnew^2 =
 \Exp\big[\supZbar^{(1)}(\delcritnew)\big] +
 \Exp\big[\supZbar^{(2)}(\delcritnew)\big] \leq 12\sqrt{2} \,
 (1-\discounteff) \, \newrad^2 \, \Big\{ \delcrit^2 +
 \frac{\delcrit\delcritnew}{\newrad} \Big\} \, .
\end{align*}
Solving the inequality above, we find that $\delcritnew \leq
\delcritnewtil \defn \const{0} \, \newrad \, \delcrit$ for a universal
constant $\const{0} > 0$, which finishes the proof of
\Cref{lemma:Term1_exp}. Moreover, we note that our choice of scalar
$\const{0}$ ensures
\begin{align}
 \label{eq:def_delcritnewtil}
\frac{1}{40} \, (1 - \discounteff) \, \delcritnewtil^2 = 12\sqrt{2} \,
(1-\discounteff) \, \newrad^2 \, \Big\{ \delcrit^2 + \frac{\delcrit
  \delcritnewtil}{\newrad} \Big\} \, .
\end{align}
We use this relation in the proof of~\Cref{lemma:Term1_tail} to
follow.


\subsubsection{Proof of Lemma~\ref{lemma:Term1_tail}}
\label{append:proof_Term1_tail}

We first note that for any scalar $u > 0$, the triangle inequality
implies that
\begin{align}
\label{eq:supZ<}
\supZ(u) \leq \sup_{f \in \funclass(u)} \Big| \frac{1}{\numobsnew} \,
\iidY{0}(f) \Big| + \sup_{f \in \funclass(u)} \Big|
\frac{1}{\numobsnew} \, \ssum{s=0}{\hitS_1} \iidY{2s+1}(f) \Big| +
\sup_{f \in \funclass(u)} \Big| \frac{1}{\numobsnew} \,
\ssum{s=0}{\hitS_2} \iidY{2s+2}(f) \Big| + \sup_{f \in \funclass(u)}
\Big| \frac{1}{\numobsnew} \, \iidY{\Tail}(f) \Big| \, ,
 \end{align}
where $\hitS_1 = \lfloor (\hitS - 2)/2 \rfloor$ and $\hitS_2 = \lfloor
(\hitS - 3)/2 \rfloor$.  For $\variota = 1$ or $2$, we say the event
$\eventA_{\variota}$ happens if the inequality
\begin{align}
\sup_{f \in \funclass(\delcritnewtil)} \Big| \frac{1}{\numobsnew} \,
\ssum{s=0}{\hitS_{\variota}} \iidY{2s+\variota}(f) \, \Big| \; \geq \;
20 \; \Exp\big[ \supZbar^{(\variota)}(\delcritnewtil) \big] + (1 \!-\!
\discounteff) \, \tfrac{\delcritnewtil^2}{8} \log \numobs
\end{align}
holds.
We will prove that there exists a universal constant $\plaincon > 0$
such that for $\variota = 1$ and $2$,
\begin{subequations}
 \begin{align}
 \label{eq:prob1}
 & \Prob( \eventA_{\variota} ) \leq \plaincon \, \exp\big( \! -
 \Constprob \, \tfrac{\numobs \, \delcritnewtil^2}{\bou^2 \newrad^2}
 \big) \, ,
 \end{align}
where $\PlainCon = \frac{(1 - \discounteff)^2(1 - \discount)^2}{\plaincon \, (\mixtime + \Step)}$. Moreover, we will also show that
\begin{align}
 \label{eq:prob2}
 & \Prob\bigg( \sup_{f \in \funclass(\delcritnewtil)}
 \Big|\frac{1}{\numobsnew} \, \iidY{s}(f) \Big| \; \leq \;
 (1 - \discounteff) \, \tfrac{\delcritnewtil^2}{8} \log \numobs \bigg) \leq
 \exp\big( - \Constprob \, \tfrac{\numobs \,
 \delcritnewtil^2}{\bou^2\newrad^2} \big) \qquad \quad \text{for $s
 = 0$ or $\Tail$}.
\end{align}
\end{subequations}
Given inequalities~\eqref{eq:prob1}~and~\eqref{eq:prob2}, we have by
union bound and the inequality~\eqref{eq:supZ<} that
\begin{align}
 \label{eq:supZ}
 \supZ(\delcritnewtil) \leq 20 \, \big\{ \Exp\big[
 \supZbar^{(1)}(\delcritnewtil) \big] + \Exp\big[
 \supZbar^{(2)}(\delcritnewtil) \big] \big\} + (1 -
 \discounteff) \, \tfrac{\delcritnewtil^2}{2} \log \numobs \, ,
\end{align}
with probability at least $1 - 2 \, (1+\plaincon) \exp\big( -
\Constprob \, \tfrac{\numobs \,
 \delcritnewtil^2}{\bou^2\newrad^2} \big)$. We learn from
inequality~\eqref{eq:EsupZbar<} and equation~\eqref{eq:def_delcritnewtil} in \Cref{lemma:Term1_var} that the expectation $\Exp\big[
 \supZbar^{(\variota)}(\delcritnewtil) \big]$ satisfies
\begin{align}
 \label{eq:EsupZbar}
 \Exp \big[ \supZbar^{(1)}(\delcritnewtil) \big] = \Exp \big[
 \supZbar^{(2)}(\delcritnewtil) \big] & \leq 6\sqrt{2} \,
 (1-\discounteff) \, \newrad^2 \, \Big\{ \delcrit^2 +
 \frac{\delcrit\delcritnewtil}{\newrad} \Big\} = \frac{1}{80} \, (1 -
 \discounteff) \, \delcritnewtil^2 \, .
\end{align}
Substituting the
bound~\eqref{eq:EsupZbar} into inequality~\eqref{eq:supZ}
yields
\begin{align*}
 \Prob\big[ \supZ(\delcritnewtil) \leq (1 - \discounteff) \, \delcritnewtil^2 \log \numobs \big] \geq 1 - 2 \, (1+\plaincon) \exp\big( - \Constprob \, \tfrac{\numobs \, \delcritnewtil^2}{\bou^2\newrad^2} \big) \, ,
\end{align*}
as stated in \Cref{lemma:Term1_tail}. \\

\noindent It remains to prove the high-probability
bounds~\eqref{eq:prob1} and~\eqref{eq:prob2}.


\paragraph{Proof of inequality~\eqref{eq:prob1}:}

By symmetry, it suffices to prove the bound with $\variota = 1$.  We
first show that with high probability, the random variable is upper
bounded as $\hitS_1 \leq \hitSbar = \lfloor \numobsnew / (\mixtime +
\Step) \rfloor$. In fact, since $\Exp[\hittime_{t+1} - \hittime_t] =
\mixtime + \Step$ and \mbox{$\psinorm[\big]{1}{(\hittime_{t+1} -
    \hittime_t) - \Exp[\hittime_{t+1} - \hittime_t]} \lesssim
  \mixtime$} due to the property of geometric random variables, it
follows from the Bernstein's inequality that there exists a universal
constant $\const{1} > 0$ such that
\begin{align}
 \Prob\big( \hitS_1 \geq \hitSbar \big) \leq \Prob\big(
 \hittime_{2\hitSbar+2} \leq \numobsnew \big) & \leq \Prob\bigg(
 \frac{1}{2\hitSbar+1}\ssum{t=1}{2\hitSbar+1}(\hittime_{t+1} -
 \hittime_t) \leq \frac{2}{3} \, (\mixtime + \Step) \bigg) \notag \\
 & \leq \exp(-\const{1} \numobs) \leq \exp\big( - \Constprob \, \tfrac{\numobs \, \delcritnewtil^2}{\bou^2\newrad^2} \big) \, .  \label{eq:prob_hitS}
\end{align}

Now we consider the probability of event $\eventA_1$, conditioned on
$\hitS_1 \leq \hitSbar$. We leverage Corollary~4 from the
paper~\cite{montgomery1993comparison} to connect the sum up to $s =
\hitS_1$ with the sum up to $s = \hitSbar$.  (See also Lemma~4 in the
paper~\cite{adamczak2008tail}.) Note that the mapping $g \mapsto
\sup\nolimits_{f \in \funclass(\delcritnewtil)} \hilin{f}{g}$ defines
a metric on the space~$\RKHS$, and therefore
\begin{align}
\Prob\big(\eventA_1 \bigm| \hitS_1 \leq \hitSbar\big)
& \leq \Prob\bigg( \sup_{0 \leq \hitStil \leq \hitSbar} \; \sup_{f \in
  \funclass(\delcritnewtil)} \Big| \frac{1}{\numobsnew} \,
\ssum{s=0}{\hitStil} \iidY{2s+1}(f) \Big| \geq 20 \,
\Exp\big[\supZbar^{(1)}(\delcritnewtil)\big] + 
(1-\discounteff) \tfrac{\delcritnewtil^2}{8} \log \numobs \Bigm| \hitS_1 \leq
\hitSbar \, \Big) \notag \\ 
& \leq 10 \; \Prob\Big( \,
\supZbar^{(1)}(\delcritnewtil) \geq 2 \,
\Exp\big[\supZbar^{(1)}(\delcritnewtil)\big] +
(1-\discounteff) \tfrac{\delcritnewtil^2}{80} \log \numobs \Big) \, .
 \label{eq:prob3}
\end{align}
Recall that $\supZbar^{(1)}(\delcritnewtil) = \sup\nolimits_{f \in
 \funclass(\delcritnewtil)} \big| \frac{1}{\numobsnew} \,
\ssum{s=0}{\hitSbar} \iidY{2s+1}(f) \big|$.

We next apply Theorem~4 from the paper~\cite{adamczak2008tail} so as
to upper bound the right-hand side of the
inequality~\eqref{eq:prob3}. Notice that the random variables $\{
\iidY{2 s + 1}(f) \}_{s=0}^{\hitSbar}$ are i.i.d. and $\Exp[ \,
  \iidY{1}(f) \, ] = 0$. According to~\Cref{lemma:Term1_var}, we have
\begin{align*}
\sup_{f \in \funclass(\delcritnewtil)} \! \Exp \big[ \iidY{1}^2(f)
  \big] & \lesssim \tfrac{(\mixtime + \Step)^2}{(1-\discount)^2} \,
\bou^2 \newrad^2 \!\! \sup_{f \in \funclass(\delcritnewtil)}
\munorm{f}^2 \; \big\{1 + \log \tfrac{\supnorm{f}}{\munorm{f}} \big\}
\\ & \stackrel{(i)}{\leq} \tfrac{(\mixtime +
  \Step)^2}{(1-\discount)^2} \cdot \delcritnewtil^2 \; \bou^2
\newrad^2 \big\{ 1 + \log \tfrac{\bou\newrad}{\delcritnewtil} \big\}
\lesssim \tfrac{(\mixtime + \Step)^2}{(1-\discount)^2} \cdot
\delcritnewtil^2 \; \bou^2 \newrad^2 \, \log \numobs \, ,
\end{align*}
where we have used the properties $\sup_{f \in
  \funclass(\delcritnewtil)} \munorm{f} \leq \delcritnewtil$ and
$\sup_{f \in \funclass(\delcritnewtil)} \supnorm{f} \leq \bou \,
\newrad$ in step $(i)$. Moreover, since \mbox{$\supnorm{f \, \termone}
  \leq \tfrac{2}{1-\discount} \, \bou^2 \newrad^2$} for any $f \in
\funclass(\delcritnewtil)$ due to the bound~\eqref{eq:supnorm_termone}
on $\supnorm{\termone}$, we have
\begin{align*}
 & \psinorm[\big]{1}{\sup\nolimits_{f \in \funclass(\delcritnewtil)}
    |\iidY{1}(f)|} \leq \tfrac{2 \, \bou^2\newrad^2}{1-\discount} \,
  \psinorm{1}{\hittime_2 - \hittime_1} \lesssim
  \tfrac{\bou^2\newrad^2}{1-\discount} \, (\mixtime + \Step) \, .
\end{align*}
Applying Theorem 4 from the paper~\cite{adamczak2008tail}, there
exists a universal constant $\plaincon > 0$ such that
\begin{multline}
\label{eq:prob4}
\Prob \Big( \, \supZbar^{(1)}(\delcritnewtil) \geq 2 \,
\Exp\big[\supZbar^{(1)}(\delcritnewtil)\big] + (1-\discounteff)
\tfrac{\delcritnewtil^2}{80} \log \numobs \Big) \\
\leq 4 \, \exp\bigg( - \plaincon \; \frac{(1-\discounteff)^2
  (1-\discount)^2}{\mixtime + \Step} \, \frac{\numobs \,
  \delcritnewtil^2}{\bou^2 \newrad^2} \bigg) \leq 4 \, \exp\Big( -
\Constprob \, \, \frac{\numobs \, \delcritnewtil^2}{\bou^2 \newrad^2}
\Big) \, ,
\end{multline}
where $\Constprob \defn \plaincon \; \frac{(1 - \discounteff)^2 (1 -
  \discount)^2}{\mixtime + \Step}$.  Substituting the
bound~\eqref{eq:prob4} into inequality~\eqref{eq:prob3} yields
\begin{align}
\label{eq:PeventA}
\Prob \big(\eventA_1 \bigm| \hitS_1 \leq \hitSbar\big) \leq 40 \, \exp
\big( - \Constprob \, \, \tfrac{\numobs \, \delcritnewtil^2}{\bou^2
  \newrad^2} \big) \, .
\end{align}
Finally, using the relation $\Prob(\eventA_1) \leq \Prob(\eventA_1
\mid \hitS_1 \leq \hitSbar) + \Prob(\hitS_1 \geq \hitSbar)$ to combine
inequalities~\eqref{eq:prob_hitS}~and~\eqref{eq:PeventA} yields the
claimed bound~\eqref{eq:prob1}.

\paragraph{Proof of inequality~\eqref{eq:prob2}:}

We first note that $\supnorm{f} \leq \bou \, \newrad$ for any function $f \in \funclass(\delcritnewtil)$ and $\supnorm{\termone} \leq \frac{2}{1-\discount} \, \bou \, \newrad$ due to bound \eqref{eq:supnorm_termone}, therefore,
\begin{align*}
 |\iidY{0}(f)| = \Big| \sum_{t=1}^{\hittime_1} f(\state_t) \, \termone\big(
 \state_t^{t+\Step} \big) \Big| \leq \hittime_1 \; \supnorm{f}
 \supnorm{\termone} \leq \tfrac{2 \, \bou^2\newrad^2}{1-\discount} \,
 \hittime_1 \, .
\end{align*}
We introduce the shorthand $\hittimelb \defn (1 - \discounteff) \, (1-\discount)
\, \tfrac{\numobsnew \, \delcritnewtil^2}{32 \,
 \bou^2\newrad^2}$. If the sample size is large enough such that
\mbox{$\hittimelb \geq \max\{ 2 \, \Step, \; 20 \; \mixtime \log
 \mixtime \}$}, then we have
\begin{multline}
 \label{eq:prob0}
 \Prob\Big( \sup\nolimits_{f \in \funclass(\delcritnewtil)}
 \big|\tfrac{1}{\numobsnew} \, \iidY{0}(f) \big| \leq (1
 - \discounteff) \, \tfrac{\delcritnewtil^2}{16} \log \numobs  \Big) \leq \Prob(
 \hittime_1 \geq \hittimelb ) \leq \Prob\big( \hittime_1 - \Step \geq
 \lfloor\hittimelb/2\rfloor \big) \\ \leq ( 1 - \mixtime^{-1} )^{\hittimelb/4} =
 \exp\Big( \tfrac{1}{128} \, \log ( 1 - \mixtime^{-1} ) \cdot (1 - \discounteff) \,
 (1-\discount) \, \tfrac{\numobs \,
 \delcritnewtil^2}{\bou^2\newrad^2} \Big) \, .
 \end{multline}
Since $\log ( 1 - \mixtime^{-1} ) \geq (\mixtime + \Step)^{-1}$, we have $\log ( 1 - \mixtime^{-1} ) \cdot (1 - \discounteff) \, (1-\discount) \, \tfrac{\numobs \, \delcritnewtil^2}{\bou^2\newrad^2} \leq - \,
\Constprob \, \tfrac{\numobs \,
 \delcritnewtil^2}{\bou^2\newrad^2}$. It follows from
inequality~\eqref{eq:prob0} that inequality~\eqref{eq:prob2} holds
with $s = 0$.

We now turn to prove inequality~\eqref{eq:prob2} with $s =
\Tail$. Similar to \eqref{eq:prob0}, we can show that
\begin{align*}
 \Prob\Big( \sup\nolimits_{f \in
 \funclass(\delcritnewtil)}
 \big|\tfrac{1}{\numobsnew} \, \iidY{\Tail}(f) \big|
 \leq \log \numobs \, (1 - \discounteff) \,
 \tfrac{\delcritnewtil^2}{16} \Big) \leq
 \Prob(\numobsnew - \hittime_{\hitS} \geq \hittimelb)
 \, .
\end{align*}
Moreover, we have
\begin{align}
 \Prob(\numobsnew - \hittime_{\hitS} \geq \hittimelb) & \leq
 \ssum{t=\lceil\hittimelb\rceil}{\numobsnew-1}
 \Prob\big(\hittime_{\hitS} = \numobsnew - t, \, \hittime_{\hitS+1} -
 \hittime_{\hitS} \geq t + 1 \big) \leq
 \ssum{t=\lceil\hittimelb/2\rceil}{\infty}
 \Prob\big(\hittime_{\hitS+1} - \hittime_{\hitS} - \Step > t \big)
 \notag \\ 
 & \leq \ssum{t=\lceil\hittimelb/2\rceil}{\infty} ( 1 - \mixtime^{-1} )^t =
\mixtime \, ( 1 - \mixtime^{-1} )^{\lceil\hittimelb/2\rceil} \leq
 \mixtime \, ( 1 - \mixtime^{-1} )^{\hittimelb/4}  \notag \\ 
 & \leq \exp\Big( \tfrac{1}{256}
 \, \log ( 1 - \mixtime^{-1} ) \cdot (1 - \discounteff) \, (1-\discount) \,
 \tfrac{\numobs \, \delcritnewtil^2}{\bou^2\newrad^2}
 \Big)
 \leq \exp\big( - \Constprob \, \tfrac{
 \numobs \, \delcritnewtil^2}{\bou^2\newrad^2} \big) \, , \label{eq:prob_Tail}
\end{align}
which verifies inequality~\eqref{eq:prob2} with $s = \Tail$.


\subsection{Proof of Lemma~\ref{lemma:Term3}}
 \label{sec:proof_Term3}

 Recall that the term $\Term_3$ is given by $\Term_3 = \hilin[\big]{\Deltahat}{(\Gamma - \GammaHat) \, \Deltahat}$ with $\GammaHat = \CovOphat - \CrOpwhat$ and $\Gamma = \CovOp - \CrOpw$.
 By letting
 \begin{align}
   \label{eq:def_gff}
   g[f, f'](\state_{t}^{t+\Step}) \defn f(\state_t) \, f'(\state_t) - \ssum{\step=1}{\Step} \, \weight{\step} \, \discount^{\step} \, f(\state_t) \, f'(\state_{t+\step}) \, ,
 \end{align}
 we can write
 \begin{align*}
 \Term_3 = \frac{1}{\numobsnew} \sum_{t=1}^{\numobsnew} \big\{ \Exp\big[ g[\Deltahat,\Deltahat](\State_{0}^{\Step}) \big] - g[\Deltahat, \Deltahat](\state_{t}^{t+\Step}) \big\} \, .
 \end{align*}
 Given the partition $\{ \Idxset{s} \}_{s = 0}^{\hitS - 1} \cup \{ \Idxset{\Tail} \}$ of indices $[\numobsnew]$ given in equation~\eqref{eq:def_block}, we define random variables
 \begin{align}
 \label{eq:iidXi}
 \iidXi{s}(f,f') \defn \!\! \sum_{t \in \Idxset{s}} \!\! \big\{ g[f,f'](\state_{t}^{t+\Step}) - \Exp\big[ g[f,f'](\State_{0}^{\Step}) \big] \big\}
 \end{align}
 for any $f,f' \in \RKHS$ and $s = 0,1,\ldots,\hitS-1$ and $\Tail$, and reform the term $\Term_3$ as
 \begin{align*}
 \Term_3 = - \frac{1}{\numobsnew} \, \iidXi{0}(\Deltahat,\Deltahat) - \frac{1}{\numobsnew}\sum_{s=0}^{\hitS_1} \iidXi{2s+1}(\Deltahat,\Deltahat) - \frac{1}{\numobsnew}\sum_{s=0}^{\hitS_2} \iidXi{2s+2}(\Deltahat,\Deltahat) - \frac{1}{\numobsnew} \, \iidXi{\Tail}(\Deltahat,\Deltahat) \, ,
 \end{align*}
 where $\hitS_1 \defn \lfloor(\hitS - 2)/2 \rfloor$ and $\hitS_2 \defn \lfloor (\hitS - 3)/2 \rfloor$. For any fixed function $f$, the two   groups of random variables $\{ \iidXi{2s+1}(f,f) \}_{s=0}^{\hitS_1}$ and $\{ \iidXi{2s+2}(f,f) \}_{s=0}^{\hitS_2}$ are i.i.d., respectively.
 
 The following \Cref{lemma:Term3_var} provides bounds on the expectation and variance of the random variable $\iidXi{1}(f,f)$ for any $f \in \RKHS$.
 \begin{lemma} \label{lemma:Term3_var}
 For any function $f \in \RKHS$, we have $\Exp\big[\iidXi{1}(f,f)\big] = 0$. The second moment is upper bounded by
 \begin{align}
 \label{eq:Term3_var}
 \Exp\big[\iidXi{1}^2(f,f)\big] & \leq 24 \; (\mixtime + \Step)^2 \; \bou^2 \hilnorm{f}^2 \, \specfun^2(f) \, \big\{ 1 + \log \tfrac{\bou \, \hilnorm{f}}{\specfun(f)} \big\} \, .
 \end{align}
 \end{lemma}
 \noindent The proof of the claim is shown in \Cref{sec:proof:lemma:Term3_var}. \\

 For scalars $v > 0$, we define function classes $\funclasstil(v) \defn \big\{ f \in \RKHS \bigm| \specfun(f) \leq v, \, \hilnorm{f} \leq \newrad \big\}$ and random variables
 \begin{align}
 & \supHbar^{(\variota)}(v) \defn \sup_{f \in \funclasstil(v)} \Big| \frac{1}{\numobs}\sum_{s=0}^{\hitSbar} \, \iidXi{2s+\variota}(f,f) \Big| \qquad \text{for $\variota \in \{1,2\}$} \, ,
 \end{align}
 where $\hitSbar = \lfloor \numobsnew / (\mixtime + \Step) \rfloor$ is a deterministic number and serves as high-probability upper bounds on random variables $\hitS_1$ and $\hitS_2$.
 Since $\{ \iidXi{2s+1}(f,f) \}_{s=0}^{\hitSbar}$ and $\{ \iidXi{2s+2}(f,f) \}_{s=0}^{\hitSbar}$ are two groups of i.i.d. random variables that are identically distributed, $\supHbar^{(1)}(v)$ and $\supHbar^{(2)}(v)$ have the same distribution.
 Let $\delcritv > 0$ be the smallest positive solution to inequality
 \begin{align}
 \label{eq:def_delcritv}
 \Exp\big[ \supHbar^{(1)}(v) \big] = \Exp\big[ \supHbar^{(2)}(v) \big] \leq \tfrac{1}{160} \, v^2 \, .
 \end{align}
 We show that the radius $\delcritv$ is connected with $\delcrit = \delcrit(\noise)$ via inequality~\eqref{eq:delcritv} in \Cref{lemma:Term1_exp} below.
 \begin{lemma} \label{lemma:Term3_exp}
 There is a universal constant $\const{0} \geq 1$ such that
 \begin{align}
 \label{eq:delcritv}
 \delcritv \leq \delcritvtil \defn \const{0} \, \newrad \sqrt{1-\discounteff} \; \delcrit \, ,
 \end{align}
 where $\delcrit$ is the smallest positive solution to any critical
 inequality~\ref{eq:critineq} with $\noise \geq \noisebase$.
 \end{lemma}
 \noindent See \Cref{sec:proof:lemma:Term3_exp} for the proof. \\

Consider the family of random variables
\begin{align*}
\supH(v) \defn \sup_{f \in \funclasstil(v)} \Big|
\frac{1}{\numobsnew}\sum_{t=1}^{\numobsnew} \,
g[f,f](\state_{t}^{t+\Step}) - \Exp\big[ g[f,f](\State_{0}^{\Step})
  \big] \Big|.
\end{align*}
In~\Cref{lemma:Term1_tail} below, we provide a high probability upper
bound on $\supH\big( \delcritvtil \sqrt{\log \numobs} \big)$.
\begin{lemma}
\label{lemma:Term3_tail}
There is a universal constant $\const{1}, \const{2} > 0$ such that
\begin{align}
\label{eq:Term3_tail}
 \Prob\Big[ \supH\big( \delcritvtil \sqrt{\log \numobs} \big) \geq
   \tfrac{\delcritvtil^2}{2} \log \numobs \Big] \leq \const{1} \,
 \exp\big( - \tfrac{\const{2}}{\mixtime + \Step} \, \tfrac{\numobs \,
   \delcritvtil^2}{\bou^2 \newrad^2} \big) = \const{1} \, \exp\big( -
 \tfrac{\const{0}^2 \, \const{2}}{\mixtime + \Step} \, \tfrac{\numobs
   \, \delcrit^2}{\bou^2} \big) \, .
\end{align}
\end{lemma}
\noindent See \Cref{append:proof_Term3_tail} for the proof. \\
 
Following the same arguments as Lemma~11 in the
paper~\cite{duan2021optimal}, we find that
\mbox{$\supH\big(\delcritvtil \sqrt{\log \numobs}\big) \leq
  \tfrac{\delcritvtil^2}{2} \log \numobs$} implies that
\mbox{$|\Term_3| = \big| \hilin[\big]{\Deltahat}{(\Gammahat -
    \Gamma)\Deltahat} \big| \leq \delcritvtil^2 \, \max \big\{ 1, \,
  \tfrac{\hilnorm{\Deltahat}^2}{\newrad^2} \big\} \log \numobs +
  \tfrac{1}{2} \, \specfun^2(\Deltahat)$,} which finishes the proof
of~\Cref{lemma:Term3}.


\subsubsection{Proof of Lemma~\ref{lemma:Term3_var}}
\label{sec:proof:lemma:Term3_var}

We use $h[f]$ to denote the conditional expectation of function $g[f,
  f] - \Exp\big[ g[f,f](\State_{0}^{\Step}) \big]$. More explicitly,
we take
\begin{align*}
\!\!\!\!\!\!\!\!\!\!\!\!\!\!\! h[f](\state_t) & \defn \Exp\big[ g[f, f](\state_{t}^{t+\Step}) \bigm|
  \state_t \big] - \Exp\big[ g[f,f](\State_{0}^{\Step}) \big] =
f(\state_t) \, \bigg\{ \! \Big( \IdOp - \! \sum_{\step=1}^{\Step}
\weight{\step} \discount^{\step} \, \TransOp^{\step} \Big) f
\bigg\}(\state_t) - \hilin[\big]{f}{\big(\CovOp - \CrOpw\big)f} \, .
 \end{align*}
Using the function $h[f]$, we decompose
$\Exp\big[\iidXi{1}^2(f,f)\big]$ as
\begin{align}
\Exp \big[\iidXi{1}^2(f,f)\big] & = \Exp\bigg[ \Big\{ \ssum{t =
    \hittime_1+1}{\hittime_2} \big( g[f,f](\state_{t}^{t+\Step}) -
  \Exp\big[ g[f,f](\State_{0}^{\Step}) \big] \big) \Big\}^2 \bigg]
\notag \\
\label{eq:EiidXisq}
& = \underbrace{\Exp\bigg[ \Big\{ \ssum{t = \hittime_1+1}{\hittime_2}
    \big( g[f,f](\state_{t}^{t+\Step}) - h[f](\state_t) \big) \Big\}^2
    \bigg]}_{B_1} + \underbrace{\Exp\bigg[ \Big\{ \ssum{t =
      \hittime_1+1}{\hittime_2} h[f](\state_t) \Big\}^2 \bigg]}_{B_2}.
 \end{align}
In the sequel, we bound each of the two terms $B_1$ and $B_2$ in turn.

\paragraph{Bounding the term $B_1$:}

We partition the index set $\{ \hittime_1 + 1, \hittime_1 + 2, \ldots,
\hittime_2 \}$ into $(\Step + 1)$ subsets
 \begin{align*}
   \idxset{\step} \defn \big\{ t = \hittime_1 + \step + (\Step+1)s
   \bigm| s \in \Int, \, \hittime_1 + 1 \leq t \leq \hittime_2 \big\}
   \qquad \text{for $\step = 1,2,\ldots,\Step+1$ .}
 \end{align*}
 Using the Cauchy--Schwarz inequality, we find that
 \begin{align}
 B_1 & = \Exp\bigg[ \Big\{ \sum_{\step = 1}^{\Step + 1} \sum_{t \in \idxset{\step}} \big( g[f,f](\state_{t}^{t+\Step}) - h[f](\state_t) \big) \Big\}^2 \bigg] \notag \\
 & \leq (\Step + 1) \; \sum_{\step=1}^{\Step+1} \, \Exp\bigg[ \Big\{ \sum_{t \in \idxset{\step}} \big( g[f,f](\state_{t}^{t+\Step}) - h[f](\state_t) \big) \Big\}^2 \bigg] \, . \label{eq:B1_1}
 \end{align}
 For any indices $t, t'$ in a same index set $\idxset{\step}$, the random variables \mbox{$\big( g[f,f](\state_{t}^{t+\Step}) - h[f](\state_t) \big)$} and $\big( g[f,f](\state_{t'}^{t'+\Step}) - h[f](\state_{t'}) \big)$ are independent and have zero mean if $t \neq t'$. It then follows from inequality~\eqref{eq:B1_1} that
 \begin{align}
 \!\!\!\!\!\!\! B_1 & \leq (\Step + 1) \sum_{\step=1}^{\Step+1} \Exp\bigg[
   \sum_{t \in \idxset{\step}} \big( g[f,f](\state_{t}^{t+\Step}) -
   h[f](\state_t) \big)^2 \bigg] = (\Step + 1) \; \Exp\bigg[ \sum_{t =
     \hittime_1+1}^{\hittime_2} \big( g[f,f](\state_{t}^{t+\Step}) -
   h[f](\state_t) \big)^2 \bigg] \, . \label{eq:B1_2}
 \end{align}
Using similar arguments as in equation~\eqref{eq:EiidY}, we can show
for any uniformly bounded function \mbox{$\widetilde{f} :
  \StateSp^{\Step+1} \rightarrow \Real$}, it holds that \mbox{$\Exp
  \big[ \sum_{t=\hittime_1+1}^{\hittime_2}
    \widetilde{f}(\state_t^{t+\Step}) \big] = \Exp[\hittime_2 -
    \hittime_1] \; \Exp\big[\widetilde{f}(\State_0^{\Step})\big]$}. As
a corollary, we learn from inequality~\eqref{eq:B1_2} that
 \begin{align}
 B_1 & \leq (\Step + 1) \; \Exp[\hittime_2 - \hittime_1] \; \Exp\Big[ \big( g[f,f](\State_{0}^{\Step}) - h[f](\State_0) \big)^2 \Big] \notag \\
 & \leq (\Step + 1) \, (\mixtime + \Step) \; \Exp_{\State_0 \sim \distr} \big[ \{ g[f, f](\State_{0}^{\Step}) \}^2 \big] \, . \label{eq:B1<} 
 \end{align}

 We next prove the claim
 \begin{align}
 \label{eq:specfun}
 \Exp_{\State_0 \sim \distr} \big[ \{ g[f, f](\State_{0}^{\Step}) \}^2 \big] \leq 2 \, \bou^2 \hilnorm{f}^2 \, \specfun^2(f) \, .
 \end{align}
 Recall the definition of $g[f,f]$ in equation~\eqref{eq:def_gff}. We have the relation 
 \begin{align*}
   g[f,f](\State_0^{\Step}) = f(\State_0) \; \bigg\{ f(\State_0) - \sum_{\step=1}^{\Step} \weight{\step} \discount^{\step} \, f(\State_{\step}) \bigg\} \, .
 \end{align*}
 Note that
 \begin{multline*}
 \Exp_{\State_0 \sim \distr}\bigg[ \Big\{\sum_{\step=1}^{\Step} \weight{\step} \discount^{\step} \, f(\State_{\step}) \Big\}^2 \bigg] = \sum_{\step, \step' = 1}^{\Step} \weight{\step} \, \weight{\step'} \, \discount^{\step + \step'} \; \Exp_{\State_0 \sim \distr}\big[ f(\State_{\step}) \, f(\State_{\step'}) \big] \\
 \stackrel{(*)}{\leq} \sum_{\step, \step' = 1}^{\Step} \weight{\step} \, \weight{\step'} \, \discount^{\step + \step'} \; \Exp_{\State_0 \sim \distr}\big[ f^2(\State_0) \big] = \bigg\{ \sum_{\step=1}^{\Step} \weight{\step} \discount^{\step} \bigg\}^2 \; \Exp\big[ f^2(\State_0) \big] \leq \Exp\big[ f^2(\State_0) \big] \, ,
 \end{multline*}
 where step $(*)$ is due to the Cauchy--Schwarz inequality.
 It then follows that
 \begin{align*}
 & \Exp_{\State_0 \sim \distr}\bigg[ \Big\{ f(\State_0) - \ssum{\step=1}{\Step} \weight{\step} \discount^{\step} \, f(\State_{\step}) \Big\}^2 \bigg] \\
 & \qquad = \Exp\big[ f^2(\State_0) \big] - 2 \, \ssum{\step=1}{\Step} \weight{\step} \discount^{\step} \, \Exp\big[ f(\State_0) \, f(\State_{\step}) \big] + \Exp_{\State_0 \sim \distr}\bigg[ \Big\{\sum_{\step=1}^{\Step} \weight{\step} \discount^{\step} \, f(\State_{\step}) \Big\}^2 \bigg] \\
 & \qquad \leq 2 \, \Exp\big[ f^2(\State_0) \big] - 2 \, \ssum{\step=1}{\Step} \weight{\step} \discount^{\step} \, \Exp\big[ f(\State_0) \, f(\State_{\step}) \big] = 2 \, \specfun^2(f) \, .
 \end{align*}
 For any $f \in \RKHS$, we have $\supnorm{f} \leq \bou \hilnorm{f}$. It further implies
 \begin{align*}
 \Exp_{\State_0 \sim \distr} \big[ \{ g[f, f](\State_{0}^{\Step}) \}^2 \big] = \Exp_{\State_0 \sim \distr}\bigg[ f^2(\State_0) \; \Big\{ f(\State_0) - \ssum{\step=1}{\Step} \weight{\step} \discount^{\step} \, f(\State_{\step}) \Big\}^2 \bigg] \leq 2 \, \bou^2 \hilnorm{f}^2 \, \specfun^2(f) \, ,
 \end{align*}
as claimed in inequality~\eqref{eq:specfun}.

Combining bounds~\eqref{eq:B1<} and \eqref{eq:specfun}, we conclude
that
\begin{align}
\label{eq:B1}
B_1 \leq 2 \, (\Step + 1) \, (\mixtime + \Step) \; \bou^2
\hilnorm{f}^2 \, \specfun^2(f) \, .
\end{align}

 
\paragraph{Bounding $B_2$:}
Similar to the proof of equation~\eqref{eq:iidYsq1}, we can show that
\begin{align*}
\Exp \bigg[ \Big\{ \sum_{t = \hittime_1+1}^{\hittime_2} h[f](\state_t)
  \Big\}^2 \bigg] \big/ \, \Exp[\hittime_2 - \hittime_1] & =
\Exp\bigg[ h[f](\State_0) \; \ssum{t=-\infty}{\infty} h[f](\State_t)
  \bigg] \\ & = \Exp\big[ ( h[f](\State_0) )^2 \big] + 2 \,
\ssum{t=1}{\infty} \Exp\big[ h[f](\State_0) \, \Exp\big[h[f](\State_t)
    \bigm| \State_0 \big] \big] \, .
 \end{align*}
 We apply the Cauchy--Schwarz inequality and find that
 \begin{align}
 \label{eq:B2_0}
 B_2 = \Exp \bigg[ \Big\{ \sum_{t = \hittime_1+1}^{\hittime_2} h[f](\state_t) \Big\}^2 \bigg] 
 & \leq 2 \; \Exp[\hittime_2 - \hittime_1] \; \Exp\big[ ( h[f](\State_0) )^2 \big]^{\frac{1}{2}} \, \sum_{t=0}^{\infty} \Exp\Big[\Exp\big[h[f](\State_t) \bigm| \State_0 \big]^2 \Big]^{\frac{1}{2}} \, .
 \end{align}
 
 Note that for any function $f \in \RKHS$, we have $\supnorm{h[f]} \leq 4 \, \supnorm{f}^2 \leq 4 \, \bou^2 \hilnorm{f}^2$.
 According to the uniform ergodicity property \eqref{eq:geoerg} of Markov chain, we have
 \begin{align}
   \label{eq:B2_1}
   \Exp\Big[\Exp\big[h[f](\State_t) \bigm| \State_0 \big]^2 \Big]^{\frac{1}{2}} \leq \supnorm[\big]{\Exp\big[h[f](\State_t) \bigm| \State_0 \big]} \leq 2 \, ( 1 - \mixtime^{-1} )^t \, \supnorm{h[f]} \leq 8 \, ( 1 - \mixtime^{-1} )^t \, \bou^2 \hilnorm{f}^2 \, .
 \end{align}
 Additionally, since $h[f]$ is the conditional expectation of $g[f, f]$ and $\Exp_{\State_0 \sim \distr} \big[ \{ g[f, f](\State_{0}^{\Step}) \}^2 \big]$ has an upper bound~in inequality~\eqref{eq:specfun}, we have
 \begin{align}
 \label{eq:Ehfsq2}
 \munorm{h[f]}^2 = \Exp\big[ ( h[f](\State_0) )^2 \big] \leq \Exp_{\State_0 \sim \distr} \big[ \{ g[f, f](\State_{0}^{\Step}) \}^2 \big] \leq 2 \, \bou^2 \hilnorm{f}^2 \, \specfun^2(f) \, . 
 \end{align}
 It further implies
 \begin{align}
   \label{eq:B2_2}
   \Exp\Big[\Exp\big[h[f](\State_t) \bigm| \State_0 \big]^2 \Big]^{\frac{1}{2}} \leq \Exp\big[ ( h[f](\State_0) )^2 \big] \leq \sqrt{2} \; \bou \, \hilnorm{f} \, \specfun(f) \, .
 \end{align}
 Combining the bounds~\eqref{eq:B2_1}~and~\eqref{eq:B2_2}, we find that for any $\hittime \in \Natural$,
 \begin{align*}
   \sum_{t=0}^{\infty} \Exp\Big[\Exp\big[h[f](\State_t) \bigm| \State_0 \big]^2 \Big]^{\frac{1}{2}} & \leq \hittime \cdot \sqrt{2} \; \bou \, \hilnorm{f} \, \specfun(f) + \sum_{t=\hittime}^{\infty} 8 \, ( 1 - \mixtime^{-1} )^t \, \bou^2 \hilnorm{f}^2  \\ & = \bou \, \hilnorm{f} \, \big\{ \sqrt{2} \, \hittime \specfun(f) + 8 \, ( 1 - \mixtime^{-1} )^{\hittime} \mixtime \, \bou \, \hilnorm{f} \big\} \, .
 \end{align*}
 By letting $\hittime \defn \mixtime \log \frac{\bou \, \hilnorm{f}}{\specfun(f)}$, we derive that
 \begin{align}
 \label{eq:B2_3}
 \sum_{t=0}^{\infty} \Exp\Big[\Exp\big[h[f](\State_t) \bigm| \State_0 \big]^2 \Big]^{\frac{1}{2}} & \leq \mixtime \; \bou \, \hilnorm{f} \, \specfun(f) \, \big\{ 8 + \sqrt{2} \, \log \tfrac{\bou \, \hilnorm{f}}{\specfun(f)} \big\} \, .
 \end{align}
 
 Substituting the series in the right-hand side of inequality~\eqref{eq:B2_0} with its upper bound~\eqref{eq:B2_3} and applying inequality~\eqref{eq:Ehfsq2}, we find that
 \begin{align}
   \label{eq:B2_5}
   B_2
   & \leq 4 \;  \mixtime \, (\mixtime + \Step) \; \bou^2 \hilnorm{f}^2 \, \specfun^2(f) \, \big\{ 4\sqrt{2} + \log \tfrac{\bou \, \hilnorm{f}}{\specfun(f)} \big\} \, .
 \end{align}
 
 Combining inequalities~\eqref{eq:EiidXisq}, \eqref{eq:B1} and \eqref{eq:B2_5}, we obtain the bound~\eqref{eq:Term3_var} on the second moment $\Exp\big[ \iidXi{1}^2(f,f) \big]$, as stated in \Cref{lemma:Term3_var}.


\subsubsection{Proof of Lemma~\ref{lemma:Term3_exp}}
 \label{sec:proof:lemma:Term3_exp}

 We first note that $\funclasstil(v)$ belongs to the ellipse $\Ellipse(v) \defn \big\{ f \in \RKHS \bigm| \munorm{f} \leq \frac{v}{\sqrt{1-\discounteff}}, \, \hilnorm{f} \leq \newrad \big\}$. Define another rescaled ellipse $\Ellipsetil(v) \defn \tfrac{1}{\newrad} \, \Ellipse(v) = \big\{ f \in \RKHS \bigm| \munorm{f} \leq \frac{v}{\newrad \sqrt{1-\discounteff}}, \, \hilnorm{f} \leq 1 \big\}$. For either \mbox{$\variota = 1$ or $2$}, we have
 \begin{align}
 \label{eq:supHbar0}
 \Exp\big[ \supHbar^{(\variota)}(v) \big] & = \Exp\Bigg[ \sup_{f \in \funclasstil(v)} \Big| \frac{1}{\numobsnew}\sum_{s=0}^{\hitSbar} \, \iidXi{2s+\variota}(f,f) \Big| \Bigg] \leq \frac{\newrad^2}{\numobs} \; \Exp \Bigg[ \sup_{f \in \Ellipsetil(v)} \Big| \sum_{s=0}^{\hitSbar} \, \iidXi{2s+\variota}(f,f) \Big| \Bigg] \, .
 \end{align}
 For any $f \in \Ellipsetil(v)$ with $f = \sum_{j=1}^{\infty} f_j \,
 \feature{j}$, the coefficients $\{ f_j \}_{j=1}^{\infty}$ satisfy
 \mbox{$\sum_{i=1}^{\infty} f_j^2 \leq
   \frac{v^2}{\newrad^2(1-\discounteff)}$} and
 \mbox{$\sum_{j=1}^{\infty} \frac{f_j^2}{\eig{j}} \leq 1$}. Due to the
 linearity of function $\iidXi{s}(f,f')$ in $f$ and $f'$, the sum
 $\sum_{s=0}^{\hitSbar} \, \iidXi{2s+\variota}(f,f)$ can be rewritten
 as
 \begin{align*}
   \sum_{s=0}^{\hitSbar} \, \iidXi{2s+\variota}(f,f) =
   \sum_{i,j=1}^{\infty} f_i \, f_j \sum_{s=0}^{\hitSbar} \,
   \iidXi{2s+\variota}(\feature{i},\feature{j}) \, .
 \end{align*}
Bounding the constrained optimization problem over the sequence $\{
f_j \}_{j=1}^{\infty}$ yields
\begin{align*}
 \sup_{f \in \Ellipsetil(v)} \Big| \sum_{s=0}^{\hitSbar} \,
 \iidXi{2s+\variota}(f,f) \Big|^2 \leq 4 \, \sum_{i,j=1}^{\infty} \!
 \min\big\{ \tfrac{v^2}{\newrad^2 (1-\discounteff)}, \eig{i} \big\}
 \min\big\{ \tfrac{v^2}{\newrad^2 (1-\discounteff)}, \eig{j} \big\} \;
 \bigg\{ \ssum{s=0}{\hitSbar}
 \iidXi{2s+\variota}(\feature{i},\feature{j}) \bigg\}^2 \, .
 \end{align*}
Combining the bound~\eqref{eq:supHbar0} with the Cauchy--Schwarz
inequality yields
\begin{align}
 \Exp \big[ \supHbar^{(\variota)}(v) \big] & \leq \frac{2 \,
   \newrad^2}{\numobs} \; \Exp\Bigg[ \bigg\{ \! \sum_{i,j=1}^{\infty}
   \! \min\big\{ \tfrac{v^2}{\newrad^2 (1-\discounteff)}, \eig{i}
   \big\} \min\big\{ \tfrac{v^2}{\newrad^2 (1-\discounteff)}, \eig{j}
   \big\} \, \Big\{ \ssum{s=0}{\hitSbar}
   \iidXi{2s+\variota}(\feature{i},\feature{j}) \Big\}^2
   \bigg\}\!\bigg.^{\tfrac{1}{2}} \Bigg] \notag \\
 \label{eq:supHbar1} 
 & \leq \frac{2 \, \newrad^2}{\numobs} \, \Bigg\{
 \sum_{i,j=1}^{\infty} \! \min\big\{ \tfrac{v^2}{\newrad^2
   (1-\discounteff)}, \eig{i} \big\} \min\big\{ \tfrac{v^2}{\newrad^2
   (1-\discounteff)}, \eig{j} \big\} \; \Exp\bigg[ \Big\{
   \ssum{s=0}{\hitSbar} \iidXi{2s+\variota}(\feature{i},\feature{j})
   \Big\}^2 \bigg] \Bigg\}^{\tfrac{1}{2}}.
 \end{align}
It only remains to bound the expectation $\Exp\big[ \big\{
  \ssum{s=0}{\hitSbar} \iidXi{2s+\variota}(\feature{i},\feature{j})
  \big\}^2 \big]$ .
 
 According to condition~\eqref{eq:def_bou}, we have $\supnorm{\feature{j}} \leq \unibou$. It then follows from the definition~\eqref{eq:def_gff} of function $g$ that
 $\supnorm[\big]{g[\feature{i}, \feature{j}]} \leq 2 \, \supnorm{\feature{i}} \supnorm{\feature{j}} \leq 2 \, \unibou^2$.
 Recall the expression for random variable $\iidXi{1}(\feature{i}, \feature{j})$ in equation~\eqref{eq:iidXi}. We find that
 \begin{align*}
 \Exp \big[ \iidXi{1}^2(\feature{i},\feature{j}) \big]
 & \leq 4 \, \supnorm[\big]{g[\feature{i},\feature{j}]}^2 \, \Exp \big[ (\hittime_{s+1} - \hittime_s)^2 \big] \leq 32 \, \unibou^4 \, (\mixtime + \Step)^2 \, ,
 \end{align*}
 where the second inequality follows from the property of geometric random variable \mbox{$(\hittime_{s+1} - \hittime_s - \Step)$}.
 By our construction of the regenerated chain, $\big\{ \iidXi{2s+\variota}(\feature{i},\feature{j}) \big\}_{s=0}^{\hitSbar}$ are i.i.d. random variables, therefore,
 \begin{align}
 \label{eq:supHbar2}
 \Exp\bigg[ \Big\{ \sum_{s=0}^{\hitSbar} \iidXi{2s+\variota}(\feature{i},\feature{j})\Big\}^2 \bigg]
 & = (\hitSbar + 1) \, \Exp\big[ \iidXi{1}^2(\feature{i},\feature{j}) \big] \leq 64 \, \unibou^4 \, \numobs \, (\mixtime + \Step) \, ,
 \end{align}
 where we have used the definition $\hitSbar = \lfloor \numobsnew / (\mixtime + \Step) \rfloor$.
 
 Combining the inequalities~\eqref{eq:supHbar1}~and~\eqref{eq:supHbar2}, we  conclude that for any scalar $v > 0$ and \mbox{$\variota = 1$ or $2$},
 \begin{align}
 \label{eq:supHbar3}
 \Exp\big[ \supHbar^{(\variota)}(v) \big]
 & \leq 16 \, \unibou^2 \, \newrad^2 \sqrt{\frac{\mixtime + \Step}{\numobs}} \; \sum_{j=1}^{\infty} \min\big\{ \tfrac{v^2}{\newrad^2 (1-\discounteff)}, \eig{j} \big\} \, .
 \end{align}
 In what follows, we derive an upper bound on the solution $\delcritv$
 to critical inequality~\eqref{eq:def_delcritv}, based on the sample
 size constraint~\eqref{EqnSampleLowerBound}.
 
 Recall that $\delcrit = \delcrit(\noise)$ is the smallest positive solution to inequality~\ref{eq:critineq}. The solution~$\delcritv$ to inequality~\eqref{eq:def_delcritv} then satisfies
 \begin{align*}
 \tfrac{1}{160} \, \delcritv^2 = \Exp\big[
   \supHbar^{(\variota)}(\delcritv) \big] & \leq 16 \, \unibou^2 \,
 \newrad^2 \sqrt{\frac{\mixtime + \Step}{\numobs}} \;
 \sum_{j=1}^{\infty} \min\big\{ \tfrac{\delcritv^2}{\newrad^2
   (1-\discounteff)}, \eig{j} \big\} \\ & \leq \frac{16 \, \newrad^2
   \sqrt{(\mixtime + \Step) \, \numobs}}{(1-\discounteff) \, \noise^2}
 \; \newrad^2 (1-\discounteff) \, \Big\{ \delcrit^2 +
 \frac{\delcrit\delcritv}{\newrad\sqrt{1-\discounteff}} \Big\}^2 \\
 & \leq \frac{\newrad^2 (1-\discounteff)}{320} \, \Big\{ \delcrit +
 \frac{\delcritv}{\newrad\sqrt{1-\discounteff}} \Big\}^2 \, ,
 \end{align*}
where the last inequality follows from
condition~\eqref{EqnSampleLowerBound} on sample size $\numobs$.
Solving the inequality, we find that $\delcritv \leq \delcritvtil
\defn \const{0} \, \newrad \sqrt{1-\discounteff} \; \delcrit$ for
$\const{0} = \sqrt{2} + 1$, as stated in \Cref{lemma:Term3_exp}.



\subsubsection{Proof of Lemma~\ref{lemma:Term3_tail}}
 \label{append:proof_Term3_tail}

In order to prove the non-asymptotic bound~\eqref{eq:Term3_tail} on
$\supH\big( \delcritvtil \sqrt{\log \numobs} \big)$, we decompose
$\supH\big( \delcritvtil \sqrt{\log \numobs} \big)$ into $4$ parts and
bound them separately. In particular, for any scalar $v \geq 0$, we
have
\begin{align}
 \label{eq:supH}
 & \supH(v) \leq \sup_{f \in \funclasstil(v)} \Big|
 \frac{1}{\numobsnew} \, \iidXi{0}(f,f) \Big| + \sum_{\variota \in
   \{1,2\}}\sup_{f \in \funclasstil(v)} \Big| \frac{1}{\numobsnew}
 \sum_{s=0}^{\hitS_{\variota}} \iidXi{2s+\variota}(f,f) \Big| +
 \sup_{f \in \funclasstil(v)} \Big| \frac{1}{\numobsnew} \,
 \iidXi{\Tail}(f) \Big| \, .
 \end{align}
 For $\variota = 1$ or $2$, we say event $\eventB_{\variota}$ happens if the following inequality holds true:
 \begin{align}
   \label{eq:def_event1}
   \sup_{f \in \funclasstil(\delcritvtil \sqrt{\log \numobs})} \Big| \, \frac{1}{\numobsnew} \sum_{s=0}^{\hitS_{\variota}} \iidXi{2s+\variota}(f,f) \, \Big| \geq 20 \, \Exp\big[ \supHbar^{(\variota)}(\delcritvtil) \big] \log \numobs + \tfrac{\delcritvtil^2}{16} \log \numobs \, .
 \end{align}
 We will show that there exists a universal constant $\plaincon > 0$ such that for $\variota = 1$ or $2$,
 \begin{subequations}
 \begin{align}
 \Prob(\eventB_{\variota})
 \leq \plaincon \, \exp\big( \! - \tfrac{\plaincon^{-1}}{\mixtime + \Step} \, \tfrac{\numobs \, \delcritvtil^2}{\bou^2 \newrad^2} \big) \, ; \label{eq:Term3_prob1}
 \end{align}
 and for $s = 0$ or $\Tail$,
 \begin{align}
 \Prob\bigg( \sup_{f \in \funclasstil(\delcritvtil \sqrt{\log \numobs})} \Big|\frac{1}{\numobsnew} \, \iidXi{s}(f,f) \Big| \leq  \tfrac{\delcritvtil^2}{16} \log \numobs \bigg) \leq \exp\big( - \tfrac{\plaincon^{-1}}{\mixtime + \Step} \, \tfrac{\numobs \, \delcritvtil^2}{\bou^2 \newrad^2} \big) \, . \label{eq:Term3_prob2}
 \end{align}
 \end{subequations}
 Given that inequalities~\eqref{eq:Term3_prob1}~and~\eqref{eq:Term3_prob2} are satisfied, we can prove inequality~\eqref{eq:Term3_tail} in \Cref{lemma:Term1_tail}.
 According to our choice of radius $\delcritvtil$ in \Cref{lemma:Term1_exp}, we have $\Exp\big[ \supHbar^{(1)}(\delcritvtil) \big] = \Exp\big[ \supHbar^{(2)}(\delcritvtil) \big] \leq \tfrac{1}{160} \, \delcritvtil^2$.
 It follows from the decomposition~\eqref{eq:supH} and the union bound that
 \begin{align*}
 \supH\big(\delcritvtil \sqrt{\log \numobs} \big) \leq 20  \, \big\{ \Exp\big[ \supHbar^{(1)}(\delcritvtil) \big] + \Exp\big[ \supHbar^{(2)}(\delcritvtil) \big]\big\} \log \numobs + \tfrac{\delcritvtil^2}{4} \log \numobs \leq \tfrac{\delcritvtil^2}{2} \log \numobs
 \end{align*}
 with probability at least $1 - 2 \, (1+\plaincon) \, \exp\big( - \frac{\plaincon^{-1}}{\mixtime + \Step} \, \tfrac{\numobs \, \delcritvtil^2}{\bou^2\newrad^2} \big)$, which establishes inequality~\eqref{eq:Term3_tail} in \Cref{lemma:Term3_tail}. \\

 It only remains to prove inequalities~\eqref{eq:Term3_prob1}~and~\eqref{eq:Term3_prob2}.We will deal with them in turn.

\paragraph{Proof of inequality~\eqref{eq:Term3_prob1}:}

 Due to the symmetry of random variables $\{ \iidXi{2s+1}(f,f) \}_{s = 0}^{\hitS_1}$ and $\{ \iidXi{2s+2}(f,f) \}_{s = 0}^{\hitS_2}$, we only need to consider $\variota = 1$.
 
 Notice that the upper limit of summation in equation~\eqref{eq:def_event1} is a random variable \mbox{$\hitS_1 = \lfloor (\hitS - 2)/2 \rfloor$}. With high probability, $\hitS_1$ is upper bounded by a deterministic scalar $\hitSbar \defn \lfloor \numobsnew/(\mixtime + \Step) \rfloor$. Based on the observation, we conduct a decomposition of the probability $\Prob(\eventB_{1})$ and leverage the following inequality to derive an upper bound:
 \begin{align}
   \label{eq:PeventA_decomp}
   \Prob(\eventB_1) \leq \Prob\big( \eventB_{1} \bigm| \hitS_{1} \leq \hitSbar \big) + \Prob(\hitS_{1} > \hitSbar) \, .
 \end{align}
 Recall that we have proved in inequality~\eqref{eq:prob_hitS} an exponential tail bound on $\Prob(\hitS_{1} > \hitSbar)$. In the sequel, it only remains to control the conditional probability $\Prob\big( \eventB_{1} \bigm| \hitS_{1} \leq \hitSbar \big)$.
 
 We first note that if $\log \numobs \geq 1$, then
 \begin{align*}
   \Exp\big[ \supHbar^{(1)}\big( \delcritvtil \sqrt{\log \numobs} \big) \big] \leq \Exp\big[ \supHbar^{(1)}(\delcritvtil) \big] \log \numobs \, ,
 \end{align*}
 therefore, the probability $\Prob\big(\eventB_{\variota} \bigm| \hitS \leq \hitSbar\big)$ has an upper bound
 \begin{align*}
   \Prob\big(\eventB_1 \bigm| \hitS \leq \hitSbar\big) & \leq \Prob\bigg( \! \sup_{f \in \funclasstil(\delcritvtil \sqrt{\log \numobs})} \Big| \, \frac{1}{\numobsnew} \sum_{s=0}^{\hitS_{1}} \iidXi{2s+1}(f,f) \, \Big| \geq 20 \, \Exp\big[ \supHbar^{(1)}(\delcritvtil \sqrt{\log \numobs}) \big] + \tfrac{\delcritvtil^2}{16} \log \numobs \biggm| \hitS_1 \leq \hitSbar \bigg) \, .
 \end{align*}
 
We next leverage Corollary~4 from the
paper~\cite{montgomery1993comparison} to replace the random variable
$\hitS_1$ in the upper limit of summation by the scalar $\hitSbar$.
Let $\mathcal{L}(\RKHS)$ denote the collection of bounded linear
mappings on RKHS~$\RKHS$.  It can be seen that the mapping
\mbox{$\mathcal{L}(\RKHS) \rightarrow \Real: G \mapsto
  \sup\nolimits_{f \in \funclasstil(\delcritvtil \sqrt{\log \numobs})}
  \hilin{f}{G \, f}$} is a metric on space~$\mathcal{L}(\RKHS)$. It
follows from Corollary~4 in the paper~\cite{montgomery1993comparison}
that
\begin{multline}
\Prob \big(\eventB_1 \bigm| \hitS \leq \hitSbar\big) \\ \leq \Prob
\bigg( \sup_{0 \leq \hitStil \leq \hitSbar} \; \sup_{f \in
  \funclasstil(\delcritvtil \sqrt{\log \numobs})} \Big| \,
\frac{1}{\numobsnew} \sum_{s=0}^{\hitStil} \, \iidXi{2s+1}(f,f) \,
\big| \geq 20 \, \Exp\big[\supHbar^{(1)}\big(\delcritvtil \sqrt{\log
    \numobs}\big)\big] + \tfrac{\delcritvtil^2}{16} \log \numobs
\biggm| \hitS \leq \hitSbar \bigg) \\
 \label{eq:Term3_prob11}
\leq 10 \; \Prob\Big( \supHbar^{(1)}\big(\delcritvtil \sqrt{\log
  \numobs}\big) \geq 2 \, \Exp\big[\supHbar^{(1)}\big(\delcritvtil
  \sqrt{\log \numobs} \big)\big] + \tfrac{\delcritvtil^2}{160} \log
\numobs \Big).
 \end{multline}
Recall that $\supHbar^{(1)}\big(\delcritvtil \sqrt{\log \numobs}\big)
= \sup_{f \in \funclasstil(\delcritvtil\sqrt{\log \numobs})} \big|
\frac{1}{\numobs}\sum_{s=0}^{\hitSbar} \, \iidXi{2s+1}(f,f) \big|$.
 
We now apply Theorem~4 in the paper~\cite{adamczak2008tail}, which
gives a Talagrand-type inequality for sub-exponential random
variables, to bound the probability on the right-hand side of
equation~\eqref{eq:Term3_prob11}. Recall from \Cref{lemma:Term3_var}
that $\{ \iidXi{2s+1}(f) \}_{s=0}^{\hitSbar}$ are i.i.d. and $\Exp[
  \iidXi{1}(f) ] = 0$. The inequality~\eqref{eq:Term3_var} ensures
that
 \begin{align*}
 \sup\nolimits_{f \in \funclasstil(\delcritvtil \sqrt{\log \numobs})} \Exp\big[ (\iidXi{1}(f,f))^2 \big] & \leq 24 \; (\mixtime + \Step)^2 \; \bou^2 \newrad^2 \; \delcritvtil^2 \log \numobs \, \big\{ 1 + \log \tfrac{\bou \, \newrad}{\delcritvtil \sqrt{\log \numobs}} \big\} \\
  & \lesssim (\mixtime + \Step)^2 \; \bou^2 \newrad^2 \; \delcritvtil^2 \, \log^2 \numobs \, .
 \end{align*}
 Moreover, since $\supnorm{f} \leq \bou\newrad$ for any $f \in \funclasstil(\delcritvtil \sqrt{\log \numobs})$ and $\psinorm{1}{\hittime_2 - \hittime_1} \lesssim \mixtime + \Step$, we have
 \begin{align*}
 & \psinorm[\big]{1}{\sup\nolimits_{f \in \funclasstil(\delcritvtil \sqrt{\log \numobs})} |\iidXi{1}(f,f)|} \leq 4 \, \bou^2\newrad^2 \, \psinorm{1}{\hittime_2 - \hittime_1} \lesssim \bou^2\newrad^2 \, (\mixtime + \Step) \, .
 \end{align*}
 Therefore, we have
 \begin{align}
   \Prob\Big( \supHbar^{(1)}\big(\delcritvtil \sqrt{\log \numobs}\big)
   \geq 2 \, \Exp\big[\supHbar^{(1)}\big(
     \delcritvtil \sqrt{\log \numobs} \big) \big] + \tfrac{\delcritvtil^2}{160} \log \numobs
   \Bigm| \hitS \leq \hitSbar \Big) \leq 4 \, \exp\bigg( - \frac{\plaincon^{-1}}{\mixtime + \Step} \, \frac{\numobs \,
     \delcritvtil^2}{\bou^2 \newrad^2} \bigg) \label{eq:Term3_prob12}
 \end{align}
for some universal constant $\plaincon > 0$. Combining
inequalities~\eqref{eq:Term3_prob11}, \eqref{eq:Term3_prob12} and
\eqref{eq:prob_hitS} yields the claimed bound~\eqref{eq:Term3_prob1}.

\paragraph{Proof of inequality~\eqref{eq:Term3_prob2}:}
For any $f \in \funclasstil\big(\delcritvtil \sqrt{\log \numobs} \big)$, we have $\supnorm{f} \leq \bou \, \newrad$, which implies
\begin{align}
\label{eq:iidXi0}
|\iidXi{0}(f,f)| \leq 2 \, \bou^2\newrad^2 \, \hittime_1 \qquad \text{and} \qquad |\iidXi{\Tail}(f,f)| \leq 2 \, \bou^2 \newrad^2 \, (\numobsnew - \hittime_{\hitS}) \, .
\end{align}
We then derive high-probability upper bounds on the hitting time $\hittime_1$ and $\numobsnew - \hittime_{\hitS}$ using arguments similar to the proof of inequality~\eqref{eq:prob2}. The only difference is to replace $\hittimelb$ in inequalities~\eqref{eq:prob0}~and~\eqref{eq:prob_Tail} with \mbox{$\hittimelb' \defn \tfrac{\numobsnew \, \delcritvtil^2}{32 \, \bou^2\newrad^2}$}. Combining the bounds on $\hittime_1$ and $\numobsnew - \hittime_{\hitS}$ with inequality~\eqref{eq:iidXi0}, we establish inequality~\eqref{eq:Term3_prob2} for \mbox{$s = 0$ and $\Tail$}.



\section{Auxiliary results for minimax lower bound}

We provide some auxiliary results for minimax lower bound in this part.
In \Cref{sec:proof:lb_statement}, we prove that conditions~\eqref{cond:norm_ratio}~and~\eqref{cond:mixtimebar} are natural and mild.
In \Cref{sec:lb_construction}, we present the
constructions of MRPs and RKHS space used in the lower bound proof.

\subsection{Proof of the claims in Section~\ref{sec:lb_statement}}
\label{sec:proof:lb_statement}
We first consider the upper bound~$(ii)$ on $\lbnormperp$ in equation~\eqref{cond:norm_ratio}.
For any MRP instances satisfying \mbox{$\stdfun^2(\Vstar) \leq \Var_{\distr}[\reward]$} and the $\Lmu$-geometric ergodicity condition~\eqref{assump:Lmu_ergo} with parameter $(1 - \mixtimebar^{-1})$, we apply inequalities~\eqref{eq:Vstarperp<stdfun}~and~\eqref{eq:Vstarperp<reward} and find that
\begin{align*}
\munorm{\Vstarperp} \lesssim \stdfun(\Vstar) \, \min\big\{ (1-\discount)^{-1}, \sqrt{\mixtimebar} \big\} \, .
\end{align*}
Therefore, the constraint~$(ii)$ in inequality~\eqref{cond:norm_ratio} is quite natural.

We next show that the bound~$(i)$ in equation~\eqref{cond:norm_ratio} and the condition~\eqref{cond:mixtimebar} only preclude instances with $\lbstdmtg \gtrsim \lbapproxerr$.
Suppose the lower bound~$(i)$ on $\lbnormperp$ breaks, then
\begin{align*}
\lbstdmtg = (1 - \discount)^{-1} \stdfunbar \gtrsim
\lbnormperp\sqrt{\numobs/\statdim} \stackrel{(*)}{\gtrsim} \lbnormperp
\, \lbnoise / (\newradbar \, \delcrit) \stackrel{(**)}{\gtrsim}
\sqrt{\mixtimebar} \, \lbnormperp = \lbapproxerr \, ,
\end{align*}
where the step~$(*)$ is ensured by the critical
inequality~\eqref{eq:CI_lb}; the step~$(**)$ follows from the sample
size condition~\eqref{cond:n>}. Therefore, the violation of
bound~$(i)$ in equation~\eqref{cond:norm_ratio} implies $\lbstdmtg
\gtrsim \lbapproxerr$.  If condition~\eqref{cond:mixtimebar} fails to
hold, then due to the bound $\lbnormperp \lesssim \sqrt{\mixtimebar}
\, \stdfunbar$ in condition~\eqref{cond:norm_ratio}, we also have
\begin{align*}
\lbapproxerr = \sqrt{\mixtimebar} \, \lbnormperp \lesssim \mixtimebar \, \stdfunbar \leq (1 - \discount)^{-1} \stdfunbar = \lbstdmtg \, .
\end{align*}
We conclude that the lower bound~$(i)$ in condition~\eqref{cond:norm_ratio} and condition~\eqref{cond:mixtimebar} are not stringent either.


\subsection{Constructions for the minimax lower bound}
\label{sec:lb_construction}

In the sequel, we present the detailed constructions of the MRP instances $\{ \MRP_{\idxpack} \}_{\idxpack \in [\PackNum]}$ and the RKHS~$\RKHS$ that are used in the lower bound proof. 

\subsubsection{Construction of MRPs $\{ \MRP_{\idxpack} \}_{\idxpack \in [\PackNum]}$}
We construct a group of MRP instances $\{ \MRP_{\idxpack} \}_{\idxpack
  \in [\PackNum]}$ with state space $\StateSp = [0,1)$.  Recall that
  for a given radius $\delcrit > 0$, the associated statistical
  dimension is defined as \mbox{$\statdim \equiv \statdim(\delcrit)
    \defn \max \big\{ j \mid \eig{j} \geq \delcrit^2 \big\}$}. We
  introduce the shorthand $\numitv \defn 2^{\lfloor \log_2 \statdim
    \rfloor - 1}$, and let $\{ \packvec{\idxpack} \}_{\idxpack \in
    [\PackNum]} \subset \{0, 1\}^{\numitv}$ be a
  $\tfrac{1}{4}$-(maximal) packing of the set $\big\{ \packvec{} \in
  \{0, 1\}^{\numitv} \mid \tfrac{1}{\numitv} \ssum{\idxitv=1}{\numitv}
  \pack{}{\idxitv} = \tfrac{1}{2} \big\}$ with respect to the
  (rescaled) Hamming metric
\begin{align}
\label{eq:def_Hamming}
\Hamming( \packvec{\idxpack}, \packvec{\idxpack'}) \defn
\frac{1}{\numitv} \sum_{\idxitv=1}^{\numitv} \big|
\pack{\idxpack}{\idxitv} - \pack{\idxpack'}{\idxitv} \big| \, .
\end{align}
It is known that the packing number satisfies the lower bound
$\log\PackNum \geq \frac{\numitv}{11} \geq \frac{\statdim}{45}$.  Our
construction of each MRP instance $\MRP_{\idxpack}$ is based on the
boolean vector $\packvec{\idxpack}$. \\

We take a partition $\big\{ \interval{\idxitv}{\idxstate} \big\}_{\idxstate \in [3], \, \idxitv\in[\numitv]}$ of the state space $\StateSp = [0,1)$.
For $\idxitv = 1,2, \ldots, \numitv$, let
\begin{align}
\label{eq:def_interval}
\interval{\idxitv}{1} \defn \big[ \tfrac{\idxitv-1}{4 \numitv}, \tfrac{\idxitv}{4 \numitv} \big) \, , \qquad \interval{\idxitv}{2} \defn \big[ \tfrac{1}{4} + \tfrac{\idxitv-1}{4 \numitv}, \tfrac{1}{4} + \tfrac{\idxitv}{4 \numitv} \big) \, \quad \text{and} \quad \interval{\idxitv}{3} \defn \big[ \tfrac{1}{2} + \tfrac{\idxitv-1}{2 \numitv}, \tfrac{1}{2} + \tfrac{\idxitv}{2 \numitv} \big) \, .
\end{align}
For MRP $\MRP_{\idxpack}$, the transitions among intervals $\big\{ \interval{\idxitv}{\idxstate} \big\}_{\idxstate \in [3]}$ follow a local Markov chain $\TransMt_{\idxpack}^{(\idxitv)} \in \Real^{3 \times 3}$. More concretely, we set a parameter $\invmix \defn \mixtimebar^{-1}/8$ and define the transition kernel $\TransOpm{\idxpack}$ as
\begin{align}
\label{eq:def_TransOp}
\TransOpm{\idxpack}(\statenew \mid \state) \defn \begin{cases}
\frac{1}{\big| \interval{\idxitv}{\idxstatenew} \big|} \, \big\{ (1 - \invmix) \, \TransMt_{\idxpack}^{(\idxitv)}(\idxstatenew \mid \idxstate) + \tfrac{\invmix}{\numitv} \, \distrMt_{\idxpack}^{(\idxitv)}(\idxstatenew) \big\} & \text{if $\state \in \interval{\idxitv}{\idxstate}$ and $\statenew \in \interval{\idxitv}{\idxstatenew}$}, \, \\
\frac{1}{\numitv \, \big| \interval{\idxitv}{\idxstatenew} \big|} \; \invmix \, \distrMt_{\idxpack}^{(\idxitv)}(\idxstatenew) & \text{if $\state \notin \interval{\idxitv}{\idxstate}$ and $\statenew \in \interval{\idxitv}{\idxstatenew}$}.
\end{cases}
\end{align}
Here $\distrMt_{\idxpack}^{(\idxitv)} \in \Real^3$ represents the stationary distribution of matrix $\TransMt_{\idxpack}^{(\idxitv)}$, which satisfies equation \mbox{$(\distrMt_{\idxpack}^{(\idxitv)})^{\top} \TransMt_{\idxpack}^{(\idxitv)} = (\distrMt_{\idxpack}^{(\idxitv)})^{\top}$}. $\big| \interval{\idxitv}{i'} \big|$ is a scalar that stands for the length of interval $\interval{\idxitv}{i'}$. We have $\big| \interval{\idxitv}{\idxstatenew} \big| = \begin{cases} \tfrac{\numitv}{4} & \text{if $\idxstatenew=1$ or $2$} \\ \tfrac{\numitv}{2} & \text{if $\idxstatenew=3$} \end{cases}$
by our construction.
The local models $\big\{ \TransMt_{\idxpack}^{(\idxitv)} \big\}_{\idxitv \in [\numitv]}$ above are given by
\begin{subequations}
	\label{eq:def_TransMtmk}
	\begin{align} \TransMt_{\idxpack}^{(\idxitv)} \defn \TransMt\bigg( \frac{\pack{\idxpack}{\idxitv}}{60} \sqrt{\frac{\statdim}{\numobs \, \invmix}}, \; \frac{\pack{\idxpack}{\idxitv}}{60} \sqrt{\frac{\statdim}{\numobs}} \; \bigg)
	\end{align} with
	\begin{align}
	\label{eq:def_TransMt}
	\TransMt(\Dp, \Dq) \defn
	\begin{pmatrix}
	\{ \tfrac{1}{2} - \invmix(1 - \Dp)\} (1 + \Dq) & \{\tfrac{1}{2} - \invmix(1 - \Dp)\}(1 - \Dq) & 2\invmix (1 - \Dp) \\
	\{ \tfrac{1}{2} - \invmix(1 - \Dp)\} (1 + \Dq) & \{\tfrac{1}{2} - \invmix(1 - \Dp)\}(1 - \Dq) & 2\invmix (1 - \Dp) \\
	\invmix(1 + \Dp) (1 + \Dq) & \invmix(1 + \Dp) (1 - \Dq) & 1 - 2 \invmix(1 + \Dp)
	\end{pmatrix}
	\end{align}
	for any scalars $\Dp, \Dq \in\Real$. We plot an illustration of the base model $\TransMt_0 \equiv \TransMt(0, 0)$ in \Cref{fig:3stateMRPbase}.
\end{subequations}

In our construction, the local MRPs $\MRPMt_{\idxpack}^{(\idxitv)} =
\MRP\big( \rewardMt, \TransMt_{\idxpack}^{(\idxitv)}, \discount \big)$
over intervals $\big\{ \interval{\idxitv}{\idxstate} \big\}_{\idxstate
  \in [3]}$ share a common reward function $\rewardMt \in \Real^{3}$
for all $\idxitv \in [\numitv]$. The full-scale reward function takes
the form
\begin{align}
\label{eq:def_reward}
\reward(\state) = \rewardMt(1) \cdot \one \big\{ \state \in
       [0,\tfrac{1}{4}) \big\} + \rewardMt(2) \cdot \one \big\{ \state
         \in [\tfrac{1}{4}, \tfrac{1}{2}) \big\} + \rewardMt(3) \cdot
           \one \big\{ \state \in [\tfrac{1}{2}, 1) \big\} \, ,
\end{align}
where the vector $\rewardMt$ is specified
later~\eqref{eq:def_rewardMt}.

\begin{figure}[ht]
	\begin{center}
		\widgraph{0.4\textwidth}{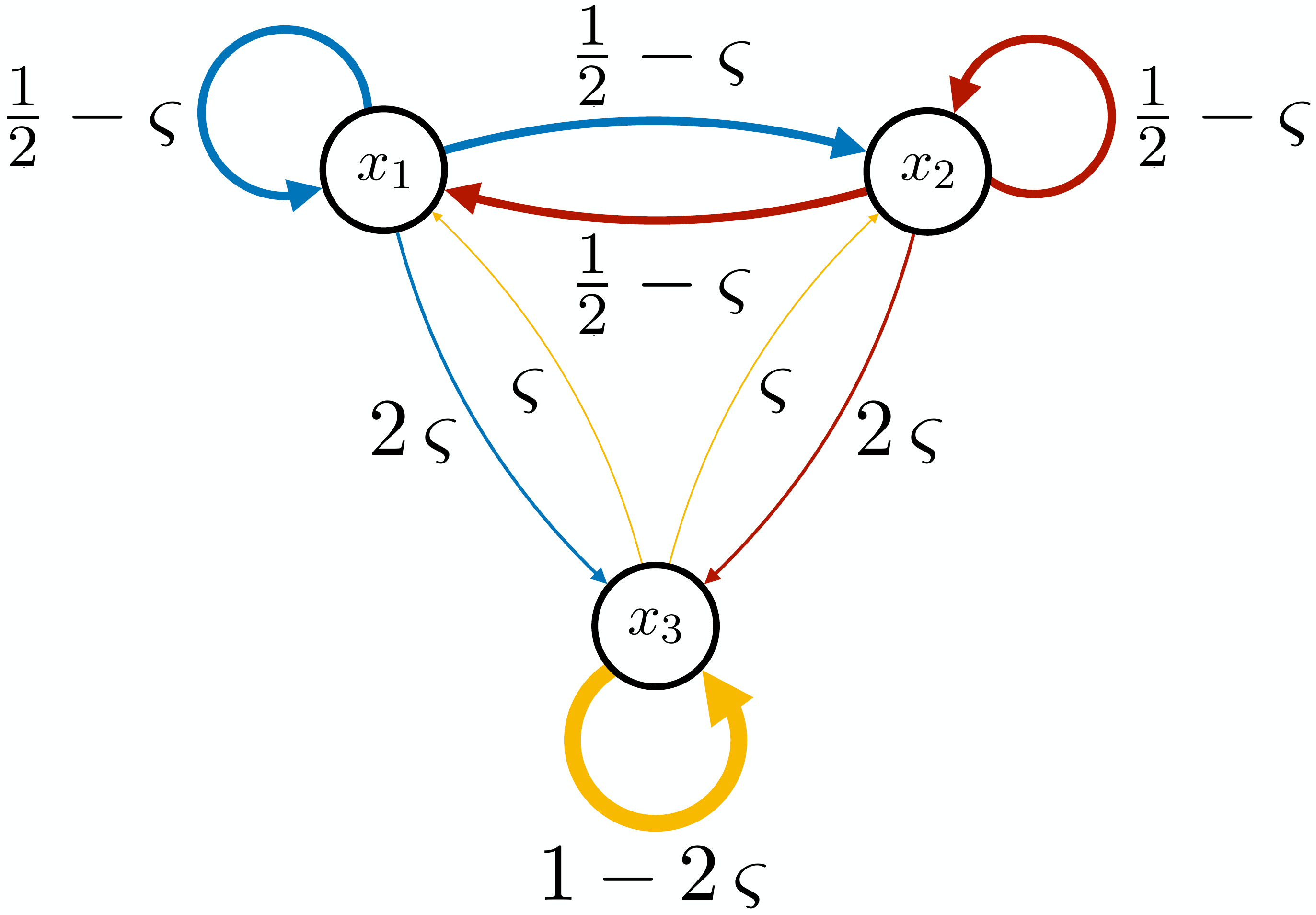}
		\caption{Illustration of the base MRP model $\TransMt_0 \equiv \TransMt(0, 0)$.}
		\label{fig:3stateMRPbase}
	\end{center}
\end{figure}

\subsubsection{Construction of RKHS $\RKHS$}
We construct an RKHS $\RKHS$ that is especially amenable to our analysis. The bases $\{ \feature{j} \}_{j=1}^{\infty}$ of $\RKHS$ is designed based on the Walsh system. Recall that the $j^{th}$ Walsh function is given by
\[ \walsh{j}(x) \defn (-1)^{\sum_{i=0}^{\infty} k_i x_{i+1}} \quad \text{for } j = \sum_{i=0}^{\infty} k_i \, 2^i , ~ x = x_0 + \sum_{i=1}^{\infty} x_i \, 2^{-i} \text{ with $k_i, x_i \in \{0, 1\}$ and $x_0 \in \Int$}. \]
We take bases $\{ \feature{j} \}_{j=1}^{\infty}$ as
\begin{subequations} \label{eq:def_feature}
	\begin{align}
	\feature{2j+1} & \defn \tfrac{1}{2} \big\{ \walsh{2j} - \walsh{2j+1} + \walsh{4j} + \walsh{4j+1} \big\} \, , \\
	\feature{2j+2} & \defn \tfrac{1}{2} \big\{ \walsh{2j} - \walsh{2j+1} + \walsh{4j} + \walsh{4j+1} \big\} \cdot \big\{ \walsh{1} \, \cos\lbtheta \; + \; \tfrac{1}{\sqrt{2}} \, ( \walsh{2} + \walsh{3} ) \, \sin\lbtheta \big\} \, ,
	\end{align}
for $j = 0, 1, 2, \ldots$, where the angle $\lbtheta$ is set as
\begin{align}
\label{eq:def_lbtheta}
\lbtheta \defn \frac{\pi}{2} - \frac{1}{2} \arcsin \Big\{ \frac{4 \, \lbnormperp \, (1 - \discount + 5 \discount\invmix-4 \discount\invmix^2)}{\stdfunbar \; \discount(1-\invmix)(1 - 4 \invmix)} \Big\} \, .
\end{align}
\end{subequations}
Condition~\eqref{cond:norm_ratio} implies that $4 \, \lbnormperp \, (1
- \discount + 5 \discount \invmix - 4\discount\invmix^2) \leq
\stdfunbar \; \discount(1-\invmix)(1 - 4 \invmix)$ so that parameter
$\lbtheta$ is well defined.  Our choice of $\lbtheta$ ensures that the
model mis-specification error is approximately $\lbnormperp$.  Using
the features $\{ \feature{j} \}_{j \in \Int_+}$, we construct a kernel
\begin{align*}
\Ker(\state, \statey) \; \defn \; \sum_{j=1}^{\infty} \; \eig{j} \; \feature{j}(\state) \, \feature{j}(\statey) \; .
\end{align*}
Kernel $\Ker$ has eigen pairs $\{ (\eig{j}, \feature{j}) \}_{j \in \Int_+}$ associated with $\distrbar$. We let $\RKHS$ be the RKHS induced by kernel $\Ker$. For any functions
\mbox{$f = \ssum{j=1}{\infty} f_j \, \feature{j}$} and \mbox{$g = \ssum{j=1}{\infty} g_j \, \feature{j} \in \RKHS$}, their inner product in $\RKHS$ is defined as \mbox{$\hilin{f}{g} \defn \ssum{j=1}{\infty} \eig{j}^{-1} \, f_j g_j$}.

We remark that the basis functions in definition~\eqref{eq:def_feature} is obtained by tensorizing the following feature vectors $\featureMt{1}, \featureMt{2} \in \Real^3$ :
\begin{align}
\label{eq:def_featureMt}
(\featureMt{1}, \, \featureMt{2}) \defn \bUbase
\begin{pmatrix}
1 & 0 \\ 0 & \cos\lbtheta \\ 0 & \sin\lbtheta
\end{pmatrix} \qquad \text{where~~}\bUbase \defn \begin{pmatrix}
1 & 1 & \sqrt{2} \\ 1 & 1 & -\sqrt{2} \\ 1 & -1 & 0
\end{pmatrix} \in \Real^{3 \times 3} \, .
\end{align}
More concretely, for any $j \in \Natural$ and $\variota = 1$ or $2$, we have
\begin{align}
\label{eq:feature_property}
\feature{2j + \variota} =
\begin{cases}
\featureMt{\variota}(1) \cdot \walsh{j}(4\state) & \text{if $\state \in [0,\tfrac{1}{4})$\,,} \\
\featureMt{\variota}(2) \cdot \walsh{j}(4\state-1) & \text{if $\state \in [\tfrac{1}{4}, \tfrac{1}{2})$\,,} \\
\featureMt{\variota}(3) \cdot \walsh{j}(2\state-1) & \text{if $\state \in [\tfrac{1}{2},1)$\,.}
\end{cases}
\end{align}
Thanks to the structures above, the projection onto RKHS $\RKHS$ has simple forms. For any piecewise constant function $f$ adapted to partition $\big\{ \interval{\idxitv}{\idxstate} \big\}_{\idxitv \in [\numitv], \, \idxstate \in [3]}$, we can treat it as $\numitv$ independent models, each defined on intervals $\big\{ \interval{\idxitv}{\idxstate} \big\}_{\idxstate \in [3]}$, and project each local function onto vectors $\featureMt{1}$ and $\featureMt{2}$.
The explicit form~\eqref{eq:feature_property} of features $\{\feature{j}\}_{j=1}^{\infty}$ also shows that $\supnorm{\feature{j}} \leq 2 = \unibou$. \\

We now construct a reward function $\reward$ belonging to the space
$\RKHS$. Define the $3$-dimensional vector
\begin{subequations}
  \begin{align}
  \label{eq:def_rewardMt}
  \rewardMt \defn (\stdfunbar/4) \cdot \featureMt{2} = \bUbase \; \bweightr \in \Real^3 \qquad
  \text{with}~~
  \bweightr \defn \begin{pmatrix}
  0 \\ (\stdfunbar/4) \; \cos\lbtheta \\ (\stdfunbar/4) \; \sin\lbtheta
  \end{pmatrix} \in \Real^3 \, .
  \end{align}
\end{subequations}
Equivalently, we can write $\reward = (\stdfunbar/4) \cdot \feature{2}
\in \RKHS$.


\section{Proof of claims from Section~\ref{sec:proof_lb_overview}}
\label{append:proof_lb}

In this section, we validate that the MRP instances and RKHS
constructed in \Cref{sec:lb_construction} satisfy the inequalities
\eqref{cond:densityratio>}, \eqref{eq:KL<0} and \eqref{eq:valuegap>0}
required in the lower bound proof in \Cref{sec:proof_lb_overview}.

Specifically, we show that our constructions of $\{ \MRP_{\idxpack}
\}_{\idxpack \in [\PackNum]}$ and $\RKHS$ possess some useful
properties, which we formalize in the following claims:
\mystatement{DEF}{claim:well_define}{ The previously described
  constructions ensure that \mbox{$\MRP_{\idxpack} \in
    \MRPclass(\newradbar, \stdfunbar, \lbnormperp, \mixtimebar)$} and
  the density ratio condition~\eqref{cond:densityratio>} holds.  }
\mystatement{KL}{claim:KL}{
  Our construction ensures that
	\begin{align}
	\label{eq:KL<}
	\kull[\big]{\TransOpm{\idxpack}^{1:\numobs}}{\TransOpm{\idxpacknew}^{1:\numobs}} \leq \frac{\statdim}{45} \qquad \text{for any }\idxpack, \idxpacknew \in [\PackNum] \, . \tag{\ref{eq:KL<0}}
	\end{align}
}
\mystatement{GAP}{claim:gap}{
 Our construction ensures that
 \begin{align}
   \label{eq:valuegap>}
   & \min_{\idxpack \neq \idxpacknew} \; \distrnorm[\big]{ \Vstarparm{\idxpack} \, - \, \Vstarparm{\idxpacknew} }{\distrbar}^2 \; \geq \; 8 \, \const{1} \; \newradbar^2 \, \delcrit^2 \qquad \text{for any }\idxpack, \idxpacknew \in [\PackNum] \, , \tag{\ref{eq:valuegap>0}}
 \end{align}
 where $\const{1} > 0$ is a universal constant.
}

The proofs of the claims above are developed in two steps.  The first
step (shown in~\Cref{append:connection}) is to establish the
connections between the full-scale MRP $\MRP_{\idxpack}(\reward,
\TransOpm{\idxpack}, \discount)$ and the local discrete MRPs with
transition kernels $\{ \TransMtper_{\idxpack}^{(\idxitv)} \}_{\idxitv
  \in [\numitv]}$.  We then leverage these connections to prove the
claims. In particular, we show in \Cref{append:proof:lemma:well_defn}
that the MRP models $\{\MRP_{\idxpack}\}_{\idxpack \in [\PackNum]}$
are well defined and they are members of the MRP family $\MRPclass$
satisfying
inequality~\eqref{cond:densityratio>}. In~\Cref{append:proof:lemma:KL},
we establish an upper bound on the KL divergences and prove
inequality~\eqref{eq:KL<0}. In~\Cref{append:proof:lemma:Vgap}, we
provide a lower bound on the value function difference as shown in
inequality~\eqref{eq:valuegap>0}. \Cref{append:proof:lemma:full2discrete}
contains the proof of a lemma stated in~\Cref{append:connection}.



\subsection{Connections with discrete local models}
\label{append:connection}

The Markov chain $\TransOpm{\idxpack}$ defined in
equation~\eqref{eq:def_TransOp} can be reformed as the mixture of two
transition kernels. In particular, we have
\begin{align}
\label{eq:TransOp_cvx}
\TransOpm{\idxpack}(\statenew \mid \state) = (1-\invmix) \,
\TransOptilde_{\idxpack}(\statenew \mid \state) + \invmix \,
\distrm{\idxpack}(\statenew) \, ,
\end{align}
where
\begin{align}
\TransOptilde_{\idxpack}(\statenew \mid \state) \defn \begin{cases}
  \frac{1}{\big| \interval{\idxitv}{\idxstatenew} \big|} \,
  \TransMt_{\idxpack}^{(\idxitv)}(\idxstatenew \mid \idxstate) &
  \text{if $\state \in \interval{\idxitv}{\idxstate}$ and $\statenew
    \in \interval{\idxitv}{\idxstatenew}$}, \, \\ 0 &
  \text{otherwise}.
\end{cases}
\end{align}
The kernel $\TransOptilde_{\idxpack}$ can be viewed as the combination
of a series of local discrete MRPs
$\{\TransMtper_{\idxpack}^{(\idxitv)} \}_{\idxitv \in [\numitv]}$.
Equivalently, we write that
\begin{align}
 \label{eq:TransOp_decomp}
 \TransOpm{\idxpack} = (1-\invmix) \, \TransOptilde_{\idxpack} +
 \invmix \, \one \otimes \one \, ,
\end{align}
where the outer product \mbox{$\one \otimes \one:
  \Ltwo{\distrm{\idxpack}} \rightarrow \Ltwo{\distrm{\idxpack}}$} is
defined as $(\one \otimes \one) f = \big(\int_{\StateSp} f(\state) \,
\distrm{\idxpack}(\dx)\big) \cdot \one$ for any $f \in
\Ltwo{\distrm{\idxpack}}$.

We now establish the connections between $\MRP_{\idxpack}$ and
$\{\MRPMt_{\idxpack}\}_{\idxpack \in [\PackNum]}$ in
\Cref{lemma:full2discrete}, based on
decomposition~\eqref{eq:TransOp_cvx}. The proof is deferred to
\Cref{append:proof:lemma:full2discrete}.
\begin{lemma}
\label{lemma:full2discrete}
\begin{enumerate}
 \item[(a)] \label{lemma:full2discrete_muV} For any indices $\idxpack \in
   [\PackNum]$, we have the relations
 \begin{subequations}
   \begin{align}
\label{eq:distr_full2discrete}      
  \distrm{\idxpack}(\state) & = \tfrac{1}{\numitv \, \big|
    \interval{\idxitv}{\idxstate} \big|} \,
  \distrMt_{\idxpack}^{(\idxitv)}(\idxstate),  \\
\Vstarm{\idxpack}(\state) & = \ValueMt{\idxitv}_{\idxpack}(\idxstate)
+ \scalarone \qquad \text{with~~} 0 \leq \scalarone \leq 10^{-3} \;
\newradbar \, \delcrit \, .
 \label{eq:V_full2discrete}
 \end{align}
 \end{subequations}
 for any $\state \in \interval{\idxitv}{\idxstate}$, where
 $\distrMt^{(\idxitv)}_{\idxpack}$ and $\ValueMt{\idxitv}_{\idxpack}$
 are the stationary distribution and the value function of discrete
 MRP $\MRPMt_{\idxpack} = \MRPMt\big( \rewardMt,
 \TransMt_{\idxpack}^{(\idxitv)}, \discount\,(1-\invmix) \big)$; $\scalarone \equiv \scalarone(\stdfunbar, \mixtimebar, \lbtheta)$ is a deterministic scalar equal for all indices $\idxitv \in [\PackNum]$.
 \item[(b)] \label{lemma:full2discrete_feature} The feature functions $\{
   \feature{j} \}_{j \in \Int_+}$ in definition~\eqref{eq:def_feature}
   satisfy
 \begin{subequations}
 \begin{align}
 & \feature{2j+\variota}(\state) = \walsh{j}\big(
   \tfrac{\idxitv-1}{\numitv} \big) \, \featureMt{\variota}(\idxstate)
   & \qquad & \text{for any $\state \in
     \interval{\idxitv}{\idxstate}$, $\variota \in \{1, 2\}$ and $j <
     \numitv$}; \\ & \int_{\interval{\idxitv}{\idxstate}}
   \feature{j}(\state) \, \dx = 0 & \qquad & \text{for any $j >
     2\numitv$}.
 \end{align}
 \end{subequations}
 \item[(c)] \label{lemma:full2discrete_Vstarpar} The projected value
   function $\Vstarparm{\idxpack}$ in
   definition~\eqref{eq:def_Vstarpar} satisfies
 \begin{align}
 \Vstarparm{\idxpack}(\state) =
 \ValueMtpar{(\idxitv)}{\idxpack}(\idxstate) + \scalarone \quad
 \text{and} \quad \Vstarperpm{\idxpack}(\state) =
 \ValueMtperp{(\idxitv)}{\idxpack}(\idxstate) \qquad \text{for any
   $\state \in \interval{\idxitv}{\idxstate}$},
 \end{align}
where $\ValueMtpar{(\idxitv)}{\idxpack}$ is the projected value
function of discrete MRP $\MRPMt_{\idxpack} = \MRPMt\big( \rewardMt,
\TransMt_{\idxpack}^{(\idxitv)}, \discount\,(1-\invmix) \big)$ onto
linear space $\RKHSMt \defn \Span{\featureMt{1}, \featureMt{2}}
\subset \Real^3$ with respect to stationary distribution
$\distrMtper_{\idxpack}^{(\idxitv)}$; \mbox{$\Vstarperpm{\idxpack} =
  \Vstarm{\idxpack} - \Vstarparm{\idxpack}$}. The features
$\featureMt{1}$ and $\featureMt{2}$ are defined in
equation~\eqref{eq:def_featureMt}. Parameter $\scalarone \geq 0$ is
the same scalar that appears in equation~\eqref{eq:V_full2discrete}.
 \item[(d)] \label{lemma:full2discrete_chidiv} For any indices $\idxpack,
   \idxpacknew \in [\PackNum]$, the $\chi^2$-divergences satisfy
\begin{subequations}
\label{eq:chidiv_full2discrete}
\begin{align}
  \label{eq:chidiv_full2discrete_distr}
  \chidiv[\big]{\distrm{\idxpack}}{\distrm{\idxpacknew}} =
  \frac{1}{\numitv} \sum_{\idxitv = 1}^{\numitv}
  \chidiv[\big]{\distrMtper_{\idxpack}^{(\idxitv)}}{\distrMtper_{\idxpacknew}^{(\idxitv)}}
  \quad \text{and} \quad \chidiv[\big]{\distrm{\idxpack}}{\distrbar}
  = \frac{1}{\numitv} \sum_{\idxitv = 1}^{\numitv}
  \chidiv[\big]{\distrMtper_{\idxpack}^{(\idxitv)}}{\distrMtbase} \, ,
  \\
  \Exp_{\State \sim \distrm{\idxpack}}\big[ \,
    \chidiv[\big]{\TransOptilde_{\idxpack}(\cdot \mid
      \State)}{\TransOptilde_{\idxpacknew}(\cdot \mid \State)} \big] =
  \frac{1}{\numitv} \sum_{\idxitv=1}^{\numitv} \Exp_{\idxstate \sim
    \distrMtper_{\idxpack}^{(\idxitv)}}\big[ \,
    \chidiv[\big]{\TransMtper_{\idxpack}^{(\idxitv)}(\cdot \mid
      \idxstate)}{\TransMtper_{\idxpacknew}^{(\idxitv)}(\cdot \mid
      \idxstate)} \big] \, .
 \end{align}
 \end{subequations}
 \end{enumerate}
 \end{lemma} ~\\

 Recall that the discrete Markov chains are defined by transition matrices
 \begin{align}
 \label{eq:def_TransMtmk2}
 \TransMt_{\idxpack}^{(\idxitv)} = \TransMt\big( \Dp_{\idxpack}^{(\idxitv)}, \, \Dq_{\idxpack}^{(\idxitv)} \big) \qquad \text{where~~} \Dp_{\idxpack}^{(\idxitv)} \defn \frac{\pack{\idxpack}{\idxitv}}{60} \sqrt{\frac{\statdim}{\numobs \, \invmix}} \text{~~and~~} \Dq_{\idxpack}^{(\idxitv)} \defn \frac{\pack{\idxpack}{\idxitv}}{60} \sqrt{\frac{\statdim}{\numobs}} \, .
 \end{align}
 The associated discount factor becomes $\discounttil \defn \discount(1-\invmix)$.
 We can show that the discrete MRPs $\big\{ \TransMt_{\idxpack}^{(\idxitv)} \big\}_{\idxitv \in [\numitv], \, \idxpack \in [\PackNum]}$ given in definition~\eqref{eq:def_TransMtmk} satisfy condition~\eqref{cond:Dp<} in \Cref{lemma:smallMRP_list}, therefore, we can safely apply the estimation results in \Cref{lemma:smallMRP_list}.

 We now check conditions~\eqref{cond:Dp<}. For any $\Dp_{\idxpack}^{(\idxitv)}$ defined by \eqref{eq:def_TransMtmk2}, we have
 \begin{align}
 \label{eq:Dp<}
 0 \leq \Dp_{\idxpack}^{(\idxitv)} \leq \frac{1}{60} \sqrt{\frac{\statdim}{\numobs \, \invmix}} \stackrel{(*)}{\leq} \frac{\newradbar \, \delcrit}{30 \, \big\{ \mixtimebar^{-1/2} (1-\discount)^{-1} \, \stdfunbar + \lbnormperp \big\}} \stackrel{(\#)}{\leq} \frac{\const{0}}{30} \leq \frac{1}{3} \, ,
 \end{align}
 where we have used the critical inequality~\eqref{eq:CI_lb} in
 step~$(*)$; the last step $(\#)$ follows from
 inequality~\eqref{cond:n>}: $\newradbar \, \delcrit \leq \const{0} \,
 \big\{ \mixtimebar^{-1/2} (1-\discount)^{-1} \, \stdfunbar +
 \lbnormperp \big\} $ with $\const{0} \leq 10$.  Under the condition
 \mbox{$\invmix = \frac{1}{8\mixtimebar} \leq \frac{1-\discount}{8}
   \leq \frac{1}{8}$}, inequality~\eqref{eq:Dp<} further implies
 \mbox{$0 \leq \Dq_{\idxpack}^{(\idxitv)} \leq \frac{1}{60}
   \sqrt{\frac{\statdim}{\numobs}} \leq \frac{\sqrt{\invmix}}{3} \leq
   \frac{1}{3}$} and
 \begin{align*}
 0 \leq \Dq_{\idxpack}^{(\idxitv)} \leq \frac{1}{60} \sqrt{\frac{\statdim}{\numobs}} \leq \frac{\sqrt{\invmix}}{3} \leq \frac{1-\discount}{24\sqrt{\invmix}} \leq \frac{1-\discounttil+4\discounttil\invmix}{2\sqrt{2\invmix}} \, .
 \end{align*}
 In addition, we learn from inequality~$\lbnormperp \geq \frac{1}{50}
 \; \stdfunbar \, (1-\discount)^{-1} \sqrt{\frac{\statdim}{\numobs}}$
 in condition~\eqref{cond:norm_ratio} that
 \begin{align*}
 0 \leq \Dq_{\idxpack}^{(\idxitv)} \leq \frac{1}{60}
 \sqrt{\frac{\statdim}{\numobs}} \leq \frac{5 \, \lbnormperp}{6 \,
   \stdfunbar} (1 - \discount) \sqrt{\frac{\statdim}{\numobs}} \leq
 \frac{2 \sqrt{2} \; \lbnormperp}{3 \; \stdfunbar} \, (1 -
 \discounttil + 4 \discounttil \invmix) \, .
 \end{align*}
Combining the pieces, we find that the parameters
$\big(\Dp_{\idxpack}^{(\idxitv)}, \Dq_{\idxpack}^{(\idxitv)}\big)$
satisfy condition~\eqref{cond:Dp<}. \\

In the sequel, we use~\Cref{lemma:full2discrete} as well
as~\Cref{lemma:smallMRP_list} in \Cref{append:smallMRP_properties} to
prove~the claims~\ref{claim:well_define}, \ref{claim:KL} and
\ref{claim:gap}.


\subsection{Proof of claim~\ref{claim:well_define}}
\label{append:proof:lemma:well_defn}

In this part, we show that the MRP models satisfy $\{
\MRP_{\idxpack}\}_{\idxpack \in [\PackNum]} \subset \MRPclass$. In
particular, we prove property~\eqref{eq:lb_constraint_chisqrad}, along
with
properties~\eqref{eq:lb_constraint_newradbar}-\eqref{eq:lb_constraint_mixtimebar}
separately.  In~\Cref{append:proof:lb_constraint_mixtimebar}, we show
that the Markov chain $\TransOpm{\idxpack}$ satisfies the mixing
condition~\ref{assump:mixing} with $\mixtime=\mixtimebar$, as stated
in constraint~\eqref{eq:lb_constraint_mixtimebar}.
\Cref{append:proof:lb_constraint_newradbar} is concerned with the
validation of constraint~\eqref{eq:lb_constraint_newradbar}, which
consists of the Hilbert norm bound on the projected value function
$\Vstarparm{\idxpack}$ as well as the uniform bound on the reward
function $\reward$. In \Cref{append:proof:lb_constraint_stdfunbar}, we
prove that the conditional variance has an upper bound as shown in
inequality~\eqref{eq:lb_constraint_stdfunbar}.  In
\Cref{append:proof:lb_constraint_lbnormperp}, we analyze the model
mis-specification error and prove
constraint~\eqref{eq:lb_constraint_lbnormperp}.  In
\Cref{append:proof:lb_constraint_chisqrad}, we verify that the
stationary distribution~$\distrm{\idxpack}$ is sufficiently close to
Lebesgue measure $\distrbar$ and the \mbox{$\chi^2$-divergence}
satisfies constraint~\eqref{eq:lb_constraint_chisqrad}.


\subsubsection{Minorization condition}
\label{append:proof:lb_constraint_mixtimebar}

The transition kernel $\TransOpm{\idxpack}$ from
equation~\eqref{eq:def_TransOp} is lower bounded as
\mbox{$\TransOpm{\idxpack}(\cdot \mid \state) \geq \invmix \,
  \distrm{\idxpack} = 8 \, \mixtimebar \, \distrm{\idxpack}(\cdot)$}
for any state $\state \in \StateSp$, so that the minorization
condition~\ref{assump:mixing} holds for Markov chains $\{
\TransOpm{\idxpack} \}_{\idxpack=1}^{\PackNum}$ with parameters
$\mixtime = \mixtimebar$ and $\Constdistrnew = 0$.


\subsubsection{Hilbert norm and sup-norm}
\label{append:proof:lb_constraint_newradbar}

In this section, we prove
inequality~\eqref{eq:lb_constraint_newradbar}.  We use the base model
$\MRPbase = \MRP(\reward, \TransOpbase, \discount)$ as a reference and
decompose function $\Vstarparm{\idxpack}$ as $\Vstarparm{\idxpack} =
\Vstarbasepar + (\Vstarparm{\idxpack} - \Vstarbasepar)$. It follows
from triangle inequality that
 \begin{align}
 \label{eq:Vstarpar_triangle}
 \hilnorm[\big]{\Vstarparm{\idxpack}} \leq \hilnorm[\big]{\Vstarbasepar} + \hilnorm[\big]{\Vstarparm{\idxpack} - \Vstarbasepar} \, .
 \end{align}

We first bound the norm of projected value function $\Vstarbasepar$ of
base model $\MRPbase$. Note that $\distrbar$ is the stationary
distribution of $\TransOpbase$ and $\ValueMtpar{(\idxitv)}{\idxpack}$
are the same for all $\idxitv \in [\numitv]$.
Combining~\Cref{lemma:full2discrete}(b)
and~\Cref{lemma:full2discrete}(c)as well as the expression for vector
$\VstarMtbasepar$ in~\Cref{lemma:smallMRP_list}(b) yields
\begin{align*}
\Vstarbasepar = \scalartwo \cdot \feature{2} \qquad \text{where~~} 0
\leq \scalartwo \leq \stdfunbar + \lbnormperp \, \min\big\{ 1, \,
(1-\discount) \sqrt{\mixtimebar} \big\} \, .
\end{align*}
 It then follows from inequality~\eqref{eq:def_newradbar} that
 \begin{align}
   \label{eq:V0_norm}
 \hilnorm[\big]{\Vstarbasepar} & = \tfrac{\scalartwo}{\sqrt{\eig{2}}}
 \leq \tfrac{1}{\sqrt{\eig{2}}} \big( \stdfunbar + \lbnormperp \,
 \min\big\{ 1, \, (1-\discount) \sqrt{\mixtimebar} \big\} \big) \leq
 \tfrac{1}{2} \, \newradbar.
 \end{align}
We now bound the difference term $\big(\Vstarparm{\idxpack} -
\Vstarbasepar\big)$.
Combining~\Cref{lemma:full2discrete}(b)~and~\Cref{lemma:full2discrete}(c),
we find the projected value function $\Vstarparm{\idxpack}$ can be
expressed by a linear combination of bases $\{ \feature{j} \}_{j =
  1}^{2\numitv}$. From our choice of $\numitv$, we have $2\numitv \leq
\statdim$, whence $\eig{2\numitv} \geq \eig{\statdim} \geq
\delcrit^2$. It follows that
 \begin{align*}
 \hilnorm[\big]{\Vstarparm{\idxpack} - \Vstarbasepar} \leq \eig{2
   \numitv}^{-\frac{1}{2}} \, \distrnorm[\big]{\Vstarparm{\idxpack} -
   \Vstarbasepar}{\distrbar} \leq \delcrit^{-1}
 \distrnorm[\big]{\Vstarparm{\idxpack} - \Vstarbasepar}{\distrbar}.
 \end{align*}
We next apply \Cref{lemma:smallMRP_list}(f) to bound
$\distrnorm[\big]{\Vstarparm{\idxpack} - \Vstarbasepar}{\distrbar}$.
From inequality~\eqref{eq:Dp<}, for a sufficiently small constant
$\const{0} > 0$ in condition~\eqref{cond:n>}, we are guaranteed to
have
\begin{align}
\label{eq:Dq^2}
\sqrt{\mixtimebar} \; \sup \nolimits_{\idxitv \in [\numitv]}
\Dq_{\idxpack}^{(\idxitv)} \leq \frac{\sqrt{\mixtimebar}}{60}
\sqrt{\frac{\statdim}{\numobs}} \leq \frac{ \sqrt{\mixtimebar}}{30}
\cdot \const{0} \sqrt{\invmix} = \frac{\const{0}}{60\sqrt{2}} \leq
\consttil{5} \, .
\end{align}
Since $\Dq_0^{(\idxitv)} = 0$ for all $\idxitv \in [\numitv]$,
condition~\eqref{cond:Dq^2} holds.  We then apply
inequality~\eqref{eq:Vgap_discrete} in \Cref{lemma:smallMRP_list}(f)
and find that
\begin{align*}
\distrnorm[\big]{\Vstarparm{\idxpack} - \Vstarbasepar}{\distrbar} &
\leq \supnorm[\big]{\Vstarparm{\idxpack} - \Vstarbasepar} \leq
\sup\nolimits_{\idxitv \in [\numitv]} \,
\supnorm[\big]{\ValueMtpar{(\idxitv)}{\idxpack} - \VstarMtbasepar} +
\scalarone \\ & \leq 10 \, \lbnoise \cdot \sup\nolimits_{\idxitv \in
  [\numitv]} \Dq_{\idxpack}^{(\idxitv)} + 10^{-3} \, \newradbar \,
\delcrit \, .
\end{align*}
Here the parameter $\scalarone$ was defined in
equation~\eqref{eq:V_full2discrete}. The critical
inequality~\eqref{eq:CI_lb} ensures that \mbox{$\sup
  \nolimits_{\idxitv \in [\numitv]} \Dq_{\idxpack}^{(\idxitv)} \leq
  \frac{1}{60} \sqrt{\frac{\statdim}{\numobs}} \leq \frac{\newradbar
    \, \delcrit}{60 \sqrt{2} \; \lbnoise}$,} from which it follows
that
\begin{align*}
\distrnorm[\big]{\Vstarparm{\idxpack} - \Vstarbasepar}{\distrbar} &
\leq \frac{\newradbar \, \delcrit}{8} \stackrel{(\dagger)}{\leq}
  \frac{\const{0}}{8} \, \bou \, \newradbar \, ,
\end{align*}
where step $(\dagger)$ uses condition~\eqref{cond:n>}.  When the constant $\const{0}$ is small enough, we have
\mbox{$\hilnorm[\big]{\Vstarparm{\idxpack} - \Vstarbasepar} \leq $} $
  \delcrit^{-1} \distrnorm[\big]{\Vstarparm{\idxpack} -
    \Vstarbasepar}{\distrbar} \leq \tfrac{1}{2} \, \newradbar$.

Substituting this bound and inequality~\eqref{eq:V0_norm} and into
equation~\eqref{eq:Vstarpar_triangle}, we obtain
\begin{align}
\label{eq:hilnorm_Vstarpar}
\hilnorm[\big]{\Vstarparm{\idxpack}} \leq \delcrit^{-1}
\distrnorm[\big]{\Vstarparm{\idxpack} - \Vstarbasepar}{\distrbar} \leq
\newradbar \, .
\end{align}
Moreover, by definition of the reward function $\reward$ in
equation~\eqref{eq:def_reward}, we have $\supnorm{\reward} =
(\stdfunbar/4)$, whence \mbox{$\tfrac{\supnorm{\reward}}{\bou} =
  \stdfunbar/(4\bou) \leq \newradbar$.}  Combining this bound with
inequality~\eqref{eq:hilnorm_Vstarpar}, we find that the MRP
$\MRP_{\idxpack}$ satisfies
constraint~\eqref{eq:lb_constraint_newradbar}.


\subsubsection{Conditional variance}
\label{append:proof:lb_constraint_stdfunbar}

We now prove inequality~\eqref{eq:lb_constraint_stdfunbar}.
Introducing the shorthand $\Vstartildem{\idxpack} \defn
\Vstarm{\idxpack} - \scalarone \cdot \one$, we have
the decomposition
\begin{align}
\stdfun^2(\Vstarm{\idxpack}) & = \Exp_{\State \sim
  \distrm{\idxpack}}\big[ \, \Var_{\Statenew \sim
    \TransOpm{\idxpack}(\cdot \mid \State)}[ \,
    \Vstarm{\idxpack}(\Statenew) \mid \State \, ] \, \big] =
\Exp_{\State \sim \distrm{\idxpack}}\big[ \, \Var_{\Statenew \sim
    \TransOpm{\idxpack}(\cdot \mid \State)}[ \,
    \Vstartildem{\idxpack}(\Statenew) \mid \State \, ] \, \big] \notag
\\
\label{eq:stdfun}
& = \Exp_{\State \sim \distrm{\idxpack}} \big[ \, \Exp_{\Statenew \sim
    \TransOpm{\idxpack}(\cdot \mid \State)} \big[ \,
    (\Vstartildem{\idxpack}(\Statenew))^2 \mid \State \, \big] -
  (\TransOpm{\idxpack} \Vstartildem{\idxpack})^2(\State) \, \big]
\stackrel{(\dagger)}{=}
\distrnorm{\Vstartildem{\idxpack}}{\distrm{\idxpack}}^2 -
\distrnorm{\TransOpm{\idxpack}
  \Vstartildem{\idxpack}}{\distrm{\idxpack}}^2,
 \end{align}
where step~$(\dagger)$ follows since $\distrm{\idxpack}$ is the stationary
distribution.  Recall the decomposition~\eqref{eq:TransOp_decomp} of
operator~$\TransOpm{\idxpack}$. We can write the squared norm
$\distrnorm{\TransOpm{\idxpack}
  \Vstartildem{\idxpack}}{\distrm{\idxpack}}^2$ as
 \begin{align*}
 \distrnorm{\TransOpm{\idxpack} \Vstartildem{\idxpack}}{\distrm{\idxpack}}^2
 & = \distrnorm[\big]{(1-\invmix) \, \TransOptildem{\idxpack} \Vstartildem{\idxpack} + \invmix \, \Exp_{\distrm{\idxpack}}[\Vstartildem{\idxpack}] \cdot \one}{\distrm{\idxpack}}^2 \\
 & = (1-\invmix)^2 \,\distrnorm[\big]{\TransOptildem{\idxpack} \Vstartildem{\idxpack}}{\distrm{\idxpack}}^2 + 2 \invmix (1-\invmix) \, \Exp_{\distrm{\idxpack}}[\Vstartildem{\idxpack}] \, \inprod[\big]{\TransOptildem{\idxpack} \Vstartildem{\idxpack}}{\one}_{\distrm{\idxpack}} + \invmix^2 \,\big( \Exp_{\distrm{\idxpack}}[\Vstartildem{\idxpack}] \big)^2 \, .
 \end{align*}
 Since $\distrm{\idxpack}$ is also a stationary distribution of
 kerrnel $\TransOptildem{\idxpack}$, we have
 $(\TransOptildem{\idxpack})^{\conj} \one = \one$ where
 $(\TransOptildem{\idxpack})^{\conj}$ is the Hermitian adjoint. It
 follows that $\inprod[\big]{\TransOptildem{\idxpack}
   \Vstartildem{\idxpack}}{\one}_{\distrm{\idxpack}} =
 \inprod[\big]{\Vstartildem{\idxpack}}{(\TransOptildem{\idxpack})^{\conj}\one}_{\distrm{\idxpack}}
 = \inprod[\big]{\Vstartildem{\idxpack}}{\one}_{\distrm{\idxpack}} =
 \Exp_{\distrm{\idxpack}}[ \Vstartildem{\idxpack} ]$, and therefore,
 \mbox{$\distrnorm{\TransOpm{\idxpack}
     \Vstartildem{\idxpack}}{\distrm{\idxpack}}^2 = (1-\invmix)^2
   \,\distrnorm[\big]{\TransOptildem{\idxpack}
     \Vstartildem{\idxpack}}{\distrm{\idxpack}}^2 + \invmix
   (2-\invmix) \big( \Exp_{\distrm{\idxpack}}[\Vstartildem{\idxpack}]
   \big)^2$.}  Equation~\eqref{eq:stdfun} then reduces to
 \begin{align}
 \label{eq:stdfun0}
 \stdfun^2(\Vstartildem{\idxpack}) & = (1-\invmix)^2 \big\{ \distrnorm{\Vstartildem{\idxpack}}{\distrm{\idxpack}}^2 - \distrnorm{\TransOptildem{\idxpack} \Vstartildem{\idxpack}}{\distrm{\idxpack}}^2 \big\} + \invmix(2-\invmix) \, \big\{ \distrnorm{\Vstartildem{\idxpack}}{\distrm{\idxpack}}^2 - \big( \Exp_{\distrm{\idxpack}}[\Vstartildem{\idxpack}] \big)^2 \big\} \, .
 \end{align}

 We apply \Cref{lemma:full2discrete}(a) to connect the full-scale quantities with the counterparts of local discrete models. We find that
 \begin{align*}
 \distrnorm{\Vstartildem{\idxpack}}{\distrm{\idxpack}}^2 = \frac{1}{\numitv} \sum_{\idxitv=1}^{\numitv} \, \distrnorm[\big]{\ValueMt{\idxitv}_{\idxpack}}{\distrMtper_{\idxpack}^{(\idxitv)}}^2 \qquad \text{and} \qquad \distrnorm{\TransOptildem{\idxpack} \Vstartildem{\idxpack}}{\distrm{\idxpack}}^2 = \frac{1}{\numitv} \sum_{\idxitv=1}^{\numitv} \, \distrnorm[\big]{\TransMtper^{(\idxitv)}_{\idxpack}\ValueMt{\idxitv}_{\idxpack}}{\distrMtper_{\idxpack}^{(\idxitv)}}^2 \, .
 \end{align*}
 Similar to the argument~\eqref{eq:stdfun}, we have
 $\stdfun^2\big(\ValueMt{\idxitv}_{\idxpack}\big) =
 \distrnorm[\big]{\ValueMt{\idxitv}_{\idxpack}}{\distrMtper_{\idxpack}^{(\idxitv)}}^2
 -
 \distrnorm[\big]{\TransMtper^{(\idxitv)}_{\idxpack}\ValueMt{\idxitv}_{\idxpack}}{\distrMtper_{\idxpack}^{(\idxitv)}}^2$. It
 then follows from inequality~\eqref{eq:smallMRP_stdfun} in
 \Cref{lemma:smallMRP_list}(c) that $
 \distrnorm{\Vstartildem{\idxpack}}{\distrm{\idxpack}}^2 -
 \distrnorm{\TransOptildem{\idxpack}
   \Vstartildem{\idxpack}}{\distrm{\idxpack}}^2 = \frac{1}{\numitv}
 \sum_{\idxitv=1}^{\numitv}
 \stdfun^2\big(\ValueMt{\idxitv}_{\idxpack}\big) \leq
 \frac{\stdfunbar^2}{3}$. Moreover,
 equation~\eqref{eq:smallMRP_munorm} in~\Cref{lemma:smallMRP_list}(c)
 also ensures that
 $\distrnorm{\Vstartildem{\idxpack}}{\distrm{\idxpack}}^2 \leq
 \frac{\stdfunbar^2}{3 \, \invmix}$.  Substituting these two bounds
 into equality~\eqref{eq:stdfun0} yields 
 \mbox{$\stdfun^2(\Vstarm{\idxpack}) \leq (1-\invmix)^2 \stdfunbar^2/3
   + (2-\invmix) \, \stdfunbar^2/3 \leq \stdfunbar^2$,} which shows
 that the model $\MRP_{\idxpack}$ satisfies
 constraint~\eqref{eq:lb_constraint_stdfunbar}.


\subsubsection{Model mis-specification}
 \label{append:proof:lb_constraint_lbnormperp}
 It is straightforward to see that by \Cref{lemma:full2discrete}(a)
 and \Cref{lemma:full2discrete}(c), we have the upper bound
 \mbox{$\distrnorm{\Vstarperpm{\idxpack}}{\distrm{\idxpack}}^2 =
   \frac{1}{\numitv} \sum_{\idxitv=1}^{\numitv} \,
   \distrnorm[\big]{\ValueMtperp{(\idxitv)}{\idxpack}}{\distrMtper_{\idxpack}}^2
   \leq \lbnormperp^2$,} where the upper bound follows from
 inequality~\eqref{eq:smallMRP_lbnormperp}
 in \Cref{lemma:smallMRP_list}(d).  We can conclude that the MRP model
 $\MRP_{\idxpack}$ satisfies
 constraint~\eqref{eq:lb_constraint_lbnormperp}.


\subsubsection{$\chi^2$-divergence between $\distrm{\idxpack}$ and $\distrbar$}
\label{append:proof:lb_constraint_chisqrad}

Combining inequality~\eqref{eq:chidiv_full2discrete_distr} in
\Cref{lemma:full2discrete}(d) and \eqref{eq:smallMRP_chisq_distr} in
\Cref{lemma:smallMRP_list}(e), we find that
\begin{align*}
\chidiv[\big]{\distrm{\idxpack}}{\distrbar} = \frac{1}{\numitv}
\sum_{\idxitv = 1}^{\numitv}
\chidiv[\big]{\distrMtper_{\idxpack}^{(\idxitv)}}{\distrMtbase} \leq
\frac{1}{\numitv} \sum_{\idxitv = 1}^{\numitv} \, 2 \big\{
\big(\Dp_{\idxpack}^{(\idxitv)}\big)^2 +
\big(\Dq_{\idxpack}^{(\idxitv)}\big)^2 \big\} \leq \frac{\mixtime
  \statdim}{200 \numobs} \, .
 \end{align*}
The critical inequality~\eqref{eq:CI_lb} ensures that
$\frac{\mixtimebar \statdim}{200\numobs} \leq \chisqrad$ for radius
$\chisqrad = \frac{\newradbar^2 \, \delcrit^2}{400 \, \mixtimebar^{-1}
  \, \lbnoise^2}$. Hence, we conclude that
distribution~$\distrm{\idxpack}$ satisfies
constraint~\eqref{eq:lb_constraint_chisqrad}.


\subsection{Proof of claim~\ref{claim:KL}}
\label{append:proof:lemma:KL}

For any trajectory $\traj = (\state_0, \state_1, \ldots,
\state_{\numobs}) \in \StateSp^{\numobs+1}$, we have
\begin{align*}
\TransOpm{\idxpack}^{1:\numobs}( \diff\traj) =
\distrm{\idxpack}(\diff\state_0) \; \prod_{t=1}^{\numobs} \,
\TransOpm{\idxpack}(\diff\state_t \mid \state_{t-1}) \quad \text{and}
\quad \TransOpm{\idxpacknew}^{1:\numobs}( \diff\traj) =
\distrm{\idxpacknew}(\diff\state_0) \; \prod_{t=1}^{\numobs} \,
\TransOpm{\idxpacknew}(\diff\state_t \mid \state_{t-1}) \, .
\end{align*}
Therefore, the KL divergence
$\kull[\big]{\TransOpm{\idxpack}^{1:\numobs}}{\TransOpm{\idxpacknew}^{1:\numobs}}$
admits a decomposition
\begin{align*}
 \kull[\big]{\TransOpm{\idxpack}^{1:\numobs}}{\TransOpm{\idxpacknew}^{1:\numobs}}
 & = \kull[\big]{\distrm{\idxpack}}{\distrm{\idxpacknew}} + \numobs
 \cdot \Exp_{\State \sim \distrm{\idxpack}}\big[
   \kull[\big]{\TransOpm{\idxpack}(\cdot \mid
     \State)}{\TransOpm{\idxpacknew}(\cdot \mid \State)} \big] \, .
\end{align*}
We use $\chi^2$-divergences to control the KL divergences on the
right-hand side. It follows that
\begin{align}
\label{eq:KL<chidiv}  
 \kull[\big]{\TransOpm{\idxpack}^{1:\numobs}}{\TransOpm{\idxpacknew}^{1:\numobs}}
 & \leq \chidiv[\big]{\distrm{\idxpack}}{\distrm{\idxpacknew}} +
 \numobs \cdot \Exp_{\State \sim \distrm{\idxpack}}\big[ \,
   \chidiv[\big]{\TransOpm{\idxpack}(\cdot \mid
     \State)}{\TransOpm{\idxpacknew}(\cdot \mid \State)} \big].
 \end{align}
Recall from equation~\eqref{eq:TransOp_cvx} that the forward
probabilities $\TransOpm{\idxpack}(\cdot \mid \state)$ (or
$\TransOpm{\idxpacknew}(\cdot \mid \state)$) can be written as the
mixture of $\TransOptilde_{\idxpack}(\cdot \mid \state)$ (or
$\TransOptilde_{\idxpacknew}(\cdot \mid \state)$) and stationary
distribution $\distrm{\idxpack}(\cdot)$ (or
$\distrm{\idxpacknew}(\cdot)$). We apply the generalized mean
inequality and obtain that
\begin{align}
\label{eq:chidiv_cvx}
\chidiv[\big]{\TransOpm{\idxpack}(\cdot \mid
  \state)}{\TransOpm{\idxpacknew}(\cdot \mid \state)} \leq 2 \, \big\{
(1-\invmix) \cdot \chidiv[\big]{\TransOptilde_{\idxpack}(\cdot \mid
  \state)}{\TransOptilde_{\idxpacknew}(\cdot \mid \state)} + \invmix
\cdot \chidiv[\big]{\distrm{\idxpack}}{\distrm{\idxpacknew}} \big\}.
\end{align}
Combining inequalities~\eqref{eq:KL<chidiv}~and~\eqref{eq:chidiv_cvx}
yields
\begin{align}
\label{eq:KL<chidiv2}
\kull[\big]{\TransOpm{\idxpack}^{1:\numobs}}{\TransOpm{\idxpacknew}^{1:\numobs}}
\leq ( 1 + 2 \, \numobs \, \invmix ) \cdot
\chidiv[\big]{\distrm{\idxpack}}{\distrm{\idxpacknew}} + 2 \, \numobs
\, (1-\invmix) \cdot \Exp_{\State \sim \distrm{\idxpack}}\big[ \,
  \chidiv[\big]{\TransOptilde_{\idxpack}(\cdot \mid
    \state)}{\TransOptilde_{\idxpacknew}(\cdot \mid \state)} \big].
\end{align}

We apply equations~\eqref{eq:chidiv_full2discrete}
in\Cref{lemma:full2discrete} and the bounds~\eqref{eq:chidiv_discrete}
in \Cref{lemma:smallMRP_list} to estimate the $\chi^2$-divergences in
inequality~\eqref{eq:KL<chidiv2}.  It follows that
\begin{subequations}
\label{eq:chisq<}
\begin{align}
& \chidiv[\big]{\distrm{\idxpack}}{\distrm{\idxpacknew}} \leq
  \frac{2}{\numitv} \sum_{\idxitv=1}^{\numitv} \Big\{
  \big(\Dp_{\idxpack}^{(\idxitv)} -
  \Dp_{\idxpacknew}^{(\idxitv)}\big)^2 +
  \big(\Dq_{\idxpack}^{(\idxitv)} -
  \Dq_{\idxpacknew}^{(\idxitv)}\big)^2 \Big\} \qquad \text{and} \\ &
  \Exp_{\State \sim \distrm{\idxpack}}\big[ \,
    \chidiv[\big]{\TransOptilde_{\idxpack}(\cdot \mid
      \State)}{\TransOptilde_{\idxpacknew}(\cdot \mid \State)} \big]
  \leq \frac{12}{\numitv} \sum_{\idxitv=1}^{\numitv} \Big\{ \invmix \,
  \big(\Dp_{\idxpack}^{(\idxitv)} -
  \Dp_{\idxpacknew}^{(\idxitv)}\big)^2 +
  \big(\Dq_{\idxpack}^{(\idxitv)} -
  \Dq_{\idxpacknew}^{(\idxitv)}\big)^2 \Big\} \, .
 \end{align}
 \end{subequations}
 We further apply definition~\eqref{eq:def_TransMtmk2} of parameters $\big( \Dp_{\idxpack}^{(\idxitv)}, \Dq_{\idxpack}^{(\idxitv)} \big)$ and find that
 \begin{align}
 \label{eq:DpDq<}
 \frac{1}{\numitv} \sum_{\idxitv=1}^{\numitv} \big(\Dp_{\idxpack}^{(\idxitv)} - \Dp_{\idxpacknew}^{(\idxitv)}\big)^2 \leq \frac{\statdim}{60^2 \, \numobs \, \invmix} \qquad \text{and} \qquad
 \frac{1}{\numitv} \sum_{\idxitv=1}^{\numitv} \big(\Dq_{\idxpack}^{(\idxitv)} - \Dq_{\idxpacknew}^{(\idxitv)}\big)^2 \leq \frac{\statdim}{60^2 \, \numobs} \, .
 \end{align}
 Substituting inequalities~\eqref{eq:chisq<}~and~\eqref{eq:DpDq<} into \eqref{eq:KL<chidiv2}, we derive inequality~\eqref{eq:KL<}, as stated in claim~\ref{claim:KL}.


\subsection{Proof of claim~\ref{claim:gap}}
 \label{append:proof:lemma:Vgap}

 In this part, we establish a lower bound on $\distrnorm[\big]{\Vstarparm{\idxpack} - \Vstarparm{\idxpacknew}}{\distrbar}^2$. According to \Cref{lemma:full2discrete}(c), we have
 \begin{align}
 \label{eq:full2discrete_Vgap}
 \distrnorm[\big]{\Vstarparm{\idxpack} - \Vstarparm{\idxpacknew}}{\distrbar}^2 & = \frac{1}{\numitv} \sum_{\idxitv = 1}^{\numitv} \; \distrnorm[\big]{\ValueMtpar{(\idxitv)}{\idxpack} - \ValueMtpar{(\idxitv)}{\idxpacknew}}{\distrMtbase}^2 \, ,
 \end{align}
 where $\distrMtbase = \big( \tfrac{1}{4}, \tfrac{1}{4}, \tfrac{1}{2} \big)^{\top} \in \Real^3$. We learn from inequality~\eqref{eq:Dq^2} that when $\Dq_{\idxpack}^{(\idxitv)} \neq \Dq_{\idxpacknew}^{(\idxitv)} $,
 \begin{align*}
 \sqrt{\mixtimebar} \, \sup\nolimits_{\idxitv \in [\numitv]}\big\{ \big(\Dq_{\idxpack}^{(\idxitv)}\big)^2, \, \big(\Dq_{\idxpacknew}^{(\idxitv)}\big)^2 \big\} = \frac{\sqrt{\mixtimebar} \, \statdim}{60^2 \, \numobs} \leq \frac{\consttil{5}}{60} \sqrt{\frac{\statdim}{\numobs}} = \consttil{5} \, \big| \Dq_{\idxpack}^{(\idxitv)} - \Dq_{\idxpacknew}^{(\idxitv)} \big| \, ,
 \end{align*}
 therefore, condition~\eqref{cond:Dq^2} in \Cref{lemma:smallMRP_list}(f) is satisfied. We then apply inequality~\eqref{eq:Vgap_discrete} and find that
 \begin{align*}
 \distrnorm[\big]{\ValueMtpar{(\idxitv)}{\idxpack} - \ValueMtpar{(\idxitv)}{\idxpacknew}}{\distrMtbase} & \geq \const{5} \; \big\{ (1 - \discount)^{-1} \, \stdfunbar + \sqrt{\mixtimebar} \, \lbnormperp \big\} \, \big| \Dq_{\idxpack}^{(\idxitv)} - \Dq_{\idxpacknew}^{(\idxitv)} \big| \notag \\ & = \frac{\const{5}}{60} \; \big\{ (1 - \discount)^{-1} \, \stdfunbar + \sqrt{\mixtimebar} \, \lbnormperp \big\} \, \sqrt{\frac{\statdim}{\numobs}} \cdot \one\big\{ \pack{\idxpack}{\idxitv} \neq \pack{\idxpacknew}{\idxitv} \big\} \, .
 \end{align*}
 The regularity condition~\eqref{cond:regular} on kernel $\Ker$ further implies
 \begin{align}
 \label{eq:Vgap_discrete>}
 \distrnorm[\big]{\ValueMtpar{(\idxitv)}{\idxpack} - \ValueMtpar{(\idxitv)}{\idxpacknew}}{\distrMtbase} & \geq \frac{\const{5}\sqrt{\const{}}}{60\sqrt{2}} \; \newradbar \, \delcrit \cdot \one\big\{ \pack{\idxpack}{\idxitv} \neq \pack{\idxpacknew}{\idxitv} \big\} \, .
 \end{align}
 Substituting the lower bound~\eqref{eq:Vgap_discrete>} into equality~\eqref{eq:full2discrete_Vgap} yields
 \begin{align*}
 \distrnorm[\big]{\Vstarparm{\idxpack} - \Vstarparm{\idxpacknew}}{\distrbar}^2 \; \gtrsim \; \newradbar^2 \, \delcrit^2 \cdot \frac{1}{\numitv} \sum_{\idxitv = 1}^{\numitv} \; \one\big\{ \pack{\idxpack}{\idxitv} \neq \pack{\idxpacknew}{\idxitv} \big\} \, .
 \end{align*}
 Recall that $\{\packvec{\idxpack}\}_{\idxpack \in [\PackNum]}$ is a $\tfrac{1}{4}$-packing of the Boolean hypercube $\{0, 1\}^{\numitv}$ with respect to the Hamming metric defined in equation~\eqref{eq:def_Hamming}. Therefore, we have
 $\frac{1}{\numitv} \sum_{\idxitv = 1}^{\numitv} \; \one\big\{ \pack{\idxpack}{\idxitv} \neq \pack{\idxpacknew}{\idxitv} \big\} \geq \frac{1}{4}$. It then follows that $\distrnorm[\big]{\Vstarparm{\idxpack} - \Vstarparm{\idxpacknew}}{\distrbar}^2 \; \gtrsim \; \newradbar^2 \, \delcrit^2$, as claimed in inequality~\eqref{eq:valuegap>} in claim~\ref{claim:gap}.


\subsection{Proof of Lemma~\ref{lemma:full2discrete}}
 \label{append:proof:lemma:full2discrete}

 In this part, we prove \Cref{lemma:full2discrete} in \Cref{append:connection}. Since equation \eqref{eq:distr_full2discrete} in \Cref{lemma:full2discrete}(a) and the results in \Cref{lemma:full2discrete}(b) to \ref{lemma:full2discrete_chidiv} are obvious by our construction of MRP $ \MRP_{\idxpack}$ and RKHS $\RKHS$, we only present the proof of equation \eqref{eq:V_full2discrete} below.

 Recall the transition operator $\TransOpm{\idxpack}$ satisfies
 decomposition~\eqref{eq:TransOp_cvx}. Applying the Sherman--Morrison
 formula yields
 \begin{align*}
 \big( \IdOp - \discount \, \TransOpm{\idxpack} \big)^{-1} & = \Big(
 \IdOp - \discount \, \big\{ (1-\invmix) \, \TransOptilde_{\idxpack} +
 \invmix \, \one \otimes \one \big\} \Big)^{-1} = \big( \IdOp -
 \discounttil \, \TransOptilde_{\idxpack} \big)^{-1} + \frac{\discount
   \, \invmix}{(1-\discount)(1 - \discount+ \discount\invmix)} \; \one
 \otimes \one \, ,
 \end{align*}
 with $\discounttil = \discount(1-\invmix)$. It follows that the value function $\Vstarm{\idxpack}$ satisfies
 \begin{align*}
 \Vstarm{\idxpack} = \big( \IdOp - \discount \, \TransOpm{\idxpack} \big)^{-1} \reward
 & = \big( \IdOp - \discounttil \, \TransOptilde_{\idxpack} \big)^{-1} \reward + \frac{\discount \, \invmix }{(1-\discount)(1 - \discount+ \discount\invmix)} \, \Exp_{\distrm{\idxpack}}[\reward] \cdot \one \, .
 \end{align*}
 Due to the block structure of the MRP, we have $\big[ \big( \IdOp -
   \discounttil \, \TransOptilde_{\idxpack} \big)^{-1} \reward
   \big](\state) = \ValueMt{\idxitv}_{\idxpack}(\idxstate)$ \text{for
   any $\state \in \interval{\idxitv}{\idxstate}$}.  We apply
 equality~\eqref{eq:smallMRP_exp} in \Cref{lemma:smallMRP_list}(a) and
 find that
 \begin{align*}
 \Exp_{\distrm{\idxpack}}[\reward] = \frac{1}{\numitv} \sum_{\idxitv = 1}^{\numitv} \Exp_{\distrMtper_{\idxpack}^{(\idxitv)}}[\rewardMt] & = \frac{1}{\numitv} \sum_{\idxitv=1}^{\numitv}\big\{ \scalaroneone \, \Dq_{\idxpack}^{(\idxitv)} + \scalaronetwo \, \big(\Dq_{\idxpack}^{(\idxitv)}\big)^2 \big\} \\
 & = \frac{1}{60} \sqrt{\frac{\statdim}{\numobs}} \, \bigg\{ \scalaroneone + \scalaronetwo \; \frac{1}{60} \sqrt{\frac{\statdim}{\numobs}} \bigg\} \cdot \frac{1}{\numitv} \sum_{\idxitv=1}^{\numitv} \pack{\idxpack}{\idxitv} \, .
 \end{align*}
 Since $\frac{1}{\numitv} \sum_{\idxitv = 1}^{\numitv} \,
 \pack{\idxpack}{\idxitv} = \frac{1}{2}$ by construction,
 $\Exp_{\distrm{\idxpack}}[\reward]$ are equal for any MRP
 $\MRP_{\idxpack}$. It follows that
\begin{subequations}
 \begin{align}
\label{eq:full2discrete_V1}
 \Vstarm{\idxpack}(\state) = \ValueMt{\idxitv}_{\idxpack}(\idxstate) +
 \scalarone \qquad \text{with~} \scalarone \defn \frac{\discount \,
   \invmix }{(1-\discount)(1 - \discount+ \discount\invmix)} \,
 \Exp_{\distrm{\idxpack}}[\reward]
\end{align}
for any $\state \in \interval{\idxitv}{\idxstate}$.

We now bound the scalar $\scalarone$. By inequality~\eqref{eq:Dp<} and
the bounds on $\scalaroneone, \scalaronetwo$ in
equation~\eqref{eq:smallMRP_exp}, we have \mbox{$0 \leq
  \Exp_{\distrm{\idxpack}}[\reward] \leq \frac{\stdfunbar \;
    \mixtimebar \, (1 - \discount)}{120}
  \sqrt{\frac{\statdim}{\numobs}}$,} which implies that
\begin{align}
 \label{eq:scalarone}
 0 \leq \scalarone \leq \frac{\stdfunbar}{960 \, (1 - \discount)}
 \sqrt{\frac{\statdim}{\numobs}} \leq \frac{\newradbar \,
   \delcrit}{1000} \, .
 \end{align}
\end{subequations}
Here we have used the critical inequality~\eqref{eq:CI_lb} in the last
step. Combining equation~\eqref{eq:full2discrete_V1} with
inequality~\eqref{eq:scalarone} completes the proof of
equation~\eqref{eq:V_full2discrete}.



\section{Lemmas for $3$-state MRPs in the proof of minimax lower bound}
 \label{append:proof_smallMRPs}

In this part, we prove various claims about the $3$-state MRPs that
appeared in our construction of minimax lower bound
in~\Cref{sec:lb_construction}.  We begin
in~\Cref{append:smallMRP_recap} by recalling the structures of the
discrete MRPs and the associated $2$-dimensional linear space.
In~\Cref{append:smallMRP_properties}, we provide the precise
statements of the claims to be proved.
\Cref{append:proof:lemma:smallMRP_densityratio,append:proof:lemma:smallMRP_norm,append:proof:lemma:smallMRP_lbnormperp,append:proof:lemma:chidiv_discrete,append:proof:lemma:Vgap_discrete}
are dedicated to the proofs of the lemmas in
\Cref{append:smallMRP_properties}.


\subsection{Recap of the discrete MRPs}
 \label{append:smallMRP_recap}

 The full-scale MRPs and RKHS in \Cref{sec:lb_construction} are constructed by tensorizing a specific small MRP with $3$ point state space and a $2$-dimensional linear subspace in $\Real^3$. We recall that the $3$-state MRP $\MRPMtper(\rewardMt, \TransMtper, \discounttil)$ is given by $\discounttil \defn \discount (1 - \invmix)$,
 \begin{multline}
 \TransMtper \equiv \TransMt(\Dp, \Dq) \notag \\
 \defn
 \begin{pmatrix}
 \{ \tfrac{1}{2} - \invmix(1 - \Dp)\} (1 + \Dq) & \{\tfrac{1}{2} - \invmix(1 - \Dp)\}(1 - \Dq) & 2\invmix (1 - \Dp) \\
 \{ \tfrac{1}{2} - \invmix(1 - \Dp)\} (1 + \Dq) & \{\tfrac{1}{2} - \invmix(1 - \Dp)\}(1 - \Dq) & 2\invmix (1 - \Dp) \\
 \invmix(1 + \Dp) (1 + \Dq) & \invmix(1 + \Dp) (1 - \Dq) & 1 - 2 \invmix(1 + \Dp)
 \end{pmatrix} \in \Real^{3 \times 3} \, ,
 \tag{\text{\ref{eq:def_TransMt}}}
 \end{multline}
 where $\invmix = \mixtimebar^{-1}/8$ and $\Dp, \Dq \in [0,1)$ are two
   scalars. The choice of parameters $(\Dp, \Dq)$ ensures that
   $\TransMtper$ is a valid transition matrix. We denote the
   stationary distribution of Markov chain~$\TransMtper$ by
   vector~$\distrMtper \in \Real^3$.  The transition kernel
   $\TransMtper$ is close to a base model $\TransMtbase =
   \TransMt(0,0)$, of which the stationary distribution is given by
   $\distrMtbase = \big( \frac{1}{4} , \frac{1}{4} , \frac{1}{2}
   \big)^{\top} \in \Real^3$.

 The linear subspace $\RKHSMt \subset \Real^3$ we used in the lower bound construction is the span of two feature vectors $\featureMt{1}, \featureMt{2} \in \Real^3$, which are defined as
 \begin{align}
 (\featureMt{1}, \, \featureMt{2}) \defn \bUbase
 \begin{pmatrix}
 1 & 0 \\ 0 & \cos\lbtheta \\ 0 & \sin\lbtheta
 \end{pmatrix} \qquad \text{with }\bUbase \defn \begin{pmatrix}
 1 & 1 & \sqrt{2} \\ 1 & 1 & -\sqrt{2} \\ 1 & -1 & 0
 \end{pmatrix} \in \Real^{3 \times 3} \, .
 \tag{\text{\ref{eq:def_featureMt}}}
 \end{align}
 The angle $\lbtheta \in [0,\tfrac{\pi}{2}]$ is set as
 \begin{align}
 \lbtheta \defn \frac{\pi}{2} - \frac{1}{2} \arcsin \Big\{ \frac{4 \, \lbnormperp \, (1 - \discounttil + 4 \discounttil \invmix)}{\stdfunbar \; \discounttil(1 - 4 \invmix)} \Big\} \,
 \tag{\text{\ref{eq:def_lbtheta}}}
 \end{align}
 so that $\cos\lbtheta$ satisfies the sandwich inequality
 \begin{align}
 \label{eq:cos<}
 \frac{2\lbnormperp \, (1 - \discounttil + 4 \discounttil \invmix)}{\stdfunbar \; \discounttil(1 - 4 \invmix)} = \tfrac{1}{2} \sin(2\lbtheta) \; \leq \; \cos \lbtheta \; \leq \; \tfrac{1}{\sqrt{2}} \sin(2\lbtheta) = \frac{2\sqrt{2} \, \lbnormperp \, (1 - \discounttil + 4 \discounttil \invmix)}{\stdfunbar \; \discounttil(1 - 4 \invmix)} \, .
 \end{align}
 The condition~$\lbnormperp \leq \frac{1}{108} \stdfunbar \min\{ (1-\discount)^{-1}, \sqrt{\mixtimebar} \}$ in equation~\eqref{cond:norm_ratio} ensures that $\lbtheta$ is well-defined.
 The columns of matrix $\bUbase$ are a group of orthonormal basis in space $\Ltwo{\distrMtbase}$.
 Moreover, we take the reward function $\rewardMt \in \Real^3$ as
 \begin{align}
 \rewardMt \defn \bUbase \; \bweightr \in \Real^3 \qquad
 \text{with}~~
 \bweightr \defn \begin{pmatrix}
 0 \\ (\stdfunbar/4) \; \cos\lbtheta \\ (\stdfunbar/4) \; \sin\lbtheta
 \end{pmatrix} \in \Real^3 \, ,
 \tag{\text{\ref{eq:def_rewardMt}}}
 \end{align}
 so that vector $\rewardMt$ belongs to the linear space $\RKHSMt$.

Given the MRP instance $\MRPMt(\rewardMt, \TransMtper, \discounttil)$,
its value function is given by $\VstarMtper = (\IdMt - \discounttil \,
\TransMtper)^{-1} \rewardMt \in \Real^3$, and we define
$L^2(\distrMt)$-based projection
 \begin{align}
 \label{eq:def_VstarMtperpar}   
 \VstarMtperpar \equiv \proj_{\distrMtper} (\VstarMtper) \defn
 \argmin_{\bfun \in \RKHSMt} \, \distrnorm{\bfun -
   \VstarMtper}{\distrMtper}
 \end{align}
of this value function onto $\RKHSMt$.  The \emph{projection error},
also referred to as the \emph{model mis-specification error}, is given
by \mbox{$\VstarMtperperp \defn \VstarMtper - \VstarMtperpar$.}
Similarly, we define $\VstarMtbase, \VstarMtbasepar, \VstarMtbaseperp
\in \Real^3$ as the counterparts of the base MRP model
$\MRPMtbase(\rewardMt, \TransMtbase, \discounttil)$.


\subsection{Properties of the discrete MRPs}
 \label{append:smallMRP_properties}

\noindent Let us summarize some useful properties of the 3-state MRPs:
\begin{lemma}[Properties of discrete MRPs]
\label{lemma:smallMRP_list}
Suppose that the parameters $\stdfunbar$, $\lbnormperp$ and
$\mixtimebar$ satisfy
conditions~\eqref{cond:norm_ratio}~and~\eqref{cond:mixtimebar}.
Consider an MRP $\MRPMt( \rewardMt, \TransMt, \discounttil )$ with
transition kernel $\TransMt(\Dp,\Dq)$ given in
definition~\eqref{eq:def_TransMt}. The parameters $(\Dp, \Dq)$ are
small enough such that
 \begin{align}
 \label{cond:Dp<}
 0 \leq \Dp \leq \frac{1}{3} \qquad \text{and} \qquad 0 \leq \Dq \leq \min \Big\{\frac{1}{3}, \, \frac{1 - \discounttil + 4\discounttil\invmix}{2\sqrt{2\invmix}}, \, \frac{2\sqrt{2} \; \lbnormperp}{3 \; \stdfunbar} \, (1 - \discounttil + 4 \discounttil \invmix) \Big\} \, .
 \end{align}
 Then the following statements hold:
 \begin{enumerate}
 \item[(a)] \emph{(Properties of stationary distribution $\distrMtper$)} \label{lemma:smallMRP_densityratio} The stationary distribution $\distrMtper$ satisfies
 \begin{subequations}
 \begin{align}
 & \label{eq:smallMRP_densityratio}
 \frac{1}{2} \leq \frac{\diff \, \distrMtper}{\diff \, \distrMtbase} (\idxstate) \leq 2 \qquad \text{for any state $\idxstate \in [3]$ and base measure $\distrMtbase = \big( \tfrac{1}{4}, \tfrac{1}{4}, \tfrac{1}{2} \big)^{\top}$}.
 \end{align}
 Suppose $\Dp = \frac{1}{\sqrt{\invmix}} \Dq$. Then the expectation of reward $\rewardMt$ over distribution $\distrMtper$ are given by
 \begin{align}
 & \Exp_{\distrMtper}[\rewardMt] = \scalaroneone \, \Dq + \scalaronetwo \, (\Dq)^2 \, , \label{eq:smallMRP_exp}
 \end{align}
 where $\scalaroneone \equiv \scalaroneone(\stdfunbar, \mixtimebar, \lbtheta)$ and $\scalaronetwo \equiv \scalaronetwo(\stdfunbar, \mixtimebar, \lbtheta)$ are two scalars satisfying 
 \begin{align}
 0 \leq \scalaroneone \leq \tfrac{1}{2} \, \stdfunbar \; \mixtimebar \, (1 - \discount) \qquad \text{and} \qquad 0 \leq \scalaronetwo \leq \tfrac{1}{2} \, \stdfunbar \sqrt{\mixtimebar} \, . \label{eq:smallMRP_exp_const}
 \end{align}
 \end{subequations}
 \item[(b)] \emph{(Form of the base projected value
 function)} \label{lemma:smallMRP_Vstarbasepar} The projected value
   function $\VstarMtbasepar$ of the base MRP $\MRPMtbase =
   \MRPMt(\rewardMt, \TransMtbase, \discounttil)$ takes the form
 \begin{subequations}
 \begin{align} \label{eq:smallMRP_Vstarbasepar}
 \VstarMtbasepar = \scalartwo \cdot \featureMt{2}
 \end{align}
 where the scalar $\scalartwo \equiv \scalartwo(\stdfunbar, \mixtimebar, \lbtheta)$ satisfies 
 \begin{align} \label{eq:smallMRP_Vstarbasepar_const}
 0 \leq \scalartwo \leq \stdfunbar + \lbnormperp \, \min\big\{ 1, \, (1-\discount) \sqrt{\mixtimebar} \big\} \, .
 \end{align}
 \end{subequations}
 \item[(c)] \emph{(Conditional variance and $\Ltwo{\distrMtper}$-norm of the value function)} \label{lemma:smallMRP_norm} The conditional variance $\stdfun^2(\VstarMtper)$ and the $\Ltwo{\distrMtper}$-norm of $\VstarMtper$ 
 satisfy
 \begin{subequations}
 \label{eq:smallMRP_norm}
 \begin{align}
 & \stdfun^2(\VstarMt) = \Exp_{\State \sim \distrMtper} \big[ \, \Var_{\Statenew \sim \TransMtper(\cdot \mid \State)}[ \, \VstarMtper(\Statenew) \mid \State \, ] \, \big] \leq \frac{\stdfunbar^2}{3} \, , \label{eq:smallMRP_stdfun} \\
 & \distrnorm{\VstarMt}{\distrMt}^2 \leq \frac{\stdfunbar^2}{3 \, \invmix} \, .
 \label{eq:smallMRP_munorm} 
 \end{align}
 \end{subequations}
 \item[(d)] \emph{(Upper bounds on the model mis-specification)} \label{lemma:smallMRP_lbnormperp} The model mis-specification error satisfies
 \begin{align}
 \label{eq:smallMRP_lbnormperp}
 \distrnorm{\VstarMtperperp}{\distrMtper} \; \leq \; \lbnormperp \, .
 \end{align}
 \end{enumerate}
 Additionally, consider two MRPs $\MRPMt_1$ and $\MRPMt_2$ with transition kernels $\TransMtper_1 = \TransMt(\Dp_1, \Dq_1)$ and $\TransMtper_2 = \TransMt(\Dp_2, \Dq_2)$ that satisfy conditions~\eqref{cond:Dp<}.
 Then the following statements hold:
 \begin{enumerate}
 \setcounter{enumi}{4}
 \item[(e)] \label{lemma:chidiv_discrete}
\emph{($\chi^2$-divergence of distributions)} The $\chi^2$-divergences between \mbox{stationary distributions $\distrMtper_1$ and $\distrMtper_2$} and transition kernels $\TransMtper_1$ and $\TransMtper_2$ satisfy
 \begin{subequations}
 \label{eq:chidiv_discrete}
 \begin{align}
 & \chidiv{\distrMtper_1}{\distrMtper_2} \; \leq 2 \, (\Dp_1 - \Dp_2)^2 + 2 \, (\Dq_1 - \Dp_2)^2 \, ; \label{eq:smallMRP_chisq_distr} \\
 & \Exp_{\State \sim \distrMtper_1}\big[ \, \chidiv[\big]{\TransMtper_1(\cdot \mid \State)}{\TransMtper_2(\cdot \mid \State)} \big] \; \leq \; 12 \, \invmix \, (\Dp_1 - \Dp_2)^2 + 8 \, (\Dq_1 - \Dp_2)^2 \, . \label{eq:smallMRP_chisq_P}
 \end{align}
 \end{subequations}
 \item[(f)] \label{lemma:Vgap_discrete} \emph{(Difference between projected value functions)} \begin{subequations}
 Suppose \mbox{$\Dp_1 = \frac{1}{\sqrt{\invmix}} \Dq_1 \geq 0$}, \mbox{$\Dp_2 = \frac{1}{\sqrt{\invmix}} \Dq_2 \geq 0$}. Then if $\Dq_1$ and $\Dq_2$ are small enough to ensure that
 \begin{align}
 \label{cond:Dq^2}
 \sqrt{\mixtimebar} \, \max\big\{(\Dq_1)^2, (\Dq_2)^2\big\} \leq \consttil{5} \, |\Dq_1 - \Dq_2|
 \end{align}
 for some universal constant $\consttil{5} > 0$,
 then the difference between projected value functions $\ValueMtpar{*}{1}$ and $\ValueMtpar{*}{2}$ satisfies
 \begin{align}
 \label{eq:Vgap_discrete}
 \const{5} \; \lbnoise \, |\Dq_1 - \Dq_2| \leq \distrnorm[\big]{\ValueMtpar{*}{1} - \ValueMtpar{*}{2}}{\distrMtbase} \leq \supnorm[\big]{\ValueMtpar{*}{1} - \ValueMtpar{*}{2}} \leq 10 \; \lbnoise \, |\Dq_1 - \Dq_2| \, ,
 \end{align}
\end{subequations}
where $\lbnoise \defn \lbstdmtg + \lbapproxerr$.
\end{enumerate}
\end{lemma}
We provide proofs of \Cref{lemma:smallMRP_list} in
\Cref{append:proof:lemma:smallMRP_densityratio,append:proof:lemma:smallMRP_Vstarbasepar,append:proof:lemma:smallMRP_norm,append:proof:lemma:smallMRP_lbnormperp,append:proof:lemma:chidiv_discrete,append:proof:lemma:Vgap_discrete}. Specifically,
\Cref{append:proof:lemma:smallMRP_densityratio} is concerned with the
density ratio $\frac{\diff \distrMtper}{\diff \distrMtbase}$ and
expectation $\Exp_{\distrMtper}[\rewardMt]$.
\Cref{append:proof:lemma:smallMRP_Vstarbasepar} studies the projected
value function $\VstarMtbasepar$ of the base model $\MRPMtbase =
\MRPMt(\rewardMt, \TransMtbase, \discounttil)$.
\Cref{append:proof:lemma:smallMRP_norm} focuses on analyzing the
conditional variance $\stdfun^2(\VstarMtper)$ and norm
$\distrnorm{\VstarMtper}{\distrMtper}$. In
\Cref{append:proof:lemma:smallMRP_lbnormperp}, we estimate the model
mis-specification error $\distrnorm{\VstarMtperp}{\distrMtper}$.
\Cref{append:proof:lemma:chidiv_discrete} is devoted to control the
$\chi^2$-divergence terms. In \Cref{append:proof:lemma:Vgap_discrete},
we consider the difference in projected value functions. \\


Before we present the proofs of \Cref{lemma:smallMRP_list}, we
summarize some facts of the discrete MRP $\MRPMt(\rewardMt, \TransMt,
\discounttil)$ (without proofs) that are helpful in our downstream
analyses.  We first note that the MRP $\MRPMt(\rewardMt, \TransMtper,
\discounttil)$ has a stationary distribution
\begin{subequations}
\begin{align}
\label{eq:def_distrMtper}
& \distrMtper \defn \begin{pmatrix} \frac{1}{4} (1 + \Dp)(1 + \Dq) \,
  , & \frac{1}{4} (1 + \Dp)(1 - \Dq) \, , & \frac{1}{2} (1 - \Dp)
\end{pmatrix}^{\top} \in \Real^3 \, .
\end{align}
Associated with distribution $\distrMtper$, we take an orthonormal
basis $\{ \bu_1, \bu_2, \bu_3 \} \subset \Real^3$ which are columns of
the matrix
\begin{align}
\label{eq:def_bUper}
& \bUper \defn \begin{pmatrix} 1 & \sqrt{\frac{1 - \Dp}{1 + \Dp}} &
  \sqrt{\frac{2(1 - \Dq)}{(1 + \Dp)(1 + \Dq)}} \\ 1 & \sqrt{\frac{1 -
      \Dp}{1 + \Dp}} & - \sqrt{\frac{2(1 + \Dq)}{(1 + \Dp)(1 -
      \Dq)}}\\ 1 & - \sqrt{\frac{1 + \Dp}{1 - \Dp}} & 0
 \end{pmatrix} \in \Real^{3 \times 3} \, .
\end{align}
\end{subequations}
We can conveniently express the transition matrix $\TransMtper$ and
value function $\VstarMtper$ with the help of distribution
$\distrMtper$ and basis $\bUper = [\bu_1, \bu_2, \bu_3]$. See
\Cref{lemma:tab_PandV} below.

\begin{lemma}
\label{lemma:tab_PandV}
The transition kernel $\TransMtper$ in
definition~\eqref{eq:def_TransMt} has a unique stationary distribution
$\distrMtper$ which is given in equation~\eqref{eq:def_distrMtper}.
The matrix $\TransMtper$ has an eigen decomposition
\begin{subequations}
  \begin{align}
\label{eq:Pper_eig}    
  \TransMtper = \bUper \; [ \, \diag\{ 1, \, 1-4\invmix, \, 0 \} \, ]
  \; \bUper^{\top} [\,\diag(\distrMtper)\,].
 \end{align}
Using the basis $\bUper = [\bu_1, \bu_2, \bu_3]$, we can write the
value function $\VstarMt \in \Real^3$ as
\begin{align}
\label{eq:Vstarper_decomp}  
 \VstarMt = \bUper \; \big[ \, \diag\big\{ (1-\discounttil)^{-1}, \,
   (1 - \discounttil + 4\discounttil\invmix)^{-1}, \, 1 \big\} \,
   \big] \; \bweightrper \, , ~~ \text{where } \bweightrper \defn
 \big( \, \bUper^{\top} [\,\diag(\distrMtper)\,] \, \bUbase \, \big)
 \; \bweightr.
\end{align}
\end{subequations}
\end{lemma}


\subsection{Proof of Lemma~\ref{lemma:smallMRP_list}(a)}
\label{append:proof:lemma:smallMRP_densityratio}

We first focus on the density ratio $\frac{\diff \, \distrMtper}{\diff
  \, \distrMtbase}$.  It follows from the
expression~\eqref{eq:def_distrMtper} of distribution~$\distrMtper$ and
the bounds~\eqref{cond:Dp<} on parameters $(\Dp, \Dq)$ that
\begin{align*}
\min_{\idxstate \in [3]} \frac{\diff \, \distrMtper}{\diff \,
  \distrMtbase}(\idxstate) \geq \min\Big\{ \frac{1}{1-\Dp},
\frac{1}{1-\Dq} \Big\} \geq \frac{1}{2}\, , \qquad \text{and} \qquad
\max_{\idxstate \in [3]} \frac{\diff \, \distrMtper}{\diff \,
  \distrMtbase}(\idxstate) \leq (1 + \Dp) (1 + \Dq) \leq 2
\end{align*}
which verifies inequality~\eqref{eq:smallMRP_densityratio}. \\

We now estimate the expectation $\Exp_{\distrMt}[\rewardMt]$.  Using
definitions~\eqref{eq:def_distrMtper}~and~\eqref{eq:def_bUper} of
$\distrMtper$ and $\bUper$, we find that
\begin{multline}
\label{eq:UDU}
\bUper^{\top} [\,\diag(\distrMtper)\,] \; \bUbase = \Big[ \diag\big(
  1, \sqrt{1 - (\Dp)^2}, \sqrt{1 + \Dp} \sqrt{1 - (\Dq)^2} \, \big)
  \Big] \; \IdMtper \, , \\ \text{where } \IdMtper
\defn \begin{pmatrix} 1 & \Dp & \frac{1}{\sqrt{2}} (1+\Dp) \, \Dq \\ 0
  & 1 & \tfrac{1}{\sqrt{2}} \Dq \\ 0 & 0 & 1
\end{pmatrix} \, .
\end{multline}
It follows that \mbox{$\Exp_{\distrMt}[\rewardMt] = \bu_1^{\top}
  [\,\diag(\distrMtper)\,] \; \bUbase \; \bweightr =
  \frac{\stdfunbar}{4} \, \big\{ \Dp \, \cos\lbtheta +
  \tfrac{1}{\sqrt{2}} (1+\Dp) \, \Dq \, \sin\lbtheta \big\}$.}  If we
take $\Dp = \tfrac{1}{\sqrt{\invmix}} \Dq$, then
\begin{align}
\label{eq:Expreward}
\Exp_{\distrMt}[\rewardMt] = \scalaroneone \, \Dq + \scalaronetwo \,
(\Dq)^2 \quad \text{where }\scalaroneone \defn \frac{\stdfunbar}{4} \,
\big\{ \tfrac{1}{\sqrt{\invmix}} \, \cos\lbtheta + \tfrac{1}{\sqrt{2}}
\, \sin\lbtheta \big\} \, , \quad \scalaronetwo \defn
\frac{\stdfunbar}{4\sqrt{2 \invmix}} \, \sin\lbtheta \, .
\end{align}
In order to bound $\scalaroneone$, we derive an upper bound on
$\cos\lbtheta$. By applying inequality $\lbnormperp \leq \;
\frac{1}{108} \; \sqrt{\mixtimebar} \, \stdfunbar $ in
condition~\eqref{cond:norm_ratio} to the upper bound~\eqref{eq:cos<},
we have
\begin{align}
  \label{eq:cos<?}
  \cos \lbtheta \; \leq \; \frac{2\sqrt{2} \, \lbnormperp \, (1 -
    \discounttil + 4 \discounttil \invmix)}{\stdfunbar \;
    \discounttil(1 - 4 \invmix)} \leq \frac{1 - \discounttil + 4
    \discounttil \invmix}{40 \sqrt{\invmix} \; (1 - 4 \invmix)} \leq
  \frac{\sqrt{\invmix}}{2} \; \mixtimebar \, (1 - \discount) \, .
 \end{align}
Here we have also used the definition~\mbox{$\invmix =
  \mixtimebar^{-1}/8$} and conditions~$\discounttil \in [0.4, 1)$ and
  $\mixtimebar \geq (1-\discount)^{-1}$. Substituting
  inequality~\eqref{eq:cos<?} into equation~\eqref{eq:Expreward}, we
  find that \mbox{$0 \leq \scalaroneone \leq \frac{\stdfunbar}{2} \;
    \mixtimebar \, (1 - \discount)$} and \mbox{$0 \leq \scalaronetwo
    \leq \frac{\stdfunbar}{2} \sqrt{\mixtimebar}$,} as stated in
  equation~\eqref{eq:smallMRP_exp}.


\subsection{Proof of Lemma~\ref{lemma:smallMRP_list}(b)}
\label{append:proof:lemma:smallMRP_Vstarbasepar}

We study the projected value function $\VstarMtbasepar$ of the base
model $\MRPMtbase = \MRPMt(\rewardMt, \TransMtbase, \discounttil)$. In
this case, we have $\bUper = \bUbase$ and $\distrMtper =
\distrMtbase$. From equation~\eqref{eq:Vstarper_decomp}
in~\Cref{lemma:tab_PandV}, the value function $\VstarMtbase$ takes the
form
\begin{align*}
\VstarMtbase = \frac{\stdfunbar}{4} \; \bUbase \begin{pmatrix} 0 \\ (1
  - \discounttil + 4\discounttil\invmix)^{-1} \, \cos\lbtheta
  \\ \sin\lbtheta
\end{pmatrix} \, .
\end{align*}
It then follows from definition~\eqref{eq:def_featureMt} of features
$\featureMt{1}$ and $\featureMt{2}$ that $\VstarMtbasepar = \scalartwo
\cdot \featureMt{2}$ where {$\scalartwo \defn
  \frac{\stdfunbar}{4} \big\{ \frac{\cos^2 \lbtheta}{1 -
    \discounttil + 4 \discounttil \invmix} + \sin^2 \lbtheta \big\}
  $.}

By applying the bound~\eqref{eq:cos<} on $\cos\lbtheta$, we find that
under conditions $\invmix = \mixtimebar^{-1}/8 \leq (1-\discount)/8$
and $\discounttil \in (0.4,1]$, it holds that
\begin{align*}
\scalartwo \leq \frac{2 \, \lbnormperp^2 \, (1 - \discounttil + 4
  \discounttil \invmix)}{\stdfunbar \; \discounttil^2(1 - 4
  \invmix)^2} + \frac{\stdfunbar}{4} \sin^2\lbtheta \leq 60 \;
\stdfunbar^{-1}\lbnormperp^2 \, (1 - \discount) + \frac{\stdfunbar}{4}
\, .
\end{align*}
Recall the inequality~$\lbnormperp \leq \; \frac{1}{108} \; \stdfunbar
\; \min \big\{ (1-\discount)^{-1}, \, \sqrt{\mixtimebar} \big\}$ in
condition~\eqref{cond:norm_ratio}.  We then derive an upper bound
$\scalartwo \leq \stdfunbar + \lbnormperp \, \min\big\{ 1, \,
(1-\discount) \sqrt{\mixtimebar} \big\}$ as
in~\Cref{lemma:smallMRP_Vstarbasepar}, which completes the proof
of~\Cref{lemma:smallMRP_list}(b).


\subsection{Proof of Lemma~\ref{lemma:smallMRP_list}(c)}
\label{append:proof:lemma:smallMRP_norm}

We first consider the conditional variance $\stdfun^2(\VstarMt)$. By
definition, we have
\begin{align}
\stdfun^2(\VstarMtper) & = \Exp_{\State \sim \distrMtper}\big[ \,
  \Var_{\Statenew \sim \TransMtper(\cdot \mid \State)}[ \,
    \VstarMtper(\Statenew) \mid \State \, ] \, \big] \notag \\
\label{eq:varper}
& = \Exp_{\State \sim \distrMtper}\big[ \, \Exp_{\Statenew \sim
    \TransMtper(\cdot \mid \State)}\big[ \, (\VstarMtper(\Statenew))^2
    \mid \State \, \big] - (\TransMtper \VstarMtper)^2(\State) \,
  \big] \stackrel{(*)}{=} \distrnorm{\VstarMtper}{\distrMtper}^2 -
\distrnorm{\TransMtper \VstarMtper}{\distrMtper}^2,
 \end{align}
where step~$(*)$ follows from the stationarity of distribution
$\distrMtper$. We now use the expressions of $\distrMtper$,
$\TransMtper$ and $\VstarMtper$ in
equations~\eqref{eq:def_distrMtper}, \eqref{eq:Pper_eig}
and~\eqref{eq:Vstarper_decomp} to calculate the conditional variance
$\stdfun^2(\VstarMtper)$.

We first note that
\begin{align}
\label{eq:PVstar_decomp}
\TransMtper \VstarMtper = \bUper \; \big[ \, \diag\big\{ (1 -
  \discounttil)^{-1}, \, (1-4\invmix)(1 - \discounttil + 4
  \discounttil \invmix)^{-1}, \, 0 \big\} \, \big] \; \bweightrper,
 \end{align}
where the vector $\bweightrper = \big( \, \bUper^{\top}
[\,\diag(\distrMtper)\,] \, \bUbase \, \big) \; \bweightr$ is given in
equation~\eqref{eq:Vgap_decomp}. Substituting
equations~\eqref{eq:Vstarper_decomp} and \eqref{eq:PVstar_decomp} into
equation~\eqref{eq:varper} yields
\begin{align}
\label{eq:stdfun_bweightrper}
\stdfun^2(\VstarMt) = \frac{1 - (1 - 4 \invmix)^2}{(1 - \discounttil +
  4 \discounttil \invmix)^2} \, \big( \bweightrper(2) \big)^2 + \big(
\bweightrper(3) \big)^2 \, .
 \end{align}
In the following, we bound the vector $\bweightrper$.

It follows from expression~\eqref{eq:UDU} of matrix $\bUper^{\top}
[\,\diag(\distrMtper)\,] \; \bUbase$ that under
condition~\eqref{cond:Dp<},
\begin{subequations}
\label{eq:bweightrper<}
\begin{align}
& \big( \bweightrper(2) \big)^2 = \frac{\stdfunbar^2}{16} \, \big( 1 -
  (\Dp)^2 \big) \big( \cos\lbtheta + \tfrac{1}{\sqrt{2}} \Dq
  \sin\lbtheta \big)^2 \leq \frac{\stdfunbar^2}{8} \, \big(
  \cos^2\lbtheta + \tfrac{1}{2} (\Dq)^2 \big) \, , \\ & \big(
  \bweightrper(3) \big)^2 = \frac{\stdfunbar^2}{16} \, (1+\Dp) \big( 1
  - (\Dq)^2 \big) \sin^2\lbtheta \leq \frac{\stdfunbar^2}{8} \,
  \sin^2\lbtheta \, .
\end{align}
\end{subequations}
Combining inequality~\eqref{eq:bweightrper<} with
equation~\eqref{eq:stdfun_bweightrper} yields
 \begin{align}
 & \stdfun^2(\VstarMtper) \leq \frac{\scalarzero}{8} \, \stdfunbar^2 +
   \frac{\stdfunbar^2 \, \invmix}{2 \, (1 - \discounttil +
     4\discounttil\invmix)^2} \, (\Dq)^2 \, , \quad
 \label{eq:def_scalar}
 & \text{where } \scalarzero \defn \frac{8\invmix(1-2\invmix)}{(1 -
   \discounttil +4\discounttil\invmix)^2} \cos^2\lbtheta +
 \sin^2\lbtheta \, .
 \end{align}
We use the bound $\Dq \leq \frac{1 - \discounttil +
  4\discounttil\invmix}{2\sqrt{2\invmix}}$ in
condition~\eqref{cond:Dp<}, and thereby find that
\begin{align}
\label{eq:varper<}
\stdfun^2(\VstarMtper) \leq \tfrac{1}{16} \, ( 2 \scalarzero + 1 ) \,
\stdfunbar^2 \, .
\end{align}
It remains to estimate the scalar $\scalarzero$.

Inequality~\eqref{eq:cos<?} implies that \mbox{$\scalarzero \leq 1 +
  \frac{4 \invmix (1 - 4 \invmix) (1 + \discounttil^2)}{(1 -
    \discounttil + 4\discounttil\invmix)^2} \, \cos^2\lbtheta \leq 1 +
  \frac{1 + \discounttil^2}{400 \; (1 - 4 \invmix)} \leq 1.1$,} where
the last step is ensured by $\invmix = \mixtimebar^{-1}/8 \leq
\tfrac{1}{8}$.  Combining this bound with
inequality~\eqref{eq:varper<} completes the proof of
inequality~\eqref{eq:smallMRP_stdfun}. \\

We now bound
$\distrnorm{\VstarMt}{\distrMtper}^2$. Equation~\eqref{eq:Vstarper_decomp}
implies that
\begin{align}
\label{eq:VstarMt_norm}
\distrnorm{\VstarMt}{\distrMtper}^2 & = \distrnorm[\big]{ \big[ \,
    \diag \big \{ (1-\discounttil)^{-1}, \, (1 - \discounttil + 4
    \discounttil\invmix)^{-1}, \, 1 \big\} \, \big] \;
  \bweightrper}{2}^2 \notag \\
& = \frac{1}{(1-\discounttil)^2} \, \big( \bweightrper(1) \big)^2 +
\frac{1}{(1-\discounttil + 4\discounttil\invmix)^2} \, \big(
\bweightrper(2) \big)^2 + \big( \bweightrper(3) \big)^2 \, .
\end{align}
From equation~\eqref{eq:UDU}, we have
\begin{align}
\label{eq:bweightrper1<}
\big( \bweightrper(1) \big)^2 = \frac{\stdfunbar^2}{16} \, \big( \Dp
\cos\lbtheta + \tfrac{1}{\sqrt{2}} (1 + \Dp) \, \Dq \sin\lbtheta
\big)^2 \leq \frac{\stdfunbar^2}{8} \, \big\{ (\Dp)^2 \cos^2\lbtheta +
(\Dq)^2 \big\} \, .
\end{align}
Substituting the
bounds~\eqref{eq:bweightrper<}~and~\eqref{eq:bweightrper1<} into
equation~\eqref{eq:VstarMt_norm} yields
\begin{align}
\label{eq:VstarMt_norm2}
\distrnorm{\VstarMt}{\distrMtper}^2 & \leq \frac{\stdfunbar^2}{8} \,
\Big\{ \frac{\cos^2\lbtheta}{(1-\discounttil + 4
  \discounttil\invmix)^2} + \sin^2\lbtheta \Big\} + \frac{\stdfunbar^2
  \; \cos^2\lbtheta}{8 \, (1-\discounttil)^2} \, (\Dp)^2 +
\frac{\stdfunbar^2}{4 \, (1-\discounttil)^2} \, (\Dq)^2 \, .
\end{align}
The leading term can be written as
\begin{align}
\label{eq:VstarMt_norm3}
\frac{\cos^2\lbtheta}{(1-\discounttil + 4\discounttil\invmix)^2} +
\sin^2\lbtheta \leq \frac{2(1-2\invmix)}{(1 - \discounttil +
  4\discounttil\invmix)^2} \cos^2\lbtheta + \frac{1}{8\invmix}
\sin^2\lbtheta = \frac{\scalarzero}{4\invmix} \, ,
 \end{align}
where the scalar $\scalarzero$ was previously
defined~\eqref{eq:def_scalar}.  Moreover, the bound~\eqref{eq:cos<?}
implies that
\begin{align}
\label{eq:VstarMt_norm4}
\cos \lbtheta \; \leq \; \frac{2\sqrt{2} \, \lbnormperp \, (1 -
  \discounttil + 4 \discounttil \invmix)}{\stdfunbar \; \discounttil(1
  - 4 \invmix)} \leq \frac{1 - \discounttil}{12 \sqrt{\invmix}}
 \end{align}
when the bound $\invmix = \mixtime^{-1}/8 \leq (1-\discounttil)/8$
holds. Combining
inequalities~\eqref{eq:VstarMt_norm2}~to~\eqref{eq:VstarMt_norm4}
yields
\begin{align}
\label{eq:VstarMtnorm<}
\distrnorm{\VstarMt}{\distrMtper}^2 & \leq \frac{\scalarzero \,
  \stdfunbar^2}{32 \, \invmix} + \frac{\stdfunbar^2}{144 \, \invmix}
\, (\Dp)^2 + \frac{\stdfunbar^2}{\invmix} \frac{\invmix
}{4(1-\discounttil)^2} \, (\Dq)^2
\end{align}
We then apply conditions~$\Dp \leq \tfrac{1}{3}$ and $\Dq \leq \frac{1
  - \discounttil + 4\discounttil\invmix}{2\sqrt{2\invmix}} \leq
\frac{1-\discounttil}{\sqrt{3\invmix}}$ in equation~\eqref{cond:Dp<},
thereby obtaining inequality~\eqref{eq:smallMRP_munorm} in
\Cref{lemma:smallMRP_list}(c).


\subsection{Proof of Lemma~\ref{lemma:smallMRP_list}(d)}
\label{append:proof:lemma:smallMRP_lbnormperp}

In this section, we analyze the norm of the function $\VstarMtperperp$
associated with the orthogonal complement.  We first define a vector
\begin{align}
\label{eq:def_featureperperp}
\featureMtperperp \defn \const{\perp} \cdot [\, \diag(\distrMtper)
  \,]^{-1} \, [\, \diag(\distrMtbase) \,] \; \featureMtperp \, ,
\end{align}
where $\featureMtperp \defn \bUbase \, \bweightperp$ with $\bweightperp = (0, \sin\lbtheta, -\cos\lbtheta)^{\top} \in \Real^3$; $\const{\perp} > 0$ is a scalar that ensures
$\norm{\featureMtperperp}_{\distrMtper} = 1$.  We then express
function~$\VstarMtperp$ using feature $\featureperperp$. It follows
that
\begin{align}
\label{eq:VstarMtperp_featureperperp}
\VstarMtperp =
\inprod[\big]{\VstarMtper}{\featureMtperperp}_{\distrMtper} \cdot
\featureMtperperp,
\end{align}
and therefore that
\begin{align}
\label{eq:Vstarperperp_norm}
\norm{\VstarMtperperp}_{\distrMtper} = \big| \,
\inprod[\big]{\VstarMtper}{\featureMtperperp}_{\distrMtper} \, \big| =
\const{\perp} \cdot \big| \,
\inprod[\big]{\VstarMtper}{\featureMtperp}_{\distrMtbase} \, \big|.
 \end{align}
We now bound $\const{\perp}$ and the inner product
$\inprod[\big]{\VstarMtper}{\featureMtperp}_{\distrMtbase}$ in turn.

\paragraph{Bounding the inner product:}
We first calculate the inner product
$\inprod[\big]{\VstarMtper}{\featureMtperp}_{\distrMtbase}$
explicitly. Using the expression~\eqref{eq:Vstarper_decomp} for the
function $\VstarMtper$, we have
\begin{align*}
\bUbase^{\top} [ \, \diag(\distrMtbase) \, ] \, \VstarMtper = \big( \,
\bUbase^{\top} [\,\diag(\distrMtbase)\,] \, \bUper \, \big) \; \big[
  \, \diag\big\{ (1\!-\!\discounttil)^{-1}, \, (1 \!-\!  \discounttil
  \!+\! 4\discounttil\invmix)^{-1}, \, 1 \big\} \, \big] \; \big( \,
\bUper^{\top} [\,\diag(\distrMtper)\,] \, \bUbase \, \big) \;
\bweightr \, .
 \end{align*}
 Note that $\bUbase^{\top} [\,\diag(\distrMtbase)\,] \, \bUper = \big(
 \, \bUper^{\top} [\,\diag(\distrMtper)\,] \, \bUbase \, \big)^{-1}$,
 where matrix $\bUper^{\top} [\,\diag(\distrMtper)\,] \, \bUbase$
 admits decomposition~\eqref{eq:UDU}, therefore, we find that
 \begin{multline}
 \bUbase^{\top} [ \, \diag(\distrMtbase) \, ] \, \VstarMtper = \IdMtper^{-1} \; \big[ \, \diag\big\{ (1-\discounttil)^{-1}, \, (1 - \discounttil + 4\discounttil\invmix)^{-1}, \, 1 \big\} \, \big] \; \IdMtper \; \bweightr \\
 = \big[ \, \diag\big\{ (1-\discounttil)^{-1}, \, (1 - \discounttil + 4\discounttil\invmix)^{-1}, \, 1 \big\} \, \big]
 \Bigg\{ \bweightr + \frac{\stdfunbar}{4} \begin{pmatrix}
 \frac{4 \discounttil \invmix}{1 - \discounttil + 4 \discounttil \invmix} & 1 + \frac{4 \invmix}{1 - \discounttil + 4 \discounttil \invmix} \Dp \\
 0 & 1 - 4 \invmix \\
 0 & 0
 \end{pmatrix}
 \begin{pmatrix}
 \Dp \cos\lbtheta \\ \tfrac{\discounttil}{\sqrt{2}} \Dq \sin\lbtheta
 \end{pmatrix} \Bigg\} \, . \label{eq:UVstarper}
 \end{multline}
 It follows that
 \begin{align*}
 \inprod[\big]{\VstarMtper}{\featureMtperp}_{\distrMtbase} & = \bweightperp^{\top} \bUbase^{\top} [ \, \diag(\distrMtbase) \, ] \, \VstarMtper \notag = \frac{\stdfunbar \; \discounttil(1 - 4 \invmix)}{4(1 - \discounttil + 4 \discounttil \invmix)} \cos\lbtheta \sin\lbtheta \cdot \big( 1 + \tfrac{1}{\sqrt{2}} \Dq \tan\lbtheta \big) \, ,
 \end{align*}
 where the vector $\bweightperp$ is given by $\bweightperp = (0, \sin\lbtheta, -\cos\lbtheta)^{\top} \in \Real^3$
 We recall that the definition~\eqref{eq:def_lbtheta} of angle
 $\lbtheta$ ensures that \mbox{$\cos\lbtheta \sin\lbtheta =
 	\tfrac{1}{2} \sin(2\lbtheta) = \frac{2 \, \lbnormperp \, (1 -
 		\discounttil + 4 \discounttil \invmix)}{\stdfunbar \;
 		\discounttil(1 - 4 \invmix)}$,} and hence
 \begin{align}
 \label{eq:inprod_Vstarper0}
 \inprod[\big]{\VstarMtper}{\featureMtperp}_{\distrMtbase} & = \frac{\lbnormperp}{2} \cdot \big( 1 + \tfrac{1}{\sqrt{2}} \Dq \tan\lbtheta \big) \, .
 \end{align}
 The bounds~\eqref{eq:cos<} and $\Dq \leq \frac{2\sqrt{2} \; \lbnormperp}{3 \; \stdfunbar} \, (1 - \discounttil + 4 \discounttil \invmix)$ in condition~\eqref{cond:Dp<} imply that
 \begin{align}
 \label{eq:Dp2tan<}
 \Dq \, \tan\lbtheta \leq \frac{\Dq}{\cos\lbtheta} \leq
 \frac{\stdfunbar \; \discounttil(1 - 4 \invmix)}{2\lbnormperp \, (1 -
   \discounttil + 4 \discounttil \invmix)} \, \Dq \leq
 \frac{\sqrt{2}}{3},
 \end{align}
 whence
 \begin{align}
 \label{eq:inprod_Vstarper}
 0 \leq \inprod[\big]{\VstarMtper}{\featureMtperp}_{\distrMtbase} \leq \frac{2}{3} \, \lbnormperp \, .
 \end{align}
 
\paragraph{Bounding the parameter $\const{\perp}$:} As for
 the parameter $\const{\perp}$, we have $\const{\perp}^{-2} =
 \featureMtperp^{\top} \, [\, \diag(\distrMtper) \,]^{-1} \, [\,
   \diag(\distrMtbase) \,]^2 \, \featureMtperp$ according to
 definition~\eqref{eq:def_featureperperp} of vector
 $\featureMtperperp$, which reduces to
 \begin{subequations}
   \begin{align}
     \label{eq:def_constperp}
     \const{\perp}^2 = \big(1 - (\Dp)^2 \big) \big( 1 - (\Dq)^2 \big)
     \cdot \consttil{\perp} \, ,
   \end{align}
   where the parameter $\consttil{\perp} > 0$ is given by
   \begin{align}
     \label{eq:def_consttilperp}
     \consttil{\perp}^{-1} \defn 1 - \Dp \, {\cos^2\lbtheta} +\sqrt{2} (1 -
     \Dp) \Dq \, {\cos\lbtheta} \, {\sin\lbtheta} - \tfrac{1}{2} (1 + \Dp)
     (\Dq)^2 \, {\sin^2\lbtheta}.
   \end{align}
 \end{subequations}
 We claim that
 \begin{align}
 \label{eq:temp_bound1}
 & \tfrac{1}{2} (1 + \Dp) (\Dq)^2 \, {\sin^2\lbtheta} \leq
 \tfrac{1}{3} \cdot \sqrt{2} (1 - \Dp) \Dq \, {\cos\lbtheta} \,
 {\sin\lbtheta} \, .
 \end{align}
Indeed, due to the bound~$\Dq \leq \frac{2\sqrt{2} \; \lbnormperp}{3
  \; \stdfunbar} \, (1 - \discounttil + 4 \discounttil \invmix)$ in
condition~\eqref{cond:Dp<}, we have the lower bound \mbox{$\cos
  \lbtheta \; \geq \; \frac{2 \, \lbnormperp \, (1 - \discounttil + 4
    \discounttil \invmix)}{\stdfunbar \; \discounttil(1 - 4 \invmix)}
  \geq \frac{3}{\sqrt{2}} \, \Dq$,} which implies
inequality~\eqref{eq:temp_bound1} under condition~$\Dp \leq
\frac{1}{3}$. It follows from equation~\eqref{eq:def_constperp} that
\mbox{$\const{\perp}^2 \leq \consttil{\perp} \leq \big( 1 - \Dp \,
  {\cos^2\lbtheta} \big)^{-1} \leq \frac{3}{2}$.}

Substituting this bound along with that in
equation~\eqref{eq:inprod_Vstarper} into
equation~\eqref{eq:Vstarperperp_norm} completes the proof of
inequality~\eqref{eq:smallMRP_lbnormperp}.


\subsection{Proof of Lemma~\ref{lemma:smallMRP_list}(e)}
\label{append:proof:lemma:chidiv_discrete}

We first estimate the $\chi^2$-divergence
$\chidiv[\big]{\distrMt_1}{\distrMtper_2}$ in
inequality~\eqref{eq:smallMRP_chisq_distr}. We calculate that
\begin{align*}
 \chidiv[\big]{\distrMt_1}{\distrMtper_2} = \sum_{\state \in \StateSp}
 \frac{(\distrMtper_1 -
   \distrMtper_2)^2(\state)}{\distrMtper_2(\state)} & = \frac{(\Dp_1 -
   \Dp_2)^2}{1 - (\Dp_2)^2} + \frac{(1 + \Dp_1)^2 \, (\Dq_1 -
   \Dq_2)^2}{2 \, (1 + \Dp_2) \, \big(1 - (\Dq_2)^2\big)} \, .
 \end{align*}
It follows from condition~\eqref{cond:Dp<} that
\mbox{$\chidiv[\big]{\distrMt_1}{\distrMtper_2} \; \leq \; 2 \, \big\{
  (\Dp_1 - \Dp_2)^2 + (\Dq_1 - \Dq_2)^2 \big\}$,} which completes the
proof of inequality~\eqref{eq:smallMRP_chisq_distr}.

We now deal with the other $\chi^2$-divergence $\Exp_{\State \sim
  \distrMtper_1} \big[
  \chidiv[\big]{\TransMtper_1(\cdot\mid\State)}{\TransMtper_2(\cdot\mid\State)}
  \big]$ in inequality~\eqref{eq:smallMRP_chisq_P}.  Note that by
definition, we have
\begin{align}
\Exp_{\State \sim \distrMtper_1}\big[ \,
  \chidiv[\big]{\TransMtper_1(\cdot\mid\State)}{\TransMtper_2(\cdot\mid\State)}
  \big] & = \sum_{\state, \statenew \in \StateSp}
\distrMtper_1(\state) \, \frac{(\TransMtper_1 -
  \TransMtper_2)^2(\statenew \mid \state)}{\TransMtper_2(\statenew
  \mid \state)} \notag \\
 \label{eq:smallMRP_chidiv_P1}
& \leq \supnorm[\Big]{\frac{\diff \, \distrMtper_1}{\diff \,
     \distrMtbase}} \supnorm[\Big]{\frac{\diff \, \TransMtbase}{\diff
     \, \TransMtper_2}} \sum_{\state, \statenew \in \StateSp}
 \distrMtbase(\state) \, \frac{(\TransMtper_1 -
   \TransMtper_2)^2(\statenew \mid \state)}{\TransMtbase(\statenew
   \mid \state)}.
 \end{align}
The density ratios can be bounded as
\begin{align}
 \label{eq:smallMRP_chidiv_P2}  
 \supnorm[\Big]{\frac{\diff \, \distrMtper_1}{\diff \, \distrMtbase}}
 = (1 + \Dp_1)(1 + \Dq_1) \leq 2 \qquad \text{and} \qquad
 \supnorm[\Big]{\frac{\diff \, \TransMtbase}{\diff \, \TransMtper_2}}
 \leq \big( 1 - \max\{ \Dp_2, \Dq_2\} \big)^{-1} \leq 2.
 \end{align}
By calculation, we also find that
\begin{multline*}
 \sum_{\state, \statenew \in \StateSp} \distrMtbase(\state) \,
 \frac{(\TransMtper_1 - \TransMtper_2)^2(\statenew \mid
   \state)}{\TransMtbase(\statenew \mid \state)} = \frac{2\invmix}{1 -
   2 \invmix} (\Dp_1 - \Dp_2)^2 + \frac{1}{2} \, (1 - 4\invmix)^2
 (\Dq_1 - \Dq_2)^2 \\ + \frac{\invmix}{1-2\invmix} \big\{2(1-2\invmix)
 (\Dq_1 - \Dq_2) + (\Dp_1 \Dq_1 - \Dp_2 \Dq_2) \big\}^2 .
 \end{multline*}
Under conditions~\eqref{cond:Dp<} and $\invmix = \mixtimebar^{-1}/8
\leq \tfrac{1}{8}$, we can prove that
\begin{align}
\label{eq:smallMRP_chidiv_P3}
\sum_{\state, \statenew \in \StateSp} \distrMtbase(\state) \,
\frac{(\TransMtper_1 - \TransMtper_2)^2(\statenew \mid
  \state)}{\TransMtbase(\statenew \mid \state)} \leq 3 \invmix \,
(\Dp_1 - \Dp_2)^2 + 2 \, (\Dq_1 - \Dq_2)^2 \, .
\end{align}
Substituting the
bounds~\eqref{eq:smallMRP_chidiv_P2}~and~\eqref{eq:smallMRP_chidiv_P3}
into equation~\eqref{eq:smallMRP_chidiv_P1} yields the
inequality~\eqref{eq:smallMRP_chisq_P} claimed
in~\Cref{lemma:smallMRP_list}(e).


\subsection{Proof of Lemma~\ref{lemma:smallMRP_list}(f)}
\label{append:proof:lemma:Vgap_discrete}

In this part, we analyze the difference between two projected value functions $\ValueMtpar{*}{1}$ and $\ValueMtpar{*}{2}$. We use the base MRP $\MRPMtbase$ as a reference and decompose the difference as
\begin{align}
\label{eq:Vgap_decomp0}
\ValueMtpar{*}{1} - \ValueMtpar{*}{2} = \big( \ValueMtpar{*}{1} - \VstarMtbasepar \big) - \big( \ValueMtpar{*}{2} - \VstarMtbasepar \big) \, .
\end{align}
A key step in our analysis is to carefully bound
$\big(\ValueMtpar{*}{1} - \VstarMtbasepar\big)$ and
$\big(\ValueMtpar{*}{2} - \VstarMtbasepar\big)$. So as to simplify
the notation, we use $\big(\VstarMtperpar - \VstarMtbasepar\big)$ to
refer to either function gap $\big(\ValueMtpar{*}{1} -
\VstarMtbasepar\big)$ or $\big(\ValueMtpar{*}{2} -
\VstarMtbasepar\big)$. We will add the subscripts
back when needed. \\

We now decompose the value function gap $\VstarMtperpar -
\VstarMtbasepar$ into two terms $\Terms$ and $\Terma$, and analyze
them in turn.  In particular, we have
\begin{align}
  \label{eq:Vgap_decomp}  
    \VstarMtperpar - \VstarMtbasepar =
    \underbrace{\proj_{\distrMtbase} (\VstarMtper -
      \VstarMtbase)}_{\Terms} + \underbrace{(\proj_{\distrMtper} -
      \proj_{\distrMtbase}) \VstarMtper}_{\Terma} .
  \end{align}
Here $\Terms$ reflects the fluctuation of the value function
$\VstarMt$, whereas $\Terma$ captures the difference in projection
operators, which is caused by the shift of stationary distribution
$\distrMt$. As shown in our analysis, the quantity $\Terms$
corresponds to the term with $\lbstdmtg \asymp (1-\discounttil)^{-1}
\stdfunbar$ in the bound, whereas $\Terma$ reduces to the term with
$\lbapproxerr = \sqrt{\mixtimebar} \, \lbnormperp$.

In~\Cref{append:Terms} below, we bound $\Terms$. \Cref{append:Terma}
is concerned with the analysis of $\Terma$. We combine the
results and derive bounds on $\big( \ValueMtpar{*}{1} -
\ValueMtpar{*}{2} \big)$ in~\Cref{append:Terms+Terma}.


\subsubsection{Analysis of term $\Terms$}
\label{append:Terms}

Note that function $\proj_{\distrMtbase} (\VstarMtper) $ can be
written in a vector form as
\begin{align*}
\proj_{\distrMtbase} (\VstarMtper) & = \FeatureMt(\cdot)^{\top}
\; \begin{pmatrix} 1 & 0 & 0 \\ 0 & \cos\lbtheta & \sin\lbtheta
\end{pmatrix} \; \bUbase^{\top} [ \, \diag(\distrMtbase) \, ] \, \VstarMtper.
\end{align*}
We use the expression for vector $\bUbase^{\top} [ \,
\diag(\distrMtbase) \, ] \, \VstarMtper \in \Real^3$ in
equation~\eqref{eq:UVstarper}. It follows that
\begin{align*}
\Terms(\cdot) = \big(\proj_{\distrMtbase} (\VstarMtper -
\VstarMtbase)\big)(\cdot) & = \frac{\stdfunbar}{4} \;
\FeatureMt(\cdot)^{\top}
\begin{pmatrix}
\frac{4 \discounttil \invmix}{(1-\discounttil)(1 - \discounttil + 4
	\discounttil \invmix)} & \frac{1}{1-\discounttil}\big\{ 1 + \frac{4
	\invmix}{1 - \discounttil + 4 \discounttil \invmix} \Dp \big\} \\ 0
& \frac{1 - 4 \invmix}{1-\discounttil+4\discounttil\invmix}
\cos\lbtheta
\end{pmatrix}
\begin{pmatrix}
\Dp \cos\lbtheta \\ \tfrac{\discounttil}{\sqrt{2}} \Dq \sin\lbtheta
\end{pmatrix} \, .
\end{align*}
Now we have the decomposition
$\Terms = \Termtils + \DTerms$,
where
\begin{subequations}
	\begin{align}
	& \Termtils(\cdot) \defn \FeatureMt(\cdot)^{\top}
	\begin{pmatrix}
	\frac{\stdfunbar \; \discounttil \invmix}{(1-\discounttil)(1 -
		\discounttil + 4 \discounttil \invmix)} \cos\lbtheta & \frac{
		\stdfunbar \; \discounttil}{4(1-\discounttil)} \sin\lbtheta \\ 0 &
	\lbnormperp
	\end{pmatrix}
	\begin{pmatrix}
	\Dp \\ \tfrac{1}{\sqrt{2}} \Dq
	\end{pmatrix} \, , \label{eq:def_Termtils} \\
	& \DTerms(\cdot) \defn \frac{\stdfunbar \; \discounttil
		\invmix}{\sqrt{2} \, (1-\discounttil)(1 - \discounttil + 4
		\discounttil \invmix)} \Dp \Dq \sin\lbtheta \cdot
	\featureMt{1}(\cdot) \, . \notag
	\end{align}
\end{subequations}
Term $\DTerms$ is of higher order: indeed, when we take $\Dp =
\frac{1}{\sqrt{\invmix}} \Dq$, it satisfies
\begin{align}
\label{eq:DTerms<}
\norm{\DTerms}_{\distrMtbase} \leq \frac{ \stdfunbar \; \Dp
	\Dq}{4\sqrt{2} \, (1-\discounttil)} = \frac{ \stdfunbar}{2 \, (1 -
	\discounttil)} \, \sqrt{\mixtimebar} \, (\Dq)^2 \stackrel{(i)}{\leq} \frac{ \lbstdmtg}{2} \, \sqrt{\mixtimebar} \, (\Dq)^2 \leq \frac{ \lbnoise}{2} \, \sqrt{\mixtimebar} \, (\Dq)^2 \, .
\end{align}
The step $(i)$ is due to $\lbstdmtg = (1 - \discount)^{-1} \stdfunbar \geq (1 - \discounttil)^{-1} \stdfunbar$.

We remark that the leading term $\Termtils$ ultimately leads to the
quantity $\lbstdmtg$ in our bounds. Indeed, suppose that we take $\Dp
= \frac{1}{\sqrt{\invmix}} \Dq \geq 0$ in
equation~\eqref{eq:def_Termtils}. For models $\MRPMt_1$ and
$\MRPMt_2$, we have
\begin{align}
\label{eq:def_Termtils-}  
\Termtilsi{1} - \Termtilsi{2} =
\FeatureMt(\cdot)^{\top} \tvectils \! \cdot \! (\Dq_1 \!-\! \Dq_2) ~~
\text{where } \tvectils \! \defn \! \begin{pmatrix} \! \frac{\stdfunbar \;
	\discounttil \sqrt{\invmix}}{(1-\discounttil)(1 - \discounttil + 4
	\discounttil \invmix)} \cos\lbtheta + \frac{ \stdfunbar \;
	\discounttil}{4\sqrt{2}(1-\discounttil)} \sin\lbtheta
\\ \tfrac{1}{\sqrt{2}} \lbnormperp
\end{pmatrix}\!.
\end{align}
It can be seen that $\frac{\stdfunbar \; \discounttil}{1-\discounttil}
\! \cdot \!\! \sqrt{\scalarzero} \lesssim \tvectils(1) \!\leq\!
\distrnorm{\tvectils}{2} \!\leq\! \distrnorm{\tvectils}{1} \lesssim
\frac{\stdfunbar \; \discounttil}{1-\discounttil} \!\cdot\!\!
\sqrt{\scalarzero} + \lbnormperp$, where the scalar $\scalarzero$ was
previously defined~\eqref{eq:def_scalar}.  Since $\scalarzero \asymp
1$ according to the bounds in \Cref{append:proof:lemma:smallMRP_norm} and $\lbnormperp
\lesssim (1-\discounttil)^{-1} \stdfunbar$ due to
condition~\eqref{cond:norm_ratio}, we have $\distrnorm{\tvectils}{2}
\asymp \distrnorm{\tvectils}{1} \asymp (1-\discounttil)^{-1}
\stdfunbar$. Note that $\lbstdmtg = (1-\discount)^{-1} \stdfunbar
\asymp (1 - \discounttil)^{-1} \stdfunbar$ under the
condition~$\mixtimebar \geq (1 - \discount)^{-1}$ in
equation~\eqref{cond:norm_ratio}. It follows that
\begin{align}
\label{eq:Termtils<}
\tvectils(1) \! \stackrel{(a)}{\asymp} \distrnorm{\tvectils}{2} \asymp
\distrnorm{\tvectils}{1} \asymp \lbstdmtg ~\, \mbox{and} ~~
\distrnorm[\big]{\Termtilsi{1} \!- \Termtilsi{2}}{\distrMtbase} \!\!
\stackrel{(b)}{\asymp} \supnorm[\big]{\Termtilsi{1} \!- \Termtilsi{2}}
\! \asymp \lbstdmtg \, |\Dq_1 - \Dq_2| \, .
\end{align}


\subsubsection{Analysis of term $\Terma$}
\label{append:Terma}

We first derive a vector representation of $\Terma$ based on \Cref{lemma:lb_approxerr}.
By comparing the base MRP (with stationary distribution $\distrMtbase$) and the perturbed model (with value function $\VstarMt$ and stationary measure $\distrMt$), we find that \Cref{lemma:lb_approxerr} implies
$
\inprod{\featureMt{j}}{\Terma}_{\distrMtbase} =
\inprod[\big]{\featureMt{j}}{(\proj_{\distrMtbase} - \proj_{\distrMt})
  \VstarMt}_{\distrMtbase} =
\inprod[\big]{\featureMt{j}}{\VstarMtperp}_{\distrMtbase - \distrMt} =
\big( \Exp_{\State \sim \distrMtper} - \Exp_{\State \sim \distrMtbase}
\big) \big[ \, \featureMt{j}(\State) \; \VstarMtperperp(\State) \,
  \big]
$
for each $j = 1, 2, 3$.  By our construction, the features $\{ \featureMt{j}
\}_{j=1}^3$ are orthonormal with respect to distribution
$\distrMtbase$, therefore, term~$\Terma$ takes the form
\begin{align*}
\textstyle
\Terma(\cdot) \; = \; \sum_{j} \, \inprod{\featureMt{j}}{\Terma}_{\distrMtbase} \, \featureMt{j}(\cdot) \; = \;
\FeatureMt(\cdot)^{\top} \; \big( \Exp_{\State \sim \distrMtper} -
\Exp_{\State \sim \distrMtbase} \big) \big[ \, \FeatureMt(\State) \;
\VstarMtperperp(\State) \, \big] \, .
\end{align*}
The expression~\eqref{eq:VstarMtperp_featureperperp} for
$\VstarMtperp$ reveals
$\VstarMtperperp = \const{\perp}^2
\inprod{\VstarMtper}{\featureMtperp}_{\distrMtbase} \! \cdot [\,
\diag(\distrMtper) \,]^{-1} [\, \diag(\distrMtbase) \,] \,
\featureMtperp$.
It follow that
\begin{align*}
\Terma(\cdot) \; = \; 2 \, \lbnormperp^{-1} \, \const{\perp}^2 \;
\inprod[\big]{\VstarMtper}{\featureMtperp}_{\distrMtbase} \cdot \;
\FeatureMt(\cdot)^{\top} \; \yvec \, ,
\end{align*}
where the vector $\yvec \in \Real^2$ is defined as
\begin{align*}
\yvec \defn \frac{\lbnormperp}{2} \, \big( \Exp_{\State \sim
	\distrMtper} - \Exp_{\State \sim \distrMtbase} \big) \big[ \,
\FeatureMt(\State) \; \big( [\, \diag(\distrMtper) \,]^{-1} \, [\,
\diag(\distrMtbase) \,] \; \featureMtperp \big)(\State) \, \big].
\end{align*}
The explicit forms of
$\inprod[\big]{\VstarMtper}{\featureMtperp}_{\distrMtbase}$ and
$\const{\perp}^2$ are shown in
equations~\eqref{eq:inprod_Vstarper0}~and~\eqref{eq:def_constperp}. In the sequel, we first analyze the vector $\yvec$ and decompose $\Terma$ into the leading term $\Termtila$ and higher order part~$\DTerma$. Then, we bound $\Termtila$ and $\DTerma$ in turn.

\paragraph{Analyzing the vector $\yvec$:}
We now focus on analyzing the vector $\yvec$.  By definition, we have
\begin{align*}
\yvec & = \frac{\lbnormperp}{2} \begin{pmatrix} 1 & 0 & 0 \\ 0 &
\cos\lbtheta & \sin\lbtheta
\end{pmatrix} \bZ \; \bweightperp \qquad \text{with }\bZ \defn \bUbase^{\top} \big[ \, \diag(\distrMtper - \distrMtbase) \, \big] [\, \diag(\distrMtper) \,]^{-1} \, [\, \diag(\distrMtbase) \,] \; \bUbase \, .
\end{align*}
The vector $\bweightperp$ here is given by $\bweightperp = (0, \, \sin\lbtheta, \, -\cos\lbtheta)^{\top} \in \Real^3$.
The matrix $\bZ \in \Real^{3 \times 3}$ can be written explicitly as
\mbox{$\bZ = \frac{1}{(1 - (\Dp)^2)(1 - (\Dq)^2)}
	\,\big(\bZbase + \DbZ\big)$,} where
\begin{align*}
& \bZbase \defn \begin{pmatrix} 0 & \Dp & \frac{1}{\sqrt{2}}\Dq \\ \Dp
& 0 & \frac{1}{\sqrt{2}}\Dq \\ \frac{1}{\sqrt{2}} \Dq &
\frac{1}{\sqrt{2}} \Dq & \Dp
\end{pmatrix} \, ,
\end{align*}
and the higher order term $\DbZ$ is bounded as $\opnorm{\DbZ}_F \leq 3
\, \big\{ (\Dp)^2 + (\Dq)^2 \big\}$.\footnote{The notation $\opnorm{\cdot}_F$ stands for the Frobenius norm of matrices.} Thus, we can write $\yvec =
\yvecbase + \Dyvec$ with
\begin{subequations}
	\begin{align}
	& \yvecbase \defn \frac{\lbnormperp}{2} \begin{pmatrix} 1 & 0 & 0
	\\ 0 & \cos\lbtheta & \sin\lbtheta
	\end{pmatrix} \bZbase \; \bweightperp = \frac{\lbnormperp}{2} \begin{pmatrix}
	\sin\lbtheta & - \cos\lbtheta \\ -\tfrac{1}{2} \sin(2\lbtheta) & -
	\cos(2\lbtheta)
	\end{pmatrix} \begin{pmatrix}
	\Dp \\ \tfrac{1}{\sqrt{2}} \Dq
	\end{pmatrix} \, , \text{ and} \label{eq:def_yvecbase} \\
	& \norm{\Dyvec}_2 \leq \tfrac{\lbnormperp}{2} \, \opnorm{\DbZ}_2 \leq
	\tfrac{\lbnormperp}{2} \, \opnorm{\DbZ}_F \leq \tfrac{3}{2} \,
	\lbnormperp \, \big\{ (\Dp)^2 + (\Dq)^2 \big\} \, . \label{eq:Dyec<}
	\end{align}
\end{subequations}

Based on the decomposition of vector $\yvec$, we rewrite $\Terma$ as $\Terma = \Termtila + \DTerma$,
where
\begin{subequations}
\begin{align}
  \Termtila & \defn \FeatureMt(\cdot)^{\top} \, \yvecbase \, , \qquad
  \text{and} \label{eq:def_Termtila} \\
  \DTerma & \defn \; \consttil{\perp}\, \big( 1 + \tfrac{1}{\sqrt{2}} \Dq
  \tan\lbtheta \big) \cdot \FeatureMt(\cdot)^{\top} \, \Dyvec \; + \;
  \big\{ (\consttil{\perp} - 1) + \consttil{\perp} \cdot
  \tfrac{1}{\sqrt{2}} \Dq \tan\lbtheta \, \big\} \cdot
  \FeatureMt(\cdot)^{\top} \, \yvecbase \, . \label{eq:def_DTerma}
\end{align}
\end{subequations}
We consider $\Termtila$ as the leading term and $\DTerma$ as the higher order one.

\paragraph{Connecting the leading term $\Termtila$ with $\lbapproxerr$:}
We remark that $\Termtila$ is connected with term \mbox{$\lbapproxerr = \sqrt{\mixtimebar} \, \lbnormperp$}. We learn from equation~\eqref{eq:def_Termtila} that
when taking $\Dp = \frac{1}{\sqrt{\invmix}} \Dq$, the leading term
$\big(\Termtilai{1} - \Termtilai{2}\big)$ satisfies
\begin{align}
\big(\Termtilai{1} - \Termtilai{2}\big)(\cdot) =
\FeatureMt(\cdot)^{\top} \tvectila \cdot (\Dq_1 - \Dq_2) \quad
\text{where } \tvectila \defn
\frac{\lbnormperp}{4\sqrt{\invmix}} \begin{pmatrix} 2\sin\lbtheta -
\sqrt{2 \invmix} \, \cos\lbtheta \\ - \sin(2\lbtheta) - \sqrt{2
	\invmix} \, \cos(2\lbtheta)
\end{pmatrix} \, .
\label{eq:def_Termtila-}
\end{align}
It is easy to see that $\tvectila(1) \geq 0$ and
\begin{align}
\label{eq:Termtila<}
\distrnorm{\tvectila}{2} \asymp \distrnorm{\tvectila}{1} \asymp \tvectila(1) \asymp
\sqrt{\mixtimebar} \, \lbnormperp = \lbapproxerr
\end{align}
under the conditions~$\cos\lbtheta \leq \tfrac{1}{\sqrt{2}} \leq
\sin\lbtheta$ and $\invmix = \tfrac{1}{8 \mixtime} \leq
\tfrac{1}{8}$. It further implies
\begin{align}
\label{eq:Termtila<'}
\distrnorm[\big]{\Termtilai{1} - \Termtilai{2}}{\distrMtbase}
\asymp\supnorm[\big]{\Termtilai{1} - \Termtilai{2}} \asymp
\lbapproxerr \, |\Dq_1 - \Dq_2| \, .
\end{align}

\paragraph{Bounding the higher order term $\DTerma$:}
In the following, we control the higher order term $\DTerma$ in
equation~\eqref{eq:def_DTerma} with respect to the
$\distrMtbase$-weighted norm.

We first bounding the parameter $\consttil{\perp}$. According to the bound \mbox{$\consttil{\perp} \leq \big( 1 - \Dp \,
	{\cos^2\lbtheta} \big)^{-1}$} given in \Cref{append:proof:lemma:smallMRP_lbnormperp} and the expression for constant
$\consttil{\perp}$ in equation~\eqref{eq:def_consttilperp}, we have
\begin{align*}
| \consttil{\perp} - 1| & \leq \max\Big\{ \frac{1}{1 - \Dp \,
	{\cos^2\lbtheta} } - 1, \, 1 - \frac{1}{1 + \sqrt{2} (1 - \Dp) \Dq
	\, {\cos\lbtheta} \, {\sin\lbtheta}} \Big\} \, .
\end{align*}
Using the conditions $0 \leq \Dp \leq \tfrac{1}{3}$, $0 \leq \Dq \leq
\tfrac{1}{3}$ and $0 \leq \cos\lbtheta \leq \tfrac{1}{\sqrt{2}}$, we
further have
$| \consttil{\perp} - 1|  \leq \max\big\{ \tfrac{6}{5} \, \Dp \, {\cos^2\lbtheta}, \, \sqrt{2}
(1 - \Dp) \Dq \, {\cos\lbtheta} \, {\sin\lbtheta} \big\} \leq
\max\big\{ \tfrac{6}{5} \, \Dp, \, \tfrac{1}{\sqrt{2}} \Dq \big\} \leq
\tfrac{2}{5}$.
It then follows from expression~\eqref{eq:def_DTerma} of $\DTerma$ that
\begin{align*}
  \distrnorm{\DTerma}{\distrMtbase} \leq \tfrac{7}{5} \, \big( 1 + \tfrac{1}{\sqrt{2}} \Dq
  \tan\lbtheta \big) \cdot \distrnorm{\Dyvec}{2} \; + \;
  \big\{ \tfrac{6}{5} \Dp +
  \tfrac{7}{5\sqrt{2}} \Dq \tan\lbtheta \, \big\} \cdot \distrnorm{\yvecbase}{2} \, .
\end{align*}

Moreover, recall from the
inequality~\eqref{eq:Dp2tan<} that $\tfrac{1}{\sqrt{2}} \Dq
\tan\lbtheta \leq \tfrac{1}{3}$. The bounds on $\cos\lbtheta$ in equation~\eqref{eq:cos<} also imply
$
	\tan \lbtheta \leq (\cos\lbtheta)^{-1} \leq \frac{\stdfunbar \;
		\discounttil(1 - 4 \invmix)}{\lbnormperp(1 - \discounttil + 4
		\discounttil \invmix)} \leq \lbnormperp^{-1} \; \lbstdmtg$.
Therefore, the previous upper bound on $\distrnorm{\DTerma}{\distrMtbase}$ reduces to
$
\distrnorm{\DTerma}{\distrMtbase} \leq 2 \, \distrnorm{\Dyvec}{2} \; + \;
\tfrac{6}{5} \, \big\{ \Dp + \lbnormperp^{-1} \; \lbstdmtg \, \Dq \, \big\} \, \distrnorm{\yvecbase}{2}$ .

Recall that $\distrnorm{\Dyvec}{2}$ has an upper bound~\eqref{eq:Dyec<}. Using the expression for $\yvecbase$ in
equation~\eqref{eq:def_yvecbase}, we also have \mbox{$\norm{\yvecbase}_2
  \leq \lbnormperp \, \big( \Dp + \Dq \big)$.}
Replacing $\norm{\yvecbase}_2$ and $\distrnorm{\Dyvec}{2}$ with their bounds then yields
$\distrnorm{\DTerma}{\distrMtbase} \leq 3 \, \lbnormperp
  \{ (\Dp)^2 + (\Dq)^2 \} + \tfrac{6}{5} \, \big\{ \lbnormperp
  \, \Dp + \lbstdmtg \; \Dq
  \big\} \cdot \big( \Dp + \Dq \big)$.
  
Finally, setting \mbox{$\Dp =
  \frac{1}{\sqrt{\invmix}} \, \Dq = 2\sqrt{2 \, \mixtimebar} \, \Dq
  \geq 0$} yields the conclusion
\begin{align}
\label{eq:DTerma<}
\distrnorm{\DTerma}{\distrMtbase} & \! \leq 3 \, \lbnormperp (8
\mixtimebar + 1) (\Dq)^2 \!+ \tfrac{12}{5} \sqrt{2} \, ( \sqrt{\mixtimebar}
\, \lbnormperp + \lbstdmtg ) \big(
2 \sqrt{2 \, \mixtimebar} + 1 \big) (\Dq)^2 \leq 40 \, \lbnoise \, \sqrt{ \mixtimebar} \, (\Dq)^2 \, .
\end{align}


\subsubsection{Putting together the pieces}
\label{append:Terms+Terma}

We now combine the bounds on terms $\Terms$ and $\Terma$ with the
decomposition~\eqref{eq:Vgap_decomp} of the value function gap
\mbox{$\ValueMtpar{*}{1} - \ValueMtpar{*}{2}$.} Our proof is based on
the following (claimed) relations:
\begin{subequations}
	\label{eq:Vgap_claim}
	\begin{align}
	\label{eq:Vgap_claim1}  
	\distrnorm[\big]{\big(\Termtilsi{1} - \Termtilsi{2}\big) +
		\big(\Termtilai{1} - \Termtilai{2}\big)}{\distrMtbase} & \asymp
	\lbnoise \cdot |\Dq_1 - \Dq_2|, \\
	\label{eq:Vgap_claim2}
	\sum_{i=1}^2 \big\{ \distrnorm{\DTermsi{i}}{\distrMtbase} +
	\distrnorm{\DTermai{i}}{\distrMtbase} \big\} & \leq \frac{1}{2} \,
	\distrnorm[\big]{\big(\Termtilsi{1} - \Termtilsi{2}\big) +
		\big(\Termtilai{1} - \Termtilai{2}\big)}{\distrMtbase}.
	\end{align}
\end{subequations}
Taking these claims as given for the moment, we apply them to the
decomposition
\begin{align*}
\ValueMtpar{*}{1} - \ValueMtpar{*}{2} = \big(\Termtilsi{1} -
\Termtilsi{2}\big) + \big(\Termtilai{1} - \Termtilai{2}\big) +
\big(\DTermsi{1} + \DTermai{1} - \DTermsi{2} - \DTermai{2}\big) \, .
\end{align*}
Applying the triangle inequality yields
$\distrnorm[\big]{\ValueMtpar{*}{1} - \ValueMtpar{*}{2}}{\distrMtbase}
\; \asymp \; \lbnoise \cdot |\Dq_1 - \Dq_2|$.  It is easy to see that
$\supnorm{\cdot}$ and $\distrnorm{\cdot}{\distrMtbase}$ are equivalent
in this case, in that
\begin{align}
  \label{eq:norm_equiv}
  & \tfrac{1}{2} \, \supnorm[\big]{\ValueMtpar{*}{1} -
    \ValueMtpar{*}{2}} \leq \distrnorm[\big]{\ValueMtpar{*}{1} -
    \ValueMtpar{*}{2}}{\distrMtbase} \leq
  \supnorm[\big]{\ValueMtpar{*}{1} - \ValueMtpar{*}{2}} \, .
\end{align}
Combining our earlier bound with equation~\eqref{eq:norm_equiv}
completes the proof of inequality~\eqref{eq:Vgap_discrete}
in \Cref{lemma:smallMRP_list}(f). \\


\noindent It remains to prove the
claims~\eqref{eq:Vgap_claim1}~and~\eqref{eq:Vgap_claim2}. \\

We first validate the lower bound part in
inequality~\eqref{eq:Vgap_claim1}.  Recall that the leading terms
\mbox{$\big(\Termtilsi{1} - \Termtilsi{2}\big)$} and
$\big(\Termtilai{1} - \Termtilai{2}\big)$ have feature representations
in equations~\eqref{eq:def_Termtils-} and \eqref{eq:def_Termtila-}
that involve two vectors $\tvectils, \tvectila \in \Real^2$.  The
first components of the vectors $\tvectils, \tvectila \in \Real^2$ are
both nonnegative. Moreover, they satisfy $\tvectils(1) \asymp
\lbstdmtg$, $\tvectila(1) \asymp \lbapproxerr$ according to the
bounds~\eqref{eq:Termtils<}(a)~and~\eqref{eq:Termtila<}.  It follows that
\mbox{$\distrnorm[\big]{\tvectils + \tvectila}{2} \; \geq \;
  \tvectils(1) + \tvectila(1) \; \gtrsim \; \lbstdmtg + \lbapproxerr =
  \lbnoise$,} whence
\begin{align}
\label{eq:Termtil>}
\distrnorm[\big]{\big(\Termtilsi{1} - \Termtilsi{2}\big) +
	\big(\Termtilai{1} - \Termtilai{2}\big)}{\distrMtbase} =
\distrnorm[\big]{\tvectils + \tvectila}{2} \cdot |\Dq_1 - \Dq_2|
\gtrsim \lbnoise \cdot |\Dq_1 - \Dq_2| \, .
\end{align}
As for the upper bound, combining
inequalities~\eqref{eq:Termtils<}(b)~and~\eqref{eq:Termtila<'} with
triangle inequality yields
\begin{multline}
\label{eq:Termtil<}
\distrnorm[\big]{\big(\Termtilsi{1} - \Termtilsi{2}\big) +
  \big(\Termtilai{1} - \Termtilai{2}\big)}{\distrMtbase} \leq
\distrnorm[\big]{\Termtilsi{1} - \Termtilsi{2}}{\distrMtbase} +
\distrnorm[\big]{\Termtilai{1} - \Termtilai{2}}{\distrMtbase} \\
\lesssim \big\{ \lbstdmtg + \lbapproxerr \big\} \cdot |\Dq_1 - \Dq_2| = \lbnoise \cdot |\Dq_1 - \Dq_2| \, .
\end{multline}
The bounds~\eqref{eq:Termtil>}~and~\eqref{eq:Termtil<} together
implies inequality~\eqref{eq:Vgap_claim1}.

Regarding the higher order terms $\distrnorm{\DTerms}{\distrMtbase}$ and
$\distrnorm{\DTerma}{\distrMtbase}$, we use bounds~\eqref{eq:DTerms<}
and \eqref{eq:DTerma<} and find that
\mbox{$\distrnorm{\DTermsi{i}}{\distrMtbase} +
  \distrnorm{\DTermai{i}}{\distrMtbase} \leq 50 \; \lbnoise
  \sqrt{\mixtimebar} \, (\Dq_i)^2$} \text{for $i = 1,2$}.
Condition~\eqref{cond:Dq^2} on parameters $\Dq_1$ and $\Dq_2$ further
implies that
\begin{align*}
\sum_{i=1}^2 \big\{ \distrnorm{\DTermsi{i}}{\distrMtbase} +
\distrnorm{\DTermai{i}}{\distrMtbase} \big\} & \leq 100 \; \lbnoise \sqrt{\mixtimebar} \, \max\big\{ (\Dq_1)^2, \, (\Dq_2)^2
\big\}
\leq 100 \; \consttil{5} \, \lbnoise \, \big| \Dq_1 - \Dq_2
\big| \, .
\end{align*}
Claim~\eqref{eq:Vgap_claim2} then holds for a sufficiently small
choice of the universal constant $\consttil{5} > 0$.



{\small{
\bibliographystyle{abbrv}
\bibliography{ref}

\begin{thebibliography}{10}

\bibitem{adamczak2008tail}
R.~Adamczak.
\newblock A tail inequality for suprema of unbounded empirical processes with
  applications to {Markov} chains.
\newblock {\em Electronic Journal of Probability}, 13:1000--1034, 2008.

\bibitem{adamczak2015exponential}
R.~Adamczak and W.~Bednorz.
\newblock Exponential concentration inequalities for additive functionals of
  {Markov} chains.
\newblock {\em ESAIM: Probability and Statistics}, 19:440--481, 2015.

\bibitem{antos2008learning}
A.~Antos, C.~Szepesv{\'a}ri, and R.~Munos.
\newblock Learning near-optimal policies with {Bellman}-residual minimization
  based fitted policy iteration and a single sample path.
\newblock {\em Machine Learning}, 71(1):89--129, 2008.

\bibitem{baird1995residual}
L.~Baird.
\newblock Residual algorithms: Reinforcement learning with function
  approximation.
\newblock In {\em Machine Learning Proceedings 1995}, pages 30--37. Elsevier,
  1995.

\bibitem{berlinet2011reproducing}
A.~Berlinet and C.~Thomas-Agnan.
\newblock {\em Reproducing kernel {H}ilbert spaces in probability and
  statistics}.
\newblock Springer Science \& Business Media, 2011.

\bibitem{bertsekas2011dynamic}
D.~P. Bertsekas.
\newblock {\em Dynamic programming and optimal control}, volume~II.
\newblock Athena Scientific, 3rd edition, 2011.

\bibitem{bertsekas2022abstract}
D.~P. Bertsekas.
\newblock {\em Abstract dynamic programming}.
\newblock Athena Scientific, 2022.

\bibitem{bertsekas1996neuro}
D.~P. Bertsekas and J.~N. Tsitsiklis.
\newblock {\em Neuro-dynamic programming}.
\newblock Athena Scientific, 1996.

\bibitem{bhandari2018finite}
J.~Bhandari, D.~Russo, and R.~Singal.
\newblock A finite time analysis of temporal difference learning with linear
  function approximation.
\newblock In {\em Conference on learning theory}, pages 1691--1692. PMLR, 2018.

\bibitem{boyan1999least}
J.~A. Boyan.
\newblock Least-squares temporal difference learning.
\newblock In {\em ICML}, pages 49--56, 1999.

\bibitem{bradtke1996linear}
S.~J. Bradtke and A.~G. Barto.
\newblock Linear least-squares algorithms for temporal difference learning.
\newblock {\em Machine learning}, 22(1-3):33--57, 1996.

\bibitem{chen2019information}
J.~Chen and N.~Jiang.
\newblock Information-theoretic considerations in batch reinforcement learning.
\newblock In {\em International Conference on Machine Learning}, pages
  1042--1051. PMLR, 2019.

\bibitem{chilina2006f}
O.~Chilina.
\newblock f-uniform ergodicity of {Markov} chains.
\newblock {\em Supervised Project, Unversity of Toronto}, 2006.

\bibitem{dedecker2015subgaussian}
J.~Dedecker and S.~Gou{\"e}zel.
\newblock Subgaussian concentration inequalities for geometrically ergodic
  {Markov} chains.
\newblock {\em Electronic Communications in Probability}, 20:1--12, 2015.

\bibitem{duan2021risk}
Y.~Duan, C.~Jin, and Z.~Li.
\newblock Risk bounds and {R}ademacher complexity in batch reinforcement
  learning.
\newblock In {\em International Conference on Machine Learning}, pages
  2892--2902. PMLR, 2021.

\bibitem{duan2020minimax}
Y.~Duan and M.~Wang.
\newblock Minimax-optimal off-policy evaluation with linear function
  approximation.
\newblock In {\em International Conference on Machine Learning}, pages
  2701--2709. PMLR, 2020.

\bibitem{duan2021optimal}
Y.~Duan, M.~Wang, and M.~J. Wainwright.
\newblock Optimal policy evaluation using kernel-based temporal difference
  methods.
\newblock {\em arXiv preprint arXiv:2109.12002}, 2021.

\bibitem{fan2020theoretical}
J.~Fan, Z.~Wang, Y.~Xie, and Z.~Yang.
\newblock A theoretical analysis of deep {Q}-learning.
\newblock In {\em Learning for Dynamics and Control}, pages 486--489. PMLR,
  2020.

\bibitem{farahmand2016regularized}
A.-M. Farahmand, M.~Ghavamzadeh, C.~Szepesv{\'a}ri, and S.~Mannor.
\newblock Regularized policy iteration with nonparametric function spaces.
\newblock {\em The Journal of Machine Learning Research}, 17(1):4809--4874,
  2016.

\bibitem{gheshlaghi2013minimax}
M.~Gheshlaghi~Azar, R.~Munos, and H.~J. Kappen.
\newblock Minimax {PAC} bounds on the sample complexity of reinforcement
  learning with a generative model.
\newblock {\em Machine learning}, 91(3):325--349, 2013.

\bibitem{giannoccaro2002inventory}
I.~Giannoccaro and P.~Pontrandolfo.
\newblock Inventory management in supply chains: a reinforcement learning
  approach.
\newblock {\em International Journal of Production Economics}, 78(2):153--161,
  2002.

\bibitem{Gu02}
C.~Gu.
\newblock {\em Smoothing spline {ANOVA} models}.
\newblock Springer {S}eries in {S}tatistics. Springer, New York, NY, 2002.

\bibitem{jaakkola1993convergence}
T.~Jaakkola, M.~Jordan, and S.~Singh.
\newblock Convergence of stochastic iterative dynamic programming algorithms.
\newblock {\em Advances in neural information processing systems}, 6, 1993.

\bibitem{khamaru2020temporal}
K.~Khamaru, A.~Pananjady, F.~Ruan, M.~J. Wainwright, and M.~I. Jordan.
\newblock Is temporal difference learning optimal? {A}n instance-dependent
  analysis.
\newblock {\em {S}{I}{A}{M} {J}. {M}ath. {D}ata Science}, page To appear, 2021.

\bibitem{komorowski2018artificial}
M.~Komorowski, L.~A. Celi, O.~Badawi, A.~C. Gordon, and A.~A. Faisal.
\newblock The artificial intelligence clinician learns optimal treatment
  strategies for sepsis in intensive care.
\newblock {\em Nature medicine}, 24(11):1716--1720, 2018.

\bibitem{lazaric2012finite}
A.~Lazaric, M.~Ghavamzadeh, and R.~Munos.
\newblock Finite-sample analysis of least-squares policy iteration.
\newblock {\em Journal of Machine Learning Research}, 13:3041--3074, 2012.

\bibitem{lemanczyk2021general}
M.~Lema{\'n}czyk.
\newblock General {B}ernstein-like inequality for additive functionals of
  {M}arkov chains.
\newblock {\em Journal of Theoretical Probability}, 34(3):1426--1454, 2021.

\bibitem{liu2015finite}
B.~Liu, J.~Liu, M.~Ghavamzadeh, S.~Mahadevan, and M.~Petrik.
\newblock Finite-sample analysis of proximal gradient {TD} algorithms.
\newblock In {\em Proceedings of the Thirty-First Conference on Uncertainty in
  Artificial Intelligence}, pages 504--513, 2015.

\bibitem{long20212}
J.~Long, J.~Han, and W.~E.
\newblock An ${L}^2$ analysis of reinforcement learning in high dimensions with
  kernel and neural network approximation.
\newblock {\em arXiv preprint arXiv:2104.07794}, 2021.

\bibitem{mann2016adaptive}
T.~A. Mann, H.~Penedones, S.~Mannor, and T.~Hester.
\newblock Adaptive lambda least-squares temporal difference learning.
\newblock {\em arXiv preprint arXiv:1612.09465}, 2016.

\bibitem{meyn2012markov}
S.~P. Meyn and R.~L. Tweedie.
\newblock {\em {Markov} chains and stochastic stability}.
\newblock Springer Science \& Business Media, 2012.

\bibitem{montgomery1993comparison}
S.~J. Montgomery-Smith.
\newblock Comparison of sums of independent identically distributed random
  variables.
\newblock {\em Probability and Mathematical Statistics}, 14 no.2:281--285,
  1993.

\bibitem{MouPanWai22}
W.~Mou, A.~Pananjady, and M.~J. Wainwright.
\newblock Optimal oracle inequalities for solving projected fixed-point
  equations.
\newblock {\em Mathematics of Operations Research}, 2022.
\newblock To appear; Posted originally as \emph{arXiv preprint
  arXiv:2021.05299}, 2021.

\bibitem{mou2021optimal}
W.~Mou, A.~Pananjady, M.~J. Wainwright, and P.~L. Bartlett.
\newblock Optimal and instance-dependent guarantees for {Markovian} linear
  stochastic approximation.
\newblock {\em arXiv preprint arXiv:2112.12770}, 2021.

\bibitem{munos2008finite}
R.~Munos and C.~Szepesv{\'a}ri.
\newblock Finite-time bounds for fitted value iteration.
\newblock {\em Journal of Machine Learning Research}, 9(5), 2008.

\bibitem{nian2020review}
R.~Nian, J.~Liu, and B.~Huang.
\newblock A review on reinforcement learning: Introduction and applications in
  industrial process control.
\newblock {\em Computers \& Chemical Engineering}, 139:106886, 2020.

\bibitem{pananjady2020instance}
A.~Pananjady and M.~J. Wainwright.
\newblock Instance-dependent {$\ell_{\infty}$}-bounds for policy evaluation in
  tabular reinforcement learning.
\newblock {\em IEEE Transactions on Information Theory}, 67(1):566--585, 2020.

\bibitem{shawe2004kernel}
J.~Shawe-Taylor, N.~Cristianini, et~al.
\newblock {\em Kernel methods for pattern analysis}.
\newblock Cambridge university press, 2004.

\bibitem{singh1998analytical}
S.~Singh and P.~Dayan.
\newblock Analytical mean squared error curves for temporal difference
  learning.
\newblock {\em Machine Learning}, 32(1):5--40, 1998.

\bibitem{spielberg2017deep}
S.~Spielberg, R.~Gopaluni, and P.~Loewen.
\newblock Deep reinforcement learning approaches for process control.
\newblock In {\em 2017 6th international symposium on advanced control of
  industrial processes (AdCONIP)}, pages 201--206. IEEE, 2017.

\bibitem{sutton2018reinforcement}
R.~S. Sutton and A.~G. Barto.
\newblock {\em Reinforcement learning: {An} introduction}.
\newblock MIT Press, 2018.

\bibitem{tsitsiklis1997analysis}
J.~N. Tsitsiklis and B.~Van~Roy.
\newblock An analysis of temporal-difference learning with function
  approximation.
\newblock 42(5), 1997.

\bibitem{uehara2021finite}
M.~Uehara, M.~Imaizumi, N.~Jiang, N.~Kallus, W.~Sun, and T.~Xie.
\newblock Finite sample analysis of minimax offline reinforcement learning:
  Completeness, fast rates and first-order efficiency.
\newblock {\em arXiv preprint arXiv:2102.02981}, 2021.

\bibitem{wainwright2019high}
M.~J. Wainwright.
\newblock {\em High-dimensional statistics: {A} non-asymptotic viewpoint},
  volume~48.
\newblock Cambridge University Press, 2019.

\bibitem{yang2017randomized}
Y.~Yang, M.~Pilanci, and M.~J. Wainwright.
\newblock Randomized sketches for kernels: Fast and optimal nonparametric
  regression.
\newblock {\em The Annals of Statistics}, 45(3):991--1023, 2017.

\bibitem{yu2009convergence}
H.~Yu and D.~P. Bertsekas.
\newblock Convergence results for some temporal difference methods based on
  least squares.
\newblock {\em IEEE Transactions on Automatic Control}, 54(7):1515--1531, 2009.

\bibitem{yu2010error}
H.~Yu and D.~P. Bertsekas.
\newblock Error bounds for approximations from projected linear equations.
\newblock {\em Mathematics of Operations Research}, 35(2):306--329, 2010.

\bibitem{zanette2021exponential}
A.~Zanette.
\newblock Exponential lower bounds for batch reinforcement learning: {Batch}
  {RL} can be exponentially harder than online {RL}.
\newblock In {\em International Conference on Machine Learning}, pages
  12287--12297. PMLR, 2021.

\bibitem{ZanWai22_Galerkin_Conf}
A.~Zanette and M.~J. Wainwright.
\newblock Bellman residual orthogonalization for offline reinforcement
  learning.
\newblock In {\em Neural {I}nformation {P}rocessing {S}ystems}, December 2022.
\newblock Long version posted as arxiv:2203.12786.

\bibitem{zhao2011reinforcement}
Y.~Zhao, D.~Zeng, M.~A. Socinski, and M.~R. Kosorok.
\newblock Reinforcement learning strategies for clinical trials in nonsmall
  cell lung cancer.
\newblock {\em Biometrics}, 67(4):1422--1433, 2011.

\end{thebibliography}
}}

\end{document}